\documentclass{article}
\usepackage{iclr2027_conference,times}
\iclrfinalcopy

\usepackage[T1]{fontenc}    
\usepackage{hyperref}       
\hypersetup{hidelinks, pdftitle={Optimal hypersurface decision trees}, pdfauthor={Xi He}}
\usepackage{url}            
\usepackage{booktabs}       
\usepackage{amsfonts}       
\usepackage{nicefrac}       
\usepackage{microtype}      
\usepackage{xcolor}         

\usepackage{color}
\usepackage{array}
\usepackage{amsmath}
\usepackage{amsthm}
\usepackage{amssymb}
\usepackage{graphicx}
\usepackage{varwidth}
\usepackage{float}
\usepackage{tikz}
\usepackage{multirow}
\usepackage{makecell}
\usepackage{algorithm}
\usepackage{algpseudocode}
\usepackage{subcaption}
\usepackage{listings}

\usepackage{wrapfig}
\newcommand{\red}[1]{#1} 

\providecommand{\tabularnewline}{\\}
\makeatletter
\newenvironment{cellvarwidth}[1][t]
{\begin{varwidth}[#1]{\linewidth}}
		{\@finalstrut\@arstrutbox\end{varwidth}}
\makeatother
\floatstyle{ruled}
\newfloat{algorithm}{tbp}{loa}
\providecommand{\algorithmname}{Algorithm}
\floatname{algorithm}{\protect\algorithmname}

\newtheorem{theorem}{Theorem}
\newtheorem{lemma}[theorem]{Lemma}
\newtheorem{fact}[theorem]{Fact}
\newtheorem{definition}[theorem]{Definition}

\newtheorem*{thm*}{Theorem}
\newtheorem*{straddlingdefinition}{Definition}

\title{Optimal hypersurface decision trees}

\author{%
  Xi He \\
  Department of Computer Science\\
  Peking University\\
  Beijing, China, 1000871\\
  \texttt{xihe@pku.edu.cn}
}

\begin{document}

\maketitle
\fancyhead{} 
\renewcommand{\headrulewidth}{0pt}

\begin{abstract}
Optimal decision trees have attracted increasing attention in recent
years, yet existing algorithms face two closely related challenges.
First, the trees they construct have limited expressivity because their
splitting rules are typically restricted to axis-parallel splits or
binary features. Second, these algorithms generally scale poorly to
large datasets. These challenges compound one another: more expressive
splitting rules increase the combinatorial complexity of the search
space, making the resulting optimal decision tree (ODT) problem even
more difficult to solve.

Building on the proper decision tree framework of \citet{he2025odt},
we propose, to the best of our knowledge, the first algorithm for
constructing optimal decision trees with hypersurface splits. The
algorithm has worst-case work complexity
$O\left(K!K N^{KG+1}\right)$ when $N\geq K$ and fixed-dimensional
feasibility tests are treated as constant-cost operations,
where $G=\binom{D+M}{M}-1$ depends on the hypersurface polynomial degree
$M$ and the data dimension $D$. Its structure is naturally amenable to vectorization and parallelization,
enabling efficient execution on GPUs, and the generic algorithm design can
be used to accelerate other ODT variants, such as axis-parallel
decision (ADT) trees. We further develop a safe pruning strategy that substantially reduces
the number of candidate configurations and an incremental generation
procedure that reduces the associated feasibility-checking cost from
$O(K^2)$ to $O(K)$. 

Experiments show that, for small trees, our efficient vectorized solver can explore over one million candidate trees in a matter of seconds. On both synthetic and real-world datasets, more expressive tree models consistently achieve higher predictive accuracy than approximate and optimal axis-parallel decision-tree models, often with substantially smaller tree sizes. In the most pronounced case, our model improves test accuracy by nearly 30\% points relative to the optimal ADT algorithm.
\end{abstract}

\begin{figure}[H]
	\centering
	\begin{subfigure}[b]{0.22\textwidth}
		\centering
		\includegraphics[viewport=50bp 60bp 480bp 475bp,clip,width=\textwidth]{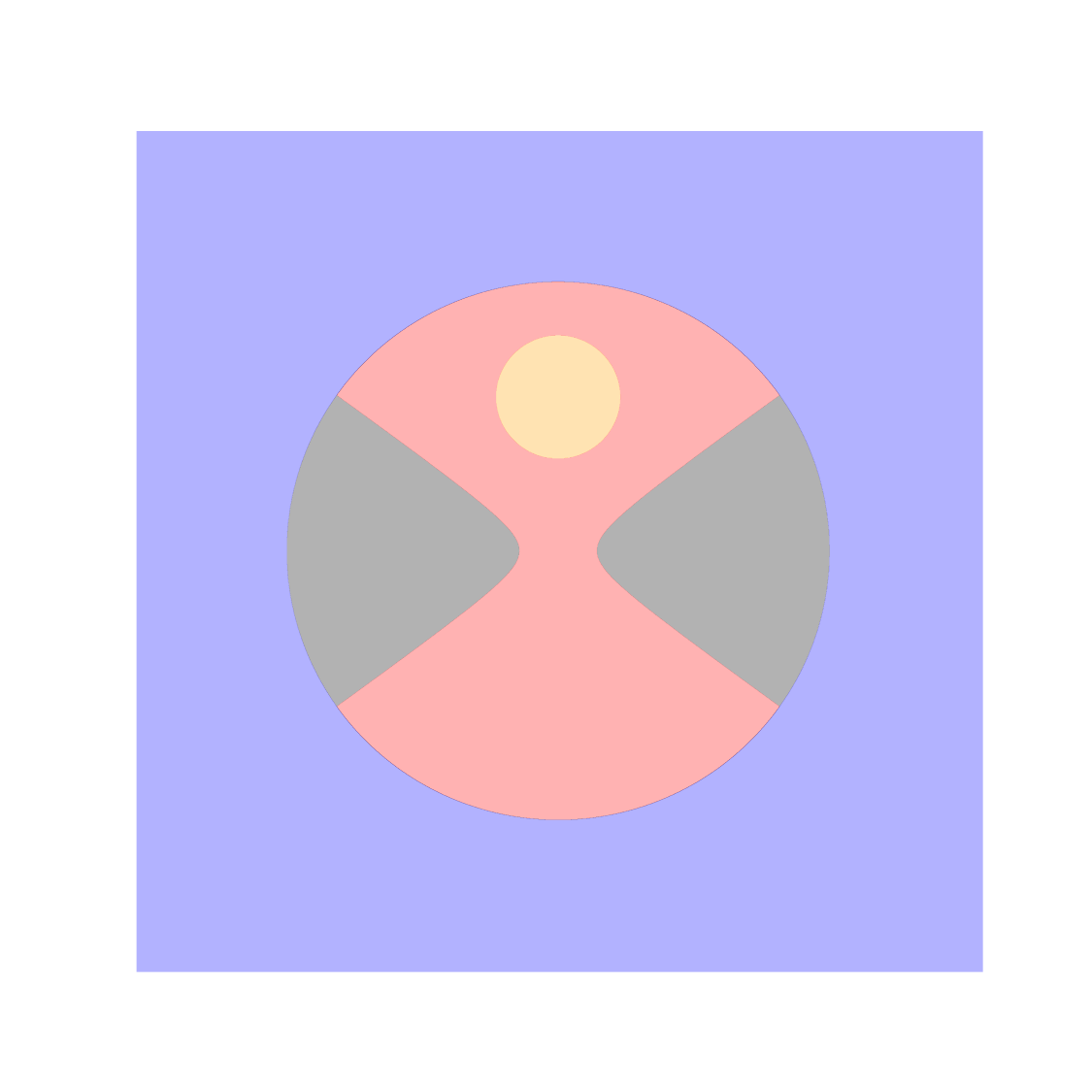}
		\caption{\centering Ground truth\\nmc: 0 $d$: 3, $K$: 3}
		\label{fig:syn_gt}
	\end{subfigure}
	\hfill
	\begin{subfigure}[b]{0.22\textwidth}
		\centering
		\includegraphics[viewport=50bp 60bp 480bp 475bp,clip,width=\textwidth]{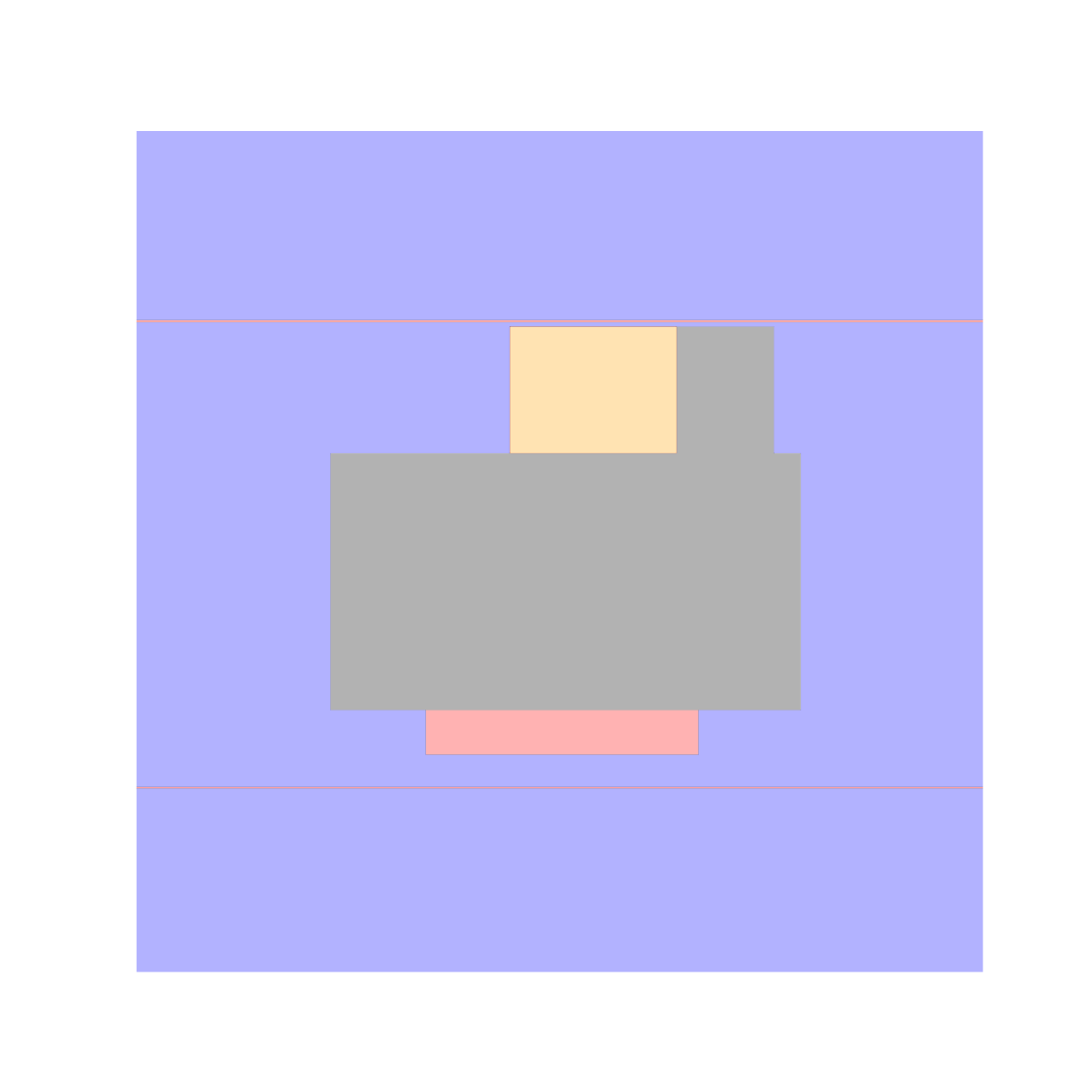}
		\caption{\centering CART\\nmc: 56 $d$: 5, $K$: 15}
		\label{fig:syn_cart}
	\end{subfigure}
	\hfill
	\begin{subfigure}[b]{0.22\textwidth}
		\centering
		\includegraphics[viewport=50bp 60bp 480bp 475bp,clip,width=\textwidth]{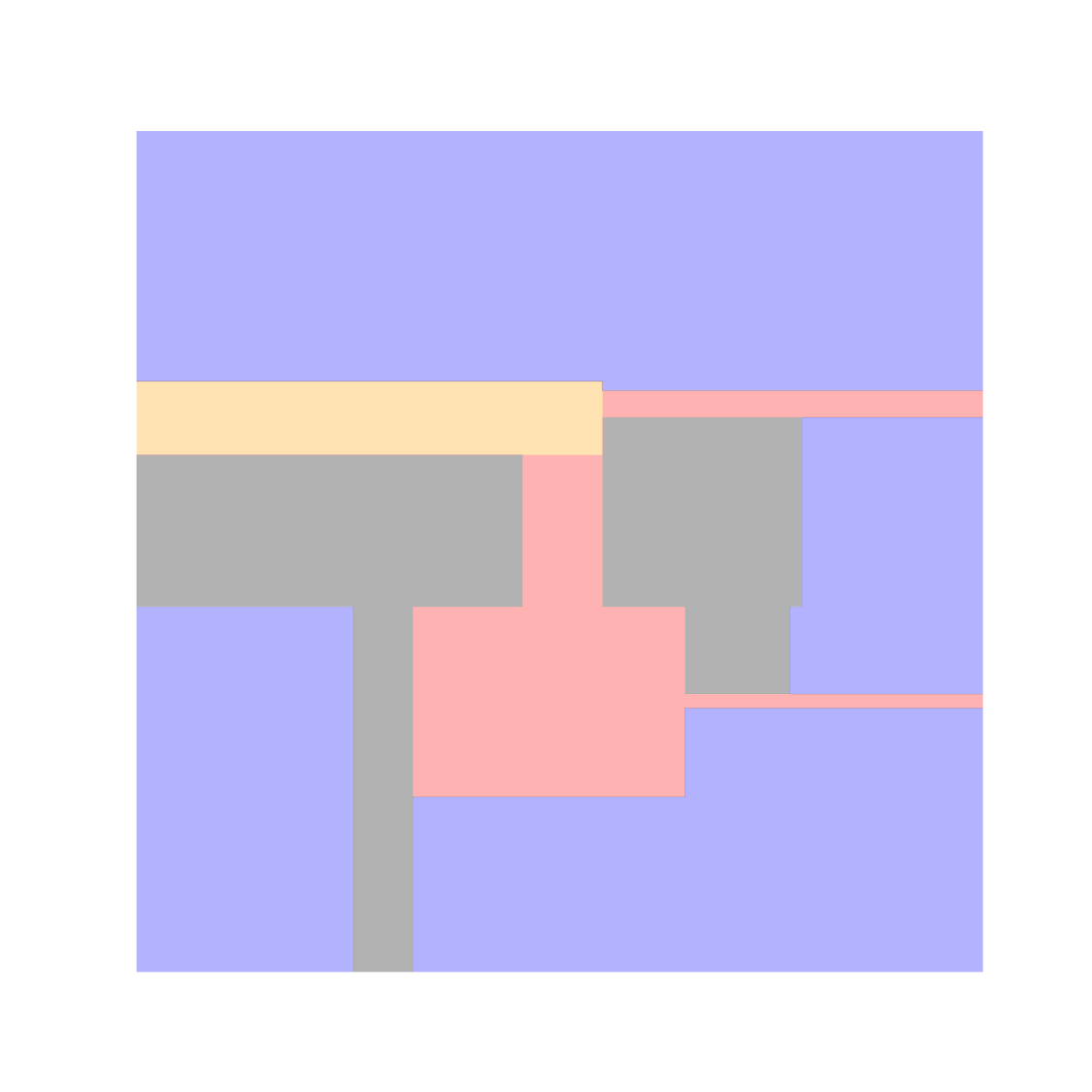}
		\caption{\centering ConTree\\nmc: 34 $d$: 4, $K$: 15}
		\label{fig:syn_contree}
	\end{subfigure}
	\hfill
	\begin{subfigure}[b]{0.22\textwidth}
		\centering
		\includegraphics[viewport=50bp 60bp 480bp 475bp,clip,width=\textwidth]{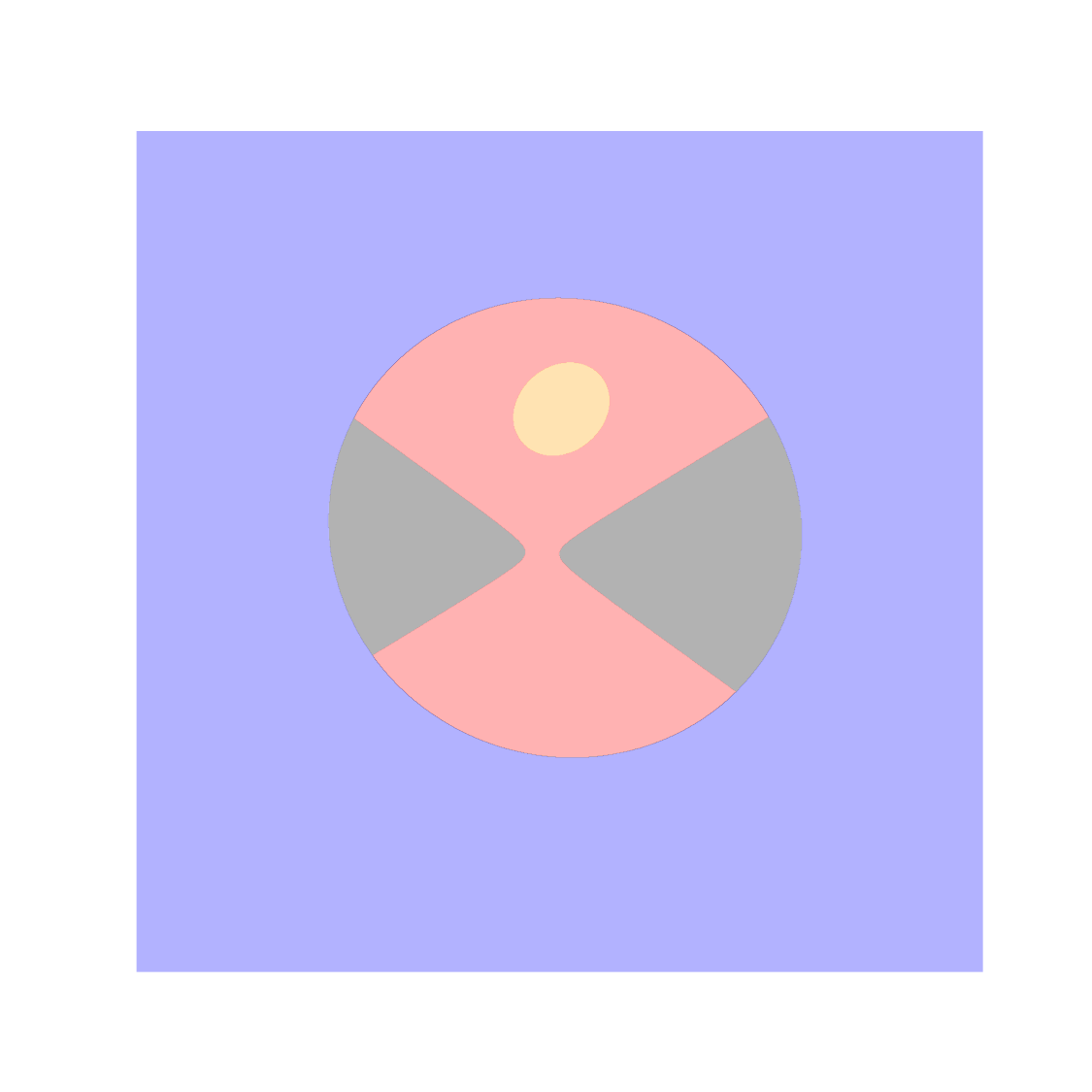}
		\caption{\centering HSDT (ours)\\nmc: 18 $d$: 3, $K$: 3}
		\label{fig:syn_ours}
	\end{subfigure}
	
	\caption{Synthetic dataset (left) generated by degree-2 polynomials; axis-parallel
		decision trees learned by CART and the state-of-the-art optimal algorithm---ConTree (middle two);
		and hypersurface decision trees (HSDT) learned using our algorithm 
		(right), with corresponding numbers of misclassifications (nmc), tree depths ($d$),
		and tree sizes ($K$). \label{fig: comparision}}
	\label{fig:comparision}
\end{figure}

\section{Introduction}

\newcommand{\odtTaxonomyScale}{0.7}

\begin{figure}
\centering
\begin{tikzpicture}[
  x=1mm,y=1mm,scale=\odtTaxonomyScale,transform shape,
  taxleaf/.style={rounded corners=2pt,draw=black!35,fill=white,
    text width=25mm,minimum height=12mm,align=center,font=\tiny,inner sep=1.3mm},
  taxadt/.style={rounded corners=2pt,draw=blue!75!black,very thick,fill=blue!10,
    text width=31mm,minimum height=12mm,align=center,font=\tiny,inner sep=1.3mm},
  taxfocus/.style={rounded corners=2pt,draw=red!75!black,very thick,fill=red!10,
    text width=25mm,minimum height=12mm,align=center,font=\tiny,inner sep=1.3mm},
  taxhsdt/.style={rounded corners=2pt,draw=red!75!black,very thick,fill=red!10,
    text width=19mm,minimum height=12mm,align=center,font=\tiny,inner sep=1.3mm},
  taxbinary/.style={rounded corners=2pt,draw=teal!75!black,fill=teal!8,
    text width=30mm,minimum height=12mm,align=center,font=\tiny,inner sep=1.3mm},
  taxtag/.style={rounded corners=1.2pt,text=white,font=\bfseries\tiny,
    inner xsep=2pt,inner ysep=.7pt}
]

  \filldraw[fill=black!2,draw=black!45,line width=.55pt]
    (0,-22) ellipse (69mm and 25mm);
  \filldraw[fill=blue!2,draw=blue!55!black,dashed,line width=.5pt]
    (-15,-25.5) ellipse (46mm and 21mm);
  \node[fill=white,inner xsep=2.5pt,font=\bfseries\scriptsize]
    at (0,3) {Classification of decision tree problems};
  \node[fill=white,inner xsep=2pt,font=\bfseries\tiny,text=blue!65!black]
    at (-15,-7) {PDT family \citep{he2025odt}};

  \node[taxadt] (adt) at (-44,-23)
    {\textbf{Axis-parallel splits (ADT)}\\[1pt]
     \textit{MIP}: \citep{bertsimas2017optimal}\\[-.2pt]
     \textit{Combinatorial}: \citep{mazumder2022quant,brita2025optimal}};
  \node[taxfocus] (hdt) at (-11,-23)
    {\textbf{Hyperplane splits (HDT)}\\[1pt]
     \citep{bertsimas2017optimal}};
  \node[taxhsdt] (hsdt) at (15,-23)
    {\textbf{Hypersurface splits (HSDT)}\\[1pt]
     No existing studies};
  \node[taxbinary] (binary) at (49,-20)
    {\textbf{Binary-feature ODT (ODT-BF)}\\[1pt]
     \citep{zhang2023optimal,aglin2021pydl8,nijssen2007mining,nijssen2010optimal,aglin2020learning,demirovic2022murtree,hu2019optimal,lin2020generalized}};

  \filldraw[fill=white,draw=black!35,line width=.45pt]
    (-33,-37) ellipse (15mm and 5.5mm);
  \node[font=\tiny,align=center] at (-33,-37)
    {\textbf{Binary space-partition tree}\\[-.2pt]
     \citep{toth2005binary,fan2018binary}};
  \filldraw[fill=white,draw=black!35,line width=.45pt]
    (10,-37) ellipse (13mm and 5.5mm);
  \node[font=\tiny,align=center] at (10,-37)
    {\textbf{K-D tree}\\[-.2pt]\citep{bentley1975multidimensional}};

  \filldraw[fill=white,draw=black!35,line width=.45pt]
    (38,-39) ellipse (14mm and 6mm);
  \node[font=\tiny,align=center,text width=25mm] at (38,-39)
    {\textbf{Binary search trees \& trees in matrix-chain multiplication}\\[-.2pt]\citep{bird2020algorithm}};

  \node[taxtag,fill=blue!75!black] at (-44,-13.5) {SPECIAL CASE};
  \node[taxtag,fill=red!75!black] at (-11,-13.5) {THIS WORK};
  \node[taxtag,fill=red!75!black] at (15,-13.5) {THIS WORK};
\end{tikzpicture}
\caption{Taxonomy of optimal decision tree problems. Much of the prior work that claims to solve ODT problems solves only the ODT-BF problem.
\red{When solving ODT problems, one can optimize either \textit{depth-constrained} trees or \textit{size-constrained} trees (also referred to as \textit{sparse decision trees}).}}
\label{fig:odt-taxonomy}
\end{figure}

Substantial recent advances have made optimal decision tree (ODT) more practical.
Despite these advances, the term ``optimal decision trees'' is often
used for several distinct tree definitions. In mainstream machine
learning (ML), the classical ODT problem concerns axis-parallel decision
trees (ADTs), which are commonly learned greedily using CART
\citep{breiman1984classification}. Finding an optimal ADT is NP-hard
\citep{laurent1976constructing}. However, only a small number of ODT studies directly
address this problem \citep{bertsimas2017optimal,mazumder2022quant,brita2025optimal};
many methods described as ODT solvers instead address the simplified
binary-feature ODT problem (ODT-BF) \citep{zhang2023optimal,demirovic2022murtree,lin2020generalized}. For fixed tree size \(K\), \citep{he2025FoODT_I} shows that the ODT-BF problem is \textit{polynomial-time solvable} in the sample size \(N\)
 and has been empirically observed to exhibit
linear with respect to $N$ \citep{hu2019optimal}. Figure \ref{fig:odt-taxonomy}
summarizes these classes and their relationships.

Existing ODT methods follow two main approaches. The first formulates
the problem as mixed-integer programming (MIP). Since the pioneering
work of \citet{bertsimas2017optimal}, this approach has motivated
substantial follow-up work
\citep{boutilier2023optimal,gunluk2021optimal,zhu2020scalable,verwer2019learning}.
MIP formulations can flexibly accommodate diverse constraints and can,
in some cases, solve large instances efficiently. However, their
computational cost is often difficult to predict, and they may perform
poorly even when a polynomial-time combinatorial solution exists.

The second approach uses \textit{combinatorial methods}, including
branch-and-bound (BnB) and dynamic programming (DP) algorithms
\citep{brita2025optimal,aglin2021pydl8,nijssen2010optimal,aglin2020learning,demirovic2022murtree,hu2019optimal,lin2020generalized,mazumder2022quant}.
Much of this work focuses on the simpler ODT-BF setting. Moreover,
its problem formulations are often informal and specified independently
of a concrete algorithm. This separation creates a gap between
theoretical characterizations and practical implementations: as
reported by \citet{he2025odt}, material discrepancies exist between
algorithmic abstractions and their actual implementations.

Moreover, existing studies still suffer from two major challenges:

\begin{itemize}
	\item \textbf{Expressivity}: Classical decision trees have limited
	expressivity because they typically use axis-parallel splitting rules.
	As illustrated in Fig.~\ref{fig: comparision}, this restriction can
	produce a coarse approximation of the underlying ground truth. In the
	example, CART and the optimal axis-parallel tree learned using
	\citet{brita2025optimal}'s algorithm require 15 splitting rules and
	have depths of 4 and 5, respectively\footnote{Additional comparisons
	in Fig.~\ref{fig:comparision-full} (Appendix) show that a plausible fit
	is obtained only when CART uses 45 splitting rules and the optimal tree
	uses 48. Trees of this size are rarely practical in interpretable
	machine learning settings.}, yet neither closely captures the target
	partition.

	\item \textbf{Scalability}: More expressive tree models are especially
	difficult to scale. Even for hyperplane splits, the number of candidate
	splits grows as $O\left(N^{D}\right)$, far exceeding the
	$O\left(N\times D\right)$ candidates considered by classical ADTs.
	Existing DP methods
	\citep{he2025FoODT_I,mazumder2022quant,brita2025optimal} rely on
	\textit{data-dependent recursion}\footnote{\emph{Data-dependent recursion}
	describes a recursive process whose number of calls, recursion depth,
	or control flow depends on the values or structure of the input data,
	rather than being fixed or statically predictable at compile time. In \citet{he2025FoODT_I,mazumder2022quant,brita2025optimal}, the recursive structure of their algorithm
	depends on the dataset $\mathcal{D}$ and on the root selected in each
	branch.} and depth-first search. These designs can have poor cache
	locality and are difficult to vectorize or parallelize effectively on
	modern hardware such as GPUs. Their pruning strategies are also often
	ad hoc and specific to particular ODT variants. For example, the
	equivalent-points bound \citep{lin2020generalized,hu2019optimal},
	which exploits duplicate data points, is useful for binary features
	but rarely applies to continuous data.
\end{itemize}

Consequently, researchers often rely on heuristics for large-scale
problems or restrict attention to easier variants, such as ODT-BF.
The combinatorial complexity of more flexible tree models, together
with the lack of general pruning strategies and hardware-efficient
algorithms, has limited their broader use.

In this paper, we address both challenges through the following
contributions:
\begin{itemize}
	\item \textbf{Exact formulation and algorithm.} We develop a \textit{unified algorithmic framework}
	for axis-parallel, hyperplane, and hypersurface decision trees.
	\item \textbf{Hardware-friendly algorithmic design.}
	For the OHSDT problem, the number of candidate splits is $O\left(N^{G}\right)$, where
	$G=\left(\begin{array}{c}
		D+M\\
		M
	\end{array}\right)-1$ and $M$ is the hypersurface polynomial degree.
	This search space is intractable even for greedy methods. We therefore
	introduce complementary acceleration strategies: a generic \textit{vectorized
	algorithm} with improved hardware compatibility and parallelizability,
	while preserving the same theoretical complexity.

	\item \textbf{Safe pruning method and algorithmic fusion.} We introduce a novel
	\textit{safe pruning method} that applies to both \textit{HDTs and HSDTs} (but not ADTs), substantially reducing cost below the theoretical upper bound. We further exploit the prefix-closed infeasibility property to enable an \textit{algorithmic fusion} that reduces the feasibility testing cost from $O(K^2)$ to $O(K)$.

	\item \textbf{Extensive empirical evaluation.} We evaluate on comprehensive synthetic and real-world datasets. On synthetic data, we assess robustness to variation and noise across \textbf{50} independently generated datasets per setting. The results show that, under controlled complexity, decision trees with more expressive splitting rules are more robust and achieve higher predictive accuracy.

\end{itemize}

\section{Background---Proper decision tree framework}

\paragraph{MIP definition of ODT problems} 
The size-constrained ODT problem can be formulated using mixed-integer programming (MIP) as follows:

\begin{equation}
	\begin{aligned}
		T^{*}= & \underset{T\in\mathcal{S}_{\text{size}}\left(K,\mathit{xs}\right)}{\text{argmin }} E\left(T\right).
	\end{aligned}
	\label{eq: MIP-ODT-size}
\end{equation}
Here, $E$ denotes the objective function; $\mathcal{S}_{\text{size}}(K,\mathit{xs})$ denotes the search space of all decision trees with $K$ internal nodes over the dataset $\mathit{xs}$; and $T^{*}$ is an optimal decision tree.

Similar definitions appear in previous studies \citep{mazumder2022quant,brita2025optimal}. Although the formulation is intuitive, constructs such as $\mathcal{S}_{\text{size}}(K,\mathit{xs})$ and $E$ are typically specified only through informal \textit{textual descriptions} rather than given formal semantics. By contrast, \citet{bertsimas2017optimal} provide a rigorous MIP formulation of \eqref{eq: MIP-ODT-size}, but their formulation spans more than three pages.

We instead introduce a alternative \textit{programmatic definition} of ODT problems that is both rigorous and concise.

\paragraph{Recursive definition of ODT problems}
In constructive algorithmics \citep{bird1996algebra}, a problem may be specified directly as a program rather than first being defined as an algorithm-independent, informal mathematical object.

The programmatic formalization of ODT problems was first introduced by
\citet{he2025FoODT_I}. They proved that any decision tree satisfying the
axioms of a \textit{proper decision tree (PDT)}---which ensure well-defined
spatial partitions and transitive ancestry relations between splitting
rules---admits two key results. First, the \emph{PDT Characterization
Theorem} (Theorem~\ref{thm: PDT-characterization}) characterizes PDTs by
\textit{$K$-permutations}. Second, the \emph{PDT Decomposition Theorem}
(Theorem~\ref{thm: PDT decomposition}) uses this characterization to
decompose the ODT problem as follows:

\begin{equation}
	\mathit{odt}\left(\mathit{rs}\right)=\mathit{min}_{E}\circ\mathit{concatMapL}\left(\mathit{sodt}\left(\mathit{xs}\right)\right)\circ\mathit{kcombs}_{K}\left(\mathit{rs}\right).\label{eq: sodt specification}
\end{equation}
Here, the algorithm $\mathit{odt}$ takes a list of rules
$\mathit{rs}:[\mathcal{R}]$ as input and returns an optimal decision tree
$t^{*}$. On the right-hand side, the formulation derived by
\citet{he2025FoODT_I} comprises three functions, where the $\circ$ denotes function composition, defined by
$f\circ g(a)=f(g(a))$.
First, $\mathit{kcombs}_{K}(\mathit{rs})$ generates all possible
$K$-combinations of the rules in $\mathit{rs}$. Next,
$\mathit{concatMapL}(\mathit{sodt}(\mathit{xs}))$ applies
$\mathit{sodt}(\mathit{xs})$ to each generated combination and
concatenates the resulting lists. Finally, $\mathit{min}_{E}$ selects an
optimal solution according to the objective function $E$.

A key consequence of \eqref{eq: sodt specification} is that the complex
$\mathit{odt}$ problem can be factorized into smaller \textbf{subproblems}
generated by $\mathit{kcombs}_{K}(\mathit{rs})$. Each corresponding
$\mathit{sodt}$ instance can be solved \textbf{independently}, providing a
natural basis for large-scale parallelization.


\paragraph{Size-/depth-constrained trees}

Most ODT solvers impose a \textbf{depth constraint}, whereas
\eqref{eq: MIP-ODT-size} and \eqref{eq: sodt specification} optimize over
\textbf{size-constrained trees}. Previous work has studied size-constrained
trees primarily as \textit{sparse decision trees} for ODT-BF
\citep{hu2019optimal,lin2020generalized}. These two formulations induce
different algorithmic structures. Specifically, \citet{he2025odt} showed
that the search space of an ODT method is determined by its
\textbf{branching method} and \textbf{split generator}. Algorithms for trees in the
PDT family share the same branching method but use different split
generators depending on whether tree size or depth is constrained.

The choice between depth and size constraints entails distinct
computational trade-offs. Depth constraints are more common in previous
work and are often combined with \textit{depth-first search}. This
strategy can quickly explore large trees and identify accurate solutions
\textit{before} certifying optimality, but it typically exhibits poor
cache locality and offers limited opportunities for parallelization. By
contrast, through $\mathit{kcombs}$ in
\eqref{eq: sodt specification}, the size-constrained search space
decomposes naturally into subproblems that can be explored independently.
Earlier sparse-tree methods have not exploited this decomposition.
Because the hypersurface decision trees studied in this paper employ
highly expressive splits, even small trees can often achieve sufficient
predictive accuracy. We therefore focus on size-constrained trees and prioritize exploring many small trees in parallel.

On the other hand, because the two paradigms define different search spaces, direct
comparison with existing depth-constrained methods is challenging.
Nevertheless, the size-constrained algorithm of \citet{he2025odt} provides the closest
size-constrained counterpart to these methods because it uses the same
branching method. We therefore use their algorithm as a
baseline for evaluating the novel algorithms developed in this paper.

\section{Optimal hypersurface decision tree algorithm\label{sec:Optimal hodt algs}}

\subsection{Rule generators for decision trees \label{par:Rules-generators-for}}

Existing studies of ODT problems often use an ad hoc algorithm to solve
ODT-BF \citep{demirovic2022murtree} and then develop a completely
different algorithm to solve ADT problems \citep{brita2025optimal}. This
problem-specific design increases the burden of understanding ODT
problems, as readers must learn distinct algorithms for distinct ODT
formulations.

Astute readers may note that our formulation of the ODT problem in
\eqref{eq: sodt specification} differs from the MIP formulation in
\eqref{eq: MIP-ODT-size}: the former takes a list of rules
$\mathit{rs}$ as input, whereas the latter defines a search space over a
dataset $\mathit{xs}$. This distinction is deliberate: it separates the \textit{rule-generation program},
which produces the rules used to define decision trees, from the \textit{program
that constructs the trees}. Consequently, different ODT variants can be
addressed by modifying only the rule-generation component while leaving
the tree-construction process unchanged. To make this separation
explicit, we define a function $\mathit{genRules}$ that generates rules
from a given dataset. The complete procedure for solving an ODT problem
is then defined as follows:
\[
	\mathit{odt}^{\prime}\left(\mathit{xs}\right)
	=
	\mathit{odt}\circ\mathit{genRules}\left(\mathit{xs}\right).
\]
Here, $\mathit{genRules}:[\mathbb{R}^{D}]\to[\mathcal{R}]$ generates a
set of rules from a dataset $\mathit{xs}$. The program
$\mathit{odt}^{\prime}$ then solves \eqref{eq: MIP-ODT-size} exactly; that is,
it provides a programmatic definition of the ODT problem.

By modifying $\mathit{genRules}$, we can address different ODT variants
while leaving the main program $\mathit{odt}$ unchanged. This separation
provides a modular formulation of ODT problems. For example, for
classical ADTs, $\mathit{genRules}$ generates all axis-aligned splits,
whereas for hyperplane decision trees, it generates all possible
hyperplanes. 

At first sight, the latter task appears infeasible because the space of
hyperplanes seems to have infinite combinatorial complexity: each
hyperplane is defined by a continuous normal vector
$\boldsymbol{w}_{k}$. Consequently, it may seem impossible for
$\mathit{genRules}$ to terminate in a finite amount of time. However,
the finiteness of the dataset imposes strong constraints. In particular,
a finite dataset admits only finitely many distinct partitions induced
by hyperplanes, thereby defining an \textbf{equivalence relation} among
hyperplanes.

Indeed, following \citet{he2023efficient}'s results, we can exhaustively enumerate the equivalence
classes of hyperplanes and hypersurfaces using the following theorem.

\begin{theorem}
	Under the general-position assumption, the 0--1 loss (accuracy)
	linear classification problem on a dataset $\mathit{xs}$ in
	$\mathbb{R}^{D}$ can be solved exhaustively by enumerating all
	$D$-point combinations of $\mathit{xs}$. Moreover, classification by
	a polynomial hypersurface\footnote{A polynomial hypersurface is
	defined by a \emph{polynomial equation} of the form
	$P\left(\mathbf{x}\right)=\sum_{i}w_{i}\mathbf{x}^{\mathbf{\alpha}_{i}}$,
	where $\mathbf{x}$ is a $D$-dimensional data vector and
	$\boldsymbol{\mathbf{\alpha}}$ is a multi-index exponent vector.}
	is isomorphic to classification by a hyperplane in a
	higher-dimensional feature space of dimension
	$G=\left(\begin{array}{c}
		M+D\\
		M
	\end{array}\right)-1$, where $M$ is the polynomial degree and $D$ is
	the data dimension.
	\label{thm: 0-1 loss classification theorem}
\end{theorem}

This geometric result unifies the rule generators for different
geometric splitting rules through the combination generator
$\mathit{genRules}\left(G,\mathit{xs}\right)
=
\mathit{kcombs}\left(G,\mathit{xs}\right)$, parameterized by $G$.

Accordingly, $\mathit{odt}^{\prime}$ can be expanded as
\[
	\mathit{odt}^{\prime}\left(\mathit{xs}\right)
	=
	\mathit{min}_{E}
	\circ
	\mathit{concatMapL}\left(\mathit{sodt}\left(\mathit{xs}\right)\right)
	\circ
	\mathit{kcombs}\left(K\right)
	\circ
	\mathit{kcombs}\left(G,\mathit{xs}\right).
\]

Interestingly, \citet{he2025CGs} show that the sequential composition of
two combination generators, as in $\mathit{odt}^{\prime}$, can be fused
into a single recursion,
$\mathit{nestedCombs}\left(K,G,\mathit{xs}\right)$. This recursion is
equivalent to
$\mathit{kcombs}\left(K\right)\circ\mathit{kcombs}\left(G,\mathit{xs}\right)$
(see \citet{he2025CGs} for a proof and the imperative definition in
Section~\ref{par:Rules-generators-for}). Therefore,
\begin{align}
	\mathit{odt}^{\prime}\left(\mathit{xs}\right)
	&=
	\mathit{min}_{E}
	\circ
	\mathit{concatMapL}\left(\mathit{sodt}\left(\mathit{xs}\right)\right)
	\circ
	\mathit{nestedCombs}\left(K,G,\mathit{xs}\right).
	\label{alg:odt-prime}
\end{align}
Each nested combination generated by
$\mathit{nestedCombs}\left(K,G,\mathit{xs}\right)$ is a
$K$-combination of rules, with each rule constructed from a
$G$-combination of data points. In the remainder of the paper, we use
the terms \emph{configurations} and \emph{nested combinations} interchangeably.

\subsection{Polynomial hypersurface decision trees}

Theorem~\ref{thm: 0-1 loss classification theorem} establishes a natural
isomorphism between hyperplane and polynomial-hypersurface
classification. We may therefore state our results for hyperplanes,
while retaining their applicability to hypersurfaces.

The following lemma establishes that decision trees defined by
hyperplanes satisfy the PDT axioms.
\begin{lemma}
	\textbf{Hyperplane decision trees are proper.} \emph{Let
		$\mathit{xs}$ be a data set and let
		$\mathit{hs}_{K}=\left[h_{1},h_{2},\ldots,h_{K}\right]$ be a list
		of $K$ hyperplanes with normal vectors
		$ws=\left[w_{1},w_{2},\ldots,w_{K}\right]$ in
		$\mathbb{R}^{D+1}$, where each hyperplane passes through exactly
		$D$ data points. For each hyperplane, define the positive and
		negative regions by }
	$h_{i}^{+}=\left\{x\mid x\in\mathbb{R}^{D},\boldsymbol{w}_{i}^{T}\bar{x}\geq0\right\}$
	\emph{ and }
	$h_{i}^{-}=\left\{x\mid x\in\mathbb{R}^{D},\boldsymbol{w}_{i}^{T}\bar{x}<0\right\}$.
	\emph{Then, decision trees constructed from these hyperplanes as
		splitting rules are proper decision trees.}
\end{lemma}
By applying the PDT framework to the lifted dataset
$\phi_M(\mathit{xs}) \subset \mathbb{R}^G$, where each polynomial
splitting rule becomes an affine hyperplane. Pulling the resulting tree
back through $\phi_M$ yields polynomial-hypersurface splits in the
original space and preserves predictions and the 0--1 empirical
objective, though not the topology of decision regions. Thus,
\citet{he2025odt}'s framework applies directly to HDTs and, through lifting, to HSDT problems.

However, the corresponding $\mathit{sodt}$ algorithm relies on
data-dependent recursion, which limits parallelization and hardware
compatibility. In the next section, we introduce a
\emph{compile-time-determined} recursion that enables full vectorization
and efficient execution on modern hardware. This addresses a key gap in
the original exposition of \citet{he2025FoODT_I} and makes the
$\mathit{odt}$ program more amenable to parallelization. We denote the
algorithm of \citet{he2025FoODT_I} by
$\mathit{sodt}_{\mathrm{rec}}$ and our new algorithm by
$\mathit{sodt}_{\mathrm{vec}}$.

\subsection{A novel hardware-friendly and parallelizable algorithm}
A size-constrained decision tree with $K$ splitting rules has $K$
internal nodes. Rather than recursively selecting a branch node and
constructing its subtrees, as in $\mathit{sodt}_{\mathit{rec}}$, we
incrementally add one internal node at a time until all $K$ nodes have
been introduced. This yields a deterministic process requiring at most
$K$ steps, so its recursion depth can be determined at compile time.
Based on this observation, we define:

\begin{equation}
	\begin{aligned}\mathit{genDTs}_{\text{vec}} & :\mathcal{D}\times\left[\mathcal{R}\right]\to\left[\mathit{DTree}\left(\mathcal{R},\mathcal{D}\right)\right]\\
		\mathit{genDTs}_{\text{vec}} & \left(\mathit{xs},\left[\;\right]\right)=\left[\mathit{DL}\left(\mathit{xs}\right)\right]\\
		\mathit{genDTs}_{\text{vec}} & \left(\mathit{xs},\mathit{rs}\right)=\mathit{concat}\circ\left[\mathit{updates}_{\boldsymbol{A}}\left(r,\mathit{rs}\backslash r,xs\right)|\left(r,\mathit{rs}\backslash r\right)\longleftarrow\mathit{candidates}\left(\mathit{rs}\right)\right]
	\end{aligned}
\end{equation}
This recursion exhaustively generates all decision trees in the search
space by appending a rule to the leaves of each partial tree through
$\mathit{updates}_{\boldsymbol{A}}$. The complete recursive definition
and imperative implementations appear in Appendix~\ref{subsec: sodt_vec}
and Algorithm~\ref{alg: sodt_vec_imperative}.

The optimal solution is obtained by selecting the best generated
candidate:
\[
\mathit{sodt}_{\text{vec}}\left(\mathit{xs},\mathit{rs}\right)=\mathit{min}_{E}\circ\mathit{genDTs}_{\text{vec}}\left(\mathit{xs},\mathit{rs}\right)
\]
Both $\mathit{sodt}_{\text{vec}}$ and $\mathit{sodt}_{\mathit{rec}}$
have factorial dependence on $K$. The vectorized construction incurs an
additional polynomial factor in sequential work, while exposing a regular
batched computation that is more suitable for parallel hardware; see
Theorem~\ref{thm: complexity of sodt_vec} for complexity analysis.

\noindent
\begin{minipage}[t]{0.48\textwidth}
\vspace{0pt}
Our definition is fully formal and concise. For example,
Figure~\ref{fig:sodt-generation-order} presents a four-rule example
generated by a \textbf{13-line} Haskell implementation, demonstrating that
$\mathit{sodt}_{\text{vec}}$ and $\mathit{sodt}_{\mathit{rec}}$
enumerate the same set of trees in different orders. By contrast,
classical ODT algorithms are often specified informally and at greater
length. For example, \citet{demirovic2022murtree}'s algorithm spans
nearly \textbf{five pages}, while leaving numerous functions undefined, and addresses
the simpler ODT-BF problem.
\end{minipage}\hfill
\begin{minipage}[t]{0.49\textwidth}
\vspace{0pt}
\newcommand{\sodtGenerationFirstTreeX}{3.75}
\newcommand{\sodtGenerationColumnGap}{4.75}
\newcommand{\sodtGenerationRowGap}{6.40}
\newcommand{\sodtGenerationTreeGridScale}{0.535}

\newcommand{\sodtTree}[4]{%
  \begin{scope}[
    every node/.style={circle,draw,line width=.5pt,minimum size=2.6mm,
      inner sep=.15pt,font=\tiny},
    every path/.style={line width=.45pt}
  ]
    \node (r) at (0,1.48) {$h_{#1}$};
    \node (l) at (-.67,.78) {$h_{#2}$};
    \node (q) at (.67,.78) {$h_{#3}$};
    \node (g) at (-1.10,.05) {$h_{#4}$};
    \draw (r) -- (l);
    \draw (r) -- (q);
    \draw (l) -- (g);
  \end{scope}%
}

\newcommand{\sodtPartitionPoints}{%
  \foreach \p in {(-.25,1.257),(.12,.592),(.52,.375),(1.25,.819),
                    (1.35,0),(2.05,0),(0,-.72),(0,-1.16)}
    \fill \p circle (2.25pt);%
}

\newcommand{\sodtPartitionOne}{%
  \draw[line width=.65pt] (-.58,1.85) -- (1.20,-1.35);       
  \draw[line width=.65pt] (.3112,.2481) -- (2.42,1.53);     
  \draw[line width=.65pt] (.4491,0) -- (2.66,0);            
  \draw[line width=.65pt] (0,-1.67) -- (0,.8073);           
  \draw[->,line width=.55pt] (-.43,1.58) -- (-.12,1.86);
  \draw[->,line width=.55pt] (2.24,1.42) -- (2.55,.88);
  \draw[->,line width=.55pt] (2.52,.02) -- (2.52,-.46);
  \draw[->,line width=.55pt] (.02,-1.48) -- (.56,-1.48);
  \node[font=\scriptsize,anchor=south east] at (-.50,1.83) {$h_1$};
  \node[font=\scriptsize,anchor=south west] at (2.35,1.48) {$h_2$};
  \node[font=\scriptsize,anchor=west] at (2.58,-.05) {$h_0$};
  \node[font=\scriptsize,anchor=north] at (0,-1.67) {$h_3$};
  \sodtPartitionPoints
}

\newcommand{\sodtPartitionTwoThree}{%
  \draw[line width=.65pt] (-1.38,-.78) -- (2.42,1.53);      
  \draw[line width=.65pt] (-.0969,0) -- (2.66,0);           
  \draw[line width=.65pt] (0,-1.67) -- (0,.0589);           
  \draw[line width=.65pt] (-.58,1.85) -- (.3112,.2481);     
  \draw[->,line width=.55pt] (-.43,1.58) -- (-.12,1.86);
  \draw[->,line width=.55pt] (2.24,1.42) -- (2.55,.88);
  \draw[->,line width=.55pt] (2.52,.02) -- (2.52,-.46);
  \draw[->,line width=.55pt] (.02,-1.48) -- (.56,-1.48);
  \node[font=\scriptsize,anchor=south east] at (-.50,1.83) {$h_1$};
  \node[font=\scriptsize,anchor=south west] at (2.35,1.48) {$h_2$};
  \node[font=\scriptsize,anchor=west] at (2.58,-.05) {$h_0$};
  \node[font=\scriptsize,anchor=north] at (0,-1.67) {$h_3$};
  \sodtPartitionPoints
}

\centering
\resizebox{0.82\linewidth}{!}{%
  \begin{tikzpicture}[x=1cm,y=1cm]
    \begin{scope}[shift={(4.85,-2.10)}]
    \begin{scope}[shift={(0,1.50)},scale=.72,
      hyperplane/.style={line width=.7pt},
      direction/.style={->,line width=.65pt}]
      \draw[hyperplane] (-1.35,0) -- (2.75,0);
      \draw[direction] (2.58,.02) -- (2.58,-.55);
      \node[anchor=west] at (2.72,-.08) {$h_0$};
      \fill (1.35,0) circle (2.0pt);
      \fill (2.05,0) circle (2.0pt);
      \draw[hyperplane] (-.58,1.85) -- (1.20,-1.35);
      \draw[direction] (-.43,1.58) -- (-.12,1.86);
      \node[anchor=south east] at (-.50,1.83) {$h_1$};
      \fill (-.25,1.257) circle (2.0pt);
      \fill (.12,.592) circle (2.0pt);
      \draw[hyperplane] (-1.38,-.78) -- (2.42,1.53);
      \draw[direction] (2.24,1.42) -- (2.55,.88);
      \node[anchor=south west] at (2.35,1.48) {$h_2$};
      \fill (.52,.375) circle (2.0pt);
      \fill (1.25,.819) circle (2.0pt);
      \draw[hyperplane] (0,-1.67) -- (0,1.38);
      \draw[direction] (.02,-1.48) -- (.59,-1.48);
      \node[anchor=north] at (0,-1.68) {$h_3$};
      \fill (0,-.72) circle (2.0pt);
      \fill (0,-1.16) circle (2.0pt);
    \end{scope}
    \node[font=\large] at (0,-2.25)
      {$\displaystyle \boldsymbol A=
        \begin{bmatrix}
          0&-1&0&1\\
          1&0&1&-1\\
          1&-1&0&1\\
          1&0&1&0
        \end{bmatrix}$};
    \end{scope}
    \draw[black!25] (7.60,1.55) -- (7.60,-5.70);

    \begin{scope}[shift={(6.40,0)},scale=\sodtGenerationTreeGridScale]

    \begin{scope}[shift={(\sodtGenerationFirstTreeX,0)}]
      \node[font=\bfseries\scriptsize,anchor=west] at (-1.42,1.70) {$t_1$};
      \sodtTree{1}{2}{3}{0}
      \begin{scope}[shift={(-.25,-2.13)},xscale=.76,yscale=.72]
        \sodtPartitionOne
      \end{scope}
    \end{scope}
    \begin{scope}[shift={({\sodtGenerationFirstTreeX+\sodtGenerationColumnGap},0)}]
      \node[font=\bfseries\scriptsize,anchor=west] at (-1.42,1.70) {$t_2$};
      \sodtTree{2}{0}{1}{3}
      \begin{scope}[shift={(-.25,-2.13)},xscale=.76,yscale=.72]
        \sodtPartitionTwoThree
      \end{scope}
    \end{scope}
    \begin{scope}[shift={({\sodtGenerationFirstTreeX+2*\sodtGenerationColumnGap},0)}]
      \node[font=\bfseries\scriptsize,anchor=west] at (-1.42,1.70) {$t_3$};
      \sodtTree{2}{3}{1}{0}
      \begin{scope}[shift={(-.25,-2.13)},xscale=.76,yscale=.72]
        \sodtPartitionTwoThree
      \end{scope}
    \end{scope}

    \draw[black!25] ({\sodtGenerationFirstTreeX-2.80},-3.55) -- ({\sodtGenerationFirstTreeX+2*\sodtGenerationColumnGap+.45},-3.55);


    \begin{scope}[shift={(\sodtGenerationFirstTreeX,-\sodtGenerationRowGap)}]
      \node[font=\bfseries\scriptsize,anchor=west] at (-1.42,1.70) {$t_3$};
      \sodtTree{2}{3}{1}{0}
      \begin{scope}[shift={(-.25,-2.13)},xscale=.76,yscale=.72]
        \sodtPartitionTwoThree
      \end{scope}
    \end{scope}
    \begin{scope}[shift={({\sodtGenerationFirstTreeX+\sodtGenerationColumnGap},-\sodtGenerationRowGap)}]
      \node[font=\bfseries\scriptsize,anchor=west] at (-1.42,1.70) {$t_1$};
      \sodtTree{1}{2}{3}{0}
      \begin{scope}[shift={(-.25,-2.13)},xscale=.76,yscale=.72]
        \sodtPartitionOne
      \end{scope}
    \end{scope}
    \begin{scope}[shift={({\sodtGenerationFirstTreeX+2*\sodtGenerationColumnGap},-\sodtGenerationRowGap)}]
      \node[font=\bfseries\scriptsize,anchor=west] at (-1.42,1.70) {$t_2$};
      \sodtTree{2}{0}{1}{3}
      \begin{scope}[shift={(-.25,-2.13)},xscale=.76,yscale=.72]
        \sodtPartitionTwoThree
      \end{scope}
    \end{scope}
    \end{scope}
  \end{tikzpicture}%
  }
\captionsetup{font=scriptsize,skip=2pt}
\captionof{figure}{Left: four hyperplanes and corresponding
  ancestry relation matrix $\boldsymbol A$; circles mark defining points and
  arrows indicate orientations. Right: the three decision trees and induced
  partitions generated by \citep{he2025odt}'s algorithm (above). Our algorithm generates the same trees in different orders (below). Both methods prune the 24 possible trees to three.}
  \label{fig:sodt-generation-order}

\end{minipage}

Since depth- and size-constrained trees have different search spaces, direct comparison of our approach with existing ODT solvers is challenging. Nevertheless, $\mathit{sodt}_{\mathit{rec}}$ is the closest size-constrained analogue of depth-constrained methods because it uses the same branching methods; only the split generator follows from the chosen constraint. We therefore compare $\mathit{sodt}_{\mathit{rec}}$ with these methods and show that it is more efficient in both sequential and parallel settings.

Extensive results reported in Subsection \ref{subsec: sodt-comparision} and Appendix \ref{subsubsec: sodt-comparision-appendix} show that our proposed method
is more efficient in both sequential and parallel settings.

\subsection{Ancestry relation matrices and safe geometric pruning}

The vectorized search described in the preceding sections implicitly uses an
ancestry relation matrix $\boldsymbol{A}\in\{-1,0,1\}^{K\times K}$ to guide
$\mathit{updates}_{\boldsymbol{A}}$. Its entry $\boldsymbol{A}_{ij}$ records
whether rule $h_j$ may occur below rule $h_i$ in the positive branch ($1$), the
negative branch ($-1$), or neither branch ($0$). However, we have not yet
explained how to generate this matrix, and doing so efficiently is a major
computational challenge. Moreover, we show that, when $\boldsymbol{A}$ satisfies
a certain geometric property, it can safely prune the decision-tree search
space without excluding any decision tree. Thus, this property serves as a
pruning certificate. In this section, we formalize the property and show how to
incorporate it efficiently into our algorithm.

For two arbitrary rules, the ancestry matrix distinguishes three cases:
\begin{enumerate}
  \item \textbf{Mutual ancestry}: both rules can serve as ancestors of the
  other, so $\boldsymbol{A}_{ij}\neq0$ and
  $\boldsymbol{A}_{ji}\neq0$.
  \item \textbf{Asymmetrical ancestry}: exactly one rule can serve as an
  ancestor of the other, so exactly one of $\boldsymbol{A}_{ij}$ and
  $\boldsymbol{A}_{ji}$ is nonzero.
  \item \textbf{No ancestry}: neither rule can serve as an ancestor of the
  other, so $\boldsymbol{A}_{ij}=\boldsymbol{A}_{ji}=0$.
\end{enumerate}
To formalize them, we first instantiate  \textit{Partition transitivity} of the PDT axioms for hyperplane splits:
\begin{equation}
  \begin{gathered}
    \text{Let $X_\ell$ be the set of data points that defines $h_\ell$.}\\
    \text{If $h_\ell$ is contained in a decision region $R$, then }X_\ell\subseteq R.
  \end{gathered}
  \label{eq:defining-points-in-region}
\end{equation}
Thus, mutual ancestry requires
\(
(X_j\subseteq h_i^+\ \text{or}\ X_j\subseteq h_i^-)
\wedge
(X_i\subseteq h_j^+\ \text{or}\ X_i\subseteq h_j^-),
\)
whereas asymmetrical ancestry requires only one of these containments. If
neither containment holds, neither rule can be an ancestor of the other.  Figure~\ref{fig:three-ancestry-relations} illustrates these cases for
hyperplane splits in $\mathbb{R}^2$. 

However, characterizing the no-ancestry case as the complement of the first two
cases is neither intuitive nor straightforward to check during computation.
Can we instead impose a geometrically clearer condition that guarantees the no ancestry relation? The answer is yes, although the resulting
condition is more intricate than it first appears.

Our observations in $\mathbb{R}^2$ suggest that two hyperplanes have no
ancestry relation when each has defining points on both sides of the other.
We formalize this intuitive crossing pattern below.
\paragraph{Mutual straddling is not a pruning certificate}

\begin{straddlingdefinition}[Straddling and crossed-hyperplanes (CH)]
Let $X\subset\mathbb{R}^D$ be a finite point set and let
$h=\{x:f(x)=0\}$ be an affine hyperplane with halfspaces $h^+$ and $h^-$.
The \emph{straddling predicate} is
\[
p_{\mathrm{str}}(X,h)
:=\bigl[X\cap h^+\neq\varnothing\bigr]
\wedge\bigl[X\cap h^-\neq\varnothing\bigr].
\]
In this case, we
say that $X$ \emph{straddles} $h$.
The crossed-hyperplane predicate is defined for two hyperplanes $h_i,h_j$ that 
\emph{mutually straddle} each other:
\begin{equation}
\begin{aligned}
p_{\mathrm{crs}}(h_i,h_j)
:={}&p_{\mathrm{str}}(X_i,h_j)
\wedge p_{\mathrm{str}}(X_j,h_i).
\end{aligned}
\label{eq:mutual-straddling}
\end{equation}

\end{straddlingdefinition}

However, the CH property alone does not guarantee that $h_i$ and $h_j$ have
no ancestry relation. In particular, it does not rule out a larger tree in
which a third rule places $h_i$ and $h_j$ in different branches. This
distinction is illustrated by the counterexample in
Figure~\ref{fig:theorem3-counterexample} for $\mathbb{R}^3$.

The limitation arises when mutual straddling is used to infer intersection of
the corresponding convex hulls. In $\mathbb{R}^2$, mutual straddling implies
that both convex hulls meet $h_i\cap h_j$. For $D\geq 3$, however, they may
meet this intersection at different locations and therefore remain disjoint.
We then strengthen mutual straddling by requiring the convex hulls of
the defining point sets to intersect. As we show next, this condition yields a
sound pruning criterion in every dimension.

\begin{theorem}[Sound fatal-pair pruning]
\label{When-two-hyperplanes}
Let $h_i$ and $h_j$ be affine hyperplane rules with defining point sets
$X_i$ and $X_j$. If the \emph{fatal-pair} predicate
\begin{equation}
\begin{aligned}
p_{\mathrm{fatal}}(h_i,h_j):={}&
p_{\mathrm{str}}(X_i,h_j)
\wedge p_{\mathrm{str}}(X_j,h_i)
\wedge
[\operatorname{conv}(X_i)\cap\operatorname{conv}(X_j)\neq\varnothing]
\end{aligned}
\label{eq:fatal-pair}
\end{equation}
holds, where $\operatorname{conv}(X_i)$ denotes the convex hull of $X_i$,
then neither rule can be an ancestor of the other; no affine hyperplane rule
can place them in different branches; and no proper decision tree satisfying
\eqref{eq:defining-points-in-region} can contain both rules.
\end{theorem}

The proof is given in Appendix \ref{subsec: proof-thm3}. The result extends
to degree-$M$ polynomial rules after applying the monomial feature map
$\phi_M:\mathbb{R}^D\to\mathbb{R}^G$ from Theorem
\ref{thm: 0-1 loss classification theorem}.

\paragraph{Safe filtering and incremental generation}

A tree of size $K$ contains $K$ rules. A configuration is safe precisely when
it contains no fatal pair. Accordingly, we define the configuration predicate
as the conjunction of the negations of all pairwise fatal-pair predicates:
\begin{equation}
p_{\mathrm{safe}}(\mathit{hs})
:=\bigwedge_{1\leq i<j\leq |\mathit{hs}|}
\neg p_{\mathrm{fatal}}(h_i,h_j),
\label{eq:safe-configuration-predicate}
\end{equation}
To replace the $\mathit{nestedCombs}$ procedure in
Algorithm~\ref{alg:odt-prime}, we introduce a filtered version that
incorporates this predicate:
\begin{equation}
\mathit{nestedCombsFA}(K,G)
=\mathit{filter}_{p_{\mathrm{safe}}}
\circ\mathit{map}_{\mathit{calARM}}
\circ\mathit{nestedCombs}(K,G).
\label{eq: nestedcombsFA}
\end{equation}
Evaluating this filter after generating a candidate requires $O(K^2)$
pairwise fatal-pair tests.

This $O(K^2)$ cost per candidate is inefficient because
$\mathit{nestedCombs}(K,G)$ produces $O(N^{KG})$ candidates. Crucially, the
predicate in \eqref{eq:safe-configuration-predicate} is
\textbf{prefix-closed}: if a prefix of $\mathit{hs}$ is unsafe, then
$\mathit{hs}$ is also unsafe. We can therefore filter configurations
incrementally during generation, rather than only after they have been fully
constructed. Accordingly, we develop Algorithm~\ref{alg:nestedCombsFA}, which
checks the safety of a configuration as it is built. The following result
shows that this algorithm has lower complexity than the original
$\mathit{nestedCombs}$ procedure in Algorithm~\ref{alg:odt-prime}.

\begin{theorem}
	The programs $\mathit{filter}_{p_{\mathrm{safe}}}$, $\mathit{map}_{\mathit{calARM}}$,
	and $\mathit{nestedCombs}$ can be fused into the single algorithm
	$\mathit{nestedCombsFA}$, as defined in Algorithm~\ref{alg:nestedCombsFA},
	thereby reducing the number of fatal-pair tests per generated configuration
	from $O\left(K^2\right)$ to $O\left(K\right)$.
\end{theorem}

Combining the preceding results, the overall worst-case work complexity of the
$\mathit{odt}^{\prime}$ algorithm is $O\left(K!K N^{KG+1}\right)$ under
the assumptions stated in Theorem~\ref{thm: odt^prime complexity}, as
proved in Theorem~\ref{thm: odt^prime complexity}.

\section{Empirical results\label{sec:Empirical-results}}

We evaluate four aspects of the proposed method: (i) the sequential and
parallel scalability of $\mathit{sodt}_{\text{vec}}$ relative to
$\mathit{sodt}_{\text{rec}}$; (ii) the effectiveness of sound fatal-pair pruning; (iii) Evaluating the robustness of HDT trees over different source of noise: tree sizes, dimensions, sample sizes, and noise on labels/data, each one is evaluated on 50 synthetic datasets; (iv) generalization performance on 30 real-world datasets. All experiments use an Intel i9-14900KF CPU and a
single GeForce RTX 4060 Ti GPU. We report the main results here and defer
additional results and discussion to Appendix~\ref{sec: experiments-appendix}.

\newlength{\mergedplotheight}
\setlength{\mergedplotheight}{0.245\linewidth}

\begin{figure}[H]
	\centering
	\begin{minipage}[t]{0.46\linewidth}
		\vspace{0pt}
		\centering
		\resizebox*{!}{\mergedplotheight}{%
			\setlength{\tabcolsep}{0pt}%
			\begin{tabular}{@{}ccc@{}}
				\includegraphics[width=0.29\textwidth]{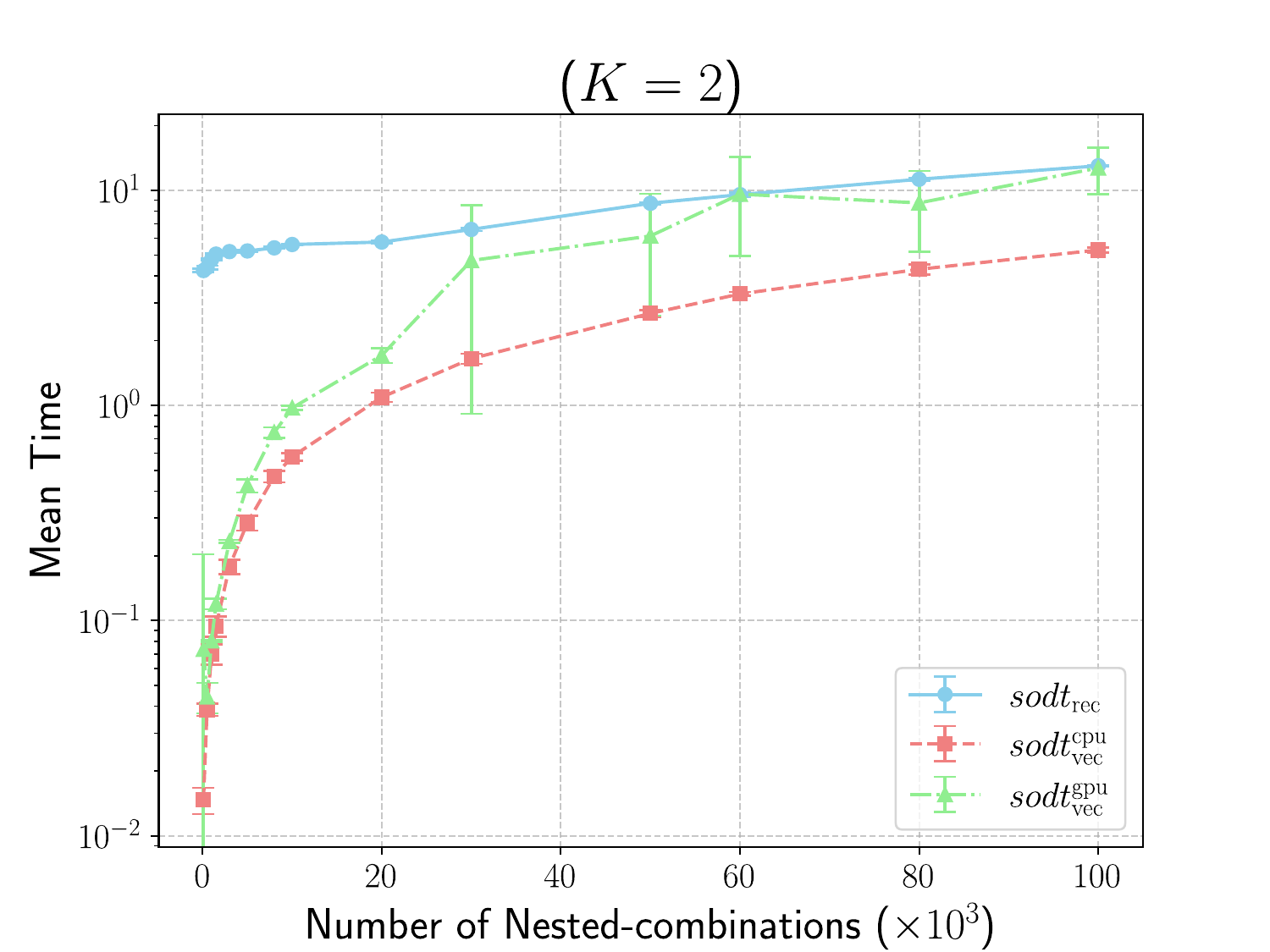} &
				\includegraphics[width=0.29\textwidth]{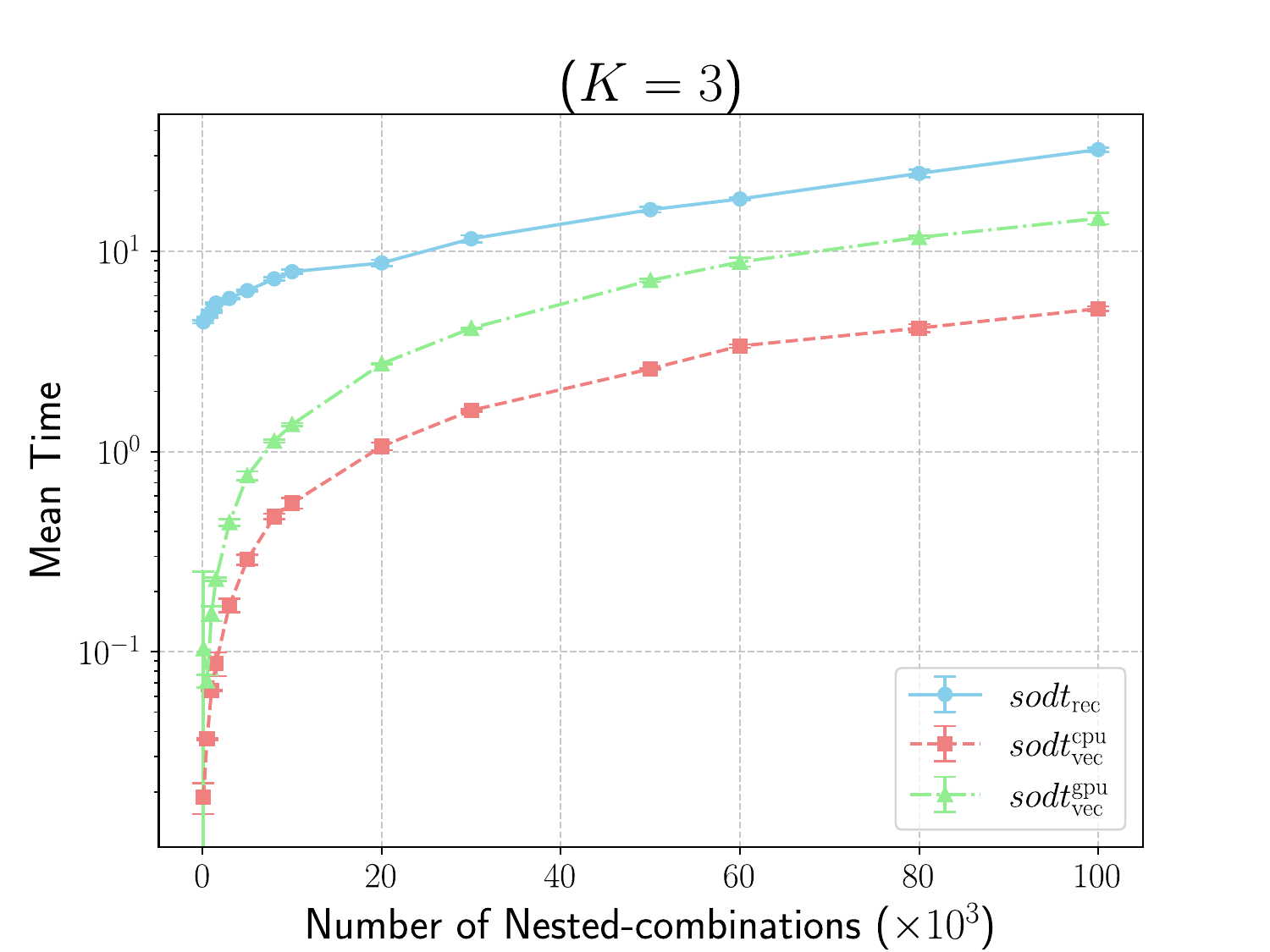} &
				\includegraphics[width=0.29\textwidth]{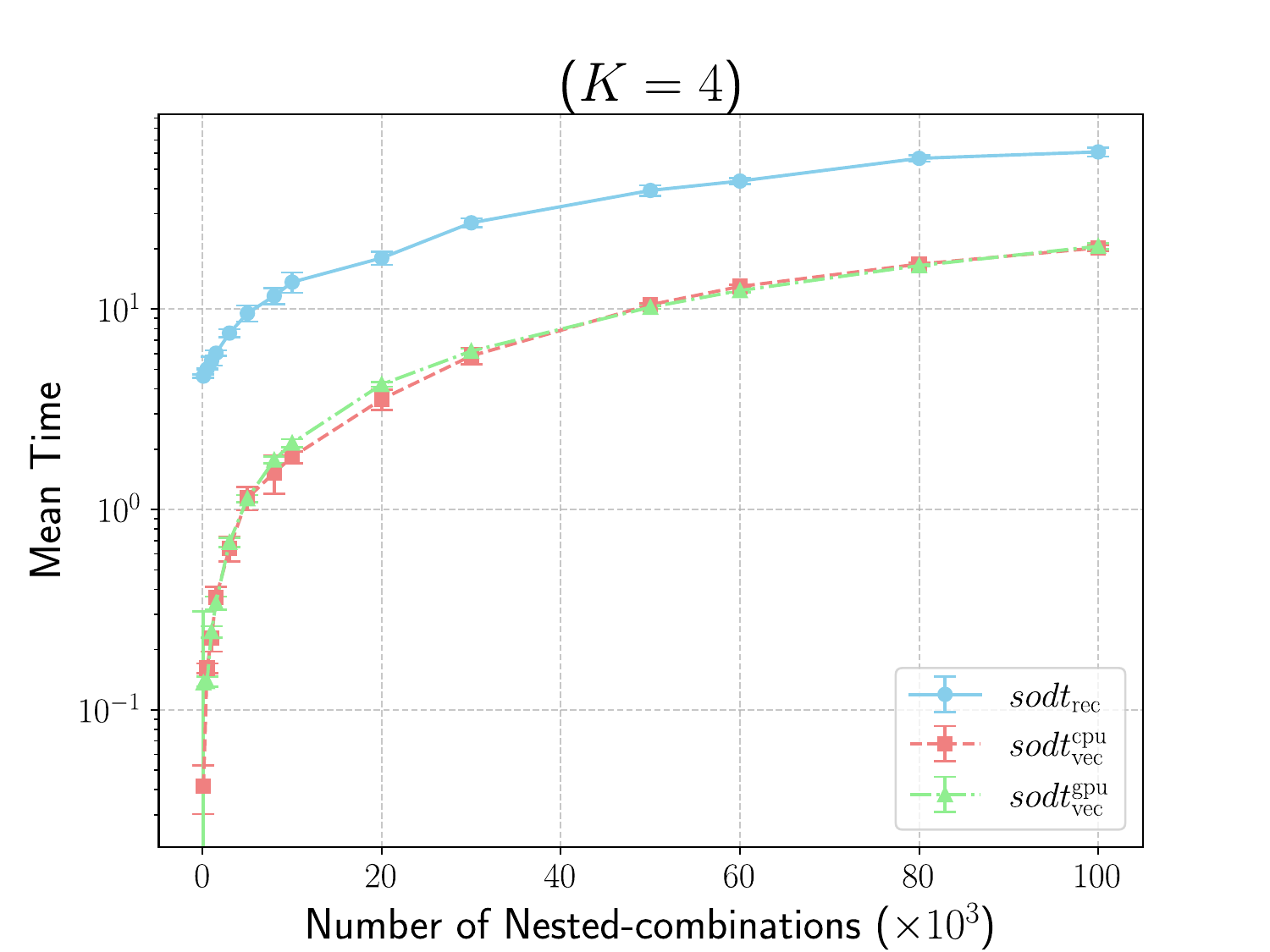} \\
				\includegraphics[width=0.29\textwidth]{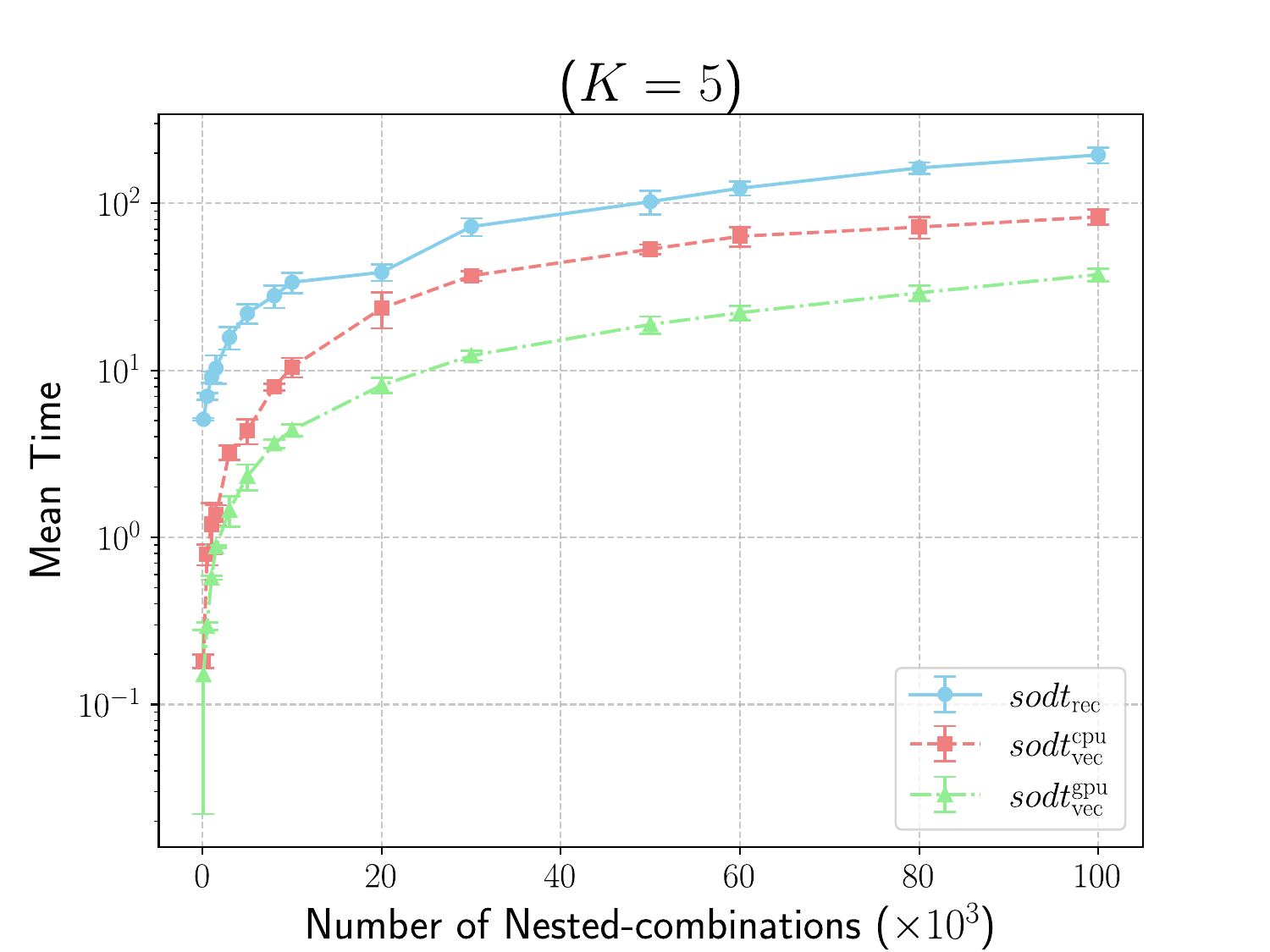} &
				\includegraphics[width=0.29\textwidth]{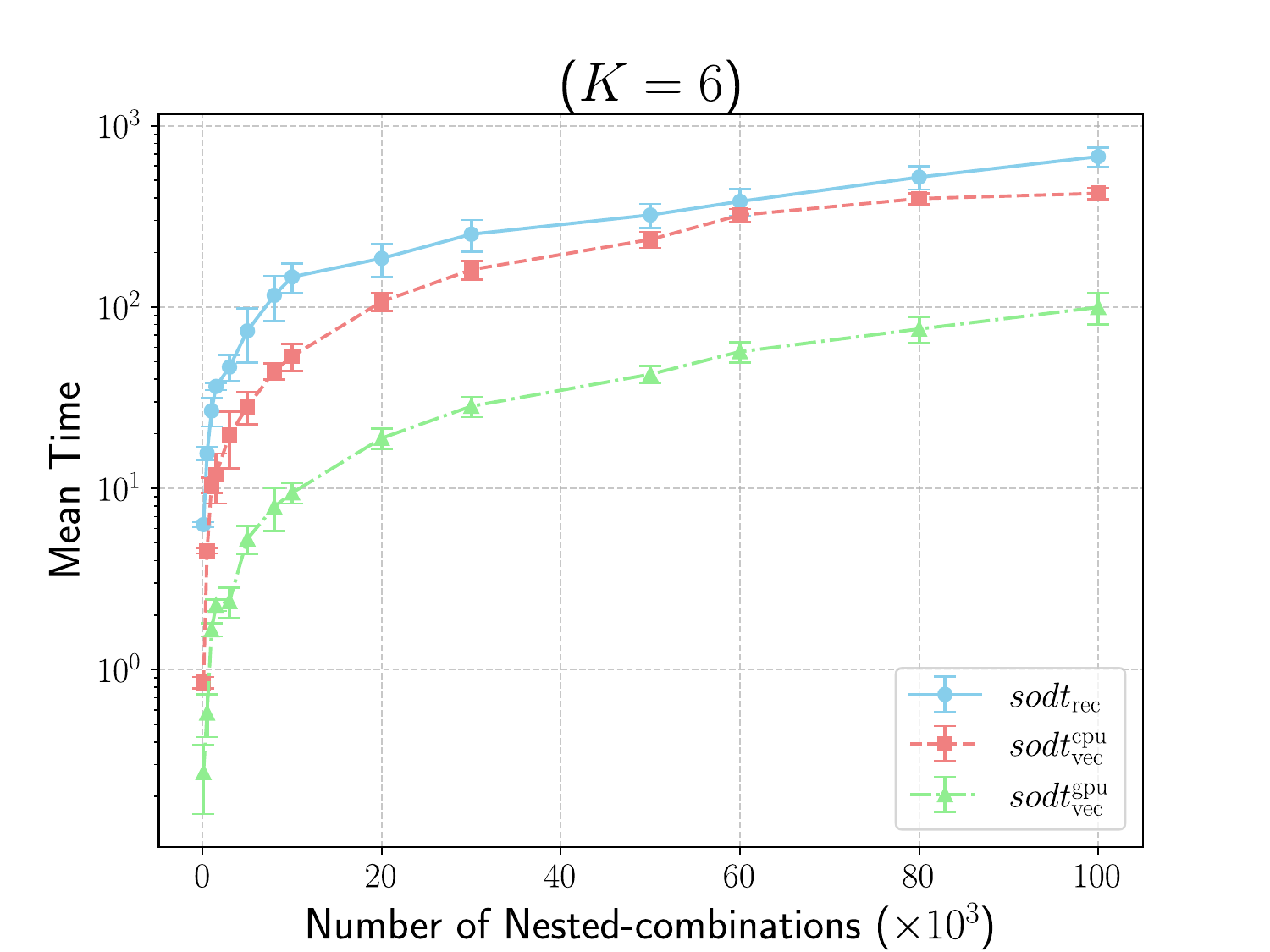} &
				\includegraphics[width=0.29\textwidth]{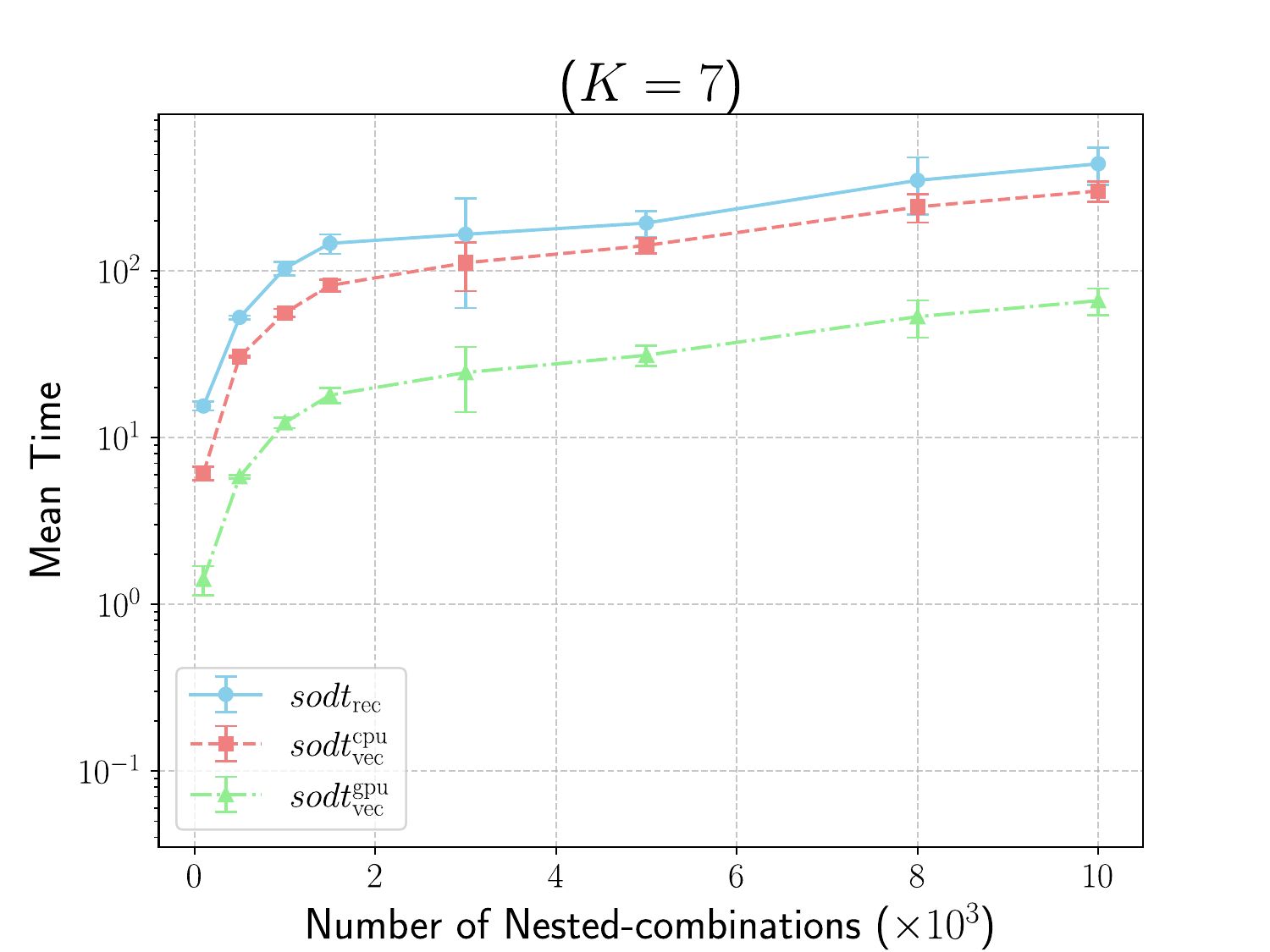}
			\end{tabular}%
		}
	\end{minipage}%
	\hfill
	\begin{minipage}[t]{0.43\linewidth}
		\vspace{0pt}
		\centering
		\includegraphics[viewport=10bp 0bp 850bp 500bp,clip,
			height=\mergedplotheight]{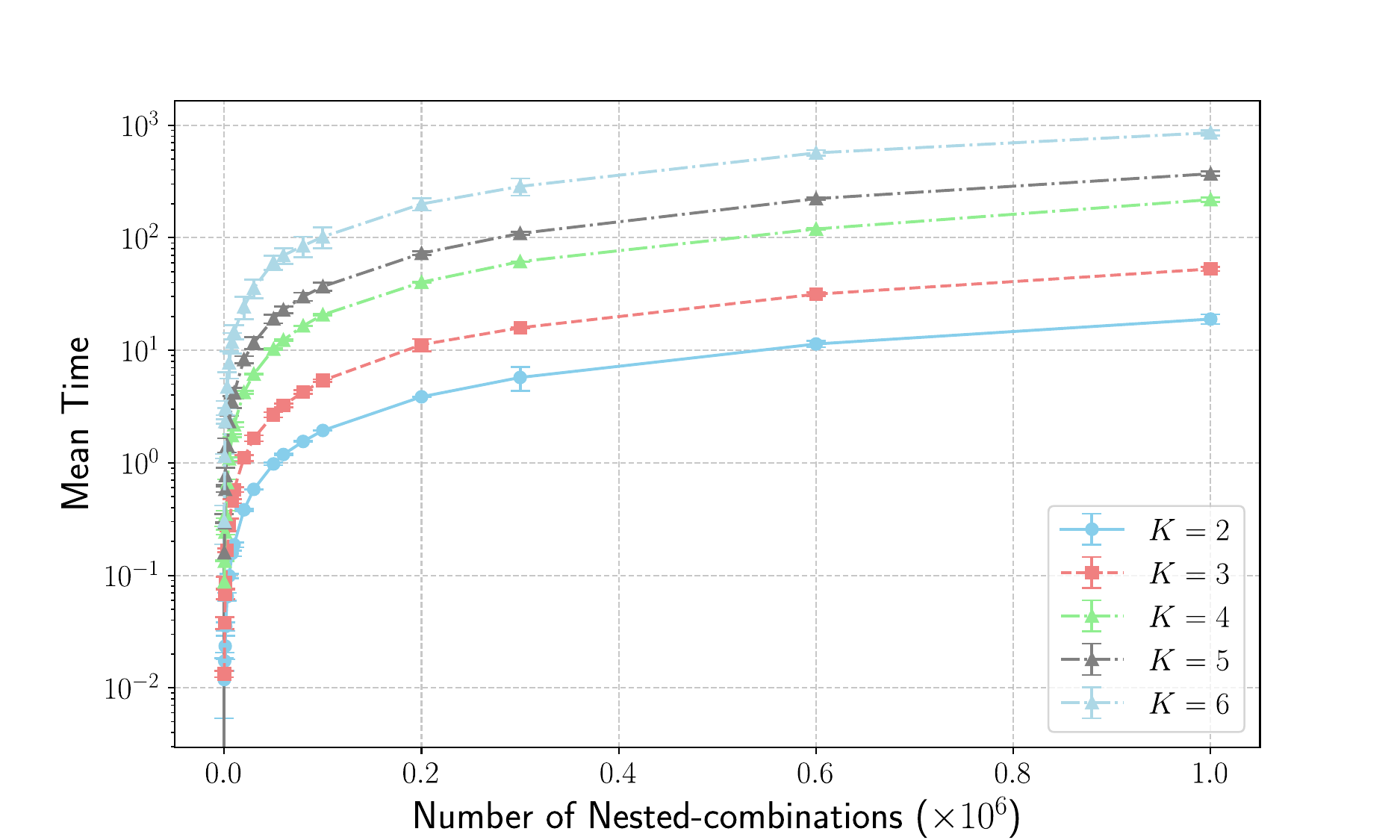}
	\end{minipage}
	\caption{Parallel run times for $\mathit{sodt}_{\text{rec}}$ and
		$\mathit{sodt}_{\text{vec}}$ across $K$ (left), and log-log run time for
		$\mathit{sodt}_{\text{vec}}$ up to $1\times10^{6}$ nested combinations
		(right).}
	\label{fig:parallel-runtime-main}
	\label{fig:efficiency of sodt_vec}
	\label{fig:merged-runtime-comparison}
\end{figure}

\subsection{Computational scalability comparison\label{subsec: sodt-comparision}}

Figure~\ref{fig:efficiency of sodt_vec} shows that
$\mathit{sodt}_{\text{vec}}$ is substantially faster than
$\mathit{sodt}_{\text{rec}}$ in parallel. For $K=2$,
$\mathit{sodt}_{\text{vec}}$ explores $0.6\times10^{6}$ feasible nested
configurations---corresponding to up to $1.2\times10^{6}$ candidate
trees---in approximately 10 seconds. Appendix
\ref{subsec:Computational-and-combinatorial analysis} presents the sequential
comparison.

\noindent
\begin{minipage}[t]{0.53\textwidth}
	\vspace{0pt}
	\subsection{Effectiveness of the safe pruning}

	As shown in Figure~\ref{fig:D2_CH}, the combinatorial complexity after
	safe filtering (bars) is substantially lower than the theoretical upper
	bound $O(N^{DK})$ (black lines). Moreover, the complexity curve is
	inverted U-shaped, indicating that the proportion of infeasible
	configurations increases with $K$ and does so at an increasing rate.
	Appendix \ref{subsec:Fatal-pair experiments} presents additional
	results for other values of $D$ and $M$.
\end{minipage}%
\hfill
\begin{minipage}[t]{0.36\textwidth}
	\vspace{0pt}
	\centering
	\includegraphics[width=0.96\linewidth]{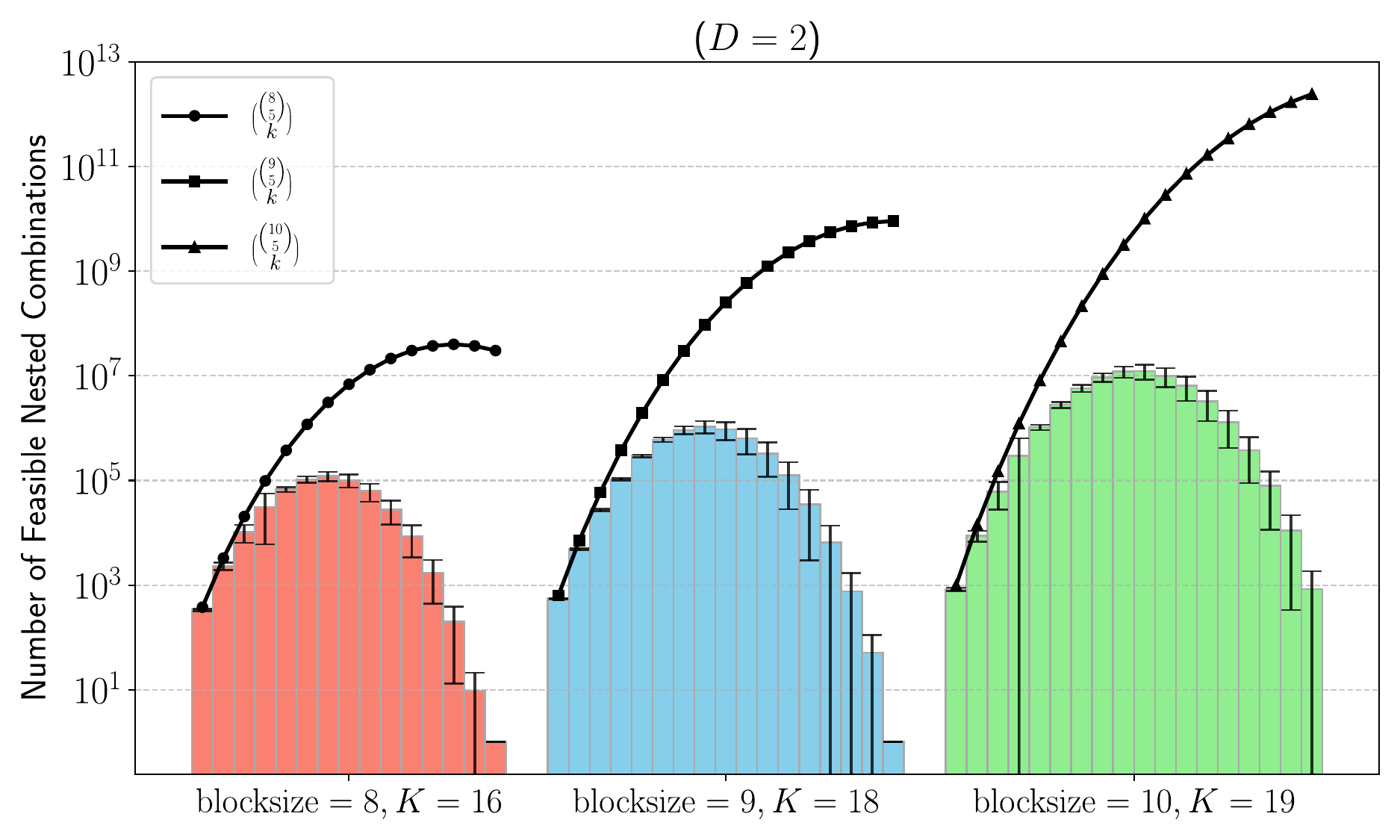}
	\captionof{figure}{Combinatorial complexity of decision tree search space}
	\label{fig:D2_CH}
\end{minipage}

\subsection{Analysis on synthetic datasets}
\begin{wraptable}{r}{0.62\textwidth}
	\centering
	\vspace{-0.8\baselineskip}
	\resizebox{\linewidth}{!}{%
		\begin{minipage}{0.98\linewidth}
			\centering
			\tiny
			\renewcommand{\arraystretch}{0.92}
			\setlength{\tabcolsep}{1.5pt}
			\begin{tabular}{clrrrrrr}
				\toprule
				Noise & Method & Train acc. & Test acc. & Size & $D_{\max}$ & $D_{\mathrm{avg}}$ & $D_{\mathrm{exp}}$ \\
				\midrule
				\multirow{5}{*}{5\%}
				& CART-depth-2 & 86.29 (4.48) & 87.15 (4.89) & 2.94 (0.24) & 2.00 (0.00) & 1.98 (0.08) & 1.99 (0.02) \\
				& CART-size-4  & 90.64 (3.07) & 91.01 (3.67) & 4.00 (0.00) & 3.20 (0.40) & 2.50 (0.16) & 2.10 (0.24) \\
				& ConTree-2    & 89.07 (3.45) & 89.87 (3.47) & 3.00 (0.00) & 2.00 (0.00) & 2.00 (0.00) & 2.00 (0.00) \\
				& HDT         & 95.00 (0.20) & 97.20 (1.40) & 2.00 (0.00) & 2.00 (0.00) & 1.68 (0.00) & 1.62 (0.29) \\
				& Ground truth & 95.00 (0.00) & 100 (0.00)   & 2.00 (0.00) & 2.00 (0.00) & 1.67 (0.00) & 1.68 (0.28) \\
				\midrule
				\multirow{5}{*}{15\%}
				& CART-depth-2 & 77.90 (3.99) & 85.99 (5.71) & 2.98 (0.14) & 2.00 (0.00) & 1.99 (0.05) & 2.00 (0.04) \\
				& CART-size-4  & 81.80 (3.10) & 90.19 (3.67) & 4.00 (0.00) & 3.22 (0.41) & 2.51 (0.17) & 2.19 (0.31) \\
				& ConTree-2    & 80.53 (3.04) & 89.10 (3.79) & 3.00 (0.00) & 2.00 (0.00) & 2.00 (0.00) & 2.00 (0.00) \\
				& HDT         & 85.69 (0.54) & 97.10 (1.84) & 2.00 (0.00) & 2.00 (0.00) & 1.67 (0.00) & 1.63 (0.29) \\
				& Ground truth & 84.99 (0.00) & 100 (0.00)   & 2.00 (0.00) & 2.00 (0.00) & 1.67 (0.00) & 1.68 (0.28) \\
				\midrule
				\multirow{5}{*}{25\%}
				& CART-depth-2 & 70.35 (3.22) & 84.59 (6.66) & 3.00 (0.00) & 2.00 (0.00) & 2.00 (0.00) & 1.99 (0.09) \\
				& CART-size-4  & 73.19 (2.54) & 87.78 (5.41) & 4.00 (0.00) & 3.26 (0.44) & 2.52 (0.17) & 2.05 (0.35) \\
				& ConTree-2    & 72.26 (2.62) & 87.06 (5.16) & 3.00 (0.00) & 2.00 (0.00) & 2.00 (0.00) & 1.99 (0.05) \\
				& HDT         & 76.37 (0.83) & 95.67 (2.93) & 2.00 (0.00) & 2.00 (0.00) & 1.67 (0.00) & 1.67 (0.78) \\
				& Ground truth & 75.00 (0.00) & 100 (0.00)   & 2.00 (0.00) & 2.00 (0.00) & 1.67 (0.00) & 1.68 (0.28) \\
				\bottomrule
			\end{tabular}
			\caption{Selected results from Table \ref{tab: the effect of noise on labels} for the effect of label noise.  Entries are means (standard deviations).\label{tab: the effect of noise on label (selected)}}
		\end{minipage}%
	}
	\vspace{-0.5\baselineskip}
\end{wraptable}
This experimental design builds on the setup of \citet{murthy1995decision},
extending its ADT experiments to the hyperplane setting.

Table~\ref{tab: the effect of noise on label (selected)} presents selected
results from Table~\ref{tab: the effect of noise on labels}, illustrating the
effect of label noise across methods. Tables~\ref{tab: the effect of tree size},
\ref{tab: Effects-of-dimensionality.}, and
\ref{tab: effect of noise on training data} in Appendix
\ref{sec: experiments-appendix} summarize the effects of the remaining
factors. As label noise increased, the accuracy of all methods declined, with
similar reductions in out-of-sample performance up to 20\% noise. Between
20\% and 25\% noise, however, the accuracy of CART-depth, CART-size, and
ConTree decreased by 1.47\%, 1.73\%, and 1.55\%, respectively, compared with
only 0.91\% for HDT. Moreover, from a noise level of 5\% onward, HDT found
solutions that \textit{exceeded the ground truth on the training dataset}
while maintaining strong out-of-sample performance. These results provide insights that high-training accuracy does not necessarily indicate overfitting, and provide
evidence against that optimal methods necessarily overfit.

\subsection{Analysis on real-world datasets }

\begin{wraptable}{r}{0.60\textwidth}
		\centering
		\vspace{-0.8\baselineskip}
		{\tiny
			\renewcommand{\arraystretch}{0.95}
			\setlength{\tabcolsep}{0.5pt}
			\resizebox{0.98\linewidth}{!}{%
			\begin{tabular}
			{
				>{\raggedright\arraybackslash}p{1.3cm}  
				>{\centering\arraybackslash}p{0.3cm}
				>{\centering\arraybackslash}p{0.3cm}
				>{\centering\arraybackslash}p{0.3cm}
				>{\centering\arraybackslash}p{1.3cm}
				>{\centering\arraybackslash}p{1.3cm}
				>{\centering\arraybackslash}p{1.3cm}
				>{\centering\arraybackslash}p{1.3cm}
				>{\centering\arraybackslash}p{1.3cm}
				>{\centering\arraybackslash}p{1.3cm}
				>{\centering\arraybackslash}p{1.3cm}
			}
			
			Dataset & $N$ & $D$ & $C$ & \begin{cellvarwidth}[t]
				\centering
				CART-depth
				
				$d=2$
			\end{cellvarwidth} & \begin{cellvarwidth}[t]
				\centering
				ConTree
				
				$d=2$
			\end{cellvarwidth} & \begin{cellvarwidth}[t]
				\centering
				CART-size
				
				$K=2$
			\end{cellvarwidth} & \begin{cellvarwidth}[t]
				\centering
				CART-size
				
				$K=3$
			\end{cellvarwidth} & \begin{cellvarwidth}[t]
				\centering
				HDT
				
				$K=2$
			\end{cellvarwidth} & \begin{cellvarwidth}[t]
				\centering
				HDT
				
				$K=3$
			\end{cellvarwidth}\tabularnewline\\
			
			\Xhline{0.9pt}\\
			haberman & 283 & 3 & 2 & \begin{cellvarwidth}[t]
				\centering
				75.58/73.68
				
				(1.62/3.14)
			\end{cellvarwidth} & \begin{cellvarwidth}[t]
				\centering
				76.99/72.98
				
				(0.93/3.06)
			\end{cellvarwidth} & \begin{cellvarwidth}[t]
				\centering
				74.25/72.63
				
				(1.26/4.52)
			\end{cellvarwidth} & \begin{cellvarwidth}[t]
				\centering
				75.31/73.68
				
				(1.82/2.94)
			\end{cellvarwidth} & \begin{cellvarwidth}[t]
				\centering
				79.38/\textbf{78.25}
				
				(1.23/3.60)
			\end{cellvarwidth} & \begin{cellvarwidth}[t]
				\centering
				\textbf{80.70}/78.24
				
				(1.39/4.87)
			\end{cellvarwidth}\tabularnewline
			BldTrns & 502 & 4 & 2 & \begin{cellvarwidth}[t]
				\centering
				75.71/72.08
				
				(1.05/2.02)
			\end{cellvarwidth} & \begin{cellvarwidth}[t]
				\centering
				77.16/71.09
				
				(0.60/1.70)
			\end{cellvarwidth} & \begin{cellvarwidth}[t]
				\centering
				75.71/72.08
				
				(1.05/02.02)
			\end{cellvarwidth} & \begin{cellvarwidth}[t]
				\centering
				76.91/74.06
				
				(0.56/03.09)
			\end{cellvarwidth} & \begin{cellvarwidth}[t]
				\centering
				79.38/\textbf{78.50}
				
				(1.16/4.37)
			\end{cellvarwidth} & \begin{cellvarwidth}[t]
				\centering
				\textbf{80.71}/78.25
				
				(1.12/3.86)
			\end{cellvarwidth}\tabularnewline
			sepsis & 975 & 3 & 2 & \begin{cellvarwidth}[t]
				\centering
				94.13/93.64
				
				(0.30/1.06)
			\end{cellvarwidth} & \begin{cellvarwidth}[t]
				\centering
				94.26/93.74
				
				(0.30/0.88)
			\end{cellvarwidth} & \begin{cellvarwidth}[t]
				\centering
				94.13/93.64
				
				(0.30/1.06)
			\end{cellvarwidth} & \begin{cellvarwidth}[t]
				\centering
				94.18/93.64
				
				(0.35/1.06)
			\end{cellvarwidth} & \begin{cellvarwidth}[t]
				\centering
				94.97/\textbf{93.85}
				
				(0.44/1.95)
			\end{cellvarwidth} & \begin{cellvarwidth}[t]
				\centering
				\textbf{95.39}/\textbf{93.85}
				
				(0.43/1.95)
			\end{cellvarwidth}\tabularnewline
			algerian & 243 & 14 & 2 & \begin{cellvarwidth}[t]
				\centering
				99.28/98.37
				
				(0.25/1.53)
			\end{cellvarwidth} & \begin{cellvarwidth}[t]
				\centering
				99.69/95.92
				
				(0.25/2.24)
			\end{cellvarwidth} & \begin{cellvarwidth}[t]
				\centering
				99.28/98.37
				
				(0.25/1.52)
			\end{cellvarwidth} & \begin{cellvarwidth}[t]
				\centering
				99.59/98.37
				
				(0.21/1.53)
			\end{cellvarwidth} & \begin{cellvarwidth}[t]
				\centering
				99.59/97.96
				
				(0.43/1.44)
			\end{cellvarwidth} & \begin{cellvarwidth}[t]
				\centering
				\textbf{100}/\textbf{98.78}
				
				(0.00/1.12)
			\end{cellvarwidth}\tabularnewline
			Cryotherapy & 89 & 6 & 2 & \begin{cellvarwidth}[t]
				\centering
				94.93/90.00
				
				(1.91/6.48)
			\end{cellvarwidth} & \begin{cellvarwidth}[t]
				\centering
				94.93/91.11
				
				(1.91/6.67)
			\end{cellvarwidth} & \begin{cellvarwidth}[t]
				\centering
				91.83/80.00
				
				(2.87/10.30)
			\end{cellvarwidth} & \begin{cellvarwidth}[t]
				\centering
				94.93/90.00
				
				(1.91/6.48)
			\end{cellvarwidth} & \begin{cellvarwidth}[t]
				\centering
				98.31/93.33
				
				(0.63/6.33)
			\end{cellvarwidth} & \begin{cellvarwidth}[t]
				\centering
				\textbf{99.16}/\textbf{93.33}
				
				(0.77/4.65)
			\end{cellvarwidth}\tabularnewline
			Caesarian & 72 & 5 & 2 & \begin{cellvarwidth}[t]
				\centering
				73.33/58.67
				
				(2.05/7.78)
			\end{cellvarwidth} & \begin{cellvarwidth}[t]
				\centering
				75.09/57.33
				
				(1.31/6.80)
			\end{cellvarwidth} & \begin{cellvarwidth}[t]
				\centering
				73.33/61.33
				
				(2.05/7.78)
			\end{cellvarwidth} & \begin{cellvarwidth}[t]
				\centering
				75.79/64.00
				
				(2.05/9.04)
			\end{cellvarwidth} & \begin{cellvarwidth}[t]
				\centering
				88.42/\textbf{85.33}
				
				(1.57/5.58)
			\end{cellvarwidth} & \begin{cellvarwidth}[t]
				\centering
				\textbf{91.23}/84.00
				
				(2.15/3.68)
			\end{cellvarwidth}\tabularnewline
			ecoli & 336 & 7 & 8 & \begin{cellvarwidth}[t]
				\centering
				80.37/79.12
				
				(1.28/2.35)
			\end{cellvarwidth} & \begin{cellvarwidth}[t]
				\centering
				81.11/\textbf{80.59}
				
				(0.80/3.65)
			\end{cellvarwidth} & \begin{cellvarwidth}[t]
				\centering
				76.04/75.00
				
				(1.38/2.94)
			\end{cellvarwidth} & \begin{cellvarwidth}[t]
				\centering
				81.11/77.35
				
				(0.80/1.77)
			\end{cellvarwidth} & \begin{cellvarwidth}[t]
				\centering
				80.73/77.64
				
				(0.97/4.08)
			\end{cellvarwidth} & \begin{cellvarwidth}[t]
				\centering
				\textbf{82.76}/79.12
				
				(1.10/3.81)
			\end{cellvarwidth}\tabularnewline
			GlsId & 213 & 9 & 6 & \begin{cellvarwidth}[t]
				\centering
				62.35/62.33
				
				(1.18/5.38)
			\end{cellvarwidth} & \begin{cellvarwidth}[t]
				\centering
				67.29/\textbf{62.79}
				
				(1.15/04.88)
			\end{cellvarwidth} & \begin{cellvarwidth}[t]
				\centering
				61.76/62.33
				
				(1.18/5.38)
			\end{cellvarwidth} & \begin{cellvarwidth}[t]
				\centering
				66.47/63.72
				
				(1.12/5.22)
			\end{cellvarwidth} & \begin{cellvarwidth}[t]
				\centering
				71.18/61.86
				
				(2.46/4.53)
			\end{cellvarwidth} & \begin{cellvarwidth}[t]
				\centering
				\textbf{74.25}/\textbf{62.79}
				
				(2.29/5.20)
			\end{cellvarwidth}\tabularnewline
			
			\end{tabular}%
			}
		}
		\caption{(Part I: $K$ in {2, 3}) 
	Five-fold cross-validation results. We compare
		the performance of our HDT model, with $K$ (number of splitting
		rules) ranging from 2 to 3 and the depth is fixed as 2, against approximate
		methods:size- and depth-constrained CART algorithms (CART-size and
		CART-depth), as well as the state-of-the-art optimal ADT algorithm,
		ConTree \citep{brita2025optimal}. Results are reported as mean accuracy on the training and test sets in the format \emph{Training Error
			/ Test Error (Standard Deviation: Train / Test)}. The best-performing
		algorithm in each row is shown in \textbf{bold}. Complete results appear in
		Table~\ref{tab:Five-fold-cross-validation-resul}.\label{tab: subset table} }
	\vspace{-0.5\baselineskip}
\end{wraptable}

This paper aims to demonstrate that more expressive tree models can learn the ground truth more accurately. Because HSDTs have intractable combinatorial complexity, they are difficult to optimize within a reasonable time on real-world datasets. We therefore compare HDTs with classical ADTs. With model complexity controlled, HDTs consistently outperform ADTs across a range of real-world and synthetic datasets. In the strongest cases, HDTs improve training accuracy by more than 20\% and test accuracy by nearly 30\% compared with the optimal ADT algorithm.

Table~\ref{tab: subset table} reports only Part I
of the experimental results. Part II is reported in
Table~\ref{tab:Five-fold-cross-validation-resul}. Appendix
\ref{subsec: real-world} provides further details on the experimental settings
and results for larger tree sizes.

\WFclear
\vspace{0.5\baselineskip}

\section{Conclusion}

We present the first optimal hypersurface decision tree algorithm, with two
complementary speed-ups. First, a generic vectorized solver for the
$\mathit{sodt}$ problem enables efficient GPU parallelization, outperforms
custom dynamic-programming approaches even on a sequential CPU, and applies
to any tree in the PDT family. Second, for HDTs and HSDTs, safe fatal-pair pruning
is fused directly into candidate generation. Experiments confirm the
efficiency of both techniques and show that, under controlled model
complexity, our algorithm significantly outperforms classical axis-parallel
decision trees on synthetic and more than 30 real-world datasets.

\subsection*{AI use statement}

In this work, we used generative AI tools for polishing English grammar and spelling, and for improving the clarity of the text. We have reviewed all AI-assisted work and take responsibility for the final content of this work, including text, claims or artifacts produced with the aid of generative AI.

\bibliographystyle{iclr2027_conference}
\bibliography{Bibliography}

\appendix

\section{Additional comparison}

\begin{figure}[H]
	\centering
	
	\begin{subfigure}[b]{0.3\textwidth}
		\centering
		\includegraphics[viewport=50bp 60bp 480bp 475bp,clip,width=\textwidth]{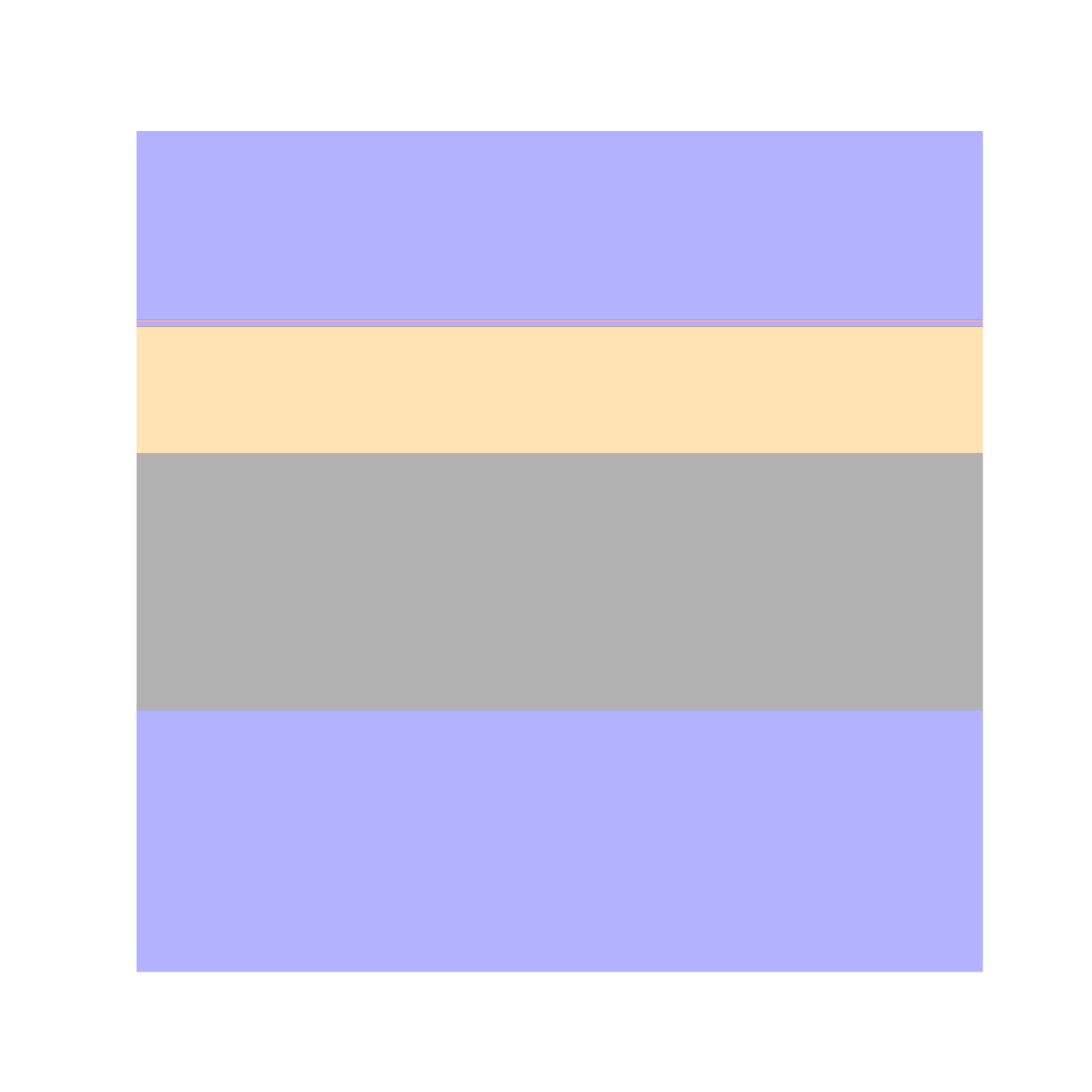}
		\caption{\centering CART\\nmc: 114 $d$: 3, $K$: 6}
		\label{fig:cart1}
	\end{subfigure}
	\hfill
	\begin{subfigure}[b]{0.3\textwidth}
		\centering
		\includegraphics[viewport=50bp 60bp 480bp 475bp,clip,width=\textwidth]{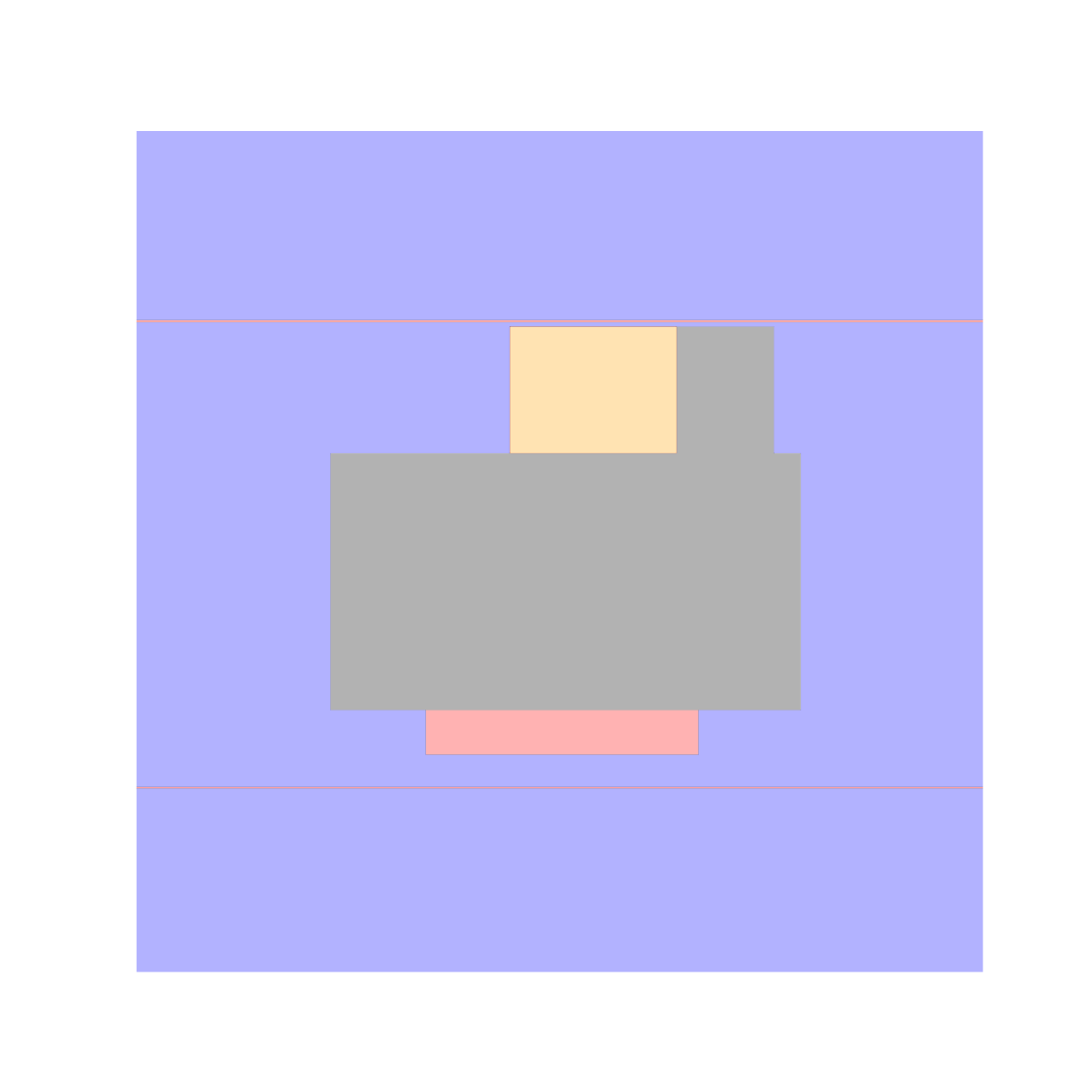}
		\caption{\centering CART\\nmc: 56 $d$: 5, $K$: 15}
		\label{fig:cart2}
	\end{subfigure}
	\hfill
	\begin{subfigure}[b]{0.3\textwidth}
		\centering
		\includegraphics[viewport=50bp 60bp 480bp 475bp,clip,width=\textwidth]{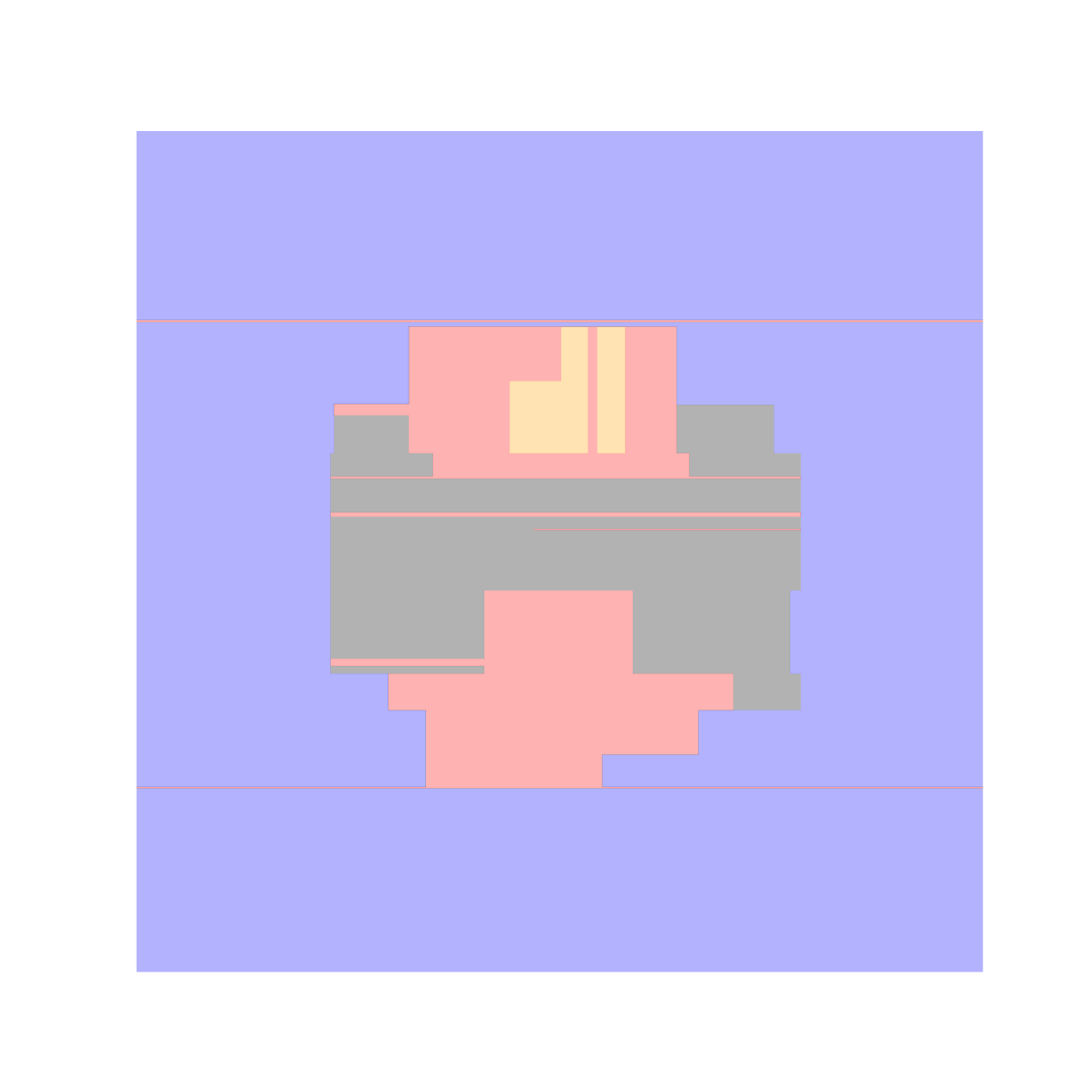}
		\caption{\centering CART\\nmc: 0 $d$: 12, $K$: 45}
		\label{fig:cart3}
	\end{subfigure}
	
	\vspace{0.5em}
	
	\begin{subfigure}[b]{0.3\textwidth}
		\centering
		\includegraphics[viewport=50bp 60bp 480bp 475bp,clip,width=\textwidth]{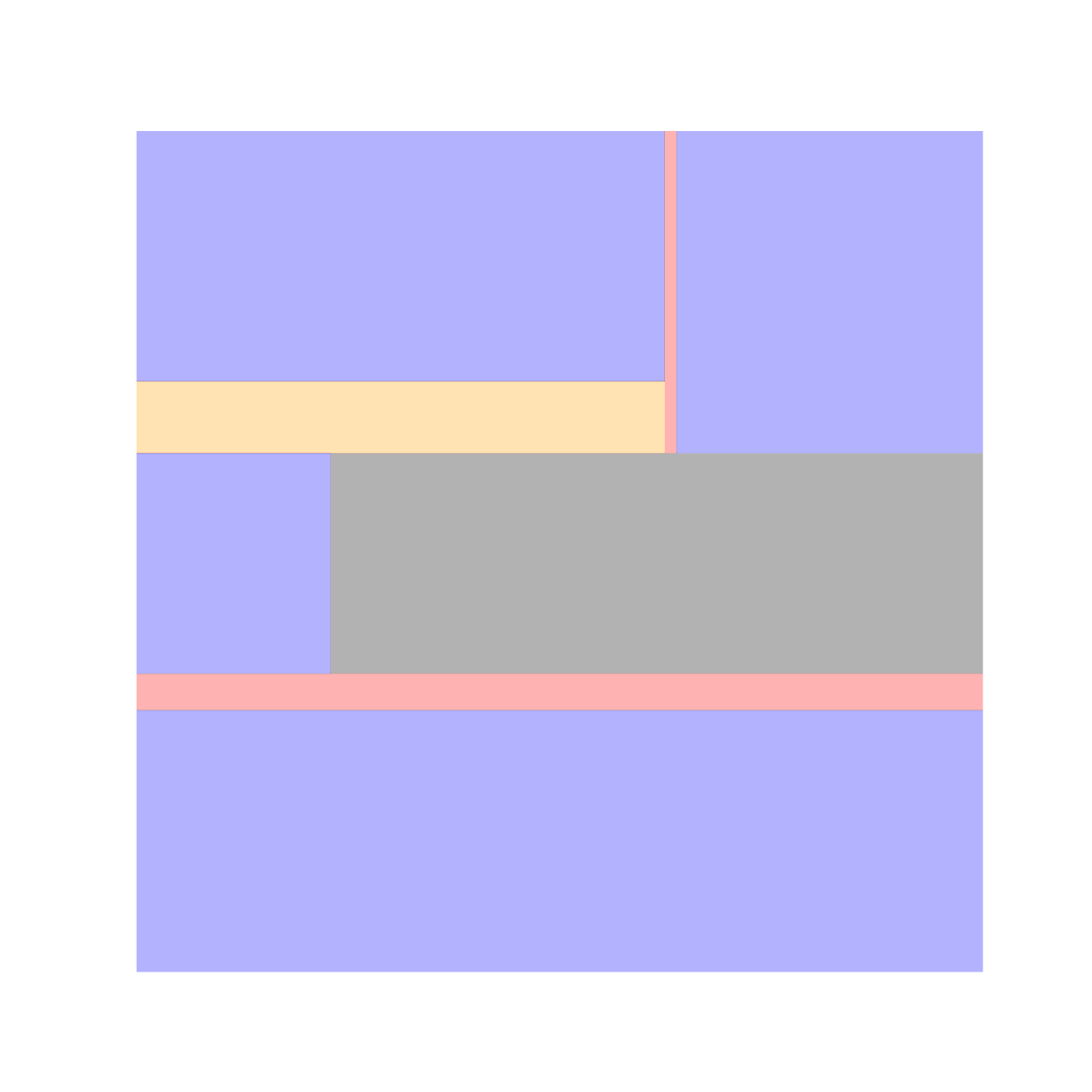}
		\caption{\centering ConTree\\nmc: 69 $d$: 3, $K$: 7}
		\label{fig:contree1}
	\end{subfigure}
	\hfill
	\begin{subfigure}[b]{0.3\textwidth}
		\centering
		\includegraphics[viewport=50bp 60bp 480bp 475bp,clip,width=\textwidth]{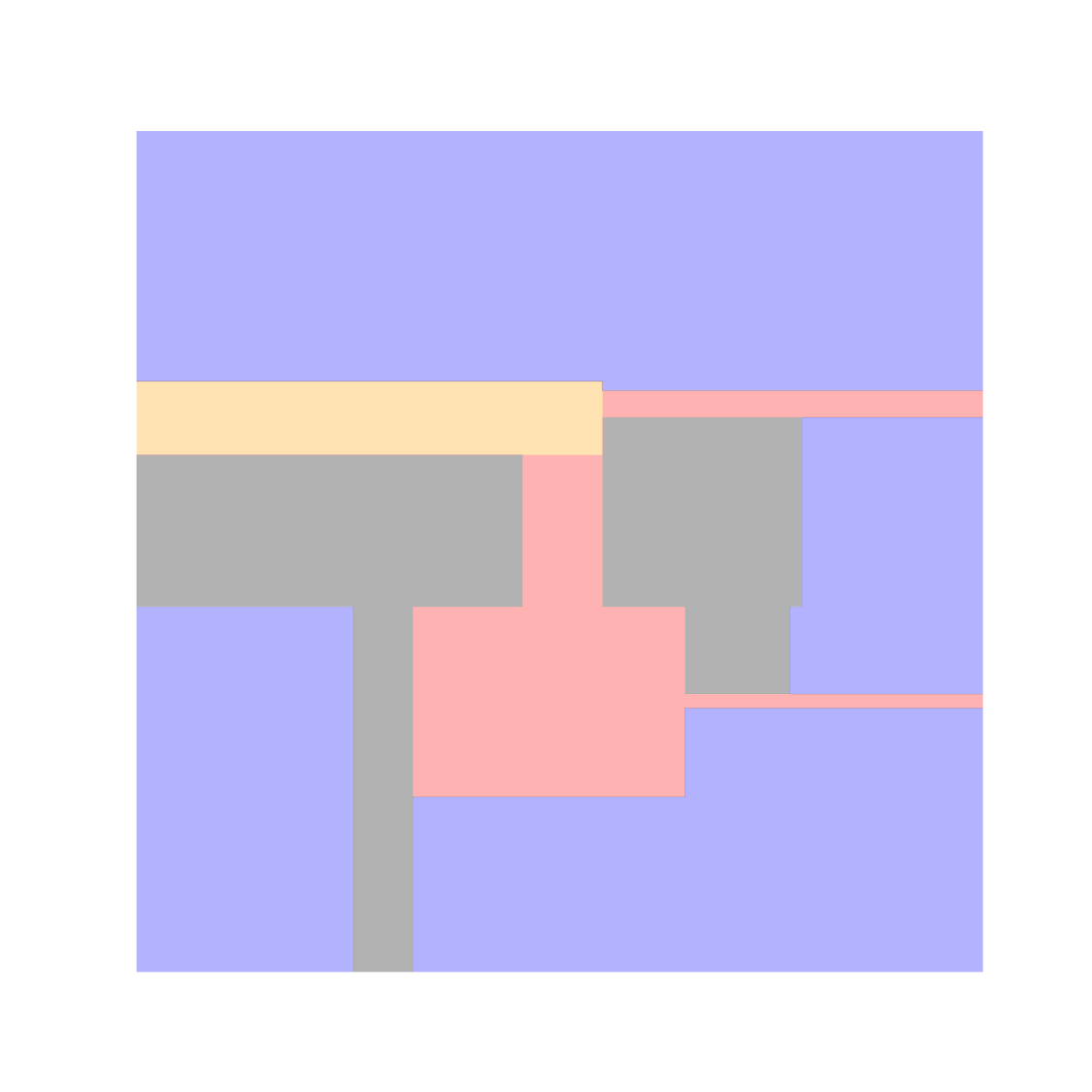}
		\caption{\centering ConTree\\nmc: 34 $d$: 4, $K$: 15}
		\label{fig:contree2}
	\end{subfigure}
	\hfill
	\begin{subfigure}[b]{0.3\textwidth}
		\centering
		\includegraphics[viewport=50bp 60bp 480bp 475bp,clip,width=\textwidth]{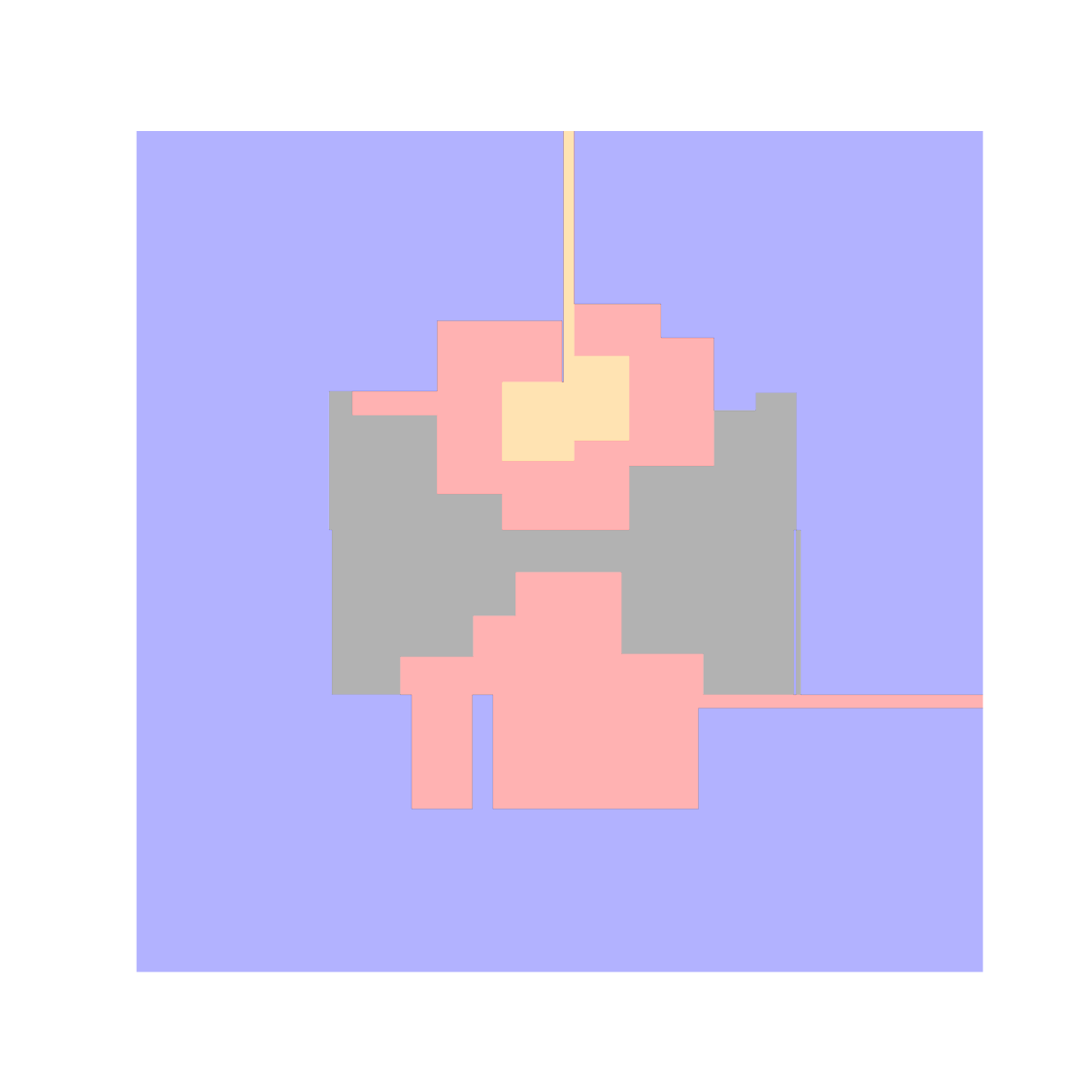}
		\caption{\centering ConTree\\nmc: 0 $d$: 6, $K$: 48}
		\label{fig:contree3}
	\end{subfigure}
	
	\caption{Same ground truth synthetic dataset as \ref{fig:comparision}. 
		The above three models are axis-parallel decision trees learned by CART and the below three are axis-parallel decision trees learned by ConTree algorithm.}
	\label{fig:comparision-full}
\end{figure}

\section{Proper decision tree \label{sec: PDT}}

\subsection{Definition of polynomial hypersurface}
 A polynomial hypersurface in $\mathbb{R}^D$ is the zero set of a polynomial function $f: \mathbb{R}^D \to \mathbb{R}$, i.e., the set of points $\mathbf{x} = (x_1, x_2, \ldots, x_D)$ such that $f(\mathbf{x}) = 0$. Let $\boldsymbol{\alpha} = (\alpha_1, \alpha_2, \ldots, \alpha_D)$ be a multi-index exponent vector, and $\mathbf{x}^{\mathbf{\alpha}}=x_{1}^{\alpha_{1}}\cdot x_{2}^{\alpha_{2}}\ldots x_{D}^{\alpha_{D}}$.
The degree of a polynomial is defined as the maximum total degree
of its monomial terms, i.e., $\left|\boldsymbol{\mathbf{\alpha}}\right|$.
For example, a hyperplane corresponds to a first-degree polynomial
and can be written as $w_{1}x_{1}+w_{2}x_{2}\ldots w_{D}x_{D}+w_{0}=0$
while a conic section (a quadratic polynomial) can be expressed as
$w_{1}x_{1}^{2}+w_{2}x_{1}x_{2}+ \ldots + w_{D}x_{D}^{2}+w_{0}=0$. 

\subsection{Axiomatic definition }

1) Each branch (internal) node of the decision tree subdivides the
ambient space into two \emph{disjoint} and connected subspaces 2)
and the leaves of the decision tree is the intersection of subspaces
for all the splitting rules 3) the relationship between splitting rules
is transitive, namely if $r_{j}$ lies in the subspace of $r_{i}$
and $r_{i}$ lies in the subspace of $r_{k}$, then $r_{j}$ lies
in the subspace of $r_{k}$; 4) Notably a distinguishing feature of \citet{he2025FoODT_I}'s
axiom is that the left-children of a splitting rule \textbf{cannot}
be the\textbf{ }right-children. 

Formally, we have following definitions 
\begin{definition}
	\emph{Axioms for proper decision trees}. We call a decision tree consists
	of splitting rules that satisfies the following axioms, a \emph{proper}
	decision tree:\label{axioms: proper decision tree}
\end{definition}
\begin{enumerate}
	\item \label{axiom:pdt-ambient-partition}\emph{Structural constraint one}
	(Ambient space partition): Each branch
	node is defined by a single splitting rule $r:\mathcal{R}$, and each
	splitting rule subdivides the ambient space into two \emph{disjoint}
	and connected subspaces, $r^{+}$ and $r^{-}$.
	\item \label{axiom:pdt-path-intersection}\emph{Structural constraint two}
	(Path intersection in leaves): Each
	leaf $L$ is defined by the intersection of subspaces $\bigcap_{p\in P_{L}}r_{p}^{\pm}$
	for all the splitting rules $\left\{ r_{p}\mid p\in P_{L}\right\} $
	in the path $P_{L}$ from the root to leaf $L$. The connected region
	(subspace) defined by $\bigcap_{p\in P_{L}}r_{p}^{\pm}$ is referred
	to as the \emph{decision} \emph{region}.
	\item \label{axiom:pdt-partition-transitivity}\emph{Structural constraint three}
	(Partition transitivity): The ancestry
	relation between any pair of splitting rules $r_{i}\left(\swarrow\vee\searrow\right)r_{j}$
	is transitive; in other words, if $r_{i}\left(\swarrow\vee\searrow\right)r_{j}$
	and $r_{j}\left(\swarrow\vee\searrow\right)r_{k}$ then $\neg(r_{i}\swarrow r_{j}) \wedge \neg(r_{i}\swarrow r_{j})$.
	Moreover, $\boldsymbol{A}_{ij}=\pm1$ if $r_{j}$ can be generated
	from $r_{i}^{\pm}$. As a result, any new decision rule $r$ added
	to a leaf must be generated within its corresponding decision region.
	\item \label{axiom:pdt-ancestry-uniqueness}\emph{Ancestral constraint}
	(Uniqueness of the ancestry relation):
	For any pair of splitting rules $r_{i}$ and $r_{j}$, only one of
	the following three cases is true: $r_{i}\swarrow r_{j}$, $r_{i}\searrow r_{j}$,
	and $r_{i}\overline{\left(\swarrow\vee\searrow\right)}r_{j}$; additionally,
	$r_{i}\overline{\left(\swarrow\vee\searrow\right)}r_{i}$ is always
	true; in other words, the possible value of $\boldsymbol{A}_{ij}\in\left\{ 1,0,-1\right\} $
	is unique determined for all $i,j$ , and $\boldsymbol{A}_{ii}=0$
	for all $i$.
\end{enumerate}
There are two important theorems that can be derived based on the
PDT axioms. 
\begin{theorem}
	Proper decision tree characterization theorem\emph{. A decision tree
		consisting of $K$ splitting rules corresponds to a unique $K$-permutation
		 if and only if it is proper.\label{thm: PDT-characterization}}
\end{theorem}
Furthermore, $K$-permutations are among the most extensively studied
combinatorial objects in the literature. By characterizing decision
trees in terms of these permutations, we obtain the following theorem,
which describes how to decompose the size-constrained ODT problem
as follows.
\begin{theorem}
	Simplified optimal decision tree problem\emph{. Given a list of rules
		$\mathit{rs}:\left[\mathcal{R}\right]$ and a size constraint $K:\mathbb{N}$.
		Let the search space }$\mathcal{S}\left(K,\mathit{rs}\right)$ \emph{of
		size-$K$ decision trees be defined by the program $\mathit{genDTKs}\left(K,\mathit{rs}\right)$.
		Then there exists a dynamic programming algorithm $\mathit{sodt}$
		such that the following equivalence holds:} \label{thm:  PDT decomposition}\emph{
		\begin{equation}
			\mathit{min}_{E}\left(\mathit{concatMapL}\left(\mathit{sodt},\mathit{kcombs}\left(K,\mathit{rs}\right)\right)\right)\subseteq\mathit{min}_{E}\left(\mathit{genDTKs}\left(K,\mathit{rs}\right)\right)\label{eq: SODT-introduction}
		\end{equation}
		where the symbol ``$\subseteq$'' indicates that the solution on the
		left-hand side is also a solution on the right-hand side. Here, $\mathit{kcombs}\left(K,\mathit{rs}\right)$
		generates all possible $K$-combinations of rules from $\mathit{rs}$.
		The function $\mathit{concatMapL}$ applies $\mathit{sodt}$ to each
		combination returned by $\mathit{kcombs}\left(K,\mathit{rs}\right)$
		and flattens the resulting list of lists into a single list. Finally,
		the operator $\mathit{min}_{E}$ is the programmatic definition of
		$\text{argmin}$, which selects the }first\emph{ optimal solution
		with respect to $E$ from a list of candidates. \label{thm: sodt-introduction}}
\end{theorem}

See \cite{he2025odt} for the proof of Theorem \ref{thm: PDT-characterization} and Theorem \ref{thm: sodt-introduction}.

\subsection{Proof of theorem 3}
\label{subsec: proof-thm3}

\begin{figure}[H]
  \centering
  \includegraphics[width=0.70\textwidth]{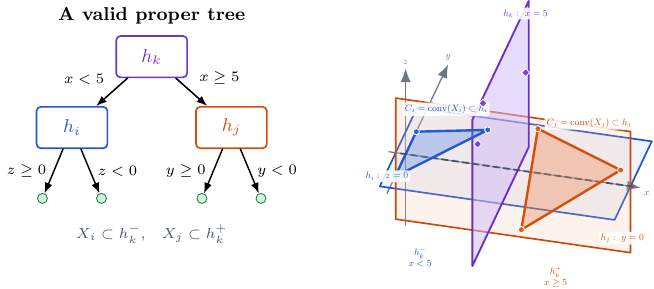}
  \caption{ Let
  $h_i=\{(x,y,z):z=0\}$ with
  $X_i=\{(4,3,0),(0,2,0),(0,-2,0)\}$,
  $h_j=\{(x,y,z):y=0\}$ with
  $X_j=\{(13,0,1),(7,0,-4),(8,0,3)\}$, and
  $h_k=\{(x,y,z):x=5\}$ with
  $X_k=\{(5,7,1),(5,-2,3),(5,-1,5)\}$.  Each set contains three affinely
  independent points, and the nine points are in general position (no four
  are coplanar).  Since the $y$-coordinates in $X_i$ are $3,2,-2$ and the
  $z$-coordinates in $X_j$ are $1,-4,3$, both
  $p_{\mathrm{str}}(X_i,h_j)$ and $p_{\mathrm{str}}(X_j,h_i)$ hold; hence
  $p_{\mathrm{crs}}(h_i,h_j)$ is true.  Nevertheless,
  $X_i\subset\{x<5\}=h_k^-$ and
  $X_j\subset\{x>5\}\subset h_k^+$.  In other words, the blue hyperplane ($h_i$) and yellow hyperplanes ($h_j$) are crossed (mutually straddling), but there exists a third hyperplane ($h_k$) that can separate them.}
  \label{fig:theorem3-counterexample}
\end{figure}

\begin{figure}[H]
  \centering
  \includegraphics[width=\textwidth]{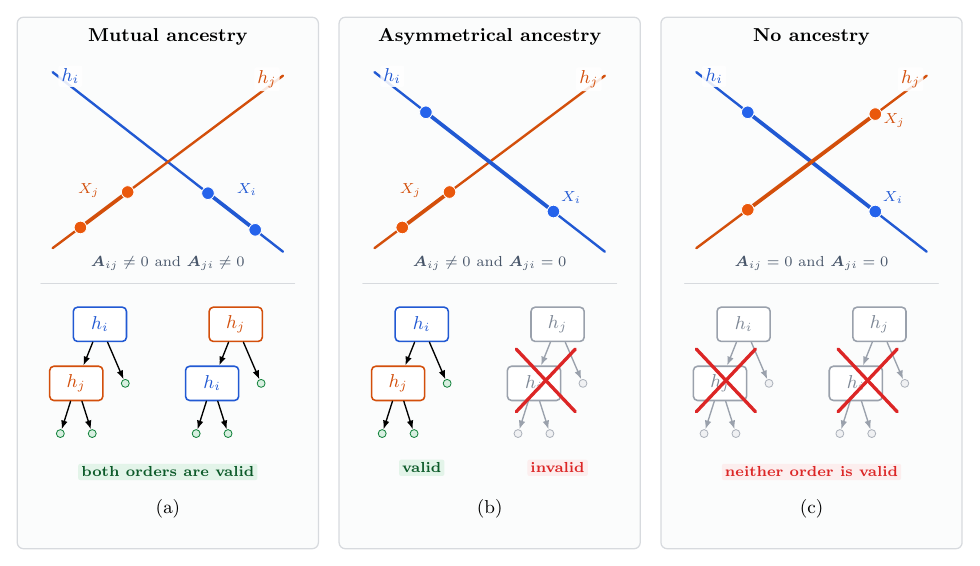}
  \caption{The three pairwise ancestry relations.  The upper row shows the
  two rules and their defining point sets; the lower row shows the
  corresponding tree structures.  (a) In mutual ancestry, both root orders
  are valid.  (b) In asymmetrical ancestry, only one root order is valid.
  (c) In no ancestry, neither rule can be a direct ancestor of the other.
  The last case concerns direct ancestry only and does not, without an
  additional geometric condition, rule out the two rules appearing in
  separate branches of a larger tree.}
  \label{fig:three-ancestry-relations}
\end{figure}

\begin{theorem}[Sound fatal-pair pruning]
\label{thm:fatal-pair-pruning-appendix}
Let $h_i$ and $h_j$ be affine hyperplane rules with defining point sets
$X_i$ and $X_j$.  If $p_{\mathrm{fatal}}(h_i,h_j)$ holds, then neither rule
can be an ancestor of the other, no affine hyperplane rule can place them in
different branches, and no proper decision tree satisfying
\eqref{eq:defining-points-in-region} can contain both rules.
\end{theorem}

\begin{proof}
Write $C_i=\operatorname{conv}(X_i)$ and
$C_j=\operatorname{conv}(X_j)$.  We prove the result in three steps.

\emph{Step 1: exclude an ancestry relation between $h_i$ and $h_j$.}
Suppose first that $h_i$ is an ancestor of $h_j$.  The node $h_j$ must then
occur in exactly one of the two branches of $h_i$.  By
\eqref{eq:defining-points-in-region}, all points defining $h_j$ must belong
to the decision region in which $h_j$ is generated.  In particular, either
\[
X_j\subseteq h_i^+ \qquad\text{or}\qquad X_j\subseteq h_i^-.
\]
This is impossible because $X_j$ contains points in both $h_i^+$ and
$h_i^-$.  Therefore $h_i$ cannot be an ancestor of $h_j$.  Interchanging
$i$ and $j$ gives the same contradiction, so $h_j$ cannot be an ancestor of
$h_i$ either.  Thus
\[
h_i\not\prec h_j \qquad\text{and}\qquad h_j\not\prec h_i.
\tag{3}
\]
This step uses only the mutual-straddling assumption.

\emph{Step 2: exclude a third rule that separates the two nodes.}
Suppose, for contradiction, that an affine hyperplane rule $h_k$ places
$h_i$ and $h_j$ in different branches.  After possibly interchanging the
branch labels, let $h_i$ lie in the $h_k^+$ branch and $h_j$ in the $h_k^-$
branch.  Repeated application of \eqref{eq:defining-points-in-region} along
the two ancestor paths gives
\[
X_i\subseteq h_k^+,\qquad X_j\subseteq h_k^-.
\tag{1}
\]
Both halfspaces are convex.  Therefore this implies
\[
C_i\subseteq h_k^+,\qquad C_j\subseteq h_k^-.
\tag{2}
\]
For completeness, if $x_1,\ldots,x_m\in h_k^-$, $\lambda_r\geq0$, and
$\sum_r\lambda_r=1$, then
\[
f_k\!\left(\sum_{r=1}^m\lambda_r x_r\right)
=\sum_{r=1}^m\lambda_r f_k(x_r)<0.
\]
so the strict negative halfspace is also closed under convex combinations.
Since $h_k^+\cap h_k^-=\varnothing$, (2) yields
$C_i\cap C_j=\varnothing$.  This contradicts \eqref{eq:fatal-pair}.
Consequently, no affine hyperplane rule can place $h_i$ and $h_j$ in
different branches.

This step is exactly where the convex-hull-intersection assumption is
needed; mutual straddling alone does not rule out such a separator in
dimensions $D\geq3$.

\emph{Step 3: exhaust the two possible positions in a rooted tree.}
Suppose, finally, that a proper rooted binary decision tree contains both
$h_i$ and $h_j$.  Any two nodes in a rooted tree fall into exactly one of the
following cases:
\begin{enumerate}
\item \emph{Comparable case.} One node is an ancestor of the other.  This
is impossible by (3).
\item \emph{Incomparable case.} Neither node is an ancestor of the other.
Let $h_k$ be their lowest common ancestor.  Then $h_i$ and $h_j$ lie below
different children of $h_k$.  Repeated application of
\eqref{eq:defining-points-in-region} along their paths places $X_i$ and
$X_j$ in opposite halfspaces of $h_k$.  Hence $h_k$ is precisely a
separating rule of the kind excluded in Step~2.
\end{enumerate}
The two cases exhaust all possible relative positions of two nodes in a
rooted tree, and both are impossible.  Therefore no proper decision tree
satisfying \eqref{eq:defining-points-in-region} can contain both $h_i$ and
$h_j$.
\end{proof}

Convex hulls in the original input space cannot replace
\eqref{eq:lifted-hull-feasibility}, because $\phi_M$ is nonlinear.
For finite defining sets, the lifted intersection is exactly the following
linear feasibility problem:
\begin{equation}
\begin{aligned}
\sum_r\alpha_r\phi_M(x_{i,r})
&=\sum_s\beta_s\phi_M(x_{j,s}),\\
\alpha_r,\beta_s&\geq0,\\
\sum_r\alpha_r&=1,\qquad \sum_s\beta_s=1.
\end{aligned}
\label{eq:lifted-hull-feasibility}
\end{equation}
This provides an implementable and dimension-independent safe test.

\section{Algorithms\label{appsec:algorithms}}

\subsection{A novel, parallelizable and generic algorithm for simplified optimal
	decision tree problem}

\subsubsection{Recursive definition\label{subsec: sodt_vec}}

Building on this idea, the following function exhuastively generate
all decision trees in the search space

\begin{equation}
	\begin{aligned}\mathit{genDTs}_{\text{vec}} & :\mathcal{D}\times\left[\mathcal{R}\right]\to\left[\mathit{DTree}\left(\mathcal{R},\mathcal{D}\right)\right]\\
		\mathit{genDTs}_{\text{vec}} & \left(\mathit{xs},\left[\;\right]\right)=\left[\mathit{DL}\left(\mathit{xs}\right)\right]\\
		\mathit{genDTs}_{\text{vec}} & \left(\mathit{xs},\mathit{rs}\right)=\mathit{concat}\circ\left[\mathit{updates}_{\mathbf{A}}\left(r,\mathit{rs}\backslash r,xs\right)|\left(r,\mathit{rs}\backslash r\right)\longleftarrow\mathit{candidates}\left(\mathit{rs}\right)\right]
	\end{aligned}
\end{equation}

where $\mathit{candidates}\left(\mathit{rs}\right)=\left[\left(r,\mathit{rs}\backslash r\right)\mid r\leftarrow\mathit{rs}\right]$
generates all possible rules in $\mathit{rs}$ and $\mathit{rs}\backslash r$
denotes eliminating $r$ from list $\mathit{rs}$, and $\mathit{updates}$
is defined as
\begin{equation}
	\begin{aligned}\mathit{updates}_{\mathbf{A}} & :\mathcal{R}\times\left[\mathcal{R}\right]\times\mathcal{D}\to\left[\mathit{DTree}\left(\mathcal{R},\mathcal{D}\right)\right]\\
		\mathit{updates}_{\mathbf{A}} & \left(r,\mathit{rs},\mathit{xs}\right)=\mathit{catMaybes}\left[\mathit{update}_{\boldsymbol{A}}\left(r,t\right)\mid t\leftarrow\mathit{genDTs}_{\text{vec}}\left(\mathit{rs},\mathit{xs}\right)\right]
	\end{aligned}
\end{equation}
which call $\mathit{genDTs}_{\text{vec}}$ recursively, and append
a new rule $r$ to every tree $t$ generated by $\mathit{genDTs}_{\text{vec}}\left(\mathit{rs}\backslash r,xs\right)$
by using $\mathit{update}$ function

\begingroup
\small
\begin{equation}
	\begin{aligned}\mathit{update}_{\boldsymbol{A}} & :\mathcal{R}\to\mathit{DTree}\left(\mathcal{R},\mathcal{D}\right)\to\mathit{Maybe}\left(\mathit{DTree}\left(\mathcal{R},\mathcal{D}\right)\right)\\
		\mathit{update}_{\boldsymbol{A}} & \left(r,\mathit{DL}\left(\mathit{xs}\right)\right)=\mathit{Just}\left(DN\left(\mathit{DL}\left(\mathit{xs}^{+}\right),r,\mathit{DL}\left(\mathit{xs}^{-}\right)\right)\right)\\
		\mathit{update}_{\boldsymbol{A}} & \left(r,\mathit{DN}\left(u,s,v\right)\right)=\begin{cases}
			\begin{cases}
				\mathit{DN}\left(\mathit{update}_{\boldsymbol{A}}\left(r,u\right),s,v\right) & \mathit{update}\left(r,u\right)\neq\mathit{Nothing}\\
				\mathit{Nothing} & \text{otherwise}
			\end{cases} & \boldsymbol{A}_{sr}=1\\
			\begin{cases}
				\mathit{DN}\left(u,s,\mathit{update}_{\boldsymbol{A}}\left(r,v\right)\right) & \mathit{update}\left(r,v\right)\neq\mathit{Nothing}\\
				\mathit{Nothing} & \text{otherwise}
			\end{cases} & \boldsymbol{A}_{sr}=-1\\
			\mathit{Nothing} & \text{otherwise}
		\end{cases}.
	\end{aligned}
\end{equation}
\endgroup
where $\mathit{xs}^{+},\mathit{xs}^{-}$ are disjoint data sets partitioned
by $r$ with respect to $\mathit{xs}$, and $\mathit{Maybe}\left(\mathcal{A}\right)=\mathit{Nothing}\mid\mathit{Just}\left(\mathcal{A}\right)$
represents a value that may or may not be present and $\texttt{\ensuremath{\mathit{catMaybes}}}:\left[\mathit{Maybe}\left(\mathcal{A}\right)\right]\to\left[\mathcal{A}\right]$
filters out $\mathit{Nothing}$ from a list of $\mathit{Maybe}$ values
and extracts the contents of$\mathit{Just}$ values into a plain list.
Although the definition of e $\mathit{update}_{\boldsymbol{A}}$ update
appears complex, it is applied recursively along the path of the old
tree $\mathit{DN}\left(u,s,v\right)$ to determine whether adding
a new splitting rule $r$ would result in a non-proper tree. If the
tree remains proper, $r$ is added to the leaf; otherwise, $\mathit{Nothing}$
is returned, indicating a non-proper decision tree. In this way, the
algorithm filtered out all non-proper decision tree during the generation
process. 

Finally, $\mathit{sodt}_{\text{vec}}$ is defined as following
\[
\mathit{sodt}_{\text{vec}}=\mathit{min}_{E}\circ\mathit{genDTs}_{\text{vec}}.
\]

\subsubsection{Prototype implementation in Haskell}

\textbf{Definition of tree}:
The binary tree is defined as follows:
\begin{lstlisting}[linerange={8-11}]
data Btree a = L | N (Btree a) a (Btree a) 
type D = [Int] -- A list of data indexes
type R = Int  -- rule indexes---- D-combination of data
type Asgn = [Int]
\end{lstlisting}
where \lstinline!Asgn! is the prediction of hyperplane / hypersurface. To simplify discussion, we assume \lstinline!D! and \lstinline!R! are indexes of data points and rules, respectively, thus can be modelled as integers.

\textbf{Complete definition of $gen_\text{vec}$}: Our tree generator can be succinctly defined in 13 lines:
\begin{lstlisting}[linerange={17-30}]
genBTs_vec :: [R] -> [Btree R]
genBTs_vec [] = [L]
genBTs_vec rs = concat [updates r rs'| (r,rs') <- candidates rs]
 where 
  updates x xs = catMaybes [update x t | t <- genBTs_vec xs]
  catMaybes xs = [x | Just x <- xs]
  candidates xs = [(x, xs \\ [x]) | x <- xs]

update :: R -> Btree R -> Maybe (Btree R)
update r L = Just (N L r L)
update r (N u s v)
  | anc_rel_mat !! s !! r == 1  = fmap (\t -> N t s v) (update r u)
  | anc_rel_mat !! s !! r == -1 = fmap (\t -> N u s t) (update r v)
  | otherwise                   = Nothing
\end{lstlisting}

\textbf{Input for reproducing Figure \ref{fig:sodt-generation-order}}:
We assume there are six data points and four splitting rules:
\begin{lstlisting}[linerange={60-67}]

-- Number of data points 
n :: Integer
n = 6 
asgns :: [Asgn]
asgns = [[1,-1,1,-1,-1,-1],[1,1,-1,1,1,-1],[1,1,1,-1,1,-1],[1,1,-1,-1,1,-1]]
-- Ancestry relation matrix
anc_rel_mat = [[0,-1, 0,1],[1,0,1,-1],[1,-1,0,1],[1,0,1,0]]
\end{lstlisting}

\textbf{Filter out duplicates}: Note that the $gen_\text{vec}$  will produce duplicate trees, which is not a problem for the purpose of illustration. This can be solved easily by incorporating following filter in the recursive execution of $gen_\text{vec}$ without affecting the correctness and asymptotic complexity of the algorithm:

\begin{lstlisting}[linerange={71-91}]
genBTs_vec_filt :: [R] -> [Btree R]
genBTs_vec_filt [] = [L]
genBTs_vec_filt rs = filt_dup_t $ concat [updates r rs'| (r,rs') <- candidates rs]
 where 
  updates x xs = catMaybes [update x t | t <- genBTs_vec_filt xs]
  catMaybes xs = [x | Just x <- xs]
  candidates xs = [(x, xs \\ [x]) | x <- xs]

traversal :: Btree a -> [a]
traversal tree = go [tree]
  where
    go [] = []  
    go (L : rest) = go rest  
    go (N u x v : rest) = x : go (rest ++ [u, v])  

nubSnd :: Eq b=> [(a, b)] -> [(a, b)]
nubSnd = nubBy (\(_, x) (_, y) -> x == y)

filt_dup_t :: [Btree R] -> [Btree R]
filt_dup_t [t] = [t]
filt_dup_t ts = map fst $ nubSnd $ zip ts (map traversal ts)
\end{lstlisting}

\textbf{Classical data-dependent recursive definition}: The generator for data dependent recursive $\mathit{sodt}_{\text{rec}}$ is given as follows:
\begin{lstlisting}[linerange={36-55}]
genBTs :: [R] -> [Btree R]
genBTs [] = [L]
genBTs [a] = [N L a L]
genBTs hs = [N u r v | (hs_pos, r, hs_neg) <- splits, u <- genBTs hs_pos, v <- genBTs hs_neg]
 where splits = [split r hs | r <- feasible_split hs]

feasible_split :: [R] -> [R]
feasible_split [a] = [a]
feasible_split hs = [ hi | hi <- hs, all (ind_mat hi) (filter (/= hi) hs)]
  where ind_mat hi hj = ((anc_rel_mat!!hi)!!hj)/=0

split :: R -> [R] -> ([R], R, [R])
split r [hj] = ([], r, [])
split r [hj,hk] 
 |r == hj = if (anc_rel_mat!!r)!!hk == 1 then ([hk], r, []) else ([], r, [hk])
 |r == hk = if (anc_rel_mat!!r)!!hj == 1 then ([hj], r, []) else ([], r, [hj])
split r hs = (hs_pos, r, hs_neg)
 where
  hs_pos = [hj | hj<-hs, (anc_rel_mat!!r)!!hj == 1]
  hs_neg = [hj | hj<-hs, (anc_rel_mat!!r)!!hj == -1]
\end{lstlisting}

\subsubsection{Imperative  definition}

The imperative pseudocode for $\mathit{sodt}_{\text{vec}}$ is presented
in Algorithm \ref{alg: sodt_vec_imperative}.

\begin{algorithm}[H]
	\textbf{Input}: $\mathit{rs}:\left[\mathcal{R}\right]$: A size $K$
	rule list $\mathit{rs}=\left[r_{1},r_{2},\ldots,r_{K}\right]$; $\mathit{xs}:\left[\mathbb{R}^{D}\right]$:
	input data list $\mathit{xs}=\left[x_{1},x_{2},\ldots,x_{N}\right]$,
	where $x_{n}\in\mathbb{R}^{D}$; $\boldsymbol{A}:K\times K$: The
	ancestry relation matrix, encode the ancestry relation between rules
	in $\mathit{rs}$.
	
	\textbf{Output}: $t:\mathit{DTree}\left(\mathcal{R},\mathbb{R}^{D}\right)$---The
	optimal decision tree of size $K$ ($K$ internal node) with respect
	to $\mathit{rs}$ and $\mathit{xs}$.
	\begin{enumerate}
		\item $\mathit{ts}=\left[DN\left(\mathit{DL}\left(xs^{+}\right),r_{i},\mathit{DL}\left(xs^{-}\right)\right)\mid r_{i}\in\mathit{rs}\right]$;
		// create one candidate tree for each rule as the root
		\item for $l=2,\ldots,K$ do:
		\item $\quad$$\mathit{ts}_{new}=\left[\:\right]$;
		\item $\quad$for $t\in\mathit{ts}$ do:
		\item $\quad$$\quad$for $r_{j}\in\mathit{rs}\setminus\mathit{rules}(t)$ do:
		// rules that have not yet been used in $t$
		\item $\quad$$\quad$$\quad$$t^{\prime}=\mathit{update}_{\boldsymbol{A}}\left(r_{j},t\right)$;
		// update the tree with new rule $r_{j}$
		\item $\quad$$\quad$$\quad$if $t^{\prime}\neq\mathit{Nothing}$ do:
		\item $\quad$$\quad$$\quad$$\quad$append $t^{\prime}$ to $\mathit{ts}_{new}$;
		// add the new tree to the new candidate tree list
		\item $\quad$$\mathit{ts}=\mathit{unique}\left(\mathit{ts}_{new}\right)$;
		\item \textbf{return} $\mathit{min}_{E}\left(\mathit{ts}\right)$
	\end{enumerate}
	\caption{$\mathit{sodt}_{\text{vec}}$: The imperative definition of the vectorized
		simplified decision tree problem \label{alg: sodt_vec_imperative}}
\end{algorithm}

\subsubsection{Complexity analysis}

The worst-case work complexity of $\mathit{sodt}_{\text{vec}}$ is given by
the following theorem.
\begin{theorem}
	Given a list $\mathit{rs}$ of $K$ rules and a dataset $\mathit{xs}$ of
	size $N$, assume that inserting a rule into a partial tree with at most
	$k-1$ internal nodes costs $O(N+k)$ and that evaluating the objective on
	a complete size-$K$ tree costs $O(NK)$. Then
	$\mathit{sodt}_{\text{vec}}(\mathit{xs},\mathit{rs})$ has worst-case
	work complexity
	\[
		O\left(K!\left(KN+K^{2}\right)\right).
	\]
	In the usual regime $N\geq K$, this simplifies to
	$O\left(K!KN\right)$.\label{thm: complexity of sodt_vec}
\end{theorem}
\begin{proof}
	Let $F_k$ be the maximum number of candidate trees generated from a
	list of $k$ rules. Each execution path selects every rule at most once
	and therefore corresponds to a permutation of the input rules. Invalid
	insertions are discarded and a valid insertion produces at most one tree,
	so $F_k\leq k!$.
	
	Let $T_k$ denote the running time of
	$\mathit{genDTs}_{\text{vec}}$ on $k$ rules. For each of the $k$ choices
	of the newly inserted rule, the algorithm recursively generates trees
	from the remaining $k-1$ rules and applies
	$\mathit{update}_{\boldsymbol{A}}$ to every generated tree. Consequently,
	\[
		T_k
		\leq kT_{k-1}+kF_{k-1}O(N+k)
		\leq kT_{k-1}+k!O(N+k).
	\]
	Dividing by $k!$ and expanding the recurrence gives
	\[
		\frac{T_k}{k!}
		\leq \frac{T_{k-1}}{(k-1)!}+O(N+k),
		\qquad
		T_K=O\left(K!\sum_{k=1}^{K}(N+k)\right)
		=O\left(K!\left(KN+K^2\right)\right).
	\]
	At most $K!$ complete candidate trees are produced. Directly evaluating
	the objective on all of them adds $O(K!NK)$ work, which is absorbed by
	the preceding bound. Hence
	$\mathit{sodt}_{\text{vec}}$ has worst-case work complexity
	$O\left(K!(KN+K^2)\right)$. When $N\geq K$, this reduces to
	$O(K!KN)$.
\end{proof}
\citet{he2025odt}'s $\mathit{sodt}_{\text{rec}}$ has a complexity
of $O\left(K!N\right)$. The two algorithms therefore share factorial
dependence on $K$, but $\mathit{sodt}_{\text{vec}}$ may perform up to an
additional polynomial amount of sequential work. Its practical advantage
comes from regular memory access and batched parallel execution rather than
from a smaller worst-case work bound.

\subsection{A novel fatal-pair-free nested-combination generator}
\subsubsection{Nested-combination generator proposed by He and Little}

The nested-combination generator is previously explored by \citet{he2025CGs}
in functional programming setting, we present the imperative version
directly here.

\begin{algorithm}[H]

	\caption{Nested Combination Generator ($\mathit{nestedCombs}(K, G, \mathit{xs})$) \label{alg:nested-combination-generator-sequential}}
	\begin{algorithmic}[1]
		\State \textbf{Input}: $\mathit{xs}$: data list, length $N$; $K$: outer combination size; $G$: inner-combination size
		\State \textbf{Output}: Array of $(K, G)$-nested-combinations
		\State $\mathit{css} \gets [[\,],[]^G]$ \Comment{Initialize combinations}
		\State $\mathit{ncss} \gets [[\,],[]^K]$ \Comment{Initialize nested-combinations}
		\State $\mathit{asgn}^{+}, \mathit{asgn}^{-} \gets \emptyset\left(\binom{N}{G}, N\right)$ \Comment{Empty $\binom{N}{G} \times N$ matrices}
		\For{$n \in \mathit{range}(0, N)$}
		\For{$j \in min(G,n+1), ..., 1$}
		\State $\mathit{updates} \gets \mathit{reverse}(\mathit{map}(\cup \phi_M(\mathit{xs})[n], \mathit{css}[j-1]))$
		\State $\mathit{css}[j] \gets \mathit{css}[j] \cup \mathit{updates}$ \Comment{Update combinations}
		\State $\mathit{asgn}^{+}, \mathit{asgn}^{-} \gets \mathit{genModels}(\mathit{css}[G], \mathit{asgn}^{+}, \mathit{asgn}^{-})$ \Comment{Generate predictions}
		\EndFor
		\State $C_1 \gets \binom{n}{G}$, $C_2 \gets \binom{n+1}{G}$
		\For{$i \in \mathit{range}(C_1, C_2)$}
		\For{$k \in min(K,i+1), ..., 1$}
		\State $\mathit{ncss}[k] \gets \mathit{map}(\cup [i], \mathit{ncss}[k-1]) \cup \mathit{ncss}[k]$ \Comment{Update nested combinations}
		\EndFor
		\EndFor
		\EndFor
		\State \textbf{return} $\mathit{ncss}[K]$
	\end{algorithmic}
\end{algorithm}

\subsubsection{Incremental ancestry relation matrix generator\label{subsec:Incremental-ancestry-relation}}

As discussed previously, $\mathit{nestedCombsFA}$ is defined by

\begin{equation}
	\mathit{nestedCombsFA}\left(K,G\right)
	=\mathit{filter}_{p_{\mathrm{safe}}}\circ\mathit{map}_{\mathit{calARM}}
	\circ\mathit{nestedCombs}\left(K,G\right).\label{eq:nestedCombsFA-appendix}
\end{equation}
The ancestry relation matrix provides a necessary condition for a fatal
pair, but it is not by itself a pruning certificate.
\begin{lemma}
	\emph{Let $\mathit{hs}_{K}=[h_{1},h_{2},\ldots,h_{K}]$ have ancestry
		relation matrix $\boldsymbol{A}$. If $h_i$ and $h_j$ form a fatal pair,
		then $\boldsymbol{A}_{ij}=\boldsymbol{A}_{ji}=0$. The converse does not
		generally hold.\label{lem:fatal-pair-and-ancestry-relation-matrix}}
\end{lemma}
\begin{proof}
	A fatal pair mutually straddles, so neither rule can be an ancestor of
	the other; hence $\boldsymbol{A}_{ij}=\boldsymbol{A}_{ji}=0$. However,
	zero entries record only the absence of direct ancestry. As shown in
	Figure~\ref{fig:theorem3-counterexample}, mutually straddling rules whose
	defining convex hulls are disjoint may still occur in different branches
	of a larger tree. The convex-hull-intersection test in
	$p_{\mathrm{fatal}}$ is therefore also required.
\end{proof}
We construct $\mathit{nestedCombsFA}$ in two steps. First, we update the
ancestry relation matrix incrementally. Second, whenever a new rule is
added, we test it against the rules already in the prefix using the full
fatal-pair predicate, including convex-hull intersection in the lifted
feature space.

\subsubsection{Incremental update of the ancestry relation matrix}

This part we examine how to fuse $\mathit{map}_{\mathit{calARM}}$
with $\mathit{nestedCombs}$, since $\mathit{nestedCombs}$ is an
incremental update process, it add one rule to each partial rule list,
we need to figure out how to update rule incrementally along the generation
process of $\mathit{nestedCombs}$. 

Since every rule combination is a nested combination, we create type
$\mathit{NC}$ (short for ``nested combinations'') as a type synonyms
for $\left[\mathcal{R}\right]$, i.e., $\mathit{rs}:\mathit{NC}$
or $\mathit{rs}:\left[\mathcal{R}\right]$. A size $K$ rule combination
using type $\mathit{rs}:\mathit{NC}$. Since the ancestry relation
matrix $\boldsymbol{A}$ for a $K$-length rule combination $\mathit{rs}_{K}$
is a $K\times K$ square matrix. Both $\boldsymbol{A}$ and $\mathit{rs}_{K}$
are integer-valued and can be stored compactly by stacking $\mathit{rs}_{K}$
above $\boldsymbol{A}$. That is, a \emph{configuration} $\mathit{ncr}:\mathit{NCR}$
is defined as

\begin{equation}
	\mathit{ncr}=\left[\begin{array}{c}
		\mathit{rs}_{K}\\
		\boldsymbol{A}
	\end{array}\right]
\end{equation}
Thus, a collection of $M$ configurations can be stored as a tensor
of size $M\times\left(K+1\right)\times K$.

Thus $\mathit{nestedCombs}\left(K,G\right):\left[\mathbb{R}^{D}\right]\to\left[\mathit{NC}\right]$,
and $\mathit{map}_{\mathit{calARM}}:\left[\mathit{NC}\right]\to\left[\mathit{NCR}\right]$,
the $\mathit{map}$ is a high-order function which applies its parameter
function $\mathit{calARM}:\mathit{NC}\to\mathit{NCR}$ to each $\mathit{rs}:\mathit{NC}$
generated by $\mathit{nestedCombs}\left(K,G\right)$. A critical observation
is that $\mathit{calARM}$ can be updated incrementally. 

Specifically, consider a partial configuration $\mathit{ncr}_{K-1}$
consisting of a list of rules $\mathit{rs}_{K-1}$ and a $\left(K-1\right)\times\left(K-1\right)$
ancestry relation matrix $\boldsymbol{A}^{\prime}$. We can extend
$\mathit{ncr}_{K-1}$ to a complete configuration $\mathit{ncr}_{K}$
as follows: first, appending a new rule $r$ to $\mathit{rs}_{K-1}$;
second, compute the ancestry relation of $r$ with each rule in $\mathit{rs}_{K-1}$,
and vice versa. This yields the following incremental (sequentially
recursive) program $\mathit{calARM}$ for constructing $\mathit{ncr}_{K}$
\begin{align*}
	\mathit{calARM} & \left(\left[r\right]\right)=\left[\begin{array}{c}
		r\\
		0
	\end{array}\right]\\
	\mathit{calARM} & \left(r:\mathit{rs}\right)=\mathit{update}_{\text{arMat}}\left(r,\mathit{calARM}\left(\mathit{rs}\right)\right)
\end{align*}
where $\left[\begin{array}{c}
	r\\
	0
\end{array}\right]$ represents a $2\times1$ matrix. Let $\mathit{calARM}\left(\mathit{rs}\right)=\left[\begin{array}{c}
	\mathit{rs}_{k}\\
	\boldsymbol{K}^{\prime}
\end{array}\right]$, the update function $\mathit{update}_{\text{arMat}}:\mathit{NCR}\to\mathit{NCR}$
is defined as 
\begin{equation}
	\begin{aligned}\mathit{update}_{\text{arMat}} & \left(r_{j},\left[\begin{array}{c}
			\mathit{rs}_{k}\\
			\boldsymbol{A}^{\prime}
		\end{array}\right]\right)= & \left[\begin{array}{ccccc}
			r_{0} & r_{1} & \ldots & r_{k-1} & r_{j}\\
			\boldsymbol{A}_{0,0}^{\prime} & \boldsymbol{A}_{01}^{\prime} & \ldots & \boldsymbol{A}_{0,k-1}^{\prime} & \mathit{AR}_{0,j}\\
			\boldsymbol{A}_{1,0}^{\prime} & \boldsymbol{A}_{1,1}^{\prime} & \ldots & \boldsymbol{A}_{1,k-1}^{\prime} & \mathit{AR}_{1,j}\\
			\vdots & \vdots & \ddots & \vdots & \vdots\\
			\boldsymbol{A}_{k-1,0}^{\prime} & \boldsymbol{A}_{k-1,1}^{\prime} & \ldots & \boldsymbol{A}_{k-1,k-1}^{\prime} & \mathit{A}_{k-1,j}\\
			\mathit{A}_{j,0} & \mathit{A}_{j,1} & \ldots & \mathit{A}_{j,k-1} & 0
		\end{array}\right]\end{aligned}
	,\label{eq: update function}
\end{equation}
which extends a $\left(k+1\right)\times k$ matrix to a $\left(k+2\right)\times\left(k+1\right)$
matrix. The updated matrix consists of:
\begin{itemize}
	\item $\boldsymbol{A}^{\prime}$: a $k\times k$ ancestry relation matrix
	from the previous step,
	\item $\mathit{rs}_{k}$: a $k\times1$ vector of the indices of the $k$
	splitting rules,
	\item $\mathit{AR}_{i,j}$: the ancestry relation between rule where $r_{i}$
	and $r_{j}$, where 
	\[
	\mathit{A}_{i,j}=\begin{cases}
		1 & \text{if \ensuremath{r_{i}} can occur in the positive branch below  \ensuremath{r_{j}}},\\
		-1 & \text{if \ensuremath{r_{i}} can occur in the negative branch below  \ensuremath{r_{j}}},\\
		0 & \text{if \ensuremath{r_{i}} otherwise}.
	\end{cases}
	\]
\end{itemize}
The incremental nature of the $\mathit{calARM}$ is crucial, as it
allows the ancestry matrix to be updated on-the-fly for each rule
combination during the recursive generation process of $\mathit{nestedCombs}$.
This fusion alone does not reduce the asymptotic cost, but it allows the
safe predicate $p_{\mathrm{safe}}$ to be evaluated incrementally in the
next subsection.

The single step update function $\mathit{update}_{\text{arMat}}$
cab be easily generalized to apply to multiple configurations in a single
step, the vectorized version is given in (\ref{alg:Update-ancestry-relation}).
Incorporating this function inside (\ref{alg:nested-combination-generator-sequential})
in the loop after $\mathit{ncss}\left[k\right]$ (line 14) can generate
all possible rule combinations with associated ancestry relation matrix
associates with it, we will provide the full algorithm after explained
the technique introduced next 

\subsubsection{Prefix-closed filtering for nested combinations}

After updating the ancestry relation matrix, we eliminate configurations
that contain a fatal pair. A pair of zero ancestry entries is a useful
necessary-condition screen, but it is not sufficient for pruning: the
defining point sets must also have intersecting convex hulls, as required
by $p_{\mathrm{fatal}}$ in \eqref{eq:fatal-pair}. A post-hoc evaluation of
$p_{\mathrm{safe}}$ checks $O(K^2)$ rule pairs after the full configuration
has already been generated. If one fatal-pair test costs
$T_{\mathrm{fatal}}$, this requires $O(K^2T_{\mathrm{fatal}})$ time per
configuration.

To resolve this, we introduce a prefix-closed filtering process. It
allows the post-hoc filtering to be \emph{fused directly into the
	incremental generation process} of $\mathit{nestedCombs}\left(K,G\right)$,
	enabling fatal pairs to be identified before partial configurations are
extended to full ones. This provides a safe incremental nested-combination
generator. Specifically, when we have
a generator function, such as $\mathit{nestedCombs}\left(K,G\right)$,
defined sequentially consume one data item in each recursive process,
we have the following filter fusion theorem applies.
\begin{theorem}
	Filter fusion theorem. \emph{Let a sequential generator $\mathit{gen}$
		be defined recursively as:
		\begin{align*}
			\mathit{gen} & \left(\left[\:\right]\right)=\mathit{alg}_{1}\left(\left[\:\right]\right)\\
			\mathit{gen} & \left(x:\mathit{xs}\right)=\mathit{alg}_{2}\left(x,\mathit{gen}\left(\mathit{xs}\right)\right),
		\end{align*}
		and consider a post-hoc filtering process. Then $\mathit{filtgen}_{q}=\mathit{filter}_{p}\circ\mathit{gen}$
		can be }fused into a single program\emph{ defined as:}
	
	\emph{
		\begin{align*}
			\mathit{filtgen}_{q} & \left(\left[\:\right]\right)=\mathit{filter}_{q}\left(\mathit{alg}_{1}\left(\left[\:\right]\right)\right)\\
			\mathit{filtgen}_{q} & \left(x:\mathit{xs}\right)=\mathit{filter}_{q}\left(\mathit{alg}_{2}\left(x,\mathit{filtgen}_{q}\left(\mathit{xs}\right)\right)\right),
		\end{align*}
		provided that the fusion condition holds:
		\begin{equation}
			\mathit{filter}_{p}\left(\mathit{alg}_{2}\left(x,\mathit{gen}\left(\mathit{xs}\right)\right)\right)=\mathit{filter}_{q}\left(\mathit{alg}_{2}\left(x,\mathit{filter}_{p}\left(\mathit{gen}\left(\mathit{xs}\right)\right)\right)\right)\label{eq: filtering fusion condition}
		\end{equation}
		In particular, if $\mathit{alg}_{2}:\mathcal{A}\times\left[\left[\mathcal{A}\right]\right]\to\left[\left[\mathcal{A}\right]\right]$
		is defined as an extension operation that prepend $a:\mathcal{A}$
		to $\mathit{as}:\left[\mathcal{A}\right]$ for all $\mathit{as}$
		in $\mathit{ass}:\left[\left[\mathcal{A}\right]\right]$, then proving
		the fusion condition (\ref{eq: filtering fusion condition}) is equivalent
		to proving the }prefix-closed property\emph{:
		\begin{equation}
			p\left(a:\mathit{as}\right)=q\left(a:\mathit{as}\right)\wedge p\left(\mathit{as}\right)\label{eq: prefix-closed property}
		\end{equation}
	}
\end{theorem}
\begin{proof}
	We prove the fusion theorem by following reasoning
	\[
	\begin{aligned} & \mathit{filtgen}_{q}\\
		= & \mathit{filter}_{p}\circ\mathit{gen}\\
		= & \mathit{filter}_{p}\left(\mathit{alg}_{2}\left(x,\mathit{gen}\left(\mathit{xs}\right)\right)\right)\\
		= & \{\text{fusion condition (\ref{eq: filtering fusion condition})}\}\\
		& \mathit{filter}_{q}\left(\mathit{alg}_{2}\left(x,\mathit{filter}_{p}\left(\mathit{gen}\left(\mathit{xs}\right)\right)\right)\right)\\
		= & \{\text{definition of \ensuremath{\mathit{filtgen}_{q}}}\}\\
		& \mathit{filter}_{q}\left(\mathit{alg}_{2}\left(x,\mathit{filtgen}_{q}\left(\mathit{xs}\right)\right)\right)
	\end{aligned}
	\]
	The equivalence between filter fusion condition (\ref{eq: filtering fusion condition})
	and prefix-closed property (\ref{eq: prefix-closed property}) is
	straight forwards. A configuration $a:\mathit{as}$ survived in $\mathit{filter}_{p}$
	will also survived in $\mathit{filter}_{q}\left(\mathit{alg}_{2}\left(a,\mathit{filter}_{p}\left(\left[\mathit{as}\right]\right)\right)\right)$
	because $p\left(a:\mathit{as}\right)=q\left(a:\mathit{as}\right)\wedge p\left(\mathit{as}\right)$.
\end{proof}
The reason the prefix-closed property (\ref{eq: prefix-closed property})
introduces an additional predicate $q$ is that, if we already know
$p\left(\mathit{as}\right)$ holds, it is often more efficient to
evaluate $q\left(a:\mathit{as}\right)\wedge p\left(\mathit{as}\right)$
rather than directly computing $p\left(a:\mathit{as}\right)$. This
explains why the fused generator $\mathit{filtgen}_{q}$ is more efficient
than the post-hoc approach $\mathit{filter}_{p}\circ\mathit{gen}$.
For example, in the classical \emph{eight queens problem}, $p$ checks
that no queen attacks any other, while the auxiliary predicate $q$
only verifies that the newly added queen does not attack the others.

Similarly, for our problem, since $\mathit{nestedCombs}\left(K,D\right)$
is also defined sequentially, we can establish a prefix-closed property
for the ancestry relation matrix, analogous to the auxiliary check
in the eight queens problem. The correctness of the prefix-closed
condition for fatal-pair filtering is given by the following fact.
\begin{fact}
	Prefix-closed property for fatal pairs. \emph{Assume that
		$\mathit{hs}_{K}$ is safe, i.e.,
		$p_{\mathrm{safe}}(\mathit{hs}_{K})=\text{True}$. When a new rule $h$
		is added, the following identity holds:}
	\[
	p_{\mathrm{safe}}\left(h:\mathit{hs}_{K}\right)
	=q_{\mathrm{fatal}}\left(h:\mathit{hs}_{K}\right)
	\wedge p_{\mathrm{safe}}\left(\mathit{hs}_{K}\right),
	\]
	\emph{where
		\[
		q_{\mathrm{fatal}}\left(h:\mathit{hs}_{K}\right)
		=\lnot\exists h_i\in\mathit{hs}_{K}:
		p_{\mathrm{fatal}}(h,h_i)=\text{True}.
		\]}
	\label{fact:Prefix-closed-property-for-fatal-pairs}
	\emph{Thus, extending a safe prefix requires testing only the new rule
		against the rules already in that prefix.}
\end{fact}
Fact~\ref{fact:Prefix-closed-property-for-fatal-pairs} reduces the number
of pairwise tests from $O(K^2)$ to $O(K)$ per extension, and hence the
filtering cost from $O(K^2T_{\mathrm{fatal}})$ to
$O(KT_{\mathrm{fatal}})$.

\begin{algorithm}

	\caption{Safe Fatal-Pair Filtering and Ancestry Update
		($\mathit{updates}_{\mathrm{safe}}$) \label{alg:Update-ancestry-relation}}
	\begin{algorithmic}[1]
		\State \textbf{Input}: $\mathit{r}_j$: index of the new rule; $\mathit{ncrs}$: partial safe configurations; $\mathit{css}[G]$: lifted defining point sets; $\mathit{asgn}^{\pm}$: rule predictions
		\State \textbf{Output}: Updated fatal-pair-free configurations $\mathit{ncrs}'$
		\State $\mathit{ncrs}' \gets [\,]$ \Comment{Initialize empty list}
		\For{$\mathit{ncr} \in \mathit{ncrs}$}
		\State $\mathit{safe} \gets \text{True}$
		\State $X_j \gets \mathit{css}[G][\mathit{r}_j]$
		\For{$\mathit{r}_i \in \mathit{ncr}[0]$}
		\State $X_i \gets \mathit{css}[G][\mathit{r}_i]$
		\State $\mathit{crossed} \gets p_{\mathrm{str}}(X_i,h_j) \wedge p_{\mathrm{str}}(X_j,h_i)$
		\If{$\mathit{crossed} \wedge \mathit{convIntersects}(X_i,X_j)$}
		\State $\mathit{safe} \gets \text{False}$ \Comment{$p_{\mathrm{fatal}}(h_i,h_j)$ holds}
		\State \textbf{break}
		\EndIf
		\EndFor
		\If{$\mathit{safe}$}
		\State $\mathit{ncr}' \gets \mathit{update}_{arMat}(\mathit{r}_j,\mathit{ncr},\mathit{asgn}^{\pm})$
		\State $\mathit{ncrs}' \gets \mathit{ncrs}' + [\mathit{ncr}']$
		\EndIf
		\EndFor
		\State \textbf{return} $\mathit{ncrs}'$
	\end{algorithmic}
\end{algorithm}

Here, $\mathit{convIntersects}(X_i,X_j)$ solves the linear feasibility
problem in \eqref{eq:lifted-hull-feasibility}. For polynomial rules,
$X_i$ and $X_j$ contain the lifted points $\phi_M(x)$; testing convex-hull
intersection in the original input space is not sufficient.

We can now incorporate $\mathit{update}_{\text{arMat}}$ along with
the safe filtering process for eliminating fatal pairs directly into
the definition of $\mathit{nestedCombs}\left(K,G\right)$. For simplicity,
we define a batched routine $\mathit{updates}_{\mathrm{safe}}$ that operates on a
list of configurations $\mathit{ncrs}$, as implemented in Algorithm
\ref{alg:Update-ancestry-relation}.

With these components, we can construct the incremental, fatal-pair-free
ancestry relation matrix generator $\mathit{nestedCombsFA}$. This
is achieved by introducing a single additional line of code after line 13 of Algorithm 2 using Algorithm \ref{alg:Update-ancestry-relation}
to update all safe nested combinations while discarding configurations
containing a certified fatal pair. The complete pseudo-code is provided in
Appendix \ref{subsec:Incremental-ancestry-relation},
Algorithm \ref{alg:nestedCombsFA}.

\begin{algorithm}[H]

	\caption{Nested Combination Generator with Ancestry ($\mathit{nestedCombsFA}(K, G, \mathit{xs})$) \label{alg:nestedCombsFA}}
	\begin{algorithmic}[1]
		\State \textbf{Input}: $\mathit{xs}$: data list, length $N$; $K$: outer combination size; $G$: inner-combination size
		\State \textbf{Output}: Array of $(K, G)$-nested-combinations
		\State $\mathit{css} \gets [[\,],[]^G]$ \Comment{Initialize combinations}
		\State $\mathit{ncss} \gets [[\,],[]^K]$ \Comment{Initialize nested-combinations}
		\State $\mathit{asgn}^{+}, \mathit{asgn}^{-} \gets \emptyset\left(\binom{N}{G}, N\right)$ \Comment{Empty $\binom{N}{G} \times N$ matrices}
		\For{$n \in \mathit{range}(0, N)$}
		\For{$j \in min(G,n+1), ..., 1$}
		\State $\mathit{updates} \gets \mathit{reverse}(\mathit{map}(\cup \phi_M(\mathit{xs})[n], \mathit{css}[j-1]))$
		\State $\mathit{css}[j] \gets \mathit{css}[j] \cup \mathit{updates}$ \Comment{Update combinations}
		\State $\mathit{asgn}^{+}, \mathit{asgn}^{-} \gets \mathit{genModels}(\mathit{css}[G], \mathit{asgn}^{+}, \mathit{asgn}^{-})$ \Comment{Generate predictions}
		\EndFor
		\State $C_1 \gets \binom{n}{G}$, $C_2 \gets \binom{n+1}{G}$
		\For{$i \in \mathit{range}(C_1, C_2)$}
		\For{$k \in min(K,i+1), ..., 1$}
		\State $\mathit{ncss}[k] \gets \mathit{updates}_{\mathrm{safe}}(i, \mathit{ncss}[k-1], \mathit{css}[G], \mathit{asgn}^{+}, \mathit{asgn}^{-})   \cup \mathit{ncss}[k]$  \Comment{Safe filtering and ancestry update, Alg. \ref{alg:Update-ancestry-relation}}
		\EndFor
		\EndFor
		\EndFor
		\State \textbf{return} $\mathit{ncss}[K]$
	\end{algorithmic}
\end{algorithm}

\subsection{Complexity of the
	\texorpdfstring{$\mathit{odt}^{\prime}$}{odt-prime} algorithm}

Putting the components together,
$\mathit{odt}^{\prime}(xs)=\mathit{min}_{E}\circ
\mathit{concatMapL}(\mathit{sodt}_{\text{vec}}(xs))\circ
\mathit{nestedCombsFA}(xs)$. Its worst-case work complexity is given by
the following theorem.
\begin{theorem}
	Let $T_{\mathrm{fatal}}$ be the cost of one fatal-pair feasibility test.
	The $\mathit{odt}^{\prime}(xs)$ algorithm has worst-case work complexity
	\[
		O\left(
		N^{KG+1}
		+KT_{\mathrm{fatal}}N^{KG}
		+K!\left(KN+K^2\right)N^{KG}
		\right).
	\]
	If $N\geq K$ and the fixed-dimensional fatal-pair test is treated as
	constant cost, this simplifies to
	$O\left(K!K N^{KG+1}\right)$.\label{thm: odt^prime complexity}
\end{theorem}
\begin{proof}
	Since $\mathit{odt}^{\prime}\left(xs\right)=\mathit{min}_{E}\circ\mathit{concatMapL}\left(\mathit{sodt}_{\text{vec}}\left(\mathit{xs}\right)\right)\circ\mathit{\mathit{nestedCombsFA}}\left(\mathit{xs}\right)$,
	the complexity of $\mathit{odt}^{\prime}$ is the sum of the costs of
	the composed components and is therefore dominated by the largest term.
	$\mathit{min}_{E}$ is linear in
	the size of the output produced by $\mathit{concatMapL}\left(\mathit{sodt}_{\text{vec}}\left(\mathit{xs}\right)\right)$,
	thus the complexity $\mathit{odt}^{\prime}$ is dominated by either
	$\mathit{concatMapL}\left(\mathit{sodt}_{\text{vec}}\left(\mathit{xs}\right)\right)$
	or $\mathit{\mathit{nestedCombsFA}}\left(\mathit{xs}\right)$.
	
	As \citet{he2025CGs} analysed, $\mathit{nestedCombs}(\mathit{xs})$
	produces $O\left(N^{KG}\right)$ configurations in $O\left(N^{KG+1}\right)$
	time where $K$ is the number of rules in each configuration, $G$
	is the embedded dimension of hypersurface and $N$ is the size of the
	input $xs$. Compared with $\mathit{\mathit{nestedCombs}}\left(\mathit{xs}\right)$,
	the $\mathit{nestedCombsFA}$ additionally checks the new rule against at
	most $K-1$ existing rules. If one fatal-pair test costs
	$T_{\mathrm{fatal}}$, the additional filtering cost is
	$O\left(KT_{\mathrm{fatal}}N^{KG}\right)$. When $K$, $G$, and the cost of
	the fixed-dimensional feasibility test are treated as constants, this
	does not change the stated exponent in $N$.
	
	The $\mathit{concatMapL}\left(\mathit{sodt}_{\text{vec}}(\mathit{xs})\right)$
	applies $\mathit{sodt}_{\text{vec}}$ to each of at most
	$O\left(N^{KG}\right)$ generated configurations. By
	Theorem~\ref{thm: complexity of sodt_vec}, its total work is therefore
	\[
		O\left(K!\left(KN+K^2\right)N^{KG}\right).
	\]
	Adding rule-combination generation and fatal-pair filtering yields
	\[
		O\left(
		N^{KG+1}
		+KT_{\mathrm{fatal}}N^{KG}
		+K!\left(KN+K^2\right)N^{KG}
		\right).
	\]
	When $N\geq K$ and $T_{\mathrm{fatal}}=O(1)$ under a fixed-dimensional
	cost model, the last term dominates, giving
	$O\left(K!K N^{KG+1}\right)$.
\end{proof}

\section{Heuristic methods\label{sec:Heuerestic-methods}}

The use of heuristics is pervasive in the study of optimal algorithms.
Representative examples include imposing a time limit with random
initialization \citep{dunn2018optimal}, employing depth-first search
with a time limit \citep{hu2019optimal,lin2020generalized}, or using
\emph{binarization} for continuous data \citet{brita2025optimal}.
Although some BnB algorithms assert that the use of depth-first search
with a time limit does not preclude finding the optimal solution if
the algorithm is allowed to continue running, this claim is largely
vacuous in practice: without a feasible bound on running time, identifying
the optimal solution may still require exponential time in the worst
case. A similar phenomenon arises in the study of algorithms based
on MIP solvers \citet{bertsimas2017optimal,bertsimas2019machine}.
Although MIP-based methods guarantee attainment of the optimal solution
upon full termination, their running times are often unpredictable
and impractical for most problems. Consequently, studies employing
MIP solvers typically impose time limits and aim to obtain near-optimal
solutions within the allotted time, which can likewise be regarded
as a heuristic.

It should be emphasized that\emph{ the role of these heuristics in
	the study of optimal algorithms is not to demonstrate their superiority---unless
	supported by rigorous performance guarantees, which would constitute
	a distinct line of research---but rather to provide rapid, plausible
	solutions, thereby illustrating that higher training accuracy under
	controlled model complexity does not necessarily result in overfitting}.
Consequently, most MIP-based algorithms, as well as BnB algorithms
that incorporate depth-first search strategies, do not provide guarantees
on the quality of the solution obtained within a given execution time.

Therefore, the establishment of rigorous approximation bounds constitutes
a separate line of research concerned with the design of approximate
algorithms endowed with provable guarantees. By contrast, heuristics
in the context of exact algorithms primarily serve to demonstrate
the practical effectiveness of exact solutions and to motivate further
investigation into scaling exact methods to larger datasets.

The HDT problem exhibits formidable combinatorics. Even for a modest
dataset with $N=100$, $D=3$, $K=3$, there are $\left(\begin{array}{c}
	\left(\begin{array}{c}
		100\\
		3
	\end{array}\right)\\
	3
\end{array}\right)\times3!\approx4\times10^{15}$ possible decision trees in the worst case.  We emphasize that \textbf{this
	is not a deficiency of our algorithm but is inherent to the problem
	itself}, as even its simplest instance---the linear classification problem
(a tree with a single internal node)---is NP-hard. Consequently, solving
this problem exactly is currently intractable for our algorithms.
In this section, we develop two simple yet effective heuristics for
addressing the HDT problem.

We do not claim novelty for these two heuristics, nor do we seek to
emphasize their intrinsic importance; rather, they are employed as
practical tools to provide insight into the behavior of solutions
when they are exact or nearly exact. In principle, any other effective
heuristics could be adopted for the same purpose. However, without
the geometric and algorithmic insights developed in this paper, the
construction of such heuristics would not be possible.

\subsection{Coreset selection method}

\begin{algorithm}

	\caption{Coreset Selection ($\mathit{hodtCoreset}(K, M, \mathit{xs}, \mathit{BS}, R, L, B_{\max}, c)$) \label{alg:Coreset-selection}}
	\begin{algorithmic}[1]
		\State \textbf{Input}: Coreset parameters: $\mathit{BS}$: block size; $R$: shuffle iterations; $L$: max-heap size; $B_{\max}$: max input size; $c \in (0,1]$: shrinking factor
		\State \hspace{4em} $\mathit{hodt}$ parameters: $K$: splitting rules; $M$: polynomial degree; $\mathit{xs}$: data list
		\State \textbf{Output}: Max-heap with top $L$ configurations and data blocks
		\State $\mathcal{C} \gets \mathit{xs}$ \Comment{Initialize coreset}
		\While{$|\mathcal{C}| \geq B_{\max}$}
		\State $\mathcal{C}_B \gets \{ C_1, C_2, \ldots, C_{\lceil |\mathcal{C}| / \mathit{BS} \rceil} \}$ \Comment{Reshuffle and divide into blocks}
		\State $\mathcal{H}_L \gets \emptyset_L$ \Comment{Initialize size $L$ max-heap}
		\For{$r \in \mathit{range}(1, R)$}
		\For{$C \in \mathcal{C}_B$}
		\State $\mathit{cnfg} \gets \mathit{hodt}(K, M, C)$ \Comment{Compute configuration}
		\State $\mathcal{H}_L.\mathit{push}(\mathit{cnfg}, C)$ \Comment{Push to heap}
		\EndFor
		\EndFor
		\State $\mathcal{C} \gets \mathit{unique}(\mathcal{H}_L)$ \Comment{Merge blocks, remove duplicates}
		\State $L \gets L \times c$ \Comment{Shrink heap size}
		\EndWhile
		\State $\mathit{cnfg} \gets \mathit{hodt}(K, M, \mathcal{C})$ \Comment{Final refinement}
		\State $\mathcal{H}_L.\mathit{push}(\mathit{cnfg}, \mathcal{C})$
		\State \textbf{return} $\mathcal{H}_L$
	\end{algorithmic}
\end{algorithm}

The first method, as reported by \citet{he2026deepice}, has demonstrated
superior performance for empirical risk minimization in two-layer
neural networks. The method is based on the following idea: instead
of computing the exact solution across the entire dataset---which is
computationally infeasible for large $K$ and $D$---the coreset method
identifies the exact solution for the most representative subsets.

The coreset method functions as a ``layer-by-layer'' data filtering
process, in each loop, we run the $\mathit{hodt}$ algorithm for a
subset of the dataset, and only keep the $L$ best solutions with
respect to the whole dataset. Since better configurations tend to
have higher training accuracy, they are more likely to ``survive''
during the selection process. By recursively reducing  the data size
until the remaining subset can be processed by running the complete
optimal algorithm. The algorithm process is detailed in Algorithm
\ref{alg:Coreset-selection}.

\subsection{\texorpdfstring{$\mathit{sodt}$}{sodt} with selected hyperplanes}

\begin{algorithm}

	\caption{$\mathit{sodt}$ with selected hyperplanes ($\mathit{sodtWSH}(K, M, \alpha, \mathit{xs})$) \label{alg:sodtWSH}}
	\begin{algorithmic}[1]
		\State \textbf{Input}: $K$: splitting rules; $M$: polynomial degree; $\alpha$: duplicate threshold; $\mathit{xs}$: data list
		\State \textbf{Output}: Best trees constructed from the selected candidate hyperplanes
		\State $t_{best} \gets \emptyset(K, 1)$ \Comment{Initialize vector}
		\State $\mathcal{H}_C \gets \mathit{hodtCoreset}(1, M, \mathit{xs}, \mathit{BS}, R, L, B_{\max}, c)$ \Comment{Select candidate hypersurfaces}
		\State $G \gets \binom{M+D}{D} - 1$
		\State $\mathit{ncss} \gets [[\,],[]^k]$ \Comment{Initialize nested-combinations}
		\For{$n \in \mathit{range}(0, N)$}
		\For{$k \in \mathit{reverse}(\mathit{range}(K, n+1))$}
		\State $\mathit{ncss}' \gets \mathit{filter}_{q(\alpha)}(\mathit{map}(\cup [n], \mathit{ncss}[k-1]))$ \Comment{Filter duplicates below $\alpha$}
		\State $t_{best}[k] \gets \min_{E_{0-1}}(\mathit{mapL}(\mathit{sodt}_{\mathit{xs}}, \mathit{ncss}'))$ \Comment{Compute best configuration}
		\State $\mathit{ncss}[k] \gets \mathit{ncss}[k] \cup \mathit{reverse}(\mathit{ncss}')$ \Comment{Update combinations}
		\State $\mathit{ncss}[K] \gets [\,]$ \Comment{Clear size $K$ combinations}
		\EndFor
		\EndFor
		\State \textbf{return} $t_{best}$
	\end{algorithmic}
\end{algorithm}

Although the coreset selection method provides plausible solutions
for low-dimensional datasets, an obvious limitation arises for high-dimensional
datasets: the number of possible splitting rules is too large, even
when the dataset is partitioned into small blocks. Moreover, defining
a decision tree with $K$ splitting rules requires at least $K\times D$
distinct data points. Running $\mathit{hodt}$ for $K\times D$ points
becomes intractable for most high-dimensional datasets. For example,
when $K=2$, and $D=20$, the block size is $2\times20$ and the number
of possible hyperplanes is $\left(\begin{array}{c}
	40\\
	25
\end{array}\right)\approx1\times10^{11}$, which is prohibitively large. In practice, our current implementation
can efficiently process at most $D+2$ or $D+3$ points when $D\geq20$,
far below the ideal block size of $K\times D$.

To address this limitation, we leverage a key advantage of the size-constrained
$\mathit{hodt}$ algorithm---its ability to decompose the original difficult
ODT problem into many smaller subproblems, each of which can be solved
efficiently using $\mathit{sodt}$. Solving the full ODT problem directly
is computationally prohibitive due to the enormous combinatorial space.
Instead, we focus on finding a set of ``good candidate'' hyperplanes,
and then construct decision trees from these candidates rather than
exploring the entire search space.

This insight motivated the development of a new heuristic, $\mathit{sodtWSH}$,
(short for ``$\mathit{sodt}$ with selected hyperplanes''), described
in Algorithm \ref{alg:sodtWSH}. The algorithm identifies candidate
hyperplanes with relatively \emph{low training loss }and constructs
optimal decision trees based on these hyperplanes. Rather than generating
hyperplanes sequentially from continuous data blocks---where each hyperplane
differs from the previous by only one or two points---we generate a
large set of candidate hyperplanes and then apply $\mathit{sodt}$
only to combinations in which the data points defining each splitting
rule are sufficiently distinct---we want to only evaluate those that contain unique data points greater than a threshold $\alpha<D\times K$.
This will help us to find sufficiently distinct hyperplanes.

\section{Experiments \label{sec: experiments-appendix}} 

The experiments aim to provide a detailed analysis along four dimensions:
\begin{enumerate}
	\item \textbf{Computational complexity and scalability}: We compare $\mathit{sodt}_{\text{vec}}$
	and $\mathit{sodt}_{\text{rec}}$ under both sequential and parallelized
	settings, evaluating the scalability of the more efficient method.
	\item \textbf{Empirical combinatorial complexity}: We analyze the combinatorial
	complexity of hyperplane and hypersurface decision tree models, demonstrating
		that the empirical search-space size---after safely filtering out fatal pairs---is
	substantially smaller than the theoretical upper bound provided in
	Part I.
	\item \textbf{Analysis over synthetic datasets}: Using synthetic datasets
	generated from hyperplane decision trees (with data lying in convex
	polygon regions), we benchmark the performance of $\mathit{\mathit{hodt}Coreset}$.
	We systematically test the effects of ground truth tree size, data
	dimensionality, dataset size, label noise, and data noise separately.
	Our results show that the hyperplane decision tree model learned by
	our algorithm is not only more accurate in prediction but also more
	robust to noise.
	\item \textbf{Generalization performance on real-world datasets}: We evaluate
	performance across 30 real-world datasets, comparing our hyperplane
	decision tree model learned by $\mathit{sodtWSH}$ against the state-of-the-art
	optimal decision tree algorithm \citet{brita2025optimal} and the
	well-known approximate algorithm CART. We demonstrate that, when model
	complexity is properly controlled, hyperplane decision trees consistently
	outperform axis-parallel models in out-of-sample tests.
\end{enumerate}

Global optimality for the full HODT problem additionally requires exhaustive enumeration of the complete rule set generated by \(\mathit{genRules}\), which becomes computationally intractable even for relatively small datasets. In contrast, \(\mathit{sodt}_{\mathrm{vec}}\) is exact for any fixed candidate rule set, because it exhaustively enumerates all proper size-\(K\) trees over that set.

The heuristic procedures \(\mathit{hodtCoreset}\) and \(\mathit{sodtWSH}\) reduce computation by selecting and filtering candidate rules before invoking the exact tree solver. Therefore, for each restricted candidate rule set passed to \(\mathit{sodt}\) or \(\mathit{sodt}_{\mathrm{vec}}\), the returned tree is optimal within that restricted search space, but it is not necessarily globally optimal over the complete rule set generated by \(\mathit{genRules}\). This is a practical trade-off imposed by the available computational resources. In principle, global optimality can be recovered by continuing candidate generation until the complete rule set has been explored, while caching previously evaluated rules and configurations to avoid redundant computation.

Also, due to the intractable combinatorics of HSDT, the experiments for points (3) and (4) are conducted exclusively for
the $M=1$ (hyperplane) case. To obtain high-quality solutions efficiently,
we employ the two heuristic methods developed in Section \ref{sec:Heuerestic-methods}.

In these experiments, we adapt \citet{brita2025optimal}'s ConTree
algorithm rather than \citet{mazumder2022quant}'s Quant-BnB algorithm.
We observed that ConTree often provides more accurate solutions while
being considerably more efficient, making it more suitable for our
comparisons. We note that the current experiments do not explore the
effects of hyperparameter tuning, such as minimum data per leaf or
tree depth when optimizing size-constrained decision trees. The performance
impact of these fine-grained controls represents an interesting avenue
for future research.

All experiments were conducted on an Intel Core i9 CPU with 24 cores
(2.4--6 GHz), 32 GB RAM, and a GeForce RTX 4060 Ti GPU.

\subsection{Computational analysis\label{subsec:Computational-and-combinatorial analysis}}

In this section, we analyze in detail the computational efficiency
of $\mathit{sodt}$ for solving the size-constrained ODT problem.
Since $\mathit{sodt}_{\text{vec}}$ and $\mathit{sodt}_{\text{kperms}}$
share similar advantages---both are fully vectorized but unable to exploit
efficient dynamic programming (DP)---we focus on comparing $\mathit{sodt}_{\text{vec}}$
against the dynamic programming algorithm $\mathit{sodt}_{\text{rec}}$.

Rather than using a binary tree data structure, which relies on pointers
to locate subtrees and suffers from poor cache performance, both $\mathit{sodt}_{\text{vec}}$
and $\mathit{sodt}_{\text{rec}}$ are implemented entirely with array
(heap) data structures. This design ensures contiguous memory allocation,
significantly reducing cache misses.

We compare their performance in two settings: 1) \emph{Sequential
	setting}: Nested combinations are processed one-by-one using a standard
for-loop. 2) \emph{Parallelized setting}. A large batch of feasible
nested combinations is processed simultaneously. Specifically, we
store $\mathit{ncrs}:\mathit{NCRs}$ in a single large tensor instead
of a list and pass it as input to $\mathit{mapL}\left(\mathit{sodt},\mathit{ncrs}\right)$.
For $\mathit{mapL}\left(\mathit{sodt}_{\text{vec}}\right)$, we implement
this as a single function ($\texttt{batch\_sodt}$). This allows us
to efficiently process the large batch of feasible nested combinations
$\mathit{ncrs}$ on both CPU and GPU. We denote the results as $\mathit{sodt}_{\text{vec}}^{\text{cpu}}$
and $\mathit{sodt}_{\text{vec}}^{\text{gpu}}$, respectively.

In contrast, since $\mathit{sodt}_{\text{rec}}$ cannot be fully vectorized,
we parallelize $\mathit{mapL}\left(\mathit{sodt}_{\text{vec}}\right)$
using multi-core CPU execution, initialized in a multi-process setting.

\subsubsection{Comparison between the vectorized and recursive implementation\label{subsubsec: sodt-comparision-appendix}}

\begin{figure}
	\begin{centering}
		\includegraphics[width=0.32\linewidth]{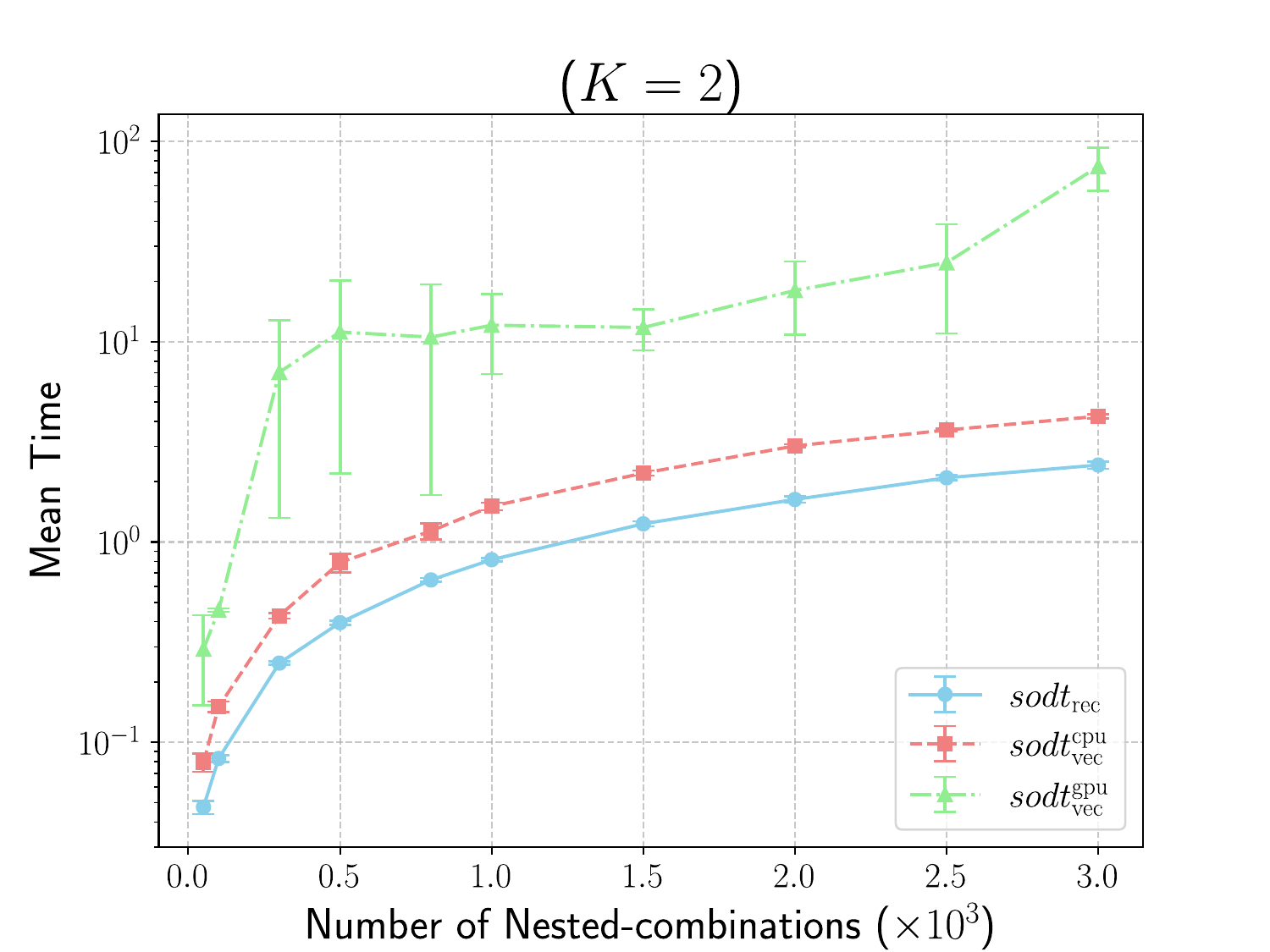}\includegraphics[width=0.32\linewidth]{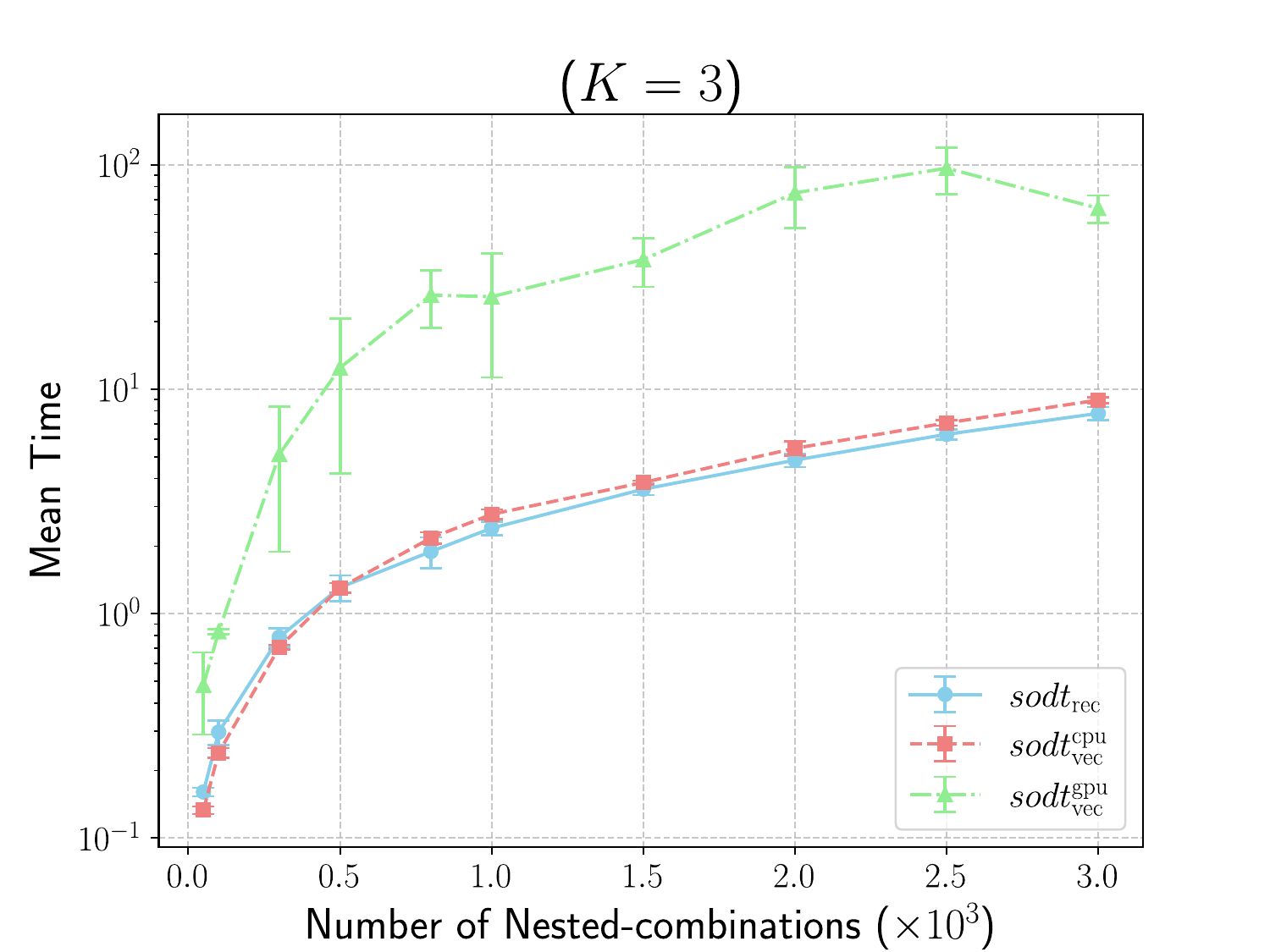}\includegraphics[width=0.32\linewidth]{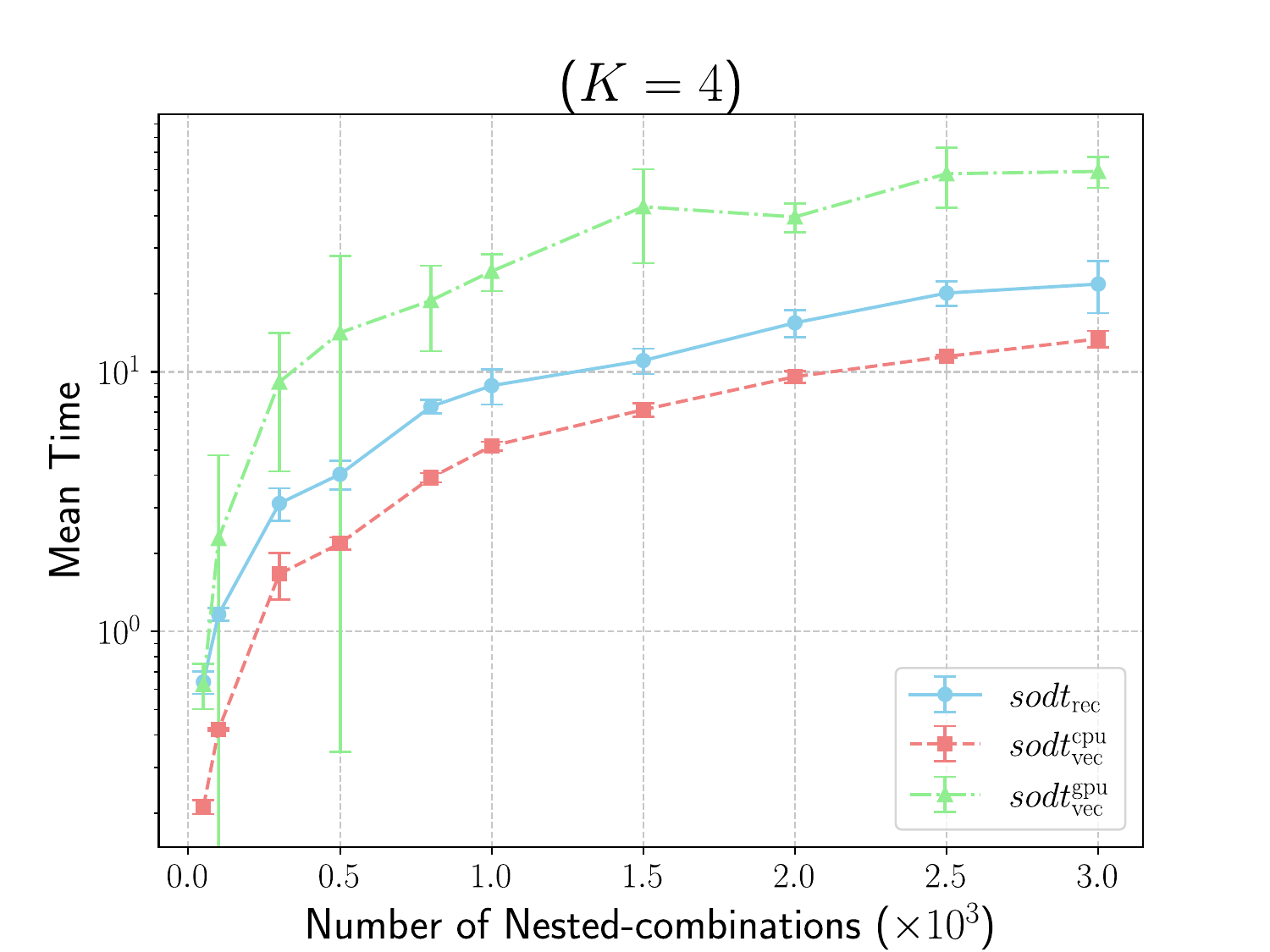}
		\par\end{centering}
	\begin{centering}
		\includegraphics[width=0.32\linewidth]{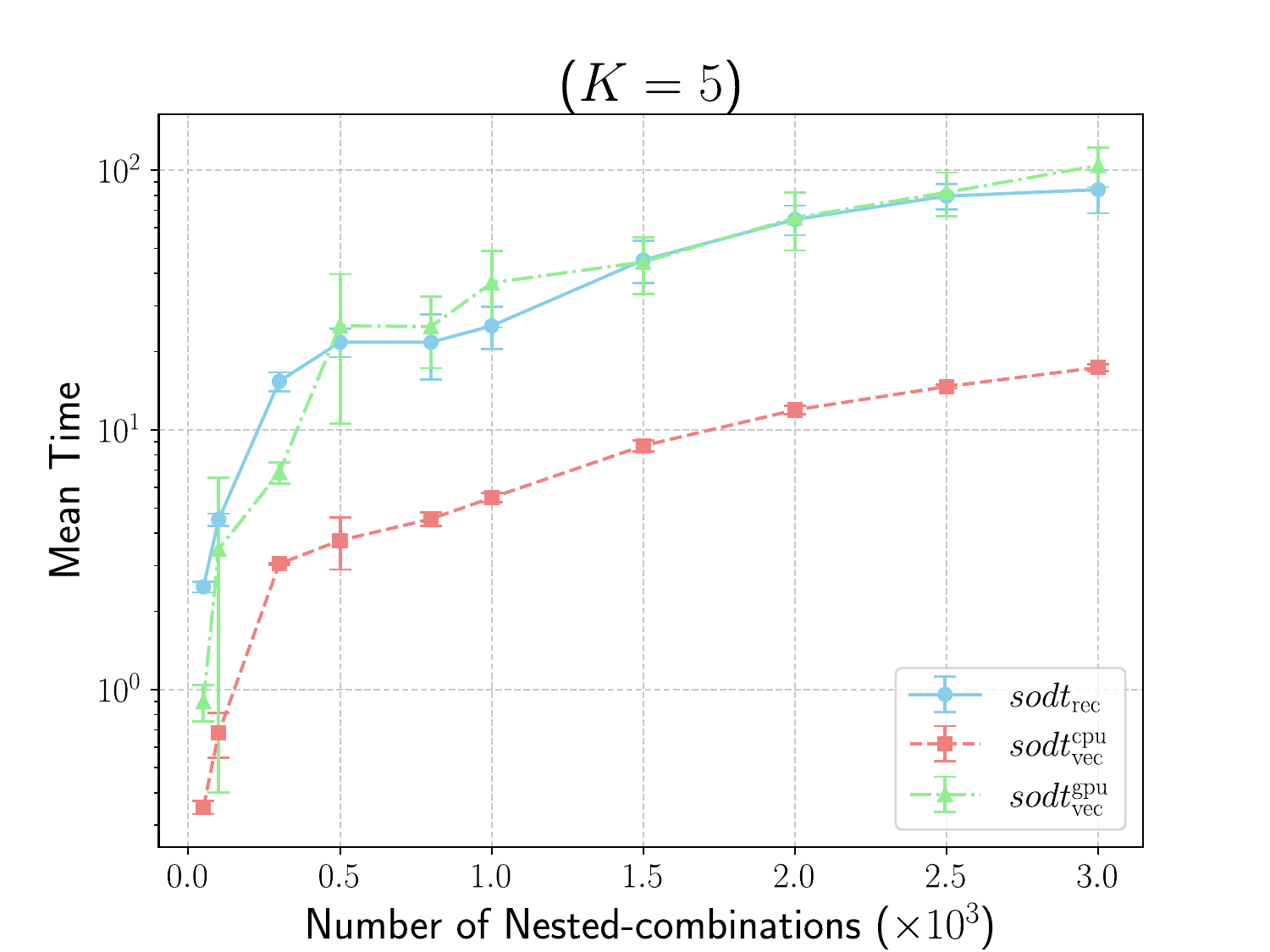}\includegraphics[width=0.32\linewidth]{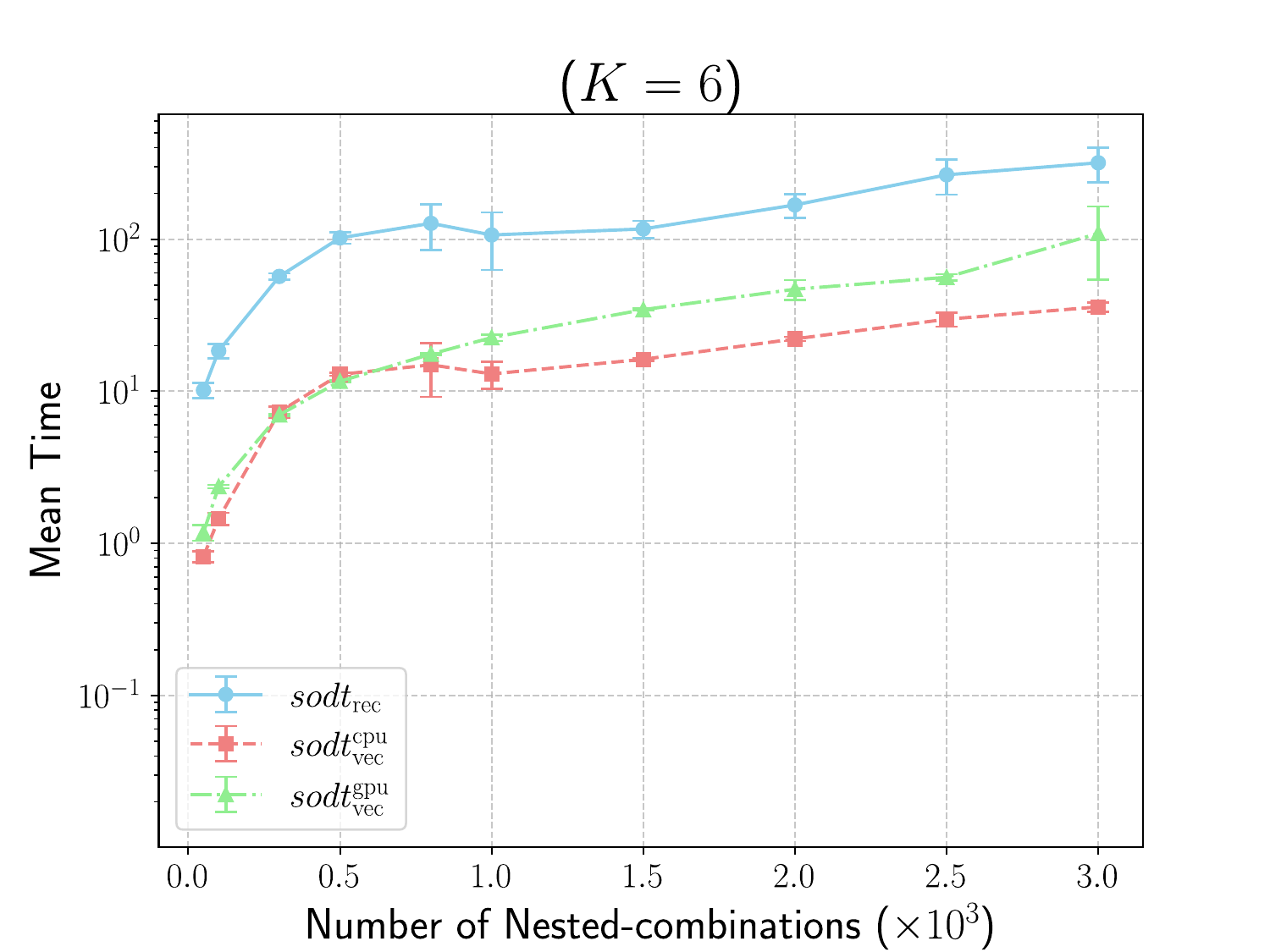}\includegraphics[width=0.32\linewidth]{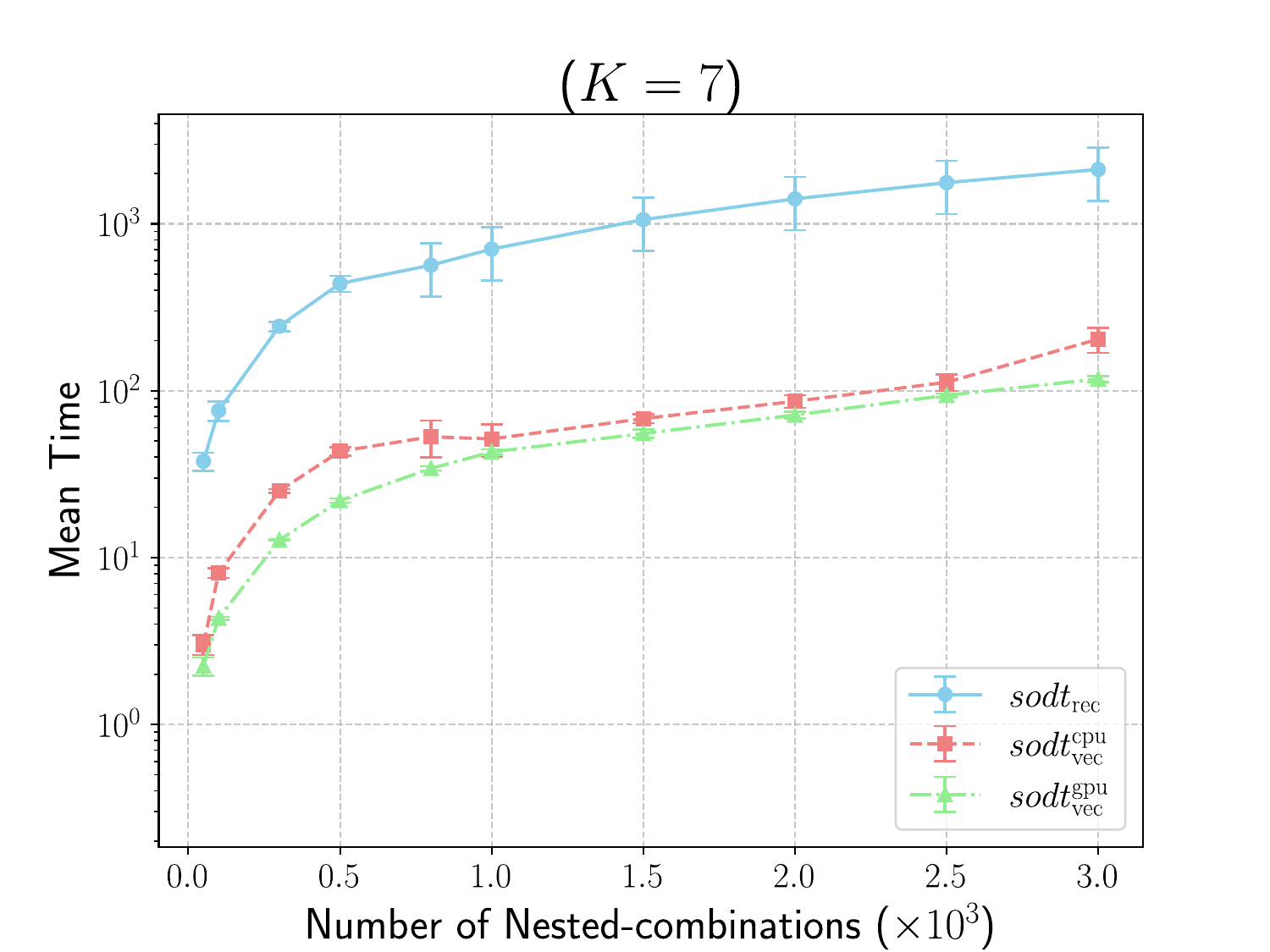}
		\par\end{centering}
	\caption{Running time comparison between $\mathit{sodt}_{\text{rec}}$ and
		$\mathit{sodt}_{\text{vec}}$ with varying $K$ on sequential setting.
		\label{fig: sequential compare between sodt}}
\end{figure}

Even though $\mathit{sodt}_{\text{rec}}$ has provably lower theoretical
complexity, the hardware compatibility of $\mathit{sodt}_{\text{vec}}$
allows it to outperform in practice, particularly as $K$ increases.
As shown in Figure \ref{fig: sequential compare between sodt}, in
the sequential setting $\mathit{sodt}_{\text{rec}}$ is the most efficient
method when $K\le3$. However, it becomes slower than $\mathit{sodt}_{\text{vec}}^{\text{cpu}}$
at $K=4$, and eventually the slowest method once $K\geqq5$. In contrast,
$\mathit{sodt}_{\text{vec}}^{\text{gpu}}$ is initially slower than
all other methods but overtakes $\mathit{sodt}_{\text{rec}}$ after
$K=5$, ultimately becoming the most efficient method at $K=7$.

\begin{figure}
	\begin{centering}
		\includegraphics[width=0.32\linewidth]{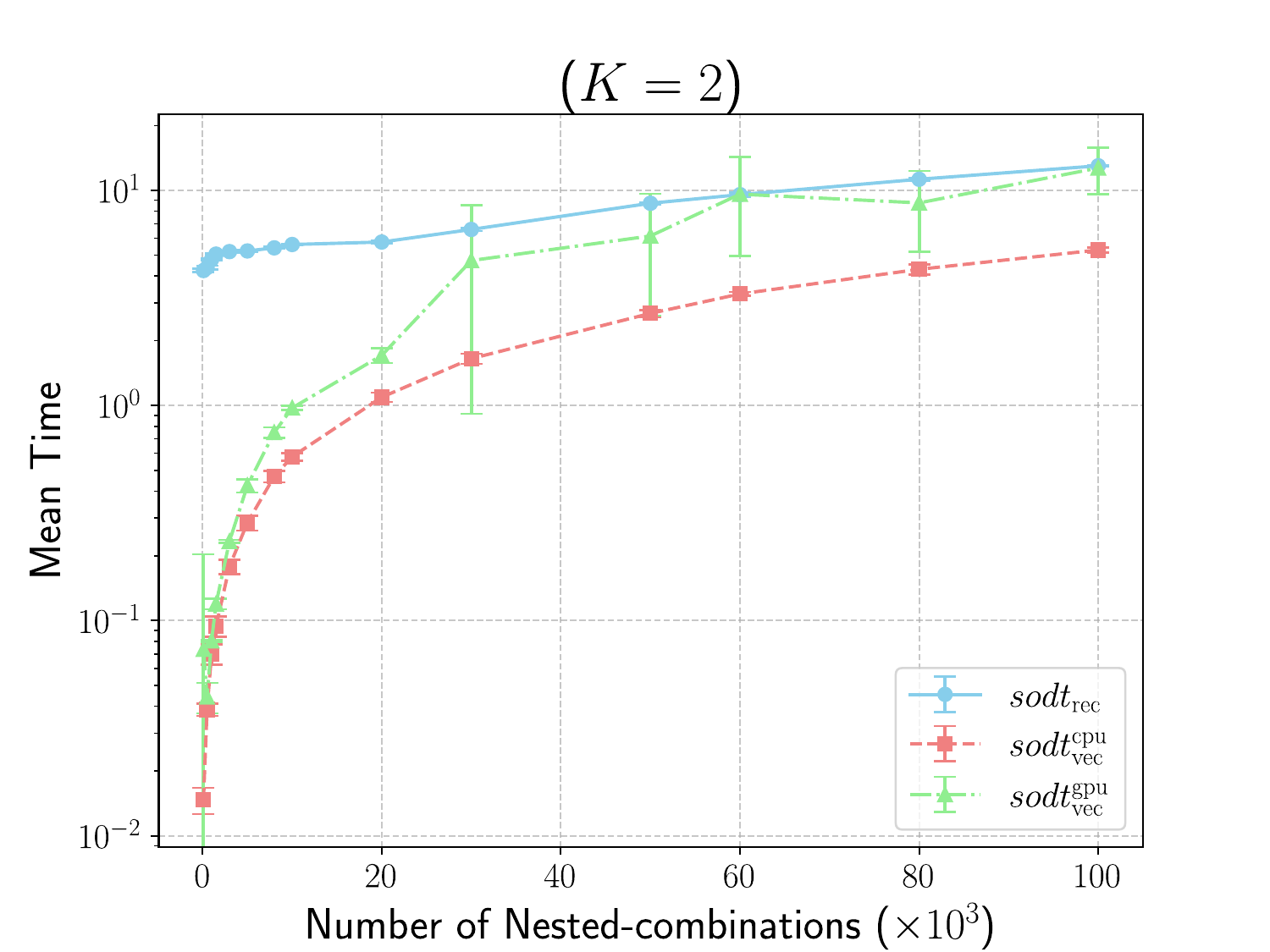}\includegraphics[width=0.32\linewidth]{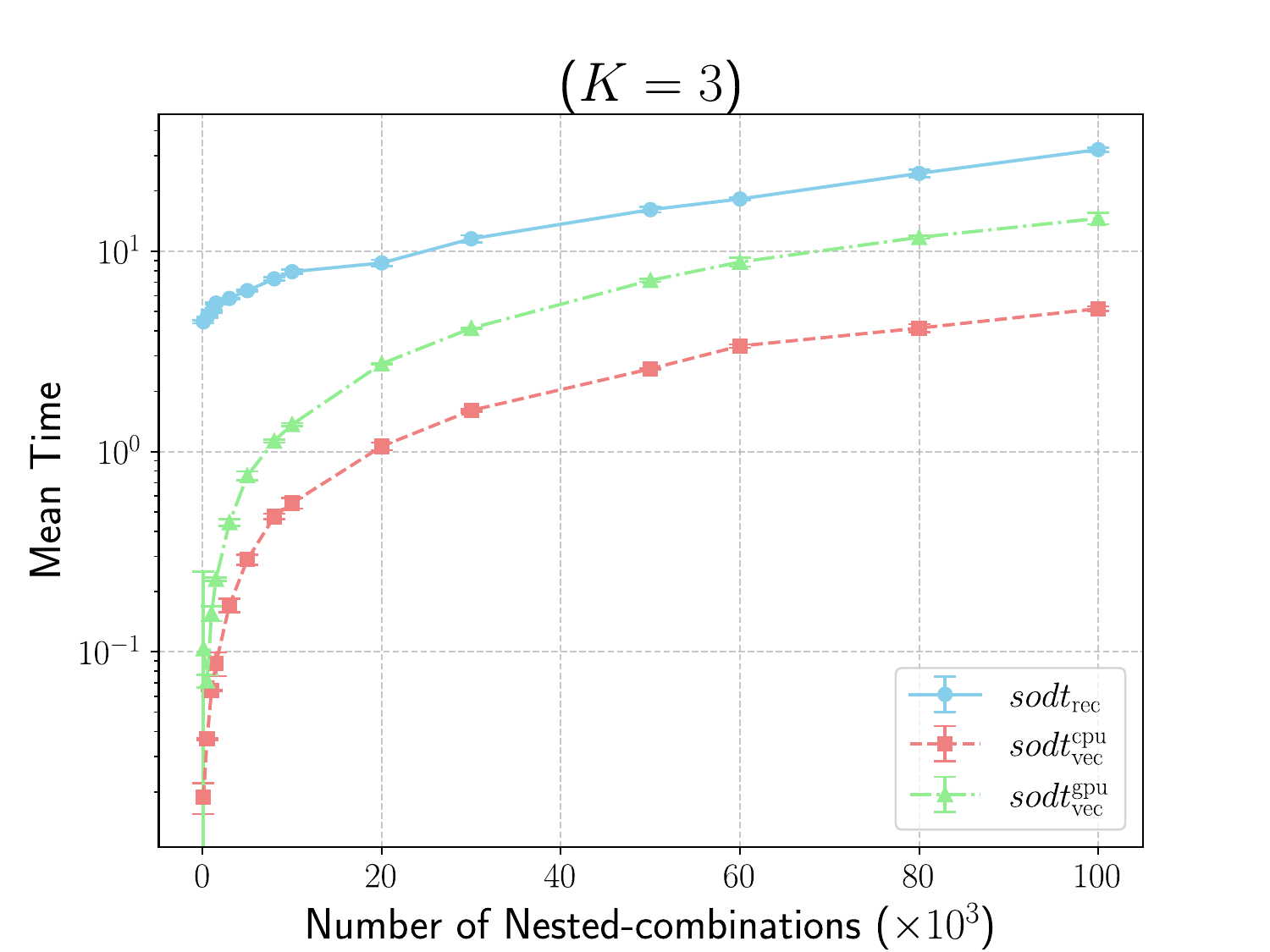}\includegraphics[width=0.32\linewidth]{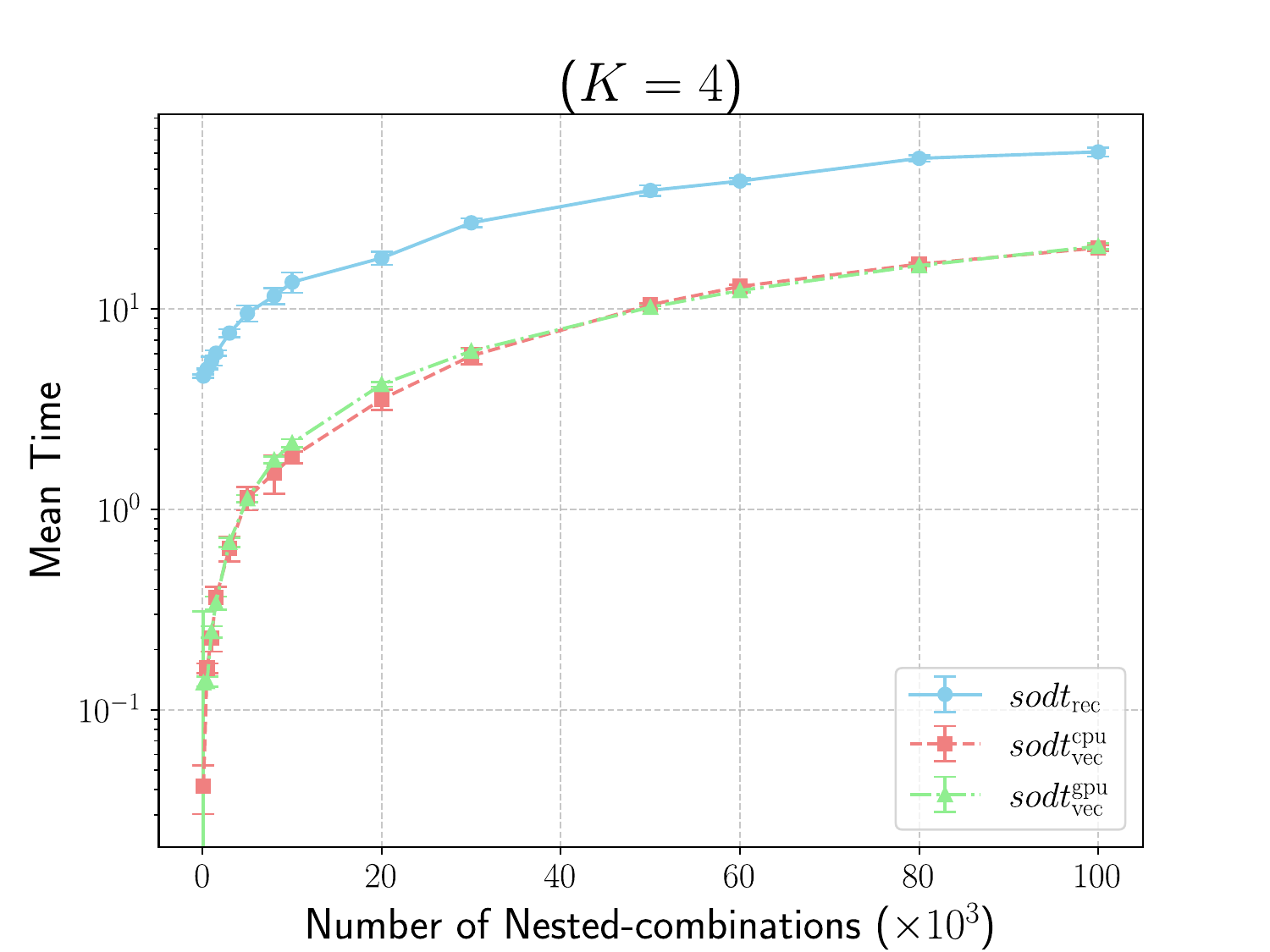}
		\par\end{centering}
	\begin{centering}
		\includegraphics[width=0.32\linewidth]{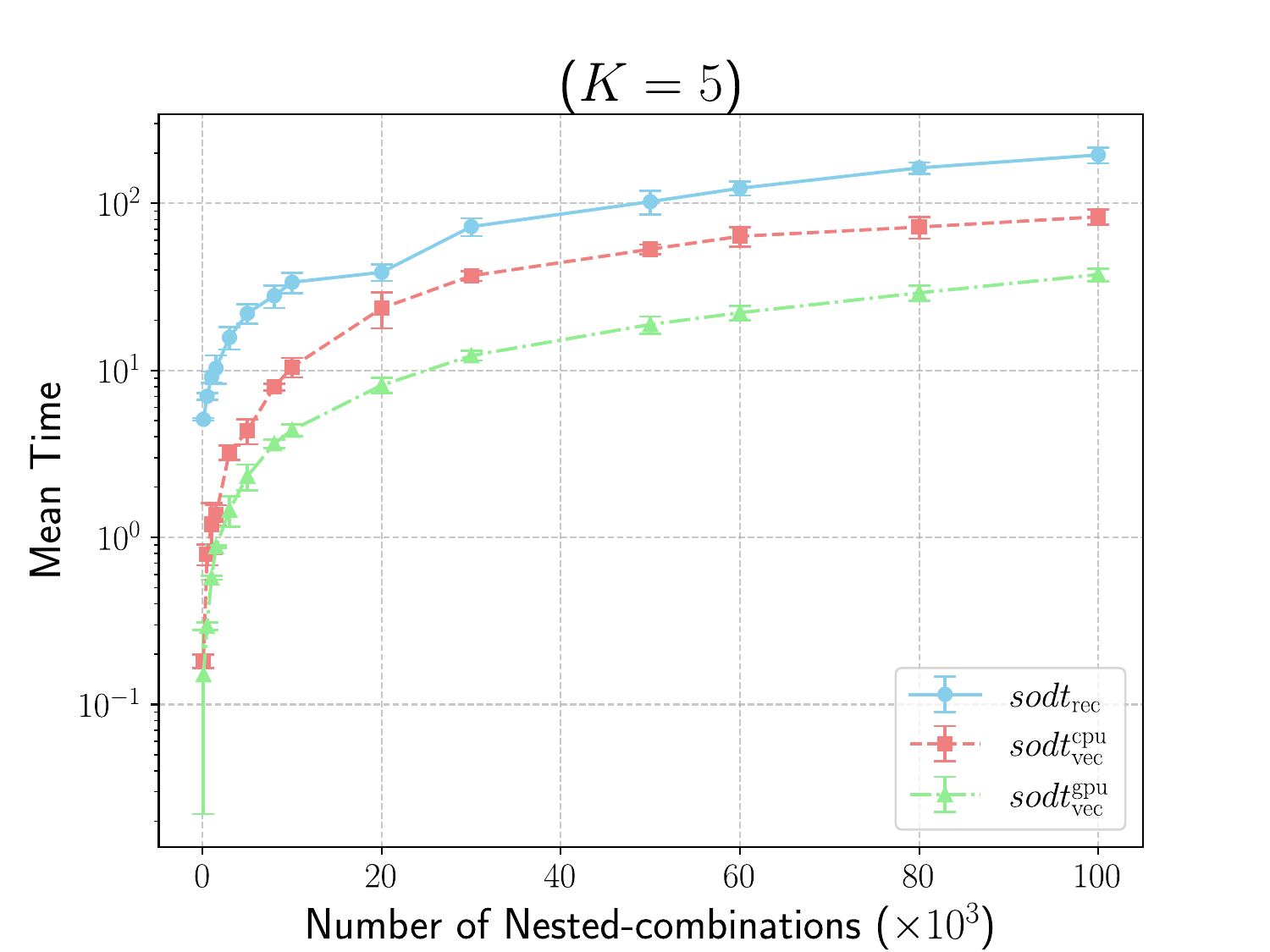}\includegraphics[width=0.32\linewidth]{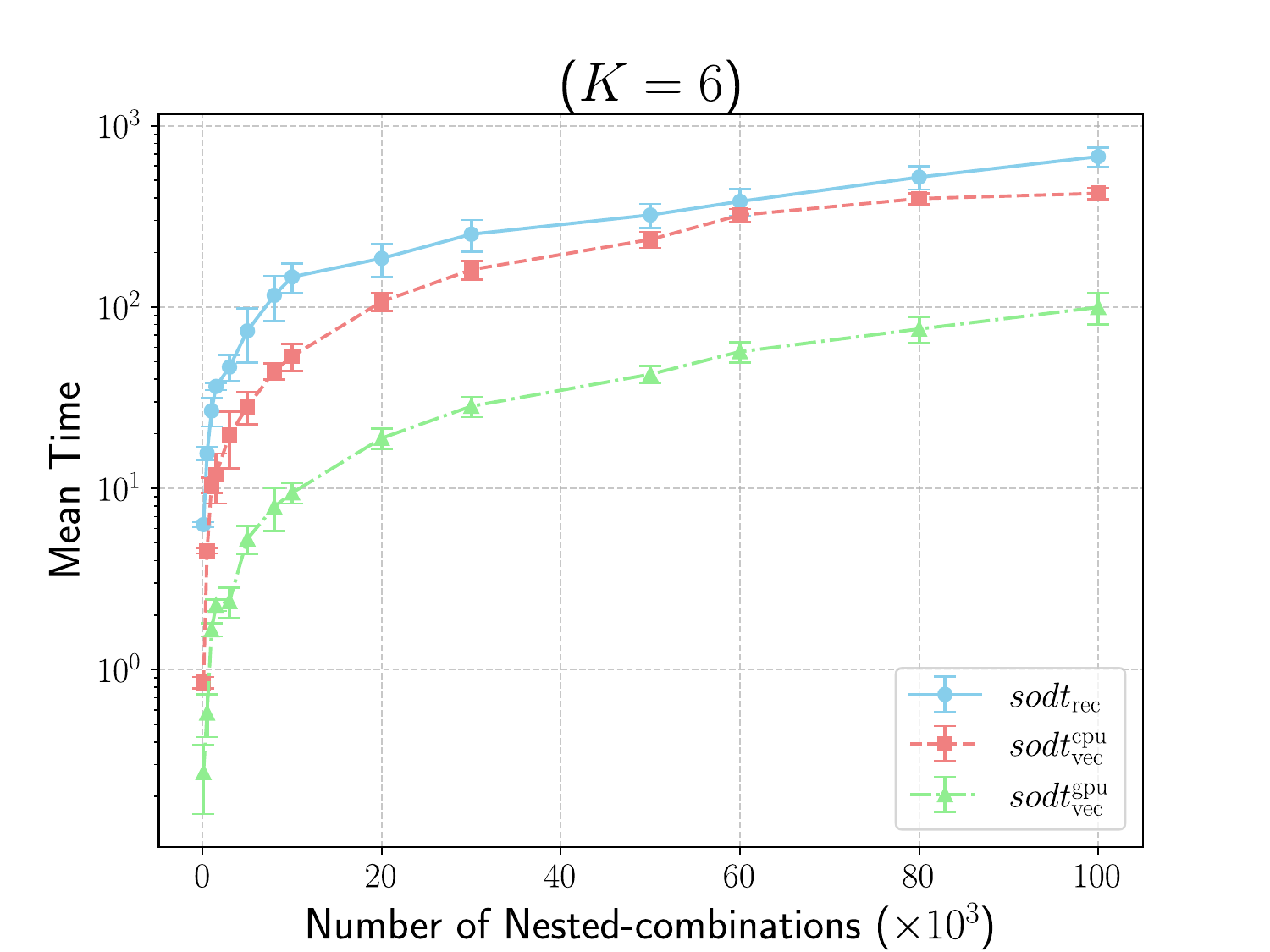}\includegraphics[width=0.32\linewidth]{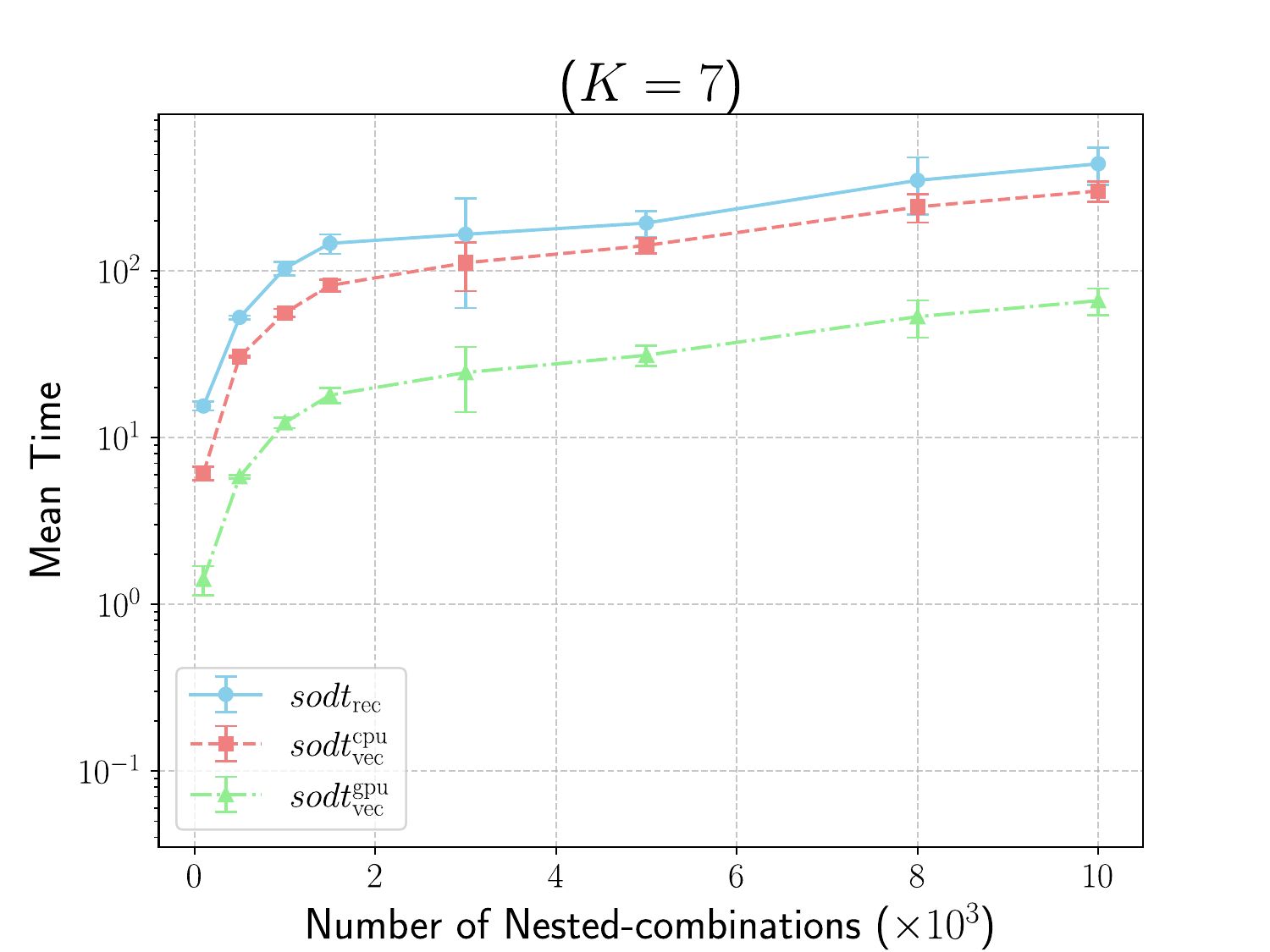}
		\par\end{centering}
	\caption{Running time comparison between $\mathit{sodt}_{\text{rec}}$ and
		$\mathit{sodt}_{\text{vec}}$ with varying $K$ on parallel setting.\label{fig: parallel compare between sodt}}
\end{figure}

In the parallel setting (Figure \ref{fig: parallel compare between sodt}),
the benefits of vectorization are even more pronounced. Parallelizing
$\mathit{mapL}\left(\mathit{sodt}_{\text{rec}}\right)$ with multiprocessing
incurs significant overhead from initializing multiple CPU cores,
making $\mathit{sodt}_{\text{rec}}$ consistently the slowest method
across all cases. Comparing $\mathit{sodt}_{\text{vec}}^{\text{cpu}}$
and $\mathit{sodt}_{\text{vec}}^{\text{gpu}}$, we observe that when
$K<4$, the CPU implementation is more efficient. For $K=4$, both
implementations achieve nearly identical performance, while for $K>4$,
$\mathit{sodt}_{\text{vec}}^{\text{gpu}}$ becomes superior. The GPU
computation does not always outperform CPU computation because $\mathit{mapL}\left(\mathit{sodt}_{\text{vec}}\right)$
in its current implementation consumes excessive memory to store intermediate
results. The resulting memory transfer overhead can outweigh computational
gains when $k$ is small. A more sophisticated low-level implementation
(e.g., in CUDA rather than Python) may reduce these overheads and
further improve GPU performance.

\subsubsection{Computational scalability of the vectorized method---the ability to
	explore one million nested combinations within a fixed time}

Following the detailed comparison between $\mathit{sodt}_{\text{rec}}$
and $\mathit{sodt}_{\text{vec}}$ in both sequential and parallel
settings, we now turn to the computational scalability of the winning
method, $\mathit{sodt}_{\text{vec}}$. We focus on analyzing the pure
computational scalability of the algorithm by benchmarking its performance
without any acceleration techniques. Specifically, we measure the
total wall-clock runtime of $\mathit{sodt}_{\text{vec}}$ for solving
$1\times10^{6}$ instances of feasible nested combinations (containing no
certified fatal pair) for different values of $K$. In our setup, cases
with $K\leq4$ are executed on the CPU, while cases with $K>4$ are
executed on the GPU.

As expected, when $K$ is fixed, the bound in
Theorem~\ref{thm: complexity of sodt_vec} is linear in $N$, and
$\mathit{sodt}_{\text{vec}}$ exhibits approximately linear runtime growth.
The
batched vectorized implementation proves highly efficient for $K=2,3$;
for example, it solves $1\times10^{6}$ feasible nested combinations
in only a few tens of seconds. Note that a single size $K$ nested
combination can generate up to $K!$ possible proper decision trees
in the worst-case. These results demonstrate the clear computational
advantages of the $\mathit{sodt}_{\text{vec}}$ algorithm.

All current experiments are implemented in Python using the PyTorch
library, and we anticipate that a lower-level implementation will
further enhance performance.

\subsection{Combinatorial complexity of the hyperplane/hypersurface decision
	tree\label{subsec:Fatal-pair experiments}}

\subsubsection{Combinatorial complexity of nested combinations after filtering out
	fatal pairs}

\begin{figure}
	\begin{centering}
		\includegraphics[viewport=10bp 0bp 900bp 600bp,clip,scale=0.17]{figures/Combinatorial_complexity_D=2}\includegraphics[viewport=10bp 0bp 900bp 600bp,clip,scale=0.17]{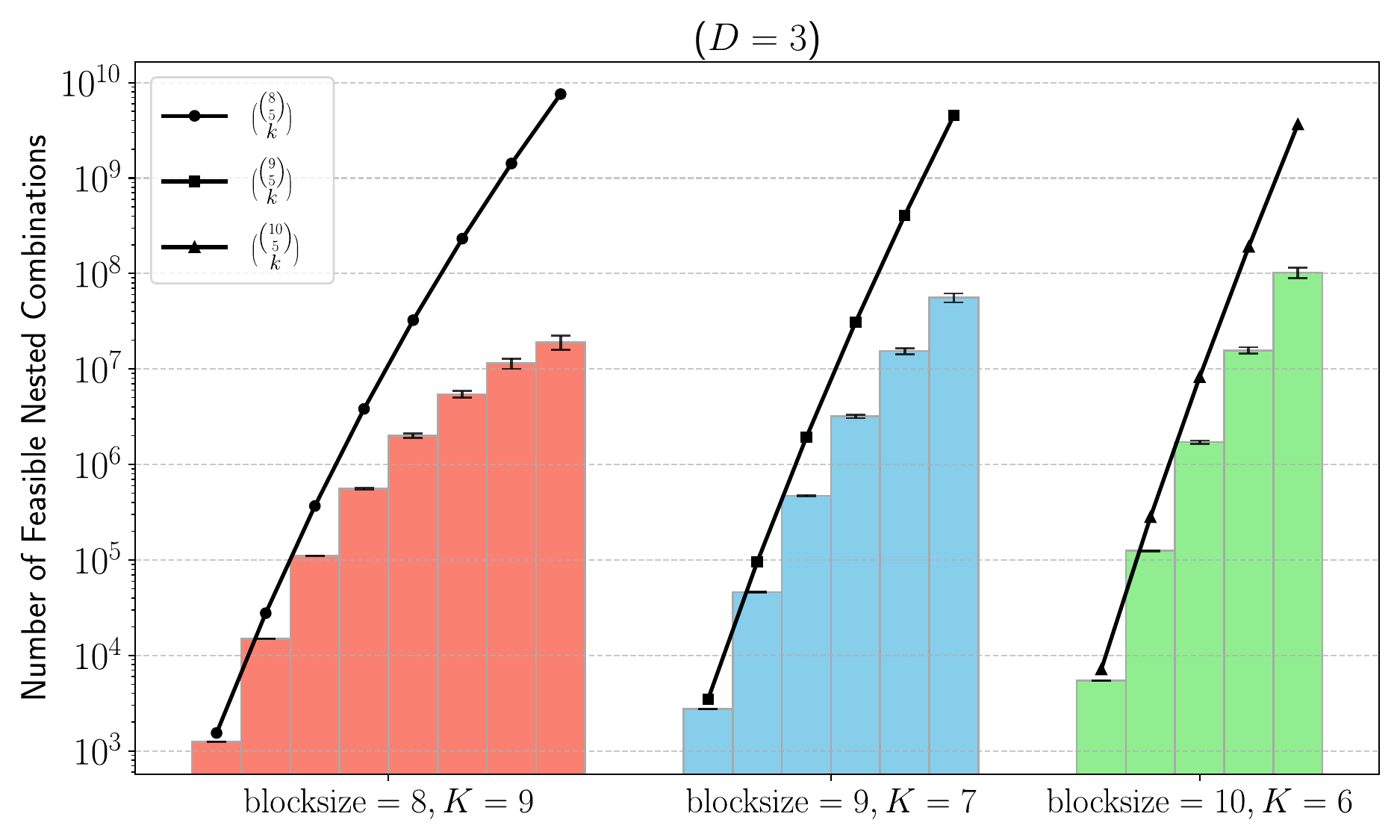}
		\par\end{centering}
	\begin{centering}
		\includegraphics[viewport=10bp 0bp 900bp 600bp,clip,scale=0.17]{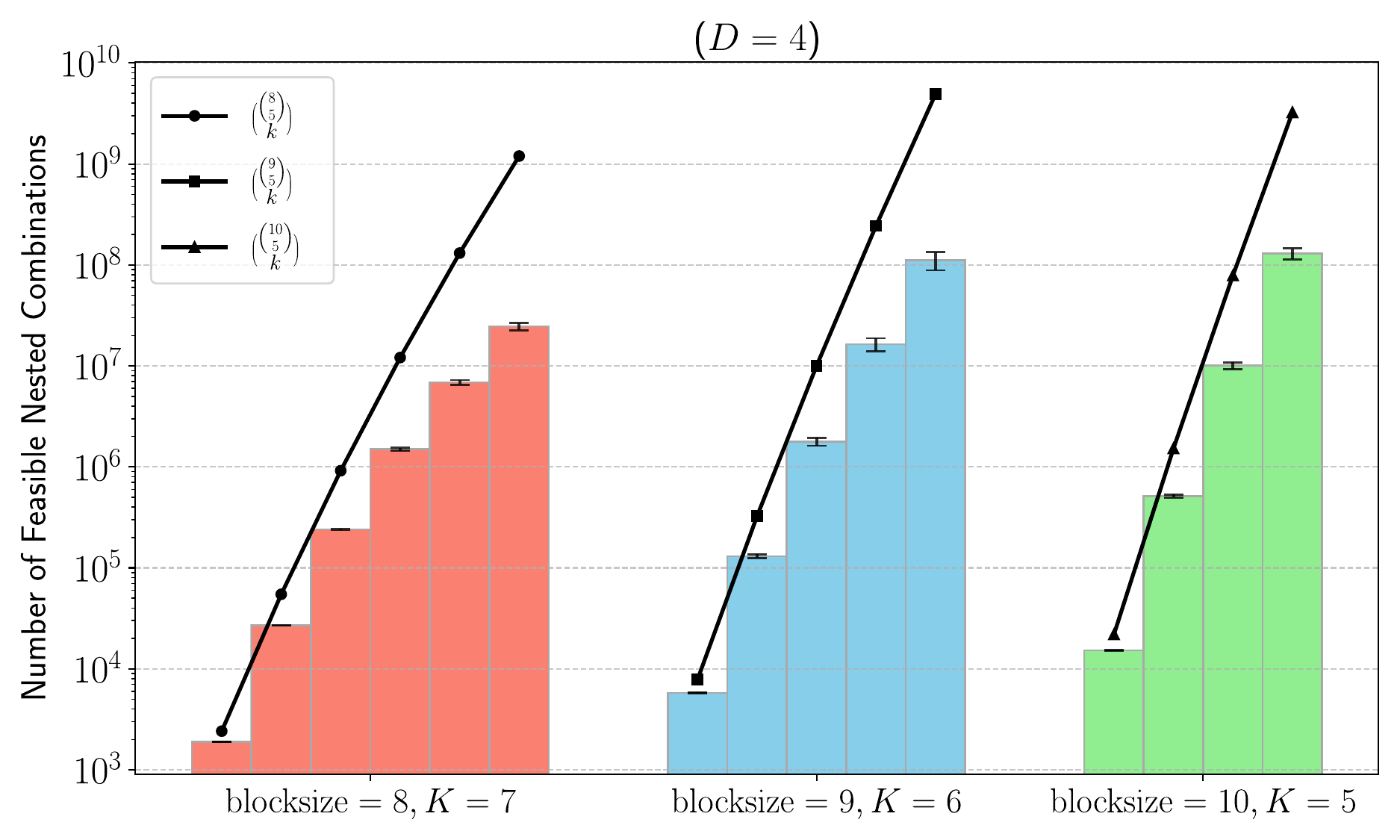}\includegraphics[viewport=10bp 0bp 900bp 600bp,clip,scale=0.17]{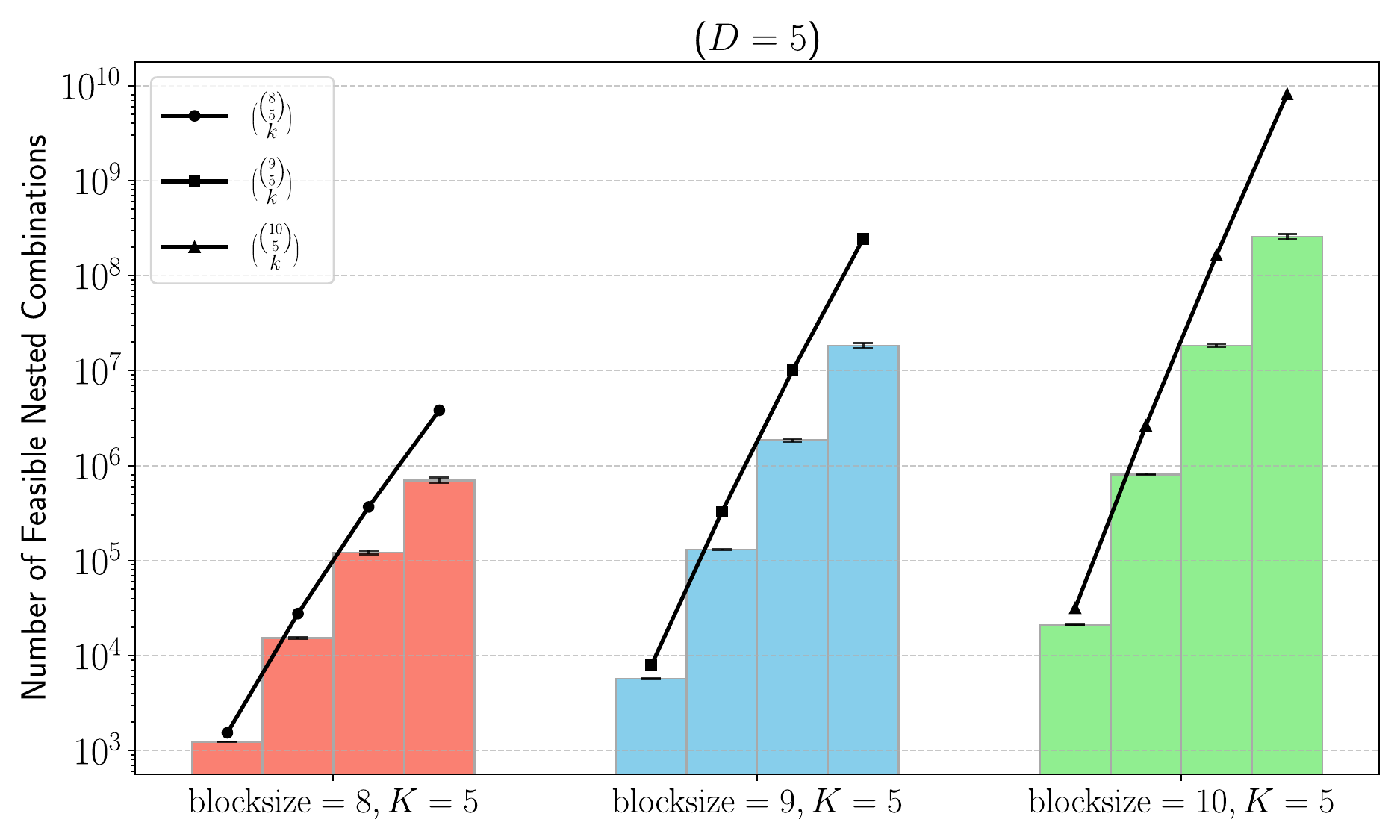}
		\par\end{centering}
	\caption{Combinatorial complexity of feasible nested combinations with varying
		$K$. Each bar corresponds to a fixed $K$ starts from $K=2$ (left-most
		bar) to the value of $K$ labeled below each group of bars. \label{fig:Combinatorial-Complexity with varing K}}
\end{figure}

As noted, the combinatorial complexity of the hyperplane ($M=1$)
ODT problem is bounded by $\left(\begin{array}{c}
	\left(\begin{array}{c}
		N\\
		D
	\end{array}\right)\\
	K
\end{array}\right)=O\left(N^{DK}\right)$. However, as established in
Theorem~\ref{When-two-hyperplanes}, any nested combination containing a
fatal pair cannot yield a complete decision tree. Consequently, all
configurations containing a certified fatal pair can be filtered out without
compromising the optimality of the algorithms.

To evaluate the true combinatorial complexity of the problem after
excluding infeasible combinations, we conducted experiments on synthetic
Gaussian datasets of size $N=2000$. Due to memory limitations, we
tested only subsets of the original dataset (denoted as ``$\text{blocksize}$''
in Figure \ref{fig:Combinatorial-Complexity with varing K}. Specifically,
a small sub-dataset was sampled from the $N=2000$ dataset and then
passed into $\mathit{nestedCombsFA}$, while hyperplane predictions
were still carried out using the full $N=2000$ dataset, but $\text{blocksize}$
ranged from 8 to 10.

For each panel in Figure \ref{fig:Combinatorial-Complexity with varing K},
we fixed the dimension $D$ and varying $N$ and $K$. We computed
the combinatorial complexity over five datasets and reported the mean
and standard deviation, illustrated as bar charts with error bars.
The results show that the true combinatorial complexity (bars), after
filtering infeasible nested combinations, is substantially smaller
than the theoretical upper bound $O\left(N^{DK}\right)$ (black lines).For
example, when $\text{blocksize}=10$, $D=2$, $K=10$, the theoretical
number of nested combinations is $3\times10^{9}$, whereas the number
of feasible nested combinations is only $1\times10^{7}$, more than
300 times difference!

An interesting pattern emerges in the $D=2$ panel: the complexity
curve forms an inverted U-shape, indicating that the combinatorial
complexity of decision trees initially increases with $K$ but eventually
decreases. This phenomenon aligns with the explanation given in \citet{he2025odt},
under mild probabilistic assumptions, the likelihood of constructing
a feasible nested combination decreases exponentially with $K$. The
inverted U-shape arises because, at first, the rapid growth in the
number of nested combinations dominates, but beyond a certain threshold,
the exponential decay in feasibility prevails. Nevertheless, in practice,
this behavior is rarely observed, as it typically requires $K$ to
reach astronomically large values even for medium-sized datasets.

\paragraph{Comparing the combinatorial complexity of hypersurface and hyperplanes
	in the same dimension}

\begin{figure}
	\begin{centering}
		\includegraphics[viewport=10bp 0bp 850bp 500bp,clip,width=0.32\linewidth]{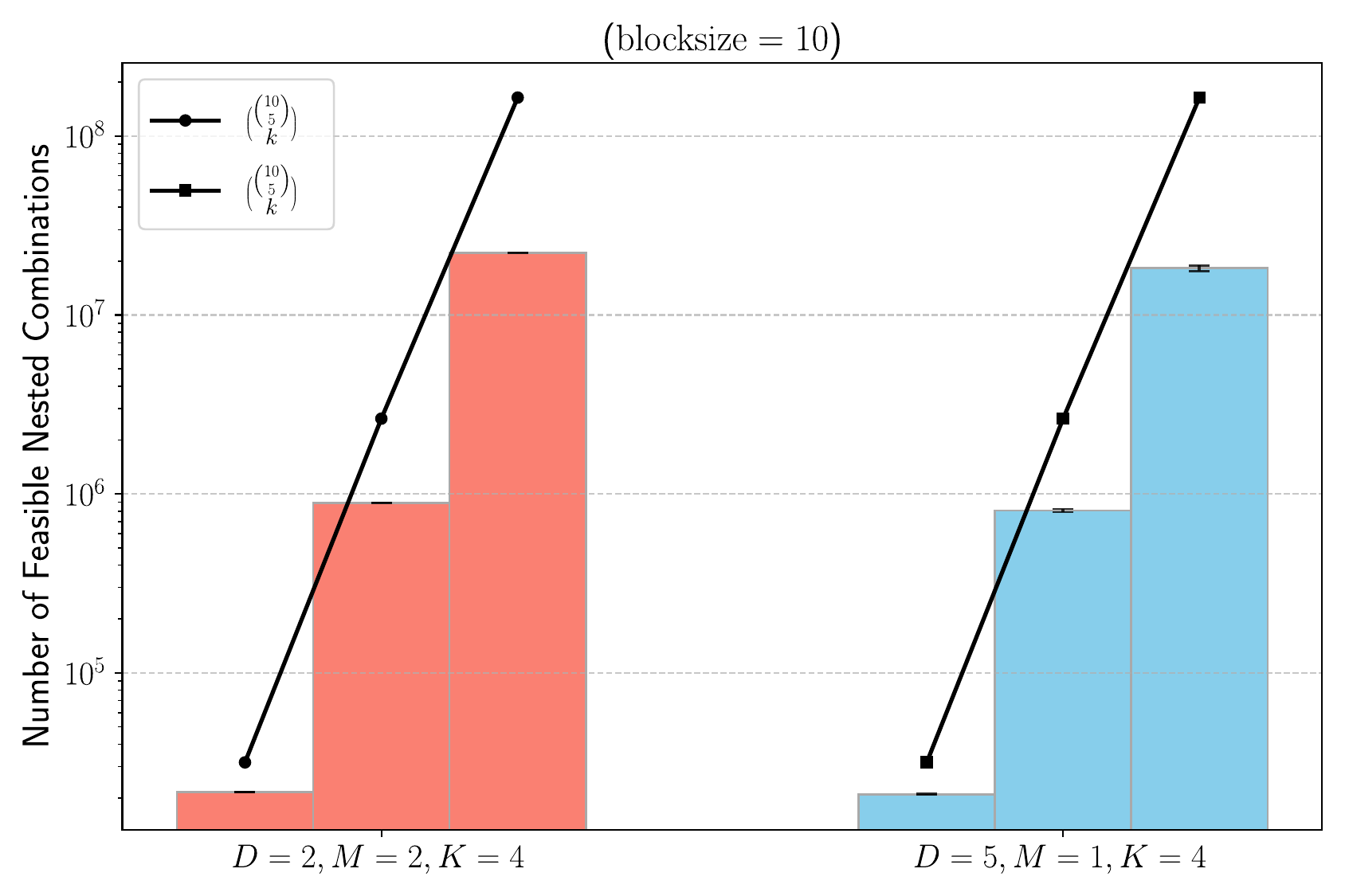}\includegraphics[viewport=10bp 0bp 850bp 500bp,clip,width=0.32\linewidth]{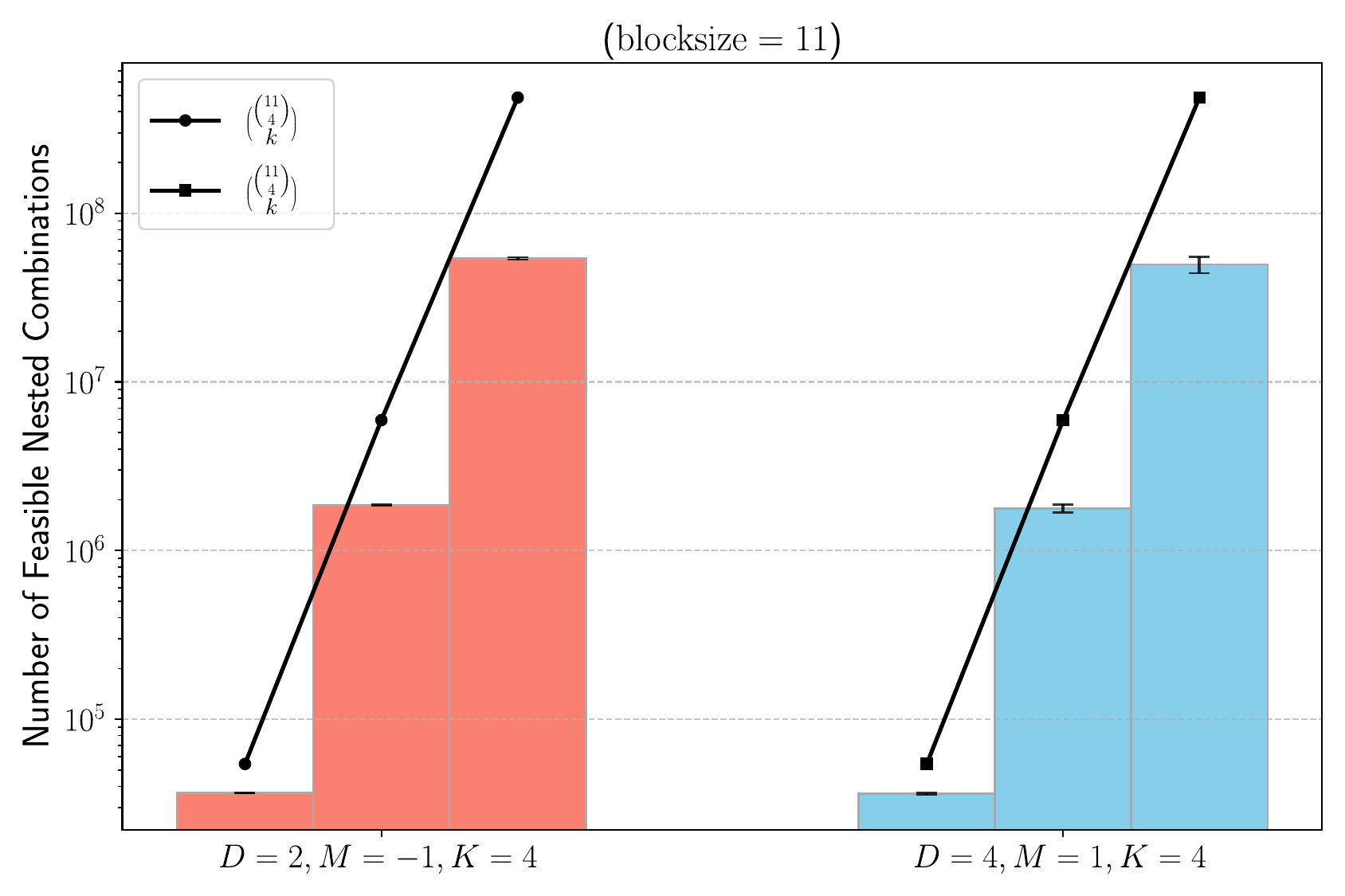}\includegraphics[viewport=10bp 0bp 850bp 500bp,clip,width=0.32\linewidth]{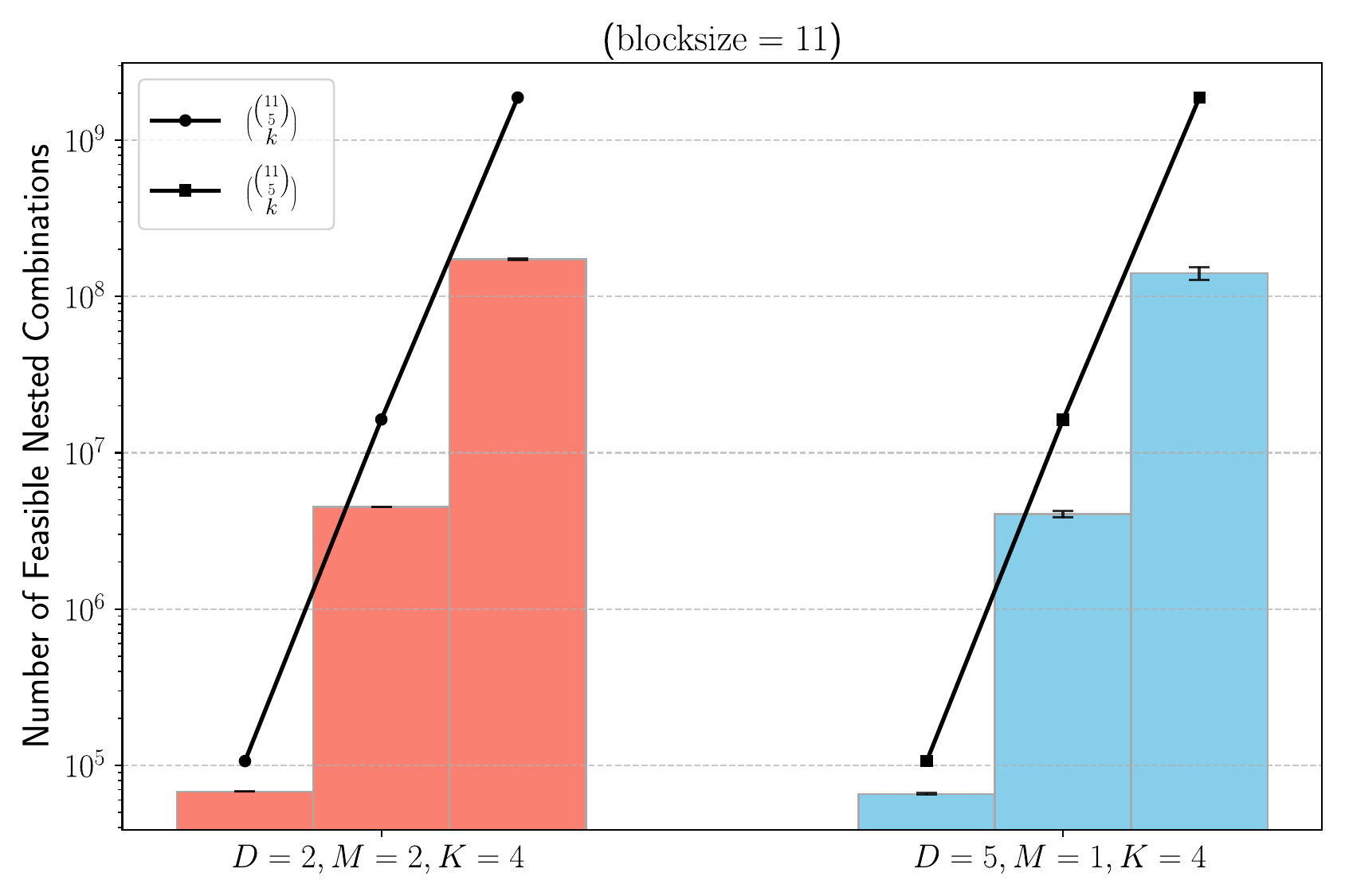}
		\par\end{centering}
	\begin{centering}
		\includegraphics[viewport=10bp 0bp 850bp 500bp,clip,width=0.32\linewidth]{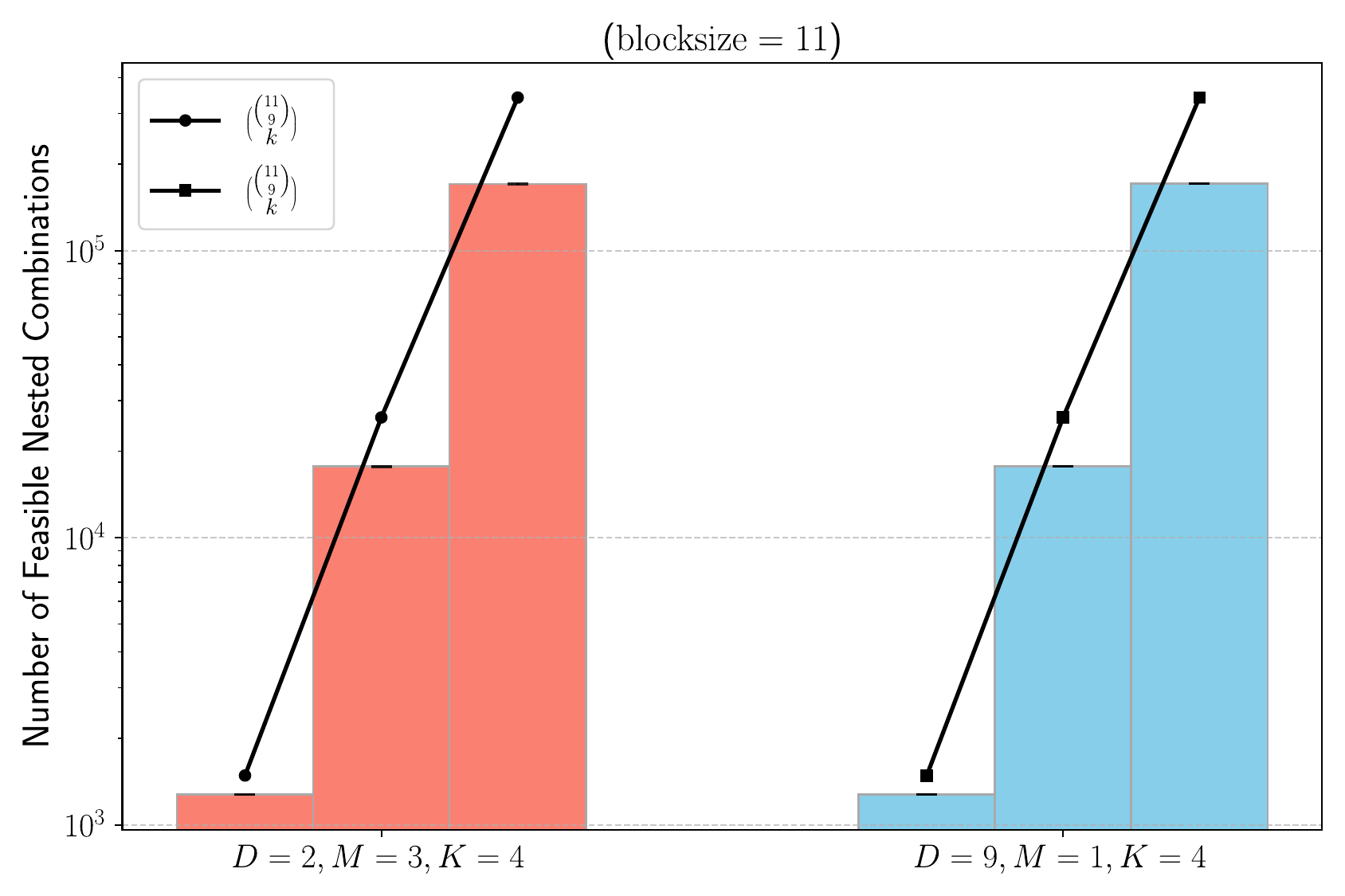}\includegraphics[viewport=10bp 0bp 850bp 500bp,clip,width=0.32\linewidth]{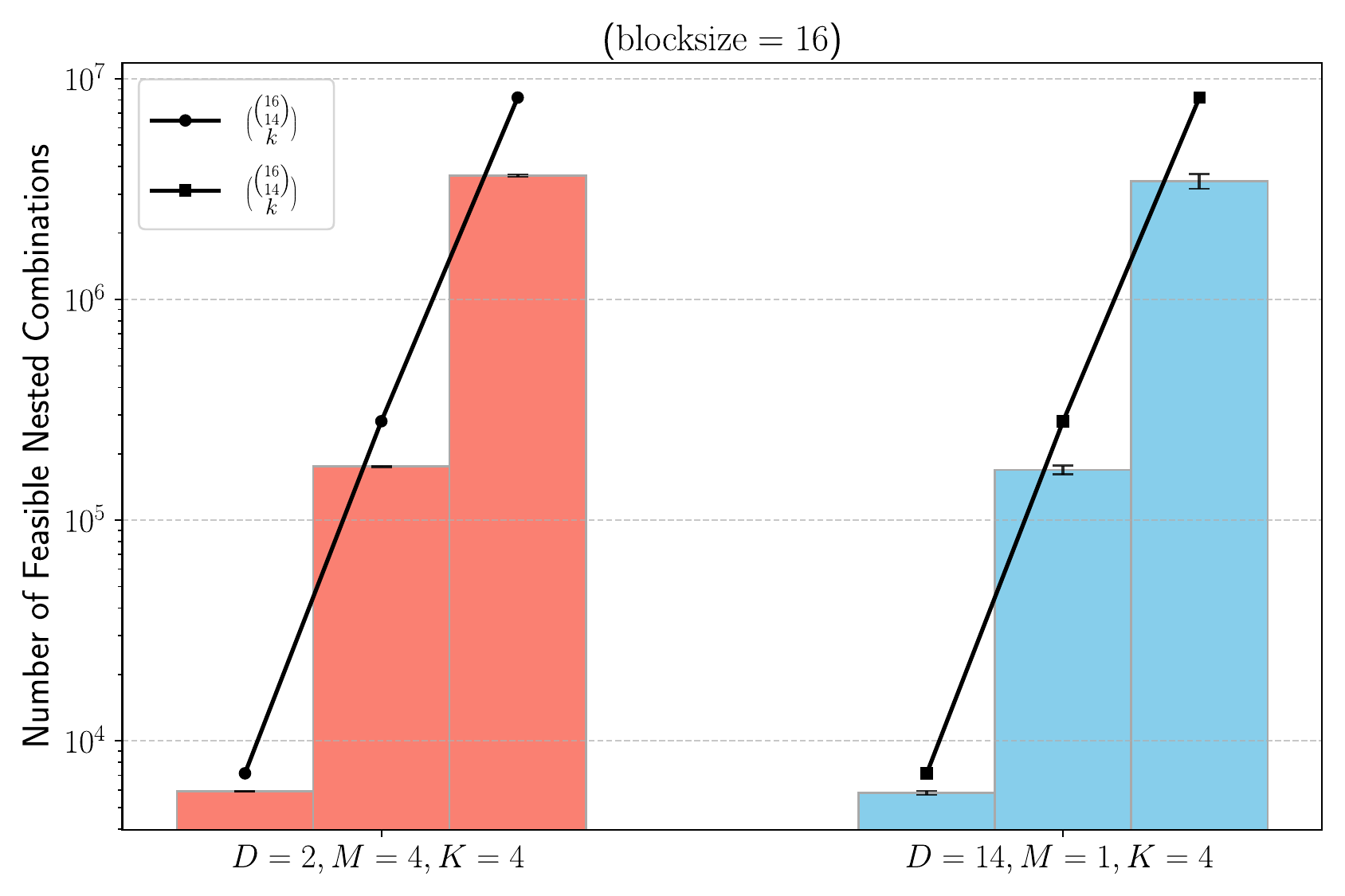}\includegraphics[viewport=10bp 0bp 850bp 500bp,clip,width=0.32\linewidth]{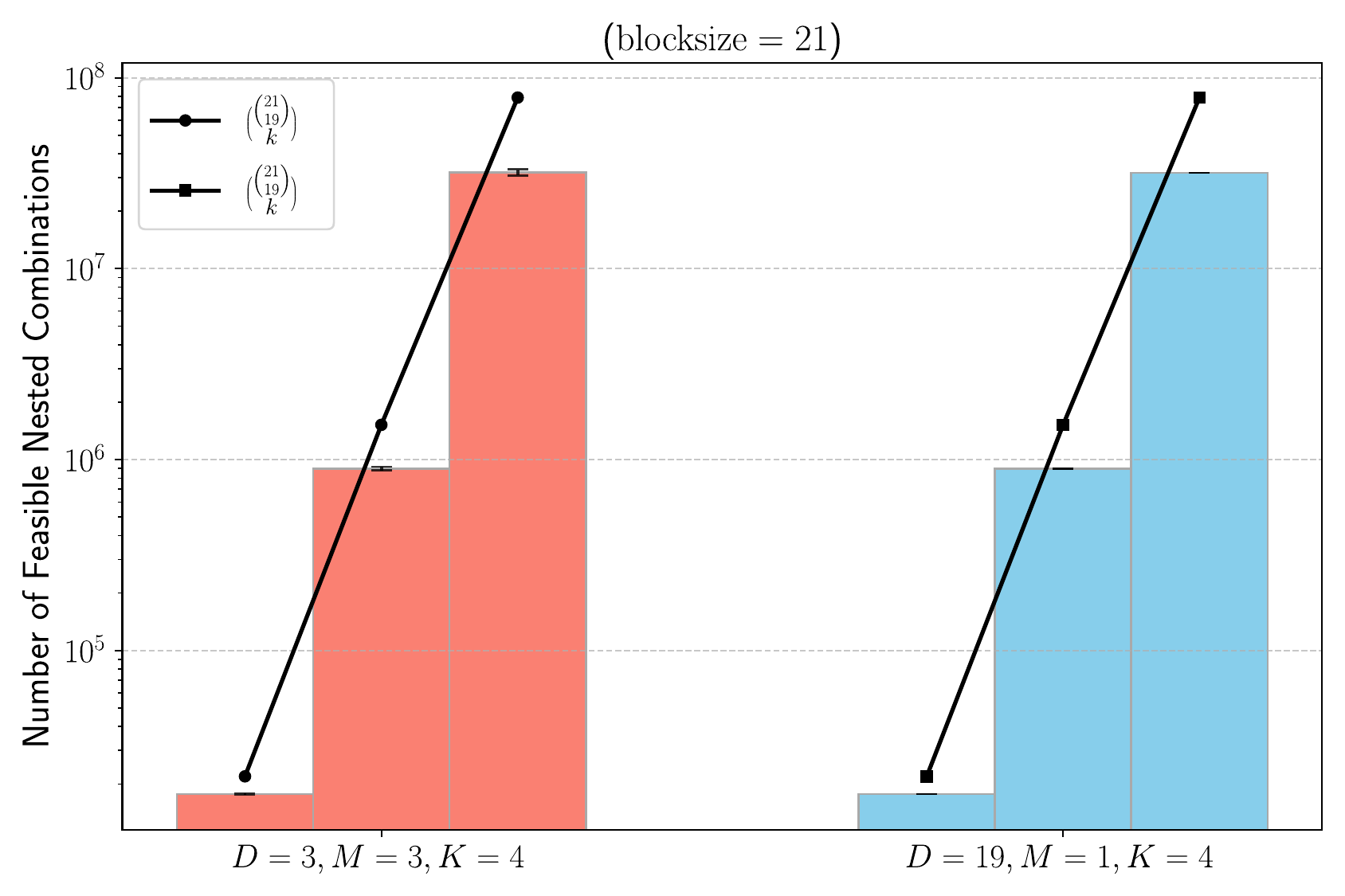}
		\par\end{centering}
	\caption{Combinatorial complexity of feasible nested combinations with varying
		$K$. Each bar corresponds to a fixed $K$ starts from $K=2$ (left-most
		bar) to the value of $K$ labeled below each group of bars.\label{fig:Combinatorial-Complexity--dual}}
\end{figure}

Having analyzed the combinatorial complexity of hyperplane splitting
rules, we now extend the analysis to hypersurface rules by comparing
their complexity with that of hyperplanes in the same dimension. As
discussed, a hypersurface defined by a degree-$M$ polynomial can
be represented as a hyperplane embedded in a space of dimension $G=\left(\begin{array}{c}
	D+M\\
	D
\end{array}\right)-1$ dimensional space.

For instance, when $D=4$ and $M=2$ we obtain $G=14$. Our goal is
to determine whether hypersurfaces in $\mathbb{R}^{D}$ (corresponding
to hyperplanes in $\mathbb{R}^{G}$) exhibit different combinatorial
complexity compared to hyperplanes. As illustrated in Figure \ref{fig:Combinatorial-Complexity--dual},
the degree-$M$ hypersurface in $\mathbb{R}^{D}$ shows slightly higher
combinatorial complexity than the corresponding hyperplane in $\mathbb{R}^{G}$,
This increase may stem from the greater expressivity of hypersurfaces
relative to hyperplanes.

\subsection{Computational experiments with synthetic datasets\label{subsec: synthetic dataset}}

In this section, we evaluate the performance of the HODT model on
a variety of synthetically generated datasets. The aim is to assess
how effectively a more flexible model can recover the underlying ground
truth, which is itself generated by a decision tree model.

Our experiments build on those of \citet{murthy1995decision} and
\citet{bertsimas2017optimal}. While their datasets were generated
using axis-parallel decision trees, we generalize the setting by using
hyperplane decision trees. Consequently, the underlying partitions
in our datasets are polygonal regions. Specifically, we generate synthetic
datasets from randomly constructed hyperplane decision trees, and
then compare the performance of different algorithms in inducing trees
that approximate the ground truth.

To construct the ground truth, we generate a decision tree of a specified
size by choosing splits at random. The leaves of the tree are assigned
unique labels so that no two leaves share the same label, ensuring
the tree is the minimal representation of the ground truth. Training
and test datasets are then generated by sampling each data point $x$
uniformly at random and assigning its label according to the ground
truth tree. In each experiment, \textbf{50} random trees are generated
as ground truths, producing 50 corresponding training-test dataset
pairs. The training set size varies across experiments, while the
test set size is fixed a $2^{d}\times\left(D-1\right)\times500$,
where $d$ is the depth of the tree and $D$ s the data dimensionality.

Following \citet{murthy1995decision,bertsimas2017optimal}, we evaluate
tree quality using six metrics:
\begin{itemize}
	\item \textbf{Training accuracy}: Accuracy on the training set.
	\item \textbf{Test accuracy}: Accuracy on the test set.
	\item \textbf{Tree size}: Number of branch nodes.
	\item \textbf{Maximum depth}: Maximum depth of the leaf nodes.
	\item \textbf{Average depth}: Mean depth of leaf nodes.
	\item \textbf{Expected depth}: Average depth of leaf nodes weighted by the
	proportion of the test set samples assigned to each leaf.
\end{itemize}
We also follow \citet{bertsimas2017optimal} in examining the effect
of five sources of variation or noise: 1) Ground-truth tree size,
2)Training data size, 3)Label noise, 4) Feature noise, 5) Data dimensionality.

Since \citet{bertsimas2017optimal}'s implementation is not publicly
available, and replicating their construction is difficult, so we
benchmark against three widely used axis-parallel decision tree algorithms:
CART-depth (depth-constrained CART algorithm), CART-size (size-constrained
CART algorithm), and the state-of-the-art optimal axis-parallel ODT
algorithm---ConTree \citep{brita2025optimal} (which only allows tree
depth to be specified). For all experiments in this section, we employ
the $\mathit{hodtCoreset}$ algorithm for generating accurate hypersurface
decision tree.

\paragraph{Effect of tree size}

\begin{table}
	\centering\scriptsize
	\renewcommand{\arraystretch}{1.1}
	\setlength{\tabcolsep}{1pt}
	\begin{tabular}{
			>{\raggedright\arraybackslash}p{1cm}  
			>{\centering\arraybackslash}p{1.8cm}
			>{\centering\arraybackslash}p{1.5cm}
			>{\centering\arraybackslash}p{1.5cm}
			>{\centering\arraybackslash}p{1.3cm}
			>{\centering\arraybackslash}p{1.3cm}
			>{\centering\arraybackslash}p{1.3cm}
			>{\centering\arraybackslash}p{1.3cm}
		}
		\hline
		\multirow{2}{1cm}{\makecell[l]{Tree size\\
				($K$)}}
		& \multirow{2}{*}{Method} 
		& \multirow{2}{*}{\makecell[t]{Train\\Acc (\%)}} 
		& \multirow{2}{*}{\makecell[t]{Test\\Acc (\%)}} 
		& \multirow{2}{*}{Tree size} 
		& Depth & & \tabularnewline
		\cline{6-8}
		& & & & & Maximum & Average & Expected\tabularnewline\\
		
		\Xhline{0.9pt}\\
		
		\makecell[l]{2\\ \\ \\ \\ \\} & %
		\begin{tabular}{c}
			CART-depth-2\tabularnewline
			CART-size-4\tabularnewline
			ConTree-2\tabularnewline
			HDT\tabularnewline
			Ground Truth\tabularnewline
		\end{tabular} & %
		\begin{tabular}{c}
			89.12 (0.04)\tabularnewline
			93.82 (3.01)\tabularnewline
			91.72 (3.83)\tabularnewline
			100 (0.00)\tabularnewline
			100 (0.00)\tabularnewline
		\end{tabular} & %
		\begin{tabular}{c}
			82.61 (0.04)\tabularnewline
			86.12 (3.73)\tabularnewline
			84.57 (4.03)\tabularnewline
			93.50 (3.21)\tabularnewline
			100 (0.00)\tabularnewline
		\end{tabular} & %
		\begin{tabular}{c}
			2.92 (0.27)\tabularnewline
			4.00 (0.00)\tabularnewline
			3.00 (0.00)\tabularnewline
			2.00 (0.00)\tabularnewline
			2.00 (0.00)\tabularnewline
		\end{tabular} & %
		\begin{tabular}{c}
			2.00 (0.00)\tabularnewline
			3.18 (0.38)\tabularnewline
			2.00 (0.00)\tabularnewline
			2.00 (0.00)\tabularnewline
			2.00 (0.00)\tabularnewline
		\end{tabular} & %
		\begin{tabular}{c}
			1.97 (0.09)\tabularnewline
			2.49 (0.16)\tabularnewline
			2.00 (0.00)\tabularnewline
			1.67 (1.20)\tabularnewline
			1.67 (1.20)\tabularnewline
		\end{tabular} & %
		\begin{tabular}{c}
			1.96 (0.15)\tabularnewline
			2.15 (0.20)\tabularnewline
			2.00 (0.00)\tabularnewline
			1.65 (0.16)\tabularnewline
			1.66 (0.16)\tabularnewline
		\end{tabular}\tabularnewline
		\makecell[l]{3\\ \\ \\ \\ \\} & %
		\begin{tabular}{c}
			CART-depth-3\tabularnewline
			CART-size-6\tabularnewline
			ConTree-3\tabularnewline
			HDT\tabularnewline
			Ground Truth\tabularnewline
		\end{tabular} & %
		\begin{tabular}{c}
			90.72 (4.71)\tabularnewline
			93.88 (3.89)\tabularnewline
			95.98 (2.93)\tabularnewline
			100 (0.00)\tabularnewline
			100 (0.00)\tabularnewline
		\end{tabular} & %
		\begin{tabular}{c}
			81.30 (5.71)\tabularnewline
			83.92 (5.55)\tabularnewline
			85.25 (4.96)\tabularnewline
			94.91 (2.36)\tabularnewline
			100 (0.00)\tabularnewline
		\end{tabular} & %
		\begin{tabular}{c}
			5.90 (0.00)\tabularnewline
			6.00 (0.00)\tabularnewline
			6.94 (0.24)\tabularnewline
			3.00 (0.00)\tabularnewline
			3.00 (0.00)\tabularnewline
		\end{tabular} & %
		\begin{tabular}{c}
			3.00 (0.00)\tabularnewline
			4.02 (0.65)\tabularnewline
			3.00 (0.00)\tabularnewline
			3.00 (0.00)\tabularnewline
			2.72 (0.45)\tabularnewline
		\end{tabular} & %
		\begin{tabular}{c}
			2.83 (0.14)\tabularnewline
			3.10 (0.24)\tabularnewline
			2.99 (0.03)\tabularnewline
			2.25 (0.00)\tabularnewline
			2.18 (0.11)\tabularnewline
		\end{tabular} & %
		\begin{tabular}{c}
			2.77 (0.19)\tabularnewline
			2.71 (0.20)\tabularnewline
			2.99 (0.04)\tabularnewline
			2.18 (0.25)\tabularnewline
			2.07 (0.24)\tabularnewline
		\end{tabular}\tabularnewline
		\makecell[l]{4\\ \\ \\ \\ \\} & %
		\begin{tabular}{c}
			CART-depth-3\tabularnewline
			CART-size-8\tabularnewline
			ConTree-3\tabularnewline
			HDT\tabularnewline
			Ground Truth\tabularnewline
		\end{tabular} & %
		\begin{tabular}{c}
			87.98 (5.26)\tabularnewline
			93.90 (3.93)\tabularnewline
			93.12 (4.02)\tabularnewline
			97.46 (1.70)\tabularnewline
			100 (0.00)\tabularnewline
		\end{tabular} & %
		\begin{tabular}{c}
			77.66 (5.38)\tabularnewline
			81.93 (4.93)\tabularnewline
			81.51 (4.78)\tabularnewline
			89.89 (3.71)\tabularnewline
			100 (0.00)\tabularnewline
		\end{tabular} & %
		\begin{tabular}{c}
			6.12 (0.86)\tabularnewline
			8.00 (0.00)\tabularnewline
			7.00 (0.00)\tabularnewline
			4.00 (0.00)\tabularnewline
			4.00 (0.00)\tabularnewline
		\end{tabular} & %
		\begin{tabular}{c}
			3.00 (0.00)\tabularnewline
			4.54 (0.57)\tabularnewline
			3.00 (0.00)\tabularnewline
			3.52 (0.50)\tabularnewline
			3.52 (0.50)\tabularnewline
		\end{tabular} & %
		\begin{tabular}{c}
			2.87 (0.14)\tabularnewline
			3.49 (0.21)\tabularnewline
			3.00 (0.00)\tabularnewline
			2.63 (0.19)\tabularnewline
			2.63 (0.19)\tabularnewline
		\end{tabular} & %
		\begin{tabular}{c}
			2.86 (0.15)\tabularnewline
			3.13 (0.21)\tabularnewline
			3.00 (0.00)\tabularnewline
			2.60 (0.31)\tabularnewline
			2.63 (0.33)\tabularnewline
		\end{tabular}\tabularnewline
		\hline 
	\end{tabular}\caption{The effect of tree size of ground truth tree. No noise in data, training
		size = 100. \label{tab: the effect of tree size}}
\end{table}

In our first experiments, we evaluated the effectiveness of each method
as problem complexity increased. We fixed the training set size at
$N=100$ and $D=2$, while varying the ground truth size $K$ from
2 to 4. Table \ref{tab: the effect of tree size} presents the results.
Since our goal was to solve the size-constrained ODT problem, we set
the depth of the depth-constrained algorithms as $d=\left\lceil \log_{2}\left(K\times D\right)\right\rceil $
for the CART-depth and ConTree algorithms. CART-size constrained the
tree size to $K\times D$. Note that a depth $d$ tree has at most
$2^{d}-1$ tree size.

The results show that our algorithm has a very high probability of
finding the optimal solution (100\% for $K=2$ and $K=3$) and HDT
significantly outperforms the other methods. Although all methods
show a decrease in both training and test accuracy as $K$ increase
to 4, this decline is expected due to the increasing problem complexity.
Even so, HDT and ConTree are more robust to changes in tree size.
When $K$ increases to 4, CART-size and CART-depth exhibit a larger
reduction in out-of-sample accuracy (4.19\% and 4.95\%, respectively)
compared to ConTree and HDT, which show smaller decreases (3.06\%
and 3.61\%). Interestingly, although the optimal axis-parallel decision
tree obtained by the ConTree algorithm is more accurate than CART-depth
under the same depth constraint, ConTree often fully exploits the
given depth, resulting in nearly maximal tree sizes in most cases.
Consequently, ConTree typically produces larger trees than CART-depth.
However, CART-size, by allowing just one additional leaf, achieves
better solutions than ConTree by providing greater flexibility in
tree depth.

In summary, the experiments demonstrate that HDT is not only more
accurate than CART and ConTree across varying tree sizes and depths,
but also closely matches the ground truth across all measures. In
contrast, axis-parallel trees learned using CART and ConTree are significantly
larger than HDT trees and produce substantially worse solutions.

\paragraph{Effect of data size}

\begin{table}[H]
	\centering\scriptsize
	\renewcommand{\arraystretch}{1.1}
	\setlength{\tabcolsep}{1pt}
	\begin{tabular}{
			>{\raggedright\arraybackslash}p{1.1cm}  
			>{\centering\arraybackslash}p{1.8cm}
			>{\centering\arraybackslash}p{1.5cm}
			>{\centering\arraybackslash}p{1.5cm}
			>{\centering\arraybackslash}p{1.3cm}
			>{\centering\arraybackslash}p{1.3cm}
			>{\centering\arraybackslash}p{1.3cm}
			>{\centering\arraybackslash}p{1.3cm}
		}
		\hline
		\multirow{2}{1.1cm}{\makecell[l]{Training \\
				set size}}
		& \multirow{2}{*}{Method} 
		& \multirow{2}{*}{\makecell[t]{Train\\Acc (\%)}} 
		& \multirow{2}{*}{\makecell[t]{Test\\Acc (\%)}} 
		& \multirow{2}{*}{Tree size} 
		& Depth & & \tabularnewline
		\cline{6-8}
		& & & & & Maximum & Average & Expected\tabularnewline\\
		
		\Xhline{0.9pt}\\
		\makecell[l]{100\\ \\ \\ \\ \\}& %
		\begin{tabular}{c}
			CART-depth-2\tabularnewline
			CART-size-4\tabularnewline
			ConTree-2\tabularnewline
			HDT\tabularnewline
			Ground Truth\tabularnewline
		\end{tabular} & %
		\begin{tabular}{c}
			89.76 (5.11)\tabularnewline
			97.62 (2.51)\tabularnewline
			93.52 (3.53)\tabularnewline
			100 (0.00)\tabularnewline
			100 (0.00)\tabularnewline
		\end{tabular} & %
		\begin{tabular}{c}
			75.46 (15.70)\tabularnewline
			80.26 (16.80)\tabularnewline
			78.62 (16.14)\tabularnewline
			84.38 (19.27)\tabularnewline
			100 (0.00)\tabularnewline
		\end{tabular} & %
		\begin{tabular}{c}
			2.88 (0.32)\tabularnewline
			4.00 (0.00)\tabularnewline
			3.00 (0.00)\tabularnewline
			2.00 (0.00)\tabularnewline
			0.00 (0.00)\tabularnewline
		\end{tabular} & %
		\begin{tabular}{c}
			2.00 (0.00)\tabularnewline
			3.92 (0.69)\tabularnewline
			2.00 (0.00)\tabularnewline
			2.00 (0.00)\tabularnewline
			2.00 (0.00)\tabularnewline
		\end{tabular} & %
		\begin{tabular}{c}
			1.96 (0.11)\tabularnewline
			3.01 (0.30)\tabularnewline
			2.00 (0.00)\tabularnewline
			1.67 (0.00)\tabularnewline
			1.67 (0.00)\tabularnewline
		\end{tabular} & %
		\begin{tabular}{c}
			1.96 (0.12)\tabularnewline
			2.52 (0.33)\tabularnewline
			2.00 (0.00)\tabularnewline
			1.66 (0.25)\tabularnewline
			1.70 (0.25)\tabularnewline
		\end{tabular}\tabularnewline
		\makecell[l]{200\\ \\ \\ \\ \\} & %
		\begin{tabular}{c}
			CART-depth-2\tabularnewline
			CART-size-4\tabularnewline
			ConTree-2\tabularnewline
			HDT\tabularnewline
			Ground Truth\tabularnewline
		\end{tabular} & %
		\begin{tabular}{c}
			89.00 (4.49)\tabularnewline
			94.49 (3.02)\tabularnewline
			92.87 (3.38)\tabularnewline
			99.94 (0.16)\tabularnewline
			100 (0.00)\tabularnewline
		\end{tabular} & %
		\begin{tabular}{c}
			84.62 (5.23)\tabularnewline
			89.57 (4.04)\tabularnewline
			88.51 (3.99)\tabularnewline
			97.97 (1.01)\tabularnewline
			100 (0.00)\tabularnewline
		\end{tabular} & %
		\begin{tabular}{c}
			2.92 (0.27)\tabularnewline
			4.00 (0.00)\tabularnewline
			3.00 (0.00)\tabularnewline
			2.00 (0.00)\tabularnewline
			2.00 (0.00)\tabularnewline
		\end{tabular} & %
		\begin{tabular}{c}
			2.00 (0.00)\tabularnewline
			3.26 (0.44)\tabularnewline
			2.00 (0.00)\tabularnewline
			2.00 (0.00)\tabularnewline
			2.00 (0.00)\tabularnewline
		\end{tabular} & %
		\begin{tabular}{c}
			1.97 (0.09)\tabularnewline
			2.50 (0.17)\tabularnewline
			2.00 (0.00)\tabularnewline
			1.67 (0.00)\tabularnewline
			1.67 (0.00)\tabularnewline
		\end{tabular} & %
		\begin{tabular}{c}
			1.98 (0.07)\tabularnewline
			2.17 (0.21)\tabularnewline
			2.00 (0.00)\tabularnewline
			1.71 (0.25)\tabularnewline
			1.70 (0.25)\tabularnewline
		\end{tabular}\tabularnewline
		\makecell[l]{400\\ \\ \\ \\ \\}& %
		\begin{tabular}{c}
			CART-depth-2\tabularnewline
			CART-size-4\tabularnewline
			ConTree-2\tabularnewline
			HDT\tabularnewline
			Ground Truth\tabularnewline
		\end{tabular} & %
		\begin{tabular}{c}
			87.45 (4.69)\tabularnewline
			92.58 (3.48)\tabularnewline
			91.44 (3.37)\tabularnewline
			99.85 (0.21)\tabularnewline
			100 (0.00)\tabularnewline
		\end{tabular} & %
		\begin{tabular}{c}
			84.62 (4.90)\tabularnewline
			89.35 (4.08)\tabularnewline
			88.78 (4.00)\tabularnewline
			98.98 (0.58)\tabularnewline
			100 (0.00)\tabularnewline
		\end{tabular} & %
		\begin{tabular}{c}
			2.98 (0.14)\tabularnewline
			4.00 (0.00)\tabularnewline
			3.00 (0.00)\tabularnewline
			2.00 (0.00)\tabularnewline
			2.00 (0.00)\tabularnewline
		\end{tabular} & %
		\begin{tabular}{c}
			2.00 (0.00)\tabularnewline
			3.24 (0.43)\tabularnewline
			2.00 (0.00)\tabularnewline
			2.00 (0.00)\tabularnewline
			2.00 (0.00)\tabularnewline
		\end{tabular} & %
		\begin{tabular}{c}
			1.99 (0.05)\tabularnewline
			2.53 (0.17)\tabularnewline
			2.00 (0.00)\tabularnewline
			1.67 (0.00)\tabularnewline
			1.67 (0.00)\tabularnewline
		\end{tabular} & %
		\begin{tabular}{c}
			1.99 (0.05)\tabularnewline
			2.15 (0.24)\tabularnewline
			2.00 (0.00)\tabularnewline
			1.71 (0.25)\tabularnewline
			1.70 (0.25)\tabularnewline
		\end{tabular}\tabularnewline
		\makecell[l]{800\\ \\ \\ \\ \\} & %
		\begin{tabular}{c}
			CART-depth-2\tabularnewline
			CART-size-4\tabularnewline
			ConTree-2\tabularnewline
			HDT\tabularnewline
			Ground Truth\tabularnewline
		\end{tabular} & %
		\begin{tabular}{c}
			86.91 (4.63)\tabularnewline
			92.21 (3.43)\tabularnewline
			91.00 (3.59)\tabularnewline
			99.87 (0.14)\tabularnewline
			100 (0.00)\tabularnewline
		\end{tabular} & %
		\begin{tabular}{c}
			85.14 (4.96)\tabularnewline
			90.12 (3.82)\tabularnewline
			89.18 (3.94)\tabularnewline
			99.40 (0.36)\tabularnewline
			100 (0.00)\tabularnewline
		\end{tabular} & %
		\begin{tabular}{c}
			3.00 (0.00)\tabularnewline
			4.00 (0.00)\tabularnewline
			3.00 (0.00)\tabularnewline
			2.00 (0.00)\tabularnewline
			2.00 (0.00)\tabularnewline
		\end{tabular} & %
		\begin{tabular}{c}
			2.00 (0.00)\tabularnewline
			3.3 (0.00)\tabularnewline
			2.00 (0.00)\tabularnewline
			2.00 (0.00)\tabularnewline
			2.00 (0.00)\tabularnewline
		\end{tabular} & %
		\begin{tabular}{c}
			2.00 (0.00)\tabularnewline
			2.53 (0.18)\tabularnewline
			2.00 (0.00)\tabularnewline
			1.67 (0.00)\tabularnewline
			1.67 (0.00)\tabularnewline
		\end{tabular} & %
		\begin{tabular}{c}
			2.00 (0.00)\tabularnewline
			2.13 (0.22)\tabularnewline
			2.00 (0.00)\tabularnewline
			1.69 (0.25)\tabularnewline
			1.70 (0.25)\tabularnewline
		\end{tabular}\tabularnewline
		\makecell[l]{1600\\ \\ \\ \\ \\} & %
		\begin{tabular}{c}
			CART-depth-2\tabularnewline
			CART-size-4\tabularnewline
			ConTree-2\tabularnewline
			HDT\tabularnewline
			Ground Truth\tabularnewline
		\end{tabular} & %
		\begin{tabular}{c}
			85.66 (5.05)\tabularnewline
			91.48 (3.69)\tabularnewline
			90.13 (3.69)\tabularnewline
			99.87 (0.12)\tabularnewline
			100 (0.00)\tabularnewline
		\end{tabular} & %
		\begin{tabular}{c}
			84.97 (5.15)\tabularnewline
			90.53 (3.71)\tabularnewline
			89.38 (3.84)\tabularnewline
			99.65 (0.22)\tabularnewline
			100 (0.00)\tabularnewline
		\end{tabular} & %
		\begin{tabular}{c}
			3.00 (0.00)\tabularnewline
			4.00 (0.00)\tabularnewline
			3.00 (0.00)\tabularnewline
			2.00 (0.00)\tabularnewline
			2.00 (0.00)\tabularnewline
		\end{tabular} & %
		\begin{tabular}{c}
			2.00 (0.00)\tabularnewline
			3.18 (0.38)\tabularnewline
			2.00 (0.00)\tabularnewline
			2.00 (0.00)\tabularnewline
			2.00 (0.00)\tabularnewline
		\end{tabular} & %
		\begin{tabular}{c}
			2.00 (0.00)\tabularnewline
			2.51 (0.16)\tabularnewline
			2.00 (0.00)\tabularnewline
			1.67 (0.00)\tabularnewline
			1.67 (0.00)\tabularnewline
		\end{tabular} & %
		\begin{tabular}{c}
			2.00 (0.00)\tabularnewline
			2.16 (0.26)\tabularnewline
			2.00 (0.00)\tabularnewline
			1.70 (0.24)\tabularnewline
			1.70 (0.25)\tabularnewline
		\end{tabular}\tabularnewline
		\hline 
	\end{tabular}\caption{The effect of training data size. Ground truth-tree size $K=2$. \label{tab: the effect of training data size}}
\end{table}

The second set of experiments demonstrates the effect of training
data size relative to problem complexity. We increased the size of
the training set while keeping the ground truth size fixed at $K=2$
and the dimension at $D=2$. ConTree and CART-depth were restricted
to depth 2, while CART-size was restricted to $K\times D=4$.

No noise was added to the data. Table \ref{tab: the effect of training data size}
shows that out-of-sample accuracy increased for all methods as the
training set grew. The improvement was most pronounced for HDT, which
achieved a 15.27\% increase in out-of-sample accuracy, whereas ConTree,
CART-size, and CART-depth improved by 10.76\%, 10.27\%, and 9.51\%,
respectively. Notably, CART-depth performed significantly worse than
the other methods.

These results demonstrate that even in data-poor environments, optimizing
an appropriate model (HDT in this case) substantially improves out-of-sample
performance. Even when axis-parallel trees do not match the ground
truth, optimizing the solution to optimality still produces significant
differences. The optimal algorithm (ConTree) achieves much higher
test accuracy than the approximate method (CART-depth) on both training
and test datasets. This observation is consistent with \citet{bertsimas2017optimal},
and our experiments provide additional evidence for scenarios in which
the ground truth does not align with the chosen model. These results
offer clear support against the notion that optimal methods necessarily
overfit the training data in data-scarce settings.

\paragraph{Effect of data dimension}

\begin{table}
	\centering\scriptsize
	\renewcommand{\arraystretch}{1.1}
	\setlength{\tabcolsep}{1pt}
	\begin{tabular}{
			>{\raggedright\arraybackslash}p{1.25cm}  
			>{\centering\arraybackslash}p{1.8cm}
			>{\centering\arraybackslash}p{1.5cm}
			>{\centering\arraybackslash}p{1.5cm}
			>{\centering\arraybackslash}p{1.3cm}
			>{\centering\arraybackslash}p{1.3cm}
			>{\centering\arraybackslash}p{1.3cm}
			>{\centering\arraybackslash}p{1.3cm}
		}
		\hline
		\multirow{2}{1.25cm}{\makecell[l]{Data \\
				dimension}}
		& \multirow{2}{*}{Method} 
		& \multirow{2}{*}{\makecell[t]{Train\\Acc (\%)}} 
		& \multirow{2}{*}{\makecell[t]{Test\\Acc (\%)}} 
		& \multirow{2}{*}{Tree size} 
		& Depth & & \tabularnewline
		\cline{6-8}
		& & & & & Maximum & Average & Expected\tabularnewline\\
		
		\Xhline{0.9pt}\\
		\makecell[l]{2\\ \\ \\ \\ \\}  & %
		\begin{tabular}{c}
			CART-depth-2\tabularnewline
			CART-size-4\tabularnewline
			ConTree-2\tabularnewline
			HDT\tabularnewline
			Ground Truth\tabularnewline
		\end{tabular} & %
		\begin{tabular}{c}
			89.00 (4.49)\tabularnewline
			94.49 (3.02)\tabularnewline
			92.87 (3.38)\tabularnewline
			99.94 (0.16)\tabularnewline
			100 (0.00)\tabularnewline
		\end{tabular} & %
		\begin{tabular}{c}
			84.62 (5.23)\tabularnewline
			89.57 (4.04)\tabularnewline
			88.51 (3.99)\tabularnewline
			97.97 (1.01)\tabularnewline
			100 (0.00)\tabularnewline
		\end{tabular} & %
		\begin{tabular}{c}
			2.92 (0.27)\tabularnewline
			4.00 (0.00)\tabularnewline
			3.00 (0.00)\tabularnewline
			2.00 (0.00)\tabularnewline
			2.00 (0.00)\tabularnewline
		\end{tabular} & %
		\begin{tabular}{c}
			2.00 (0.00)\tabularnewline
			3.26 (0.44)\tabularnewline
			2.00 (0.00)\tabularnewline
			2.00 (0.00)\tabularnewline
			2.00 (0.00)\tabularnewline
		\end{tabular} & %
		\begin{tabular}{c}
			1.97 (0.09)\tabularnewline
			2.53 (0.17)\tabularnewline
			2.00 (0.00)\tabularnewline
			1.67 (0.00)\tabularnewline
			1.67 (0.00)\tabularnewline
		\end{tabular} & %
		\begin{tabular}{c}
			1.98 (0.07)\tabularnewline
			2.17 (0.21)\tabularnewline
			2.0 (0.00)\tabularnewline
			1.71 (0.25)\tabularnewline
			1.70 (0.25)\tabularnewline
		\end{tabular}\tabularnewline
		\makecell[l]{4\\ \\ \\ \\ \\}  & %
		\begin{tabular}{c}
			CART-depth-2\tabularnewline
			CART-size-4\tabularnewline
			ConTree-2\tabularnewline
			HDT\tabularnewline
			Ground Truth\tabularnewline
		\end{tabular} & %
		\begin{tabular}{c}
			79.94 (4.53)\tabularnewline
			84.17 (4.41)\tabularnewline
			82.80 (4.16)\tabularnewline
			95.13 (1.38)\tabularnewline
			100 (0.00)\tabularnewline
		\end{tabular} & %
		\begin{tabular}{c}
			73.97 (4.74)\tabularnewline
			76.72 (4.48)\tabularnewline
			76.16 (4.82)\tabularnewline
			90.97 (2.55)\tabularnewline
			100 (0.00)\tabularnewline
		\end{tabular} & %
		\begin{tabular}{c}
			3.00 (0.00)\tabularnewline
			4.00 (0.00)\tabularnewline
			3.00 (0.00)\tabularnewline
			2.00 (0.00)\tabularnewline
			2.00 (0.00)\tabularnewline
		\end{tabular} & %
		\begin{tabular}{c}
			2.00 (0.00)\tabularnewline
			3.10 (0.3)\tabularnewline
			2.00 (0.00)\tabularnewline
			2.00 (0.00)\tabularnewline
			2.00 (0.00)\tabularnewline
		\end{tabular} & %
		\begin{tabular}{c}
			2.00 (0.00)\tabularnewline
			2.48 (0.13)\tabularnewline
			2.00 (0.00)\tabularnewline
			1.67 (0.00)\tabularnewline
			1.67 (0.00)\tabularnewline
		\end{tabular} & %
		\begin{tabular}{c}
			2.00 (0.00)\tabularnewline
			2.24 (0.18)\tabularnewline
			2.00 (0.00)\tabularnewline
			1.67 (0.13)\tabularnewline
			1.67 (0.12)\tabularnewline
		\end{tabular}\tabularnewline
		\makecell[l]{8\\ \\ \\ \\ \\} & %
		\begin{tabular}{c}
			CART-depth-2\tabularnewline
			CART-size-4\tabularnewline
			ConTree-2\tabularnewline
			HDT\tabularnewline
			Ground Truth\tabularnewline
		\end{tabular} & %
		\begin{tabular}{c}
			75.05 (5.72)\tabularnewline
			79.25 (5.66)\tabularnewline
			77.78 (5.36)\tabularnewline
			89.08 (4.26)\tabularnewline
			100 (0.00)\tabularnewline
		\end{tabular} & %
		\begin{tabular}{c}
			67.54 (7.08)\tabularnewline
			69.66 (6.75)\tabularnewline
			69.40 (6.37)\tabularnewline
			82.95 (5.70)\tabularnewline
			100 (0.00)\tabularnewline
		\end{tabular} & %
		\begin{tabular}{c}
			3.00 (0.00)\tabularnewline
			4.00 (0.00)\tabularnewline
			3.00 (0.00)\tabularnewline
			2.00 (0.00)\tabularnewline
			2.00 (0.00)\tabularnewline
		\end{tabular} & %
		\begin{tabular}{c}
			2.00 (0.00)\tabularnewline
			3.12 (0.325)\tabularnewline
			2.00 (0.00)\tabularnewline
			2.00 (0.00)\tabularnewline
			2.00 (0.00)\tabularnewline
		\end{tabular} & %
		\begin{tabular}{c}
			2.00 (0.00)\tabularnewline
			2.48 (0.14)\tabularnewline
			2.00 (0.00)\tabularnewline
			1.67 (0.00)\tabularnewline
			1.67 (0.00)\tabularnewline
		\end{tabular} & %
		\begin{tabular}{c}
			2.00 (0.00)\tabularnewline
			2.21 (0.22)\tabularnewline
			2.00 (0.00)\tabularnewline
			1.48 (0.12)\tabularnewline
			1.49 (0.13)\tabularnewline
		\end{tabular}\tabularnewline
		\makecell[l]{12\\ \\ \\ \\ \\} & %
		\begin{tabular}{c}
			CART-depth-2\tabularnewline
			CART-size-4\tabularnewline
			ConTree-2\tabularnewline
			HDT\tabularnewline
			Ground Truth\tabularnewline
		\end{tabular} & %
		\begin{tabular}{c}
			70.39 (3.82)\tabularnewline
			74.50 (4.00)\tabularnewline
			73.35 (3.49)\tabularnewline
			85.06 (5.19)\tabularnewline
			100 (0.00)\tabularnewline
		\end{tabular} & %
		\begin{tabular}{c}
			62.78 (5.20)\tabularnewline
			64.03 (4.99)\tabularnewline
			63.53 (4.78)\tabularnewline
			76.71 (6.47)\tabularnewline
			100 (0.00)\tabularnewline
		\end{tabular} & %
		\begin{tabular}{c}
			3.00 (0.00)\tabularnewline
			4.00 (0.00)\tabularnewline
			3.00 (0.00)\tabularnewline
			2.00 (0.00)\tabularnewline
			2.00 (0.00)\tabularnewline
		\end{tabular} & %
		\begin{tabular}{c}
			2.00 (0.00)\tabularnewline
			3.14 (0.35)\tabularnewline
			2.00 (0.00)\tabularnewline
			2.00 (0.00)\tabularnewline
			2.00 (0.00)\tabularnewline
		\end{tabular} & %
		\begin{tabular}{c}
			2.00 (0.00)\tabularnewline
			2.46 (0.14)\tabularnewline
			2.00 (0.00)\tabularnewline
			1.67 (0.00)\tabularnewline
			1.67 (0.00)\tabularnewline
		\end{tabular} & %
		\begin{tabular}{c}
			2.00 (0.00)\tabularnewline
			2.26 (0.13)\tabularnewline
			2.00 (0.00)\tabularnewline
			1.50 (0.11)\tabularnewline
			1.52 (0.10)\tabularnewline
		\end{tabular}\tabularnewline
	\end{tabular}
	\caption{Effects of dimensionality. Training size = 100. No noise. Ground truth
		tress are size 2.\label{tab: Effects-of-dimensionality.}}
\end{table}

The third set of experiments examines the effect of problem dimensionality
while keeping the tree size ($K=2$) and training set size ($N=100$)
constant. ConTree and CART-depth were fixed at depth 2, and CART-size
was fixed at size 4. Table \ref{tab: Effects-of-dimensionality.} shows
the effect of increasing the number of features for a fixed training
size and tree size.

Although HDT might be expected to suffer most from the increased
combinatorial complexity associated with higher dimensionality, it
remains the most robust method on both training and test datasets,
exhibiting the smallest decrease in accuracy. Increasing the number
of features significantly affects all methods: CART-depth experiences
decreases of approximately 18.61\% in training accuracy and 21.84\%
in test accuracy; CART-size decreases by 19.99\% (train) and 25.54\%
(test); ConTree decreases by 19.52\% (train) and 24.98\% (test); whereas
HDT decreases only by 14.88\% (train) to 21.26\% (test). This robustness
may be attributed to the effectiveness of $\mathit{hodtCoreset}$,
which efficiently explores configurations without being heavily affected
by combinatorial complexity.

Interestingly, our results for axis-parallel ODT algorithms on datasets
with hyperplane decision tree ground truth differ from those reported
by \citet{bertsimas2017optimal}. However, it is important to note
that our experimental setup differs from theirs. \citet{bertsimas2017optimal}
observed that the performance gap between axis-parallel ODT algorithms
and approximate CART increases in higher dimensions, with little difference
in lower dimensions. In contrast, our experiments show that at lower
dimensions, there is a significant difference between CART-depth and
ConTree (3.89\%), which decreases to 0.75\% when $D=12$.

\paragraph{Effect of noise on labels}

\begin{table}
	\centering\scriptsize
	\renewcommand{\arraystretch}{1.1}
	\setlength{\tabcolsep}{1pt}
	\begin{tabular}{
			>{\raggedright\arraybackslash}p{0.8cm}  
			>{\centering\arraybackslash}p{1.8cm}
			>{\centering\arraybackslash}p{1.5cm}
			>{\centering\arraybackslash}p{1.5cm}
			>{\centering\arraybackslash}p{1.3cm}
			>{\centering\arraybackslash}p{1.3cm}
			>{\centering\arraybackslash}p{1.3cm}
			>{\centering\arraybackslash}p{1.3cm}
		}
		\hline
		\multirow{2}{0.8cm}{\makecell[l] {Noise \\
				level}}
		& \multirow{2}{*}{Method} 
		& \multirow{2}{*}{\makecell[t]{Train\\Acc (\%)}} 
		& \multirow{2}{*}{\makecell[t]{Test\\Acc (\%)}} 
		& \multirow{2}{*}{Tree size} 
		& Depth & & \tabularnewline
		\cline{6-8}
		& & & & & Maximum & Average & Expected\tabularnewline\\
		
		\Xhline{0.9pt}\\
		\makecell[l]{0\\ \\ \\ \\ \\}  & %
		\begin{tabular}{c}
			CART-depth-2\tabularnewline
			CART-size-4\tabularnewline
			ConTree-2\tabularnewline
			HDT\tabularnewline
			Ground Truth\tabularnewline
		\end{tabular} & %
		\begin{tabular}{c}
			90.42 (4.61)\tabularnewline
			95.19 (3.11)\tabularnewline
			93.49 (3.60)\tabularnewline
			99.62 (0.68)\tabularnewline
			100 (0.00)\tabularnewline
		\end{tabular} & %
		\begin{tabular}{c}
			87.25 (5.20)\tabularnewline
			91.21 (3.94)\tabularnewline
			89.99 (3.71)\tabularnewline
			98.06 (1.21)\tabularnewline
			100 (0.00)\tabularnewline
		\end{tabular} & %
		\begin{tabular}{c}
			2.92 (0.27)\tabularnewline
			4.0 (0.00)\tabularnewline
			2.98 (0.04)\tabularnewline
			2.00 (0.00)\tabularnewline
			2.00 (0.00)\tabularnewline
		\end{tabular} & %
		\begin{tabular}{c}
			2.00 (0.00)\tabularnewline
			3.18 (0.38)\tabularnewline
			2.00 (0.00)\tabularnewline
			2.00 (0.00)\tabularnewline
			2.00 (0.00)\tabularnewline
		\end{tabular} & %
		\begin{tabular}{c}
			1.97 (0.09)\tabularnewline
			2.49 (0.16)\tabularnewline
			1.99 (0.05)\tabularnewline
			1.67 (0.00)\tabularnewline
			1.67 (0.00)\tabularnewline
		\end{tabular} & %
		\begin{tabular}{c}
			1.99 (0.03)\tabularnewline
			2.11 (0.23)\tabularnewline
			1.98 (0.11)\tabularnewline
			1.67 (0.28)\tabularnewline
			1.68 (0.28)\tabularnewline
		\end{tabular}\tabularnewline
		\makecell[l]{5\\ \\ \\ \\ \\}  & %
		\begin{tabular}{c}
			CART-depth-2\tabularnewline
			CART-size-4\tabularnewline
			ConTree-2\tabularnewline
			HDT\tabularnewline
			Ground Truth\tabularnewline
		\end{tabular} & %
		\begin{tabular}{c}
			86.29 (4.48)\tabularnewline
			90.64 (3.07)\tabularnewline
			89.07 (3.45)\tabularnewline
			95.00 (0.20)\tabularnewline
			95.00 (0.00)\tabularnewline
		\end{tabular} & %
		\begin{tabular}{c}
			87.15 (4.89)\tabularnewline
			91.01 (3.67)\tabularnewline
			89.87 (3.47)\tabularnewline
			97.20 (1.40)\tabularnewline
			100 (0.00)\tabularnewline
		\end{tabular} & %
		\begin{tabular}{c}
			2.94 (0.24)\tabularnewline
			4.0 (0.00)\tabularnewline
			3.00 (0.00)\tabularnewline
			2.00 (0.00)\tabularnewline
			2.00 (0.00)\tabularnewline
		\end{tabular} & %
		\begin{tabular}{c}
			2.00 (0.00)\tabularnewline
			3.20 (0.40)\tabularnewline
			2.00 (0.00)\tabularnewline
			2.00 (0.00)\tabularnewline
			2.00 (0.00)\tabularnewline
		\end{tabular} & %
		\begin{tabular}{c}
			1.98 (0.08)\tabularnewline
			2.50 (0.16)\tabularnewline
			2.00 (0.00)\tabularnewline
			1.68 (0.00)\tabularnewline
			1.67 (0.00)\tabularnewline
		\end{tabular} & %
		\begin{tabular}{c}
			1.99 (0.02)\tabularnewline
			2.10 (0.24)\tabularnewline
			2.00 (0.00)\tabularnewline
			1.62 (0.29)\tabularnewline
			1.68 (0.28)\tabularnewline
		\end{tabular}\tabularnewline
		\makecell[l]{10\\ \\ \\ \\ \\}  & %
		\begin{tabular}{c}
			CART-depth-2\tabularnewline
			CART-size-4\tabularnewline
			ConTree-2\tabularnewline
			HDT\tabularnewline
			Ground Truth\tabularnewline
		\end{tabular} & %
		\begin{tabular}{c}
			82.25 (3.78)\tabularnewline
			86.25 (2.67)\tabularnewline
			84.69 (3.09)\tabularnewline
			90.45 (0.49)\tabularnewline
			90.00 (0.00)\tabularnewline
		\end{tabular} & %
		\begin{tabular}{c}
			86.74 (5.36)\tabularnewline
			90.45 (4.07)\tabularnewline
			89.48 (3.96)\tabularnewline
			97.59 (1.21)\tabularnewline
			100 (0.00)\tabularnewline
		\end{tabular} & %
		\begin{tabular}{c}
			2.98 (0.14)\tabularnewline
			4.00 (0.00)\tabularnewline
			3.00 (0.00)\tabularnewline
			2.00 (0.00)\tabularnewline
			2.00 (0.00)\tabularnewline
		\end{tabular} & %
		\begin{tabular}{c}
			2.00 (0.00)\tabularnewline
			3.12 (0.32)\tabularnewline
			2.00 (0.00)\tabularnewline
			2.00 (0.00)\tabularnewline
			2.00 (0.00)\tabularnewline
		\end{tabular} & %
		\begin{tabular}{c}
			1.99 (0.05)\tabularnewline
			3.12 (0.32)\tabularnewline
			2.00 (0.00)\tabularnewline
			1.67 (0.00)\tabularnewline
			1.67 (0.00)\tabularnewline
		\end{tabular} & %
		\begin{tabular}{c}
			2.00 (0.01)\tabularnewline
			2.16 (0.26)\tabularnewline
			2.00 (0.00)\tabularnewline
			1.64 (0.30)\tabularnewline
			1.68 (0.28)\tabularnewline
		\end{tabular}\tabularnewline
		\makecell[l]{15\\ \\ \\ \\ \\}  & %
		\begin{tabular}{c}
			CART-depth-2\tabularnewline
			CART-size-4\tabularnewline
			ConTree-2\tabularnewline
			HDT\tabularnewline
			Ground Truth\tabularnewline
		\end{tabular} & %
		\begin{tabular}{c}
			77.90 (3.99)\tabularnewline
			81.80 (3.10)\tabularnewline
			80.53 (3.04)\tabularnewline
			85.69 (0.54)\tabularnewline
			84.99 (0.00)\tabularnewline
		\end{tabular} & %
		\begin{tabular}{c}
			85.99 (5.71)\tabularnewline
			90.19 (3.67)\tabularnewline
			89.10 (3.79)\tabularnewline
			97.10 (1.84)\tabularnewline
			100 (0.00)\tabularnewline
		\end{tabular} & %
		\begin{tabular}{c}
			2.98 (0.14)\tabularnewline
			4.00 (0.00)\tabularnewline
			3.00 (0.00)\tabularnewline
			2.00 (0.00)\tabularnewline
			2.00 (0.00)\tabularnewline
		\end{tabular} & %
		\begin{tabular}{c}
			2.00 (0.00)\tabularnewline
			3.22 (0.41)\tabularnewline
			2.00 (0.00)\tabularnewline
			2.00 (0.00)\tabularnewline
			2.00 (0.00)\tabularnewline
		\end{tabular} & %
		\begin{tabular}{c}
			1.99 (0.05)\tabularnewline
			2.51 (0.17)\tabularnewline
			2.00 (0.00)\tabularnewline
			1.67 (0.00)\tabularnewline
			1.67 (0.00)\tabularnewline
		\end{tabular} & %
		\begin{tabular}{c}
			2.00 (0.04)\tabularnewline
			2.19 (0.31)\tabularnewline
			2.00 (0.00)\tabularnewline
			1.63 (0.29)\tabularnewline
			1.68 (0.28)\tabularnewline
		\end{tabular}\tabularnewline
		\makecell[l]{20\\ \\ \\ \\ \\}  & %
		\begin{tabular}{c}
			CART-depth-2\tabularnewline
			CART-size-4\tabularnewline
			ConTree-2\tabularnewline
			HDT\tabularnewline
			Ground Truth\tabularnewline
		\end{tabular} & %
		\begin{tabular}{c}
			73.88 (3.16)\tabularnewline
			77.36 (2.91)\tabularnewline
			76.28 (2.81)\tabularnewline
			80.98 (0.69)\tabularnewline
			79.90 (0.00)\tabularnewline
		\end{tabular} & %
		\begin{tabular}{c}
			86.06 (4.95)\tabularnewline
			89.51 (4.59)\tabularnewline
			88.61 (4.43)\tabularnewline
			96.58 (2.06)\tabularnewline
			100 (0.00)\tabularnewline
		\end{tabular} & %
		\begin{tabular}{c}
			2.94 (0.24)\tabularnewline
			4.00 (0.00)\tabularnewline
			3.00 (0.00)\tabularnewline
			2.00 (0.00)\tabularnewline
			2.00 (0.00)\tabularnewline
		\end{tabular} & %
		\begin{tabular}{c}
			2.00 (0.00)\tabularnewline
			3.24 (0.43)\tabularnewline
			2.00 (0.00)\tabularnewline
			2.00 (0.00)\tabularnewline
			2.00 (0.00)\tabularnewline
		\end{tabular} & %
		\begin{tabular}{c}
			1.98 (0.08)\tabularnewline
			2.54 (0.17)\tabularnewline
			2.00 (0.00)\tabularnewline
			1.67 (0.00)\tabularnewline
			1.67 (0.00)\tabularnewline
		\end{tabular} & %
		\begin{tabular}{c}
			2.00 (0.01)\tabularnewline
			2.13 (0.34)\tabularnewline
			2.00 (0.00)\tabularnewline
			1.68 (0.27)\tabularnewline
			1.68 (0.28)\tabularnewline
		\end{tabular}\tabularnewline
		\makecell[l]{25\\ \\ \\ \\ \\}  & %
		\begin{tabular}{c}
			CART-depth-2\tabularnewline
			CART-size-4\tabularnewline
			ConTree-2\tabularnewline
			HDT\tabularnewline
			Ground Truth\tabularnewline
		\end{tabular} & %
		\begin{tabular}{c}
			70.35 (3.22)\tabularnewline
			73.19 (2.54)\tabularnewline
			72.26 (2.62)\tabularnewline
			76.37 (0.83)\tabularnewline
			75.00 (0.00)\tabularnewline
		\end{tabular} & %
		\begin{tabular}{c}
			84.59 (6.66)\tabularnewline
			87.78 (5.41)\tabularnewline
			87.06 (5.16)\tabularnewline
			95.67 (2.93)\tabularnewline
			100 (0.00)\tabularnewline
		\end{tabular} & %
		\begin{tabular}{c}
			3.00 (0.00)\tabularnewline
			4.00 (0.00)\tabularnewline
			3.00 (0.00)\tabularnewline
			2.00 (0.00)\tabularnewline
			2.00 (0.00)\tabularnewline
		\end{tabular} & %
		\begin{tabular}{c}
			2.00 (0.00)\tabularnewline
			3.26 (0.44)\tabularnewline
			2.00 (0.00)\tabularnewline
			2.00 (0.00)\tabularnewline
			2.00 (0.00)\tabularnewline
		\end{tabular} & %
		\begin{tabular}{c}
			2.00 (0.00)\tabularnewline
			2.52 (0.17)\tabularnewline
			2.00 (0.00)\tabularnewline
			1.67 (0.00)\tabularnewline
			1.67 (0.00)\tabularnewline
		\end{tabular} & %
		\begin{tabular}{c}
			2.00 (0.00)\tabularnewline
			2.24 (0.32)\tabularnewline
			2.00 (0.00)\tabularnewline
			1.67 (0.78)\tabularnewline
			1.68 (0.28)\tabularnewline
		\end{tabular}\tabularnewline
	\end{tabular}\caption{The effect of noise on labels. Training data size = 100. Ground truth
		trees are size 2.\label{tab: the effect of noise on labels}}
\end{table}

In the fourth set of experiments, we introduced noise in the labels.
As the noise level increased, it was noteworthy that HDT was able
to find solutions that outperformed the ground truth on the training
dataset. As before, we fixed $D=2$ and ($K=2$). Following the experimental
setup of \citet{bertsimas2017optimal} we added noise by increasing
the label of a random $k\%$ of the points by 1, where Table \ref{tab: the effect of noise on labels}
presents the results.

As the noise level increased, the accuracy of all methods tended to
decrease, with similar effects on out-of-sample performance up to
a noise level of 20\%. Beyond this point (from 20\% to 25\% noise),
differences became more pronounced: CART-depth, CART-size, and ConTree
decreased by 1.47\%, 1.73\%, and 1.55\%, respectively, whereas HDT
decreased by only 0.91\%. From a broader perspective, when increasing
noise from 0\% to 25\%, HDT's out-of-sample accuracy decreased by
2.39\%, compared with 2.93\%, 3.43\%, and 2.66\% for ConTree, CART-size,
and CART-depth, respectively. Notably, starting at a 5\% noise level,
HDT was able to find solutions that exceeded the ground truth on
the training dataset while maintaining strong performance on out-of-sample
tests. These results further refute the notion that optimal methods
are less robust to noise, highlighting the superiority of HDT.

\paragraph{Effect of noise on data}

\begin{table}
	\centering\scriptsize
	\renewcommand{\arraystretch}{1.1}
	\setlength{\tabcolsep}{1pt}
	\begin{tabular}{
			>{\raggedright\arraybackslash}p{0.8cm}  
			>{\centering\arraybackslash}p{1.8cm}
			>{\centering\arraybackslash}p{1.5cm}
			>{\centering\arraybackslash}p{1.5cm}
			>{\centering\arraybackslash}p{1.3cm}
			>{\centering\arraybackslash}p{1.3cm}
			>{\centering\arraybackslash}p{1.3cm}
			>{\centering\arraybackslash}p{1.3cm}
		}
		\hline
		\multirow{2}{0.8cm}{\makecell[l] {Noise \\
				level}}
		& \multirow{2}{*}{Method} 
		& \multirow{2}{*}{\makecell[t]{Train\\Acc (\%)}} 
		& \multirow{2}{*}{\makecell[t]{Test\\Acc (\%)}} 
		& \multirow{2}{*}{Tree size} 
		& Depth & & \tabularnewline
		\cline{6-8}
		& & & & & Maximum & Average & Expected\tabularnewline\\
		
		\Xhline{0.9pt}\\
		\makecell[l]{0\\ \\ \\ \\ \\}  & %
		\begin{tabular}{c}
			CART-depth-2\tabularnewline
			CART-size-4\tabularnewline
			ConTree-2\tabularnewline
			HDT\tabularnewline
			Ground Truth\tabularnewline
		\end{tabular} & %
		\begin{tabular}{c}
			90.29 (4.74)\tabularnewline
			94.20 (3.70)\tabularnewline
			92.73 (3.90)\tabularnewline
			99.88 (0.26)\tabularnewline
			100 (0.00)\tabularnewline
		\end{tabular} & %
		\begin{tabular}{c}
			87.23 (5.14)\tabularnewline
			90.32 (4.24)\tabularnewline
			89.55 (4.10)\tabularnewline
			98.08 (1.00)\tabularnewline
			100 (0.00)\tabularnewline
		\end{tabular} & %
		\begin{tabular}{c}
			2.96 (0.20)\tabularnewline
			4.00 (0.00)\tabularnewline
			3.00 (0.00)\tabularnewline
			2.00 (0.00)\tabularnewline
			2.00 (0.00)\tabularnewline
		\end{tabular} & %
		\begin{tabular}{c}
			2.00 (0.00)\tabularnewline
			3.2 (0.399)\tabularnewline
			2.00 (0.00)\tabularnewline
			2.00 (0.00)\tabularnewline
			2.00 (0.00)\tabularnewline
		\end{tabular} & %
		\begin{tabular}{c}
			1.99 (0.065)\tabularnewline
			2.51 (0.16)\tabularnewline
			2.00 (0.00)\tabularnewline
			1.67 (0.00)\tabularnewline
			1.67 (0.00)\tabularnewline
		\end{tabular} & %
		\begin{tabular}{c}
			1.99 (0.09)\tabularnewline
			2.09 (0.33)\tabularnewline
			2.00 (0.00)\tabularnewline
			1.63 (0.29)\tabularnewline
			1.63 (0.28)\tabularnewline
		\end{tabular}\tabularnewline
		\makecell[l]{5\\ \\ \\ \\ \\} & %
		\begin{tabular}{c}
			CART-depth-2\tabularnewline
			CART-size-4\tabularnewline
			ConTree-2\tabularnewline
			HDT\tabularnewline
			Ground Truth\tabularnewline
		\end{tabular} & %
		\begin{tabular}{c}
			90.16 (4.63)\tabularnewline
			93.98 (3.78)\tabularnewline
			92.62 (3.79)\tabularnewline
			99.62 (0.40)\tabularnewline
			99.59 (0.44)\tabularnewline
		\end{tabular} & %
		\begin{tabular}{c}
			87.16 (5.06)\tabularnewline
			90.32 (4.20)\tabularnewline
			89.50 (3.94)\tabularnewline
			98.10 (1.01)\tabularnewline
			100 (0.00)\tabularnewline
		\end{tabular} & %
		\begin{tabular}{c}
			2.98 (0.14)\tabularnewline
			4.00 (0.00)\tabularnewline
			3.00 (0.00)\tabularnewline
			2.00 (0.00)\tabularnewline
			2.00 (0.00)\tabularnewline
		\end{tabular} & %
		\begin{tabular}{c}
			2.00 (0.00)\tabularnewline
			3.24 (0.43)\tabularnewline
			2.00 (0.00)\tabularnewline
			2.00 (0.00)\tabularnewline
			2.00 (0.00)\tabularnewline
		\end{tabular} & %
		\begin{tabular}{c}
			1.99 (0.05)\tabularnewline
			2.53 (0.17)\tabularnewline
			2.00 (0.00)\tabularnewline
			1.67 (0.00)\tabularnewline
			1.67 (0.00)\tabularnewline
		\end{tabular} & %
		\begin{tabular}{c}
			2.00 (0.01)\tabularnewline
			2.05 (0.32)\tabularnewline
			2.00 (0.00)\tabularnewline
			1.64 (0.28)\tabularnewline
			1.65 (0.28)\tabularnewline
		\end{tabular}\tabularnewline
		\makecell[l]{10\\ \\ \\ \\ \\}& %
		\begin{tabular}{c}
			CART-depth-2\tabularnewline
			CART-size-4\tabularnewline
			ConTree-2\tabularnewline
			HDT\tabularnewline
			Ground Truth\tabularnewline
		\end{tabular} & %
		\begin{tabular}{c}
			90.30 (4.55)\tabularnewline
			93.98 (3.57)\tabularnewline
			92.50 (3.65)\tabularnewline
			99.07 (0.62)\tabularnewline
			98.81 (0.71)\tabularnewline
		\end{tabular} & %
		\begin{tabular}{c}
			87.36 (5.09)\tabularnewline
			90.30 (4.55)\tabularnewline
			89.38 (4.09)\tabularnewline
			97.50 (0.90)\tabularnewline
			100 (0.00)\tabularnewline
		\end{tabular} & %
		\begin{tabular}{c}
			2.96 (0.20)\tabularnewline
			4.00 (0.00)\tabularnewline
			3.00 (0.00)\tabularnewline
			2.00 (0.00)\tabularnewline
			2.00 (0.00)\tabularnewline
		\end{tabular} & %
		\begin{tabular}{c}
			2.00 (0.00)\tabularnewline
			3.28 (0.45)\tabularnewline
			2.00 (0.00)\tabularnewline
			2.00 (0.00)\tabularnewline
			2.00 (0.00)\tabularnewline
		\end{tabular} & %
		\begin{tabular}{c}
			1.99 (0.065)\tabularnewline
			2.54 (0.17)\tabularnewline
			2.00 (0.00)\tabularnewline
			1.67 (0.00)\tabularnewline
			1.67 (0.00)\tabularnewline
		\end{tabular} & %
		\begin{tabular}{c}
			1.99 (0.09)\tabularnewline
			2.03 (0.34)\tabularnewline
			2.00 (0.00)\tabularnewline
			1.58 (0.29)\tabularnewline
			1.65 (0.28)\tabularnewline
		\end{tabular}\tabularnewline
		\makecell[l]{15\\ \\ \\ \\ \\} & %
		\begin{tabular}{c}
			CART-depth-2\tabularnewline
			CART-size-4\tabularnewline
			ConTree-2\tabularnewline
			HDT\tabularnewline
			Ground Truth\tabularnewline
		\end{tabular} & %
		\begin{tabular}{c}
			89.86 (4.91)\tabularnewline
			93.78 (3.92)\tabularnewline
			92.52 (3.83)\tabularnewline
			99.05 (0.66)\tabularnewline
			98.72 (0.80)\tabularnewline
		\end{tabular} & %
		\begin{tabular}{c}
			87.06 (4.97)\tabularnewline
			90.27 (4.42)\tabularnewline
			89.31 (4.38)\tabularnewline
			97.90 (1.10)\tabularnewline
			100 (0.00)\tabularnewline
		\end{tabular} & %
		\begin{tabular}{c}
			2.96 (0.20)\tabularnewline
			4.00 (0.00)\tabularnewline
			3.00 (0.00)\tabularnewline
			2.00 (0.00)\tabularnewline
			2.00 (0.00)\tabularnewline
		\end{tabular} & %
		\begin{tabular}{c}
			2.00 (0.00)\tabularnewline
			3.32 (0.47)\tabularnewline
			2.00 (0.00)\tabularnewline
			2.00 (0.00)\tabularnewline
			2.00 (0.00)\tabularnewline
		\end{tabular} & %
		\begin{tabular}{c}
			1.99 (0.07)\tabularnewline
			2.55 (0.18)\tabularnewline
			2.00 (0.00)\tabularnewline
			1.67 (0.00)\tabularnewline
			1.67 (0.00)\tabularnewline
		\end{tabular} & %
		\begin{tabular}{c}
			2.00 (0.02)\tabularnewline
			2.03 (0.33)\tabularnewline
			2.00 (0.00)\tabularnewline
			1.63 (0.28)\tabularnewline
			1.65 (0.28)\tabularnewline
		\end{tabular}\tabularnewline
		\makecell[l]{20\\ \\ \\ \\ \\} & %
		\begin{tabular}{c}
			CART-depth-2\tabularnewline
			CART-size-4\tabularnewline
			ConTree-2\tabularnewline
			HDT\tabularnewline
			Ground Truth\tabularnewline
		\end{tabular} & %
		\begin{tabular}{c}
			89.75 (4.98)\tabularnewline
			93.75 (3.70)\tabularnewline
			92.39 (3.72)\tabularnewline
			98.65 (0.87)\tabularnewline
			98.11 (1.04)\tabularnewline
		\end{tabular} & %
		\begin{tabular}{c}
			86.61 (5.57)\tabularnewline
			90.21 (4.27)\tabularnewline
			89.25 (4.05)\tabularnewline
			97.40 (11.50)\tabularnewline
			100 (0.00)\tabularnewline
		\end{tabular} & %
		\begin{tabular}{c}
			3.00 (0.00)\tabularnewline
			4.00 (0.00)\tabularnewline
			3.00 (0.00)\tabularnewline
			2.00 (0.00)\tabularnewline
			2.00 (0.00)\tabularnewline
		\end{tabular} & %
		\begin{tabular}{c}
			2.00 (0.00)\tabularnewline
			3.18 (0.38)\tabularnewline
			2.00 (0.00)\tabularnewline
			2.00 (0.00)\tabularnewline
			2.00 (0.00)\tabularnewline
		\end{tabular} & %
		\begin{tabular}{c}
			2.00 (0.00)\tabularnewline
			2.49 (0.16)\tabularnewline
			2.00 (0.00)\tabularnewline
			1.67 (0.00)\tabularnewline
			1.67 (0.00)\tabularnewline
		\end{tabular} & %
		\begin{tabular}{c}
			2.00 (0.00)\tabularnewline
			2.08 (0.27)\tabularnewline
			2.00 (0.00)\tabularnewline
			1.64 (0.28)\tabularnewline
			1.65 (0.28)\tabularnewline
		\end{tabular}\tabularnewline
		\makecell[l]{25\\ \\ \\ \\ \\} & %
		\begin{tabular}{c}
			CART-depth-2\tabularnewline
			CART-size-4\tabularnewline
			ConTree-2\tabularnewline
			HDT\tabularnewline
			Ground Truth\tabularnewline
		\end{tabular} & %
		\begin{tabular}{c}
			89.76 (4.47)\tabularnewline
			93.58 (4.13)\tabularnewline
			92.29 (3.65)\tabularnewline
			98.35 (1.00)\tabularnewline
			97.67 (0.91)\tabularnewline
		\end{tabular} & %
		\begin{tabular}{c}
			87.13 (4.93)\tabularnewline
			90.28 (4.50)\tabularnewline
			89.55 (4.17)\tabularnewline
			97.62 (1.20)\tabularnewline
			100 (0.00)\tabularnewline
		\end{tabular} & %
		\begin{tabular}{c}
			2.98 (0.14)\tabularnewline
			4.00 (0.00)\tabularnewline
			2.98 (0.00)\tabularnewline
			2.00 (0.00)\tabularnewline
			2.00 (0.00)\tabularnewline
		\end{tabular} & %
		\begin{tabular}{c}
			2.00 (0.00)\tabularnewline
			3.24 (0.43)\tabularnewline
			2.00 (0.00)\tabularnewline
			2.00 (0.00)\tabularnewline
			2.00 (0.00)\tabularnewline
		\end{tabular} & %
		\begin{tabular}{c}
			1.99 (0.046)\tabularnewline
			2.53 (0.17)\tabularnewline
			1.99 (0.05)\tabularnewline
			1.67 (0.00)\tabularnewline
			1.67 (0.00)\tabularnewline
		\end{tabular} & %
		\begin{tabular}{c}
			1.99 (0.09)\tabularnewline
			2.05 (0.35)\tabularnewline
			1.99 (0.05)\tabularnewline
			1.67 (0.27)\tabularnewline
			1.65 (0.28)\tabularnewline
		\end{tabular}\tabularnewline
	\end{tabular}
	\caption{The effect of noise on training data. Training data size = 100. Ground
		truth trees are size 2. \label{tab: effect of noise on training data}}
\end{table}

Finally, in the last set of experiments, we examined the effect of
noise in the features of the data. We fixed $N=100$, $D=2$ and $K=2$.
Again, HDT was able to find solutions that outperformed the ground
truth on the training dataset while achieving the best out-of-sample
performance. Combined with previous experiments, these results strongly
refute the idea that optimal algorithms necessarily overfit the data.
Even when training accuracy exceeds that of the ground truth, proper
control of model complexity prevents overfitting.

In summary, across all synthetic data experiments, HDT-generated
trees most closely matched the quality metrics of the ground truth
trees. HDT not only produced more accurate results but also generated
smaller trees, demonstrating a clear advantage in scenarios where
tree size must be strictly controlled. Interestingly, by allowing
slightly more flexibility in size or depth, CART can sometimes achieve
slightly better solutions than optimal axis-parallel decision tree
algorithms. However, in contexts requiring strict control of tree
size, HDT consistently provides superior performance compared to
all other methods.

Moreover, our experiments align with the observations of \citet{bertsimas2017optimal},
countering the widely held misconception that optimal methods are
more prone to overfitting the training set at the expense of out-of-sample
accuracy. We demonstrate that optimal solutions remain robust in data-scarce
and noisy settings, even when training accuracy exceeds that of the
ground truth.

\subsection{Computational experiments on real-world datasets\label{subsec: real-world}}

	%

	\begin{table}
	\begin{centering}
		\centering\tiny
		\renewcommand{\arraystretch}{1.1}
		\setlength{\tabcolsep}{1pt}
		\begin{tabular}
			{
				>{\raggedright\arraybackslash}p{1.15cm}  
				>{\centering\arraybackslash}p{0.55cm}
				>{\centering\arraybackslash}p{0.3cm}
				>{\centering\arraybackslash}p{0.25cm}
				>{\centering\arraybackslash}p{1.3cm}
				>{\centering\arraybackslash}p{1.3cm}
				>{\centering\arraybackslash}p{1.3cm}
				>{\centering\arraybackslash}p{1.3cm}
				>{\centering\arraybackslash}p{1.3cm}
				>{\centering\arraybackslash}p{1.3cm}
				>{\centering\arraybackslash}p{1.3cm}
			}
			
			Dataset & $N$ & $D$ & $C$ & \begin{cellvarwidth}[t]
				\centering
				CART-depth
				
				$d=2$
			\end{cellvarwidth} & \begin{cellvarwidth}[t]
				\centering
				ConTree
				
				$d=2$
			\end{cellvarwidth} & \begin{cellvarwidth}[t]
				\centering
				CART-size
				
				$K=2$
			\end{cellvarwidth} & \begin{cellvarwidth}[t]
				\centering
				CART-size
				
				$K=3$
			\end{cellvarwidth} & \begin{cellvarwidth}[t]
				\centering
				HDT
				
				$K=2$
			\end{cellvarwidth} & \begin{cellvarwidth}[t]
				\centering
				HDT
				
				$K=3$
			\end{cellvarwidth}\tabularnewline\\
			
			\Xhline{0.9pt}\\
				Parkinsons & 195 & 22 & 2 & \begin{cellvarwidth}[t]
		\centering
		87.95/84.62
		
		(1.48/5.85)
	\end{cellvarwidth} & \begin{cellvarwidth}[t]
		\centering
		93.20/82.56
		
		(1.04/5.23)
	\end{cellvarwidth} & \begin{cellvarwidth}[t]
		\centering
		87.82/85.13
		
		(1.67/6.36)
	\end{cellvarwidth} & \begin{cellvarwidth}[t]
		\centering
		90.77/86.67
		
		(1.04/5.71)
	\end{cellvarwidth} & \begin{cellvarwidth}[t]
		\centering
		93.72/\textbf{89.74}
		
		(1.66/3.14)
	\end{cellvarwidth} & \begin{cellvarwidth}[t]
		\centering
		\textbf{94.49}/\textbf{89.74}
		
		(2.15/3.14)
	\end{cellvarwidth}\tabularnewline
	Diabetic & 1146 & 19 & 2 & \begin{cellvarwidth}[t]
		\centering
		65.55/64.00
		
		(0.92/3.74)
	\end{cellvarwidth} & \begin{cellvarwidth}[t]
		\centering
		67.36/63.48
		
		(0.79/3.95)
	\end{cellvarwidth} & \begin{cellvarwidth}[t]
		\centering
		65.55/64.00
		
		(0.92/3.74)
	\end{cellvarwidth} & \begin{cellvarwidth}[t]
		\centering
		67.53/62.96
		
		(1.05/2.66)
	\end{cellvarwidth} & \begin{cellvarwidth}[t]
		\centering
		79.93/\textbf{76.52}
		
		(1.17/1.71)
	\end{cellvarwidth} & \begin{cellvarwidth}[t]
		\centering
		\textbf{80.09}/76.35
		
		(1.06/2.38)
	\end{cellvarwidth}\tabularnewline
	BalScl & 625 & 4 & 3 & \begin{cellvarwidth}[t]
		\centering
		71.72/65.12
		
		(0.94/2.00)
	\end{cellvarwidth} & \begin{cellvarwidth}[t]
		\centering
		72.88/66.88
		
		(0.35/1.40)
	\end{cellvarwidth} & \begin{cellvarwidth}[t]
		\centering
		69.80/63.20
		
		(0.44/1.75)
	\end{cellvarwidth} & \begin{cellvarwidth}[t]
		\centering
		72.68/64.64
		
		(1.12/2.74)
	\end{cellvarwidth} & \begin{cellvarwidth}[t]
		\centering
		93.44/94.40
		
		(0.33/1.26)
	\end{cellvarwidth} & \begin{cellvarwidth}[t]
		\centering
		\textbf{94.28}/\textbf{94.56}
		
		(0.30/1.54)
	\end{cellvarwidth}\tabularnewline
	StatlogVS & 845 & 18 & 4 & \begin{cellvarwidth}[t]
		\centering
		53.34/52.66
		
		(0.32/2.97)
	\end{cellvarwidth} & \begin{cellvarwidth}[t]
		\centering
		62.42/61.89
		
		(0.91/4.07)
	\end{cellvarwidth} & \begin{cellvarwidth}[t]
		\centering
		49.94/48.40
		
		(0.64/3.73)
	\end{cellvarwidth} & \begin{cellvarwidth}[t]
		\centering
		53.34/52.66
		
		(0.32/2.97)
	\end{cellvarwidth} & \begin{cellvarwidth}[t]
		\centering
		65.12/63.31
		
		(0.78/4.73)
	\end{cellvarwidth} & \begin{cellvarwidth}[t]
		\centering
		\textbf{68.75}/\textbf{64.38}
		
		(0.74/2.58)
	\end{cellvarwidth}\tabularnewline
	ImgSeg & 210 & 19 & 7 & \begin{cellvarwidth}[t]
		\centering
		50.00/41.91
		
		(6.42/9.71)
	\end{cellvarwidth} & \begin{cellvarwidth}[t]
		\centering
		58.69/49.04
		
		(0.80/3.23)
	\end{cellvarwidth} & \begin{cellvarwidth}[t]
		\centering
		44.64/33.81
		
		(1.30/5.51)
	\end{cellvarwidth} & \begin{cellvarwidth}[t]
		\centering
		57.86/\textbf{52.38}
		
		(1.09/4.26)
	\end{cellvarwidth} & \begin{cellvarwidth}[t]
		\centering
		66.00/41.34
		
		(2.58/6.24)
	\end{cellvarwidth} & \begin{cellvarwidth}[t]
		\centering
		\textbf{66.47}/49.52
		
		(1.91/5.20)
	\end{cellvarwidth}\tabularnewline

		iris & 147 & 4 & 3 & \begin{cellvarwidth}[t]
		\centering
		97.09/90.67
		
		(0.68/2.49)
	\end{cellvarwidth} & \begin{cellvarwidth}[t]
		\centering
		97.09/90.66
		
		(0.68/2.49)
	\end{cellvarwidth} & \begin{cellvarwidth}[t]
		\centering
		97.09/90.67
		
		(0.68/2.49)
	\end{cellvarwidth} & \begin{cellvarwidth}[t]
		\centering
		98.12/91.33
		
		(1.26/3.40)
	\end{cellvarwidth} & \begin{cellvarwidth}[t]
		\centering
		\textbf{99.80}/\textbf{97.33}
		
		(3.49/7.14)
	\end{cellvarwidth} & \begin{cellvarwidth}[t]
		\centering
		98.97/96.77
		
		(4.18/6.31)
	\end{cellvarwidth}\tabularnewline
	MnkPrb & 432 & 6 & 2 & \begin{cellvarwidth}[t]
		\centering
		74.44/77.24
		
		(1.53/6.06)
	\end{cellvarwidth} & \begin{cellvarwidth}[t]
		\centering
		78.38/75.40
		
		(1.49/5.89)
	\end{cellvarwidth} & \begin{cellvarwidth}[t]
		\centering
		74.44/77.24
		
		(1.53/6.06)
	\end{cellvarwidth} & \begin{cellvarwidth}[t]
		\centering
		81.39/82.76
		
		(3.28/6.86)
	\end{cellvarwidth} & \begin{cellvarwidth}[t]
		\centering
		82.24/79.75
		
		(1.56/3.96)
	\end{cellvarwidth} & \begin{cellvarwidth}[t]
		\centering
		\textbf{82.26}/\textbf{86.20}
		
		(2.79/2.44)
	\end{cellvarwidth}\tabularnewline
	UKM & 403 & 5 & 5 & \begin{cellvarwidth}[t]
		\centering
		81.18/75.80
		
		(1.52/4.18)
	\end{cellvarwidth} & \begin{cellvarwidth}[t]
		\centering
		\textbf{82.17}/\textbf{76.05}
		
		(0.54/1.84)
	\end{cellvarwidth} & \begin{cellvarwidth}[t]
		\centering
		78.14/75.31
		
		(0.72/4.18)
	\end{cellvarwidth} & \begin{cellvarwidth}[t]
		\centering
		82.05/75.80
		
		(0.50/2.59)
	\end{cellvarwidth} & \begin{cellvarwidth}[t]
		\centering
		71.24/64.69
		
		(9.41/16.69)
	\end{cellvarwidth} & \begin{cellvarwidth}[t]
		\centering
		72.42/68.15
		
		(10.74/16.84)
	\end{cellvarwidth}\tabularnewline
	TchAsst & 106 & 5 & 4 & \begin{cellvarwidth}[t]
		\centering
		54.76/39.09
		
		(1.68/5.46)
	\end{cellvarwidth} & \begin{cellvarwidth}[t]
		\centering
		57.14/41.81
		
		(1.30/5.30)
	\end{cellvarwidth} & \begin{cellvarwidth}[t]
		\centering
		53.33/40.00
		
		(0.48/4.45)
	\end{cellvarwidth} & \begin{cellvarwidth}[t]
		\centering
		56.43/42.73
		
		(2.33/4.64)
	\end{cellvarwidth} & \begin{cellvarwidth}[t]
		\centering
		71.19/57.27
		
		(10.0/7.61)
	\end{cellvarwidth} & \begin{cellvarwidth}[t]
		\centering
		\textbf{75.71}/\textbf{60.90}
		
		(1.60/9.96)
	\end{cellvarwidth}\tabularnewline
	RiceCammeo & 3810 & 7 & 2 & \begin{cellvarwidth}[t]
		\centering
		92.91/92.52
		
		(0.15/0.67)
	\end{cellvarwidth} & \begin{cellvarwidth}[t]
		\centering
		93.30/92.91
		
		(0.19/0.80)
	\end{cellvarwidth} & \begin{cellvarwidth}[t]
		\centering
		92.91/92.52
		
		(0.15/0.67)
	\end{cellvarwidth} & \begin{cellvarwidth}[t]
		\centering
		92.91/92.52
		
		(0.15/0.67)
	\end{cellvarwidth} & \begin{cellvarwidth}[t]
		\centering
		94.00/93.28
		
		(0.29/1.28)
	\end{cellvarwidth} & \begin{cellvarwidth}[t]
		\centering
		\textbf{94.18}/\textbf{93.31}
		
		(0.29/1.39)
	\end{cellvarwidth}\tabularnewline
	Yeast & 1453 & 8 & 10 & \begin{cellvarwidth}[t]
		\centering
		48.80/45.77
		
		(0.35/0.83)
	\end{cellvarwidth} & \begin{cellvarwidth}[t]
		\centering
		49.81/\textbf{49.28}
		
		(0.18/1.55)
	\end{cellvarwidth} & \begin{cellvarwidth}[t]
		\centering
		46.56/44.19
		
		(0.43/0.77)
	\end{cellvarwidth} & \begin{cellvarwidth}[t]
		\centering
		48.80/45.77
		
		(0.35/0.83)
	\end{cellvarwidth} & \begin{cellvarwidth}[t]
		\centering
		48.62/45.66
		
		(0.66/1.38)
	\end{cellvarwidth} & \begin{cellvarwidth}[t]
		\centering
		\textbf{49.60}/45.57
		
		(0.70/3.09)
	\end{cellvarwidth}\tabularnewline
	WineQuality & 5318 & 11 & 7 & \begin{cellvarwidth}[t]
		\centering
		52.96/52.88
		
		(0.39/0.45)
	\end{cellvarwidth} & \begin{cellvarwidth}[t]
		\centering
		54.07/53.85
		
		(0.27/1.07)
	\end{cellvarwidth} & \begin{cellvarwidth}[t]
		\centering
		52.45/52.59
		
		(0.93/0.58)
	\end{cellvarwidth} & \begin{cellvarwidth}[t]
		\centering
		52.96/52.88
		
		(0.39/0.45)
	\end{cellvarwidth} & \begin{cellvarwidth}[t]
		\centering
		54.70/\textbf{55.23}
		
		(0.31/0.98)
	\end{cellvarwidth} & \begin{cellvarwidth}[t]
		\centering
		\textbf{55.09}/\textbf{55.23}
		
		(0.29/1.16)
	\end{cellvarwidth}\tabularnewline
	SteelOthers & 1941 & 27 & 2 & \begin{cellvarwidth}[t]
		\centering
		70.59/71.11
		
		(0.21/1.32)
	\end{cellvarwidth} & \begin{cellvarwidth}[t]
		\centering
		73.75/71.31
		
		(0.31/0.85)
	\end{cellvarwidth} & \begin{cellvarwidth}[t]
		\centering
		70.50/71.05
		
		(0.36/1.26)
	\end{cellvarwidth} & \begin{cellvarwidth}[t]
		\centering
		72.35/72.55
		
		(0.89/1.21)
	\end{cellvarwidth} & \begin{cellvarwidth}[t]
		\centering
		77.43/\textbf{76.52}
		
		(0.05/0.81)
	\end{cellvarwidth} & \begin{cellvarwidth}[t]
		\centering
		\textbf{78.09}/75.83
		
		(1.15/1.98)
	\end{cellvarwidth}\tabularnewline
	DrgCnsAlc & 1885 & 12 & 7 & \begin{cellvarwidth}[t]
		\centering
		69.07/\textbf{69.87}
		
		(0.48/1.92)
	\end{cellvarwidth} & \begin{cellvarwidth}[t]
		\centering
		69.54/69.66
		
		(0.39/1.90)
	\end{cellvarwidth} & \begin{cellvarwidth}[t]
		\centering
		69.07\textbf{/69.87}
		
		(0.48/1.92)
	\end{cellvarwidth} & \begin{cellvarwidth}[t]
		\centering
		69.07/\textbf{69.87}
		
		(0.48/1.92)
	\end{cellvarwidth} & \begin{cellvarwidth}[t]
		\centering
		70.89/69.07
		
		(0.35/1.33)
	\end{cellvarwidth} & \begin{cellvarwidth}[t]
		\centering
		\textbf{71.70}/69.23
		
		(0.36/1.48)
	\end{cellvarwidth}\tabularnewline
	DrgCnsImp & 1885 & 12 & 7 & \begin{cellvarwidth}[t]
		\centering
		40.41/39.79
		
		(0.57/2.03)
	\end{cellvarwidth} & \begin{cellvarwidth}[t]
		\centering
		41.95/38.04
		
		(0.39/0.87)
	\end{cellvarwidth} & \begin{cellvarwidth}[t]
		\centering
		40.36/39.89
		
		(0.55/2.19)
	\end{cellvarwidth} & \begin{cellvarwidth}[t]
		\centering
		40.41/39.79
		
		(0.57/2.03)
	\end{cellvarwidth} & \begin{cellvarwidth}[t]
		\centering
		45.04/42.55
		
		(0.88/3.07)
	\end{cellvarwidth} & \begin{cellvarwidth}[t]
		\centering
		\textbf{45.73}/\textbf{42.60}
		
		(2.39/1.80)
	\end{cellvarwidth}\tabularnewline
	DrgCnsSS & 1885 & 12 & 7 & \begin{cellvarwidth}[t]
		\centering
		51.54/51.88
		
		(0.36/2.75)
	\end{cellvarwidth} & \begin{cellvarwidth}[t]
		\centering
		52.73/52.20
		
		(0.35/1.70)
	\end{cellvarwidth} & \begin{cellvarwidth}[t]
		\centering
		51.53/\textbf{52.79}
		
		(0.35/1.38)
	\end{cellvarwidth} & \begin{cellvarwidth}[t]
		\centering
		51.54/51.88
		
		(0.36/2.75)
	\end{cellvarwidth} & \begin{cellvarwidth}[t]
		\centering
		54.13/52.31
		
		(0.58/1.70)
	\end{cellvarwidth} & \begin{cellvarwidth}[t]
		\centering
		\textbf{54.93}/52.10
		
		(0.51/1.28)
	\end{cellvarwidth}\tabularnewline
	EstObLvl & 2087 & 16 & 7 & \begin{cellvarwidth}[t]
		\centering
		55.36/\textbf{55.98}
		
		(0.18/0.66)
	\end{cellvarwidth} & \begin{cellvarwidth}[t]
		\centering
		\textbf{55.40}/55.31
		
		(0.14/1.13)
	\end{cellvarwidth} & \begin{cellvarwidth}[t]
		\centering
		43.14/42.39
		
		(0.30/1.33)
	\end{cellvarwidth} & \begin{cellvarwidth}[t]
		\centering
		55.36/\textbf{55.98}
		
		(0.18/0.66)
	\end{cellvarwidth} & \begin{cellvarwidth}[t]
		\centering
		46.05/44.21
		
		(1.09/2.87)
	\end{cellvarwidth} & \begin{cellvarwidth}[t]
		\centering
		51.99/50.28
		
		(1.82/2.64)
	\end{cellvarwidth}\tabularnewline
	AiDS & 2139 & 23 & 2 & \begin{cellvarwidth}[t]
		\centering
		85.44/85.70
		
		(0.39/01.43)
	\end{cellvarwidth} & \begin{cellvarwidth}[t]
		\centering
		87.54/86.73
		
		(0.11/0.58)
	\end{cellvarwidth} & \begin{cellvarwidth}[t]
		\centering
		85.44/85.70
		
		(0.39/1.43)
	\end{cellvarwidth} & \begin{cellvarwidth}[t]
		\centering
		\textbf{89.18}/\textbf{88.51}
		
		(0.20/0.68)
	\end{cellvarwidth} & \begin{cellvarwidth}[t]
		\centering
		86.90/85.79
		
		(1.24/1.93)
	\end{cellvarwidth} & \begin{cellvarwidth}[t]
		\centering
		87.13/85.65
		
		(1.31/1.94)
	\end{cellvarwidth}\tabularnewline
	AucVer & 2043 & 7 & 2 & \begin{cellvarwidth}[t]
		\centering
		90.42/\textbf{89.58}
		
		(0.55/1.57)
	\end{cellvarwidth} & \begin{cellvarwidth}[t]
		\centering
		\textbf{91.00}/\textbf{89.98}
		
		(0.40/1.59)
	\end{cellvarwidth} & \begin{cellvarwidth}[t]
		\centering
		90.27/89.49
		
		(0.40/1.61)
	\end{cellvarwidth} & \begin{cellvarwidth}[t]
		\centering
		90.42/\textbf{89.58}
		
		(0.55/1.57)
	\end{cellvarwidth} & \begin{cellvarwidth}[t]
		\centering
		89.65/88.46
		
		(0.35/1.99)
	\end{cellvarwidth} & \begin{cellvarwidth}[t]
		\centering
		89.87/88.66
		
		(0.34/1.83)
	\end{cellvarwidth}\tabularnewline
	Ai4iMF & 10000 & 6 & 2 & \begin{cellvarwidth}[t]
		\centering
		97.25/97.07
		
		(0.11/0.36)
	\end{cellvarwidth} & \begin{cellvarwidth}[t]
		\centering
		97.39/97.31
		
		(0.07/0.26)
	\end{cellvarwidth} & \begin{cellvarwidth}[t]
		\centering
		97.08/97.00
		
		(00.27/00.34)
	\end{cellvarwidth} & \begin{cellvarwidth}[t]
		\centering
		97.16/97.03
		
		(0.22/0.36)
	\end{cellvarwidth} & \begin{cellvarwidth}[t]
		\centering
		97.79/97.95
		
		(0.11/0.18)
	\end{cellvarwidth} & \begin{cellvarwidth}[t]
		\centering
		\textbf{97.90}/\textbf{98.00}
		
		(0.12/0.15)
	\end{cellvarwidth}\tabularnewline
	VoicePath & 704 & 2 & 2 & \begin{cellvarwidth}[t]
		\centering
		96.87/95.60
		
		(0.51/2.12)
	\end{cellvarwidth} & \begin{cellvarwidth}[t]
		\centering
		97.16/96.31
		
		(0.56/2.43)
	\end{cellvarwidth} & \begin{cellvarwidth}[t]
		\centering
		96.87/95.60
		
		(0.51/2.12)
	\end{cellvarwidth} & \begin{cellvarwidth}[t]
		\centering
		97.16/95.04
		
		(0.48/1.85)
	\end{cellvarwidth} & \begin{cellvarwidth}[t]
		\centering
		97.69/\textbf{97.44}
		
		(0.28/1.08)
	\end{cellvarwidth} & \begin{cellvarwidth}[t]
		\centering
		\textbf{97.90}/97.31
		
		(0.32/1.17)
	\end{cellvarwidth}\tabularnewline
	WaveForm & 5000 & 21 & 3 & \begin{cellvarwidth}[t]
		\centering
		70.69/68.50
		
		(0.25/1.33)
	\end{cellvarwidth} & \begin{cellvarwidth}[t]
		\centering
		\textbf{71.28}/\textbf{68.54}
		
		(0.34/0.81)
	\end{cellvarwidth} & \begin{cellvarwidth}[t]
		\centering
		66.31/65.02
		
		(0.74/0.87)
	\end{cellvarwidth} & \begin{cellvarwidth}[t]
		\centering
		70.69/68.50
		
		(0.25/1.33)
	\end{cellvarwidth} & \begin{cellvarwidth}[t]
		\centering
		67.05/66.62
		
		(2.44/1.55)
	\end{cellvarwidth} & \begin{cellvarwidth}[t]
		\centering
		65.12/65.80
		
		(1.82/1.65)
	\end{cellvarwidth}\tabularnewline
			
\end{tabular}
\par\end{centering}
\caption{(Part II: $K$ in {2, 3}) Five-fold cross-validation results on the UCI dataset. We compare
the performance of our HDT algorithm, with $K$ (number of splitting
rules) ranging from 2 to 3, trained using $\mathit{sodtWSH}$---against
approximate methods:size- and depth-constrained CART algorithms (CART-size
and CART-depth), as well as the state-of-the-art optimal axis-parallel
decision tree algorithm, ConTree. The depth of the CART-depth and
ConTree algorithms are fixed at 2. Results are reported as mean accuracy on the training and test sets in the format \emph{Training Accuracy
/ Test Accuracy(Standard Deviation: Train / Test)}. The best-performing
algorithm in each row is shown in \textbf{bold}. \label{tab:Five-fold-cross-validation-resul}}
\end{table}

	We now present a direct comparison between HDT, ConTree, CART-depth,
	and CART-size on real-world datasets. Our goal is to evaluate the
	effectiveness of the more flexible hyperplane decision tree model
	relative to the standard axis-parallel decision tree models. Table
	\ref{tab:Five-fold-cross-validation-resul} reports the mean out-of-sample
	accuracy for 30 classification datasets from the UCI Machine Learning
	Repository.
	
	As expected, when model complexity is controlled, HDT outperforms
	the axis-parallel methods on most datasets. Specifically, HDT achieves
	higher training accuracy on 25 datasets and better out-of-sample accuracy
	on 21 datasets over ConTree. By contrast, ConTree achieves the best training accuracy
	on only 4 datasets and the best out-of-sample accuracy on 6 datasets.
	The CART-depth algorithm achieves the highest out-of-sample accuracy
	on just 3 datasets, performing clearly worse than the optimal ConTree
	under the same constraints. Even with greater flexibility (allowing
	tree depths beyond 2), CART-depth achieves the highest accuracy on
	only 5 datasets, still falling behind HDT.
	
	In summary, the more flexible hyperplane decision tree model shows
	strong potential for interpretable learning tasks under controlled
	model complexity, often delivering substantially better performance
	than axis-parallel models. For instance, on the BalScl dataset, HDT
	improves training accuracy by over 20\% and test accuracy by nearly
	30\% compared to the optimal axis-parallel tree algorithm, ConTree.
	These results demonstrate the benefit of moving beyond axis-parallel splitting rules: even small hyperplane decision trees can represent decision boundaries that require substantially larger axis-parallel trees.
	Quantitative evaluation of higher-degree hypersurface trees remains future work; in the present paper, their empirical behavior is demonstrated only on small synthetic examples because of the rapidly growing candidate space.

\subsubsection{Experiments for \texorpdfstring{$K>3$}{K > 3} decision trees}

The experimental results for  $K>3$ decision trees are given in Table \ref{Kgreaterthanthree_I} and \ref{Kgreaterthanthree_II}.

\begin{table}
	\begin{centering}
		\centering\tiny
		\renewcommand{\arraystretch}{1.1}
		\setlength{\tabcolsep}{1pt}
		\begin{tabular}
			{
				>{\raggedright\arraybackslash}p{1.15cm}  
				>{\centering\arraybackslash}p{0.45cm}
				>{\centering\arraybackslash}p{0.25cm}
				>{\centering\arraybackslash}p{0.25cm}
				>{\centering\arraybackslash}p{1.3cm}
				>{\centering\arraybackslash}p{1.4cm}
				>{\centering\arraybackslash}p{1.3cm}
				>{\centering\arraybackslash}p{1.3cm}
				>{\centering\arraybackslash}p{1.3cm}
				>{\centering\arraybackslash}p{1.3cm}
				>{\centering\arraybackslash}p{1.3cm}
				>{\centering\arraybackslash}p{1.3cm}
				>{\centering\arraybackslash}p{1.3cm}
			}
			
			Dataset & $N$ & $D$ & $C$ & \begin{cellvarwidth}[t]
				\centering
				CART-depth
				
				$d=3$
			\end{cellvarwidth} & \begin{cellvarwidth}[t]
				\centering
				ConTree
				
				$d=3$
			\end{cellvarwidth} & \begin{cellvarwidth}[t]
				\centering
				CART-size
				
				$K=4$
			\end{cellvarwidth} & \begin{cellvarwidth}[t]
				\centering
				CART-size
				
				$K=5$
			\end{cellvarwidth} & \begin{cellvarwidth}[t]
				\centering
				CART-size
				
				$K=6$
			\end{cellvarwidth} & \begin{cellvarwidth}[t]
				\centering
				HDT
				
				$K=4$
			\end{cellvarwidth} & \begin{cellvarwidth}[t]
				\centering
				HDT
				
				$K=5$
			\end{cellvarwidth} & \begin{cellvarwidth}[t]
				\centering
				HDT
				
				$K=6$
			\end{cellvarwidth}\tabularnewline\\
			
			\Xhline{0.9pt}\\
			
			haberman & 283 & 3 & 2 & \begin{cellvarwidth}[t]
				\centering
				77.35/74.39
				
				(1.38/3.25)
			\end{cellvarwidth} & \begin{cellvarwidth}[t]
				\centering
				80.97/72.98
				
				(1.22/03.94)
			\end{cellvarwidth} & \begin{cellvarwidth}[t]
				\centering
				76.20/73.33
				
				(1.99/2.58)
			\end{cellvarwidth} & \begin{cellvarwidth}[t]
				\centering
				77.17/72.28
				
				(2.21/1.72)
			\end{cellvarwidth} & \begin{cellvarwidth}[t]
				\centering
				77.97/72.28
				
				(1.71/2.58)
			\end{cellvarwidth} & \begin{cellvarwidth}[t]
				\centering
				81.68/78.94
				
				(0.86/3.51)
			\end{cellvarwidth} & \begin{cellvarwidth}[t]
				\centering
				82.48/79.29
				
				(0.97/4.30)
			\end{cellvarwidth} & \begin{cellvarwidth}[t]
				\centering
				\textbf{82.83}/\textbf{79.65}
				
				(1.15/4.37)
			\end{cellvarwidth}\tabularnewline
			BldTrns & 502 & 4 & 2 & \begin{cellvarwidth}[t]
				\centering
				77.41/74.06
				
				(0.64/03.62)
			\end{cellvarwidth} & \begin{cellvarwidth}[t]
				\centering
				80.45/74.06
				
				(0.843/2.76)
			\end{cellvarwidth} & \begin{cellvarwidth}[t]
				\centering
				78.40/77.23
				
				(0.56/2.80)
			\end{cellvarwidth} & \begin{cellvarwidth}[t]
				\centering
				79.05/75.45
				
				(1.32/4.13)
			\end{cellvarwidth} & \begin{cellvarwidth}[t]
				\centering
				79.20/73.86
				
				(1.21/3.41)
			\end{cellvarwidth} & \begin{cellvarwidth}[t]
				\centering
				81.68/78.94
				
				(0.97/3.51)
			\end{cellvarwidth} & \begin{cellvarwidth}[t]
				\centering
				82.48/79.30
				
				(1.20/2.88)
			\end{cellvarwidth} & \begin{cellvarwidth}[t]
				\centering
				\textbf{82.83}/\textbf{79.65}
				
				(1.15/4.23)
			\end{cellvarwidth}\tabularnewline
			sepsis & 975 & 3 & 2 & \begin{cellvarwidth}[t]
				\centering
				94.23/93.74
				
				(0.35/1.19)
			\end{cellvarwidth} & \begin{cellvarwidth}[t]
				\centering
				94.56/93.44
				
				(0.34/1.10)
			\end{cellvarwidth} & \begin{cellvarwidth}[t]
				\centering
				94.28/93.54
				
				(0.33/1.11)
			\end{cellvarwidth} & \begin{cellvarwidth}[t]
				\centering
				94.39/93.54
				
				(0.34/0.95)
			\end{cellvarwidth} & \begin{cellvarwidth}[t]
				\centering
				94.44/93.54
				
				(0.29/0.95)
			\end{cellvarwidth} & \begin{cellvarwidth}[t]
				\centering
				95.77/93.85
				
				(0.43/1.95)
			\end{cellvarwidth} & \begin{cellvarwidth}[t]
				\centering
				96.10/93.85
				
				(0.43/1.95)
			\end{cellvarwidth} & \begin{cellvarwidth}[t]
				\centering
				\textbf{96.44}/\textbf{94.05}
				
				(0.46/1.72)
			\end{cellvarwidth}\tabularnewline
			algerian & 243 & 14 & 2 & \begin{cellvarwidth}[t]
				\centering
				99.59/98.37
				
				(0.21/1.53)
			\end{cellvarwidth} & \begin{cellvarwidth}[t]
				\centering
				\textbf{100}/95.51
				
				(0.00/3.00)
			\end{cellvarwidth} & \begin{cellvarwidth}[t]
				\centering
				99.79/98.37
				
				(0.25/1.53)
			\end{cellvarwidth} & \begin{cellvarwidth}[t]
				\centering
				99.90/98.37
				
				(0.21/1.53)
			\end{cellvarwidth} & \begin{cellvarwidth}[t]
				\centering
				\textbf{100}/98.37
				
				(0.00/1.53)
			\end{cellvarwidth} & \begin{cellvarwidth}[t]
				\centering
				\textbf{100}/\textbf{99.18}
				
				(0.00/1.12)
			\end{cellvarwidth} & \begin{cellvarwidth}[t]
				\centering
				\textbf{100}/97.96
				
				(0.00/1.44)
			\end{cellvarwidth} & \begin{cellvarwidth}[t]
				\centering
				\textbf{100}/96.74
				
				(0.00/1.83)
			\end{cellvarwidth}\tabularnewline
			Cryotherapy & 89 & 6 & 2 & \begin{cellvarwidth}[t]
				\centering
				94.93/90.00
				
				(1.91/06.48)
			\end{cellvarwidth} & \begin{cellvarwidth}[t]
				\centering
				99.44/78.89
				
				(0.69/12.86)
			\end{cellvarwidth} & \begin{cellvarwidth}[t]
				\centering
				94.93/90.00
				
				(1.91/6.48)
			\end{cellvarwidth} & \begin{cellvarwidth}[t]
				\centering
				96.62/86.67
				
				(1.91/9.03)
			\end{cellvarwidth} & \begin{cellvarwidth}[t]
				\centering
				97.75/86.67
				
				(1.69/9.03)
			\end{cellvarwidth} & \begin{cellvarwidth}[t]
				\centering
				99.44/\textbf{94.44}
				
				(0.77/3.93)
			\end{cellvarwidth} & \begin{cellvarwidth}[t]
				\centering
				\textbf{99.72}/93.33
				
				(0.77/3.93)
			\end{cellvarwidth} & \begin{cellvarwidth}[t]
				\centering
				\textbf{99.72}/92.22
				
				(0.49/1.95)
			\end{cellvarwidth}\tabularnewline
			Caesarian & 72 & 5 & 2 & \begin{cellvarwidth}[t]
				\centering
				77.90/58.67
				
				(1.79/7.78)
			\end{cellvarwidth} & \begin{cellvarwidth}[t]
				\centering
				82.11/58.67
				
				(1.72/4.99)
			\end{cellvarwidth} & \begin{cellvarwidth}[t]
				\centering
				77.54/66.67
				
				(2.33/4.22)
			\end{cellvarwidth} & \begin{cellvarwidth}[t]
				\centering
				78.60/56.00
				
				(3.58/12.36)
			\end{cellvarwidth} & \begin{cellvarwidth}[t]
				\centering
				80.35/57.33
				
				(2.58/6.80)
			\end{cellvarwidth} & \begin{cellvarwidth}[t]
				\centering
				92.63/85.33
				
				(1.47/2.98)
			\end{cellvarwidth} & \begin{cellvarwidth}[t]
				\centering
				\textbf{93.68}/\textbf{86.68}
				
				(20.96/0.00)
			\end{cellvarwidth} & \begin{cellvarwidth}[t]
				\centering
				\textbf{93.68}/85.33
				
				(20.96/2.98)
			\end{cellvarwidth}\tabularnewline
			ecoli & 336 & 7 & 8 & \begin{cellvarwidth}[t]
				\centering
				85.52/81.47
				
				(0.37/1.77)
			\end{cellvarwidth} & \begin{cellvarwidth}[t]
				\centering
				\textbf{87.84}/82.06
				
				(0.18/2.85)
			\end{cellvarwidth} & \begin{cellvarwidth}[t]
				\centering
				85.07/81.76
				
				(0.53/1.50)
			\end{cellvarwidth} & \begin{cellvarwidth}[t]
				\centering
				85.82/82.65
				
				(0.88/2.53)
			\end{cellvarwidth} & \begin{cellvarwidth}[t]
				\centering
				85.97/\textbf{82.94}
				
				(0.65/2.20)
			\end{cellvarwidth} & \begin{cellvarwidth}[t]
				\centering
				84.25/79.12
				
				(1.70/4.08)
			\end{cellvarwidth} & \begin{cellvarwidth}[t]
				\centering
				84.40/79.41
				
				(21.25/4.29
			\end{cellvarwidth} & \begin{cellvarwidth}[t]
				\centering
				85.90/79.12
				
				(21.25/4.46)
			\end{cellvarwidth}\tabularnewline
			GlsId & 213 & 9 & 6 & \begin{cellvarwidth}[t]
				\centering
				72.35/66.05
				
				(1.86/6.51)
			\end{cellvarwidth} & \begin{cellvarwidth}[t]
				\centering
				\textbf{80.12}/\textbf{70.70}
				
				(0.941\slash 3.78)
			\end{cellvarwidth} & \begin{cellvarwidth}[t]
				\centering
				71.29/66.98
				
				(1.72/6.14)
			\end{cellvarwidth} & \begin{cellvarwidth}[t]
				\centering
				74.24/63.72
				
				(1.46/4.79)
			\end{cellvarwidth} & \begin{cellvarwidth}[t]
				\centering
				75.88/64.65
				
				(2.10/7.98)
			\end{cellvarwidth} & \begin{cellvarwidth}[t]
				\centering
				76.35/61.86
				
				(1.97/4.53)
			\end{cellvarwidth} & \begin{cellvarwidth}[t]
				\centering
				77.77/61.86
				
				(21.92/4.53)
			\end{cellvarwidth} & \begin{cellvarwidth}[t]
				\centering
				77.29/58.61
				
				(21.64/4.47)
			\end{cellvarwidth}\tabularnewline

		\end{tabular}
		\par\end{centering}
	\caption{(Part I: $K$ in {4, 5, 6}) Five-fold cross-validation results on the UCI dataset. We compare
		the performance of our HDT algorithm, with $K$ (number of splitting
		rules) ranging from 4 to 6, trained using $\mathit{sodtWSH}$---against
		approximate methods:size- and depth-constrained CART algorithms (CART-size
		and CART-depth), as well as the state-of-the-art optimal axis-parallel
		decision tree algorithm, ConTree. The depth of the CART-depth and
		ConTree algorithms are fixed at 3. Results are reported as mean accuracy on the training and test sets in the format \emph{Training Error
			/ Test Error (Standard Deviation: Train / Test)}. The best-performing
		algorithm in each row is shown in \textbf{bold}.\label{Kgreaterthanthree_I}} 
\end{table}

\begin{table}
	\begin{centering}
		\centering\tiny
		\renewcommand{\arraystretch}{1.1}
		\setlength{\tabcolsep}{1pt}
		\begin{tabular}
			{
				>{\raggedright\arraybackslash}p{1.15cm}  
				>{\centering\arraybackslash}p{0.55cm}
				>{\centering\arraybackslash}p{0.25cm}
				>{\centering\arraybackslash}p{0.25cm}
				>{\centering\arraybackslash}p{1.3cm}
				>{\centering\arraybackslash}p{1.4cm}
				>{\centering\arraybackslash}p{1.3cm}
				>{\centering\arraybackslash}p{1.3cm}
				>{\centering\arraybackslash}p{1.3cm}
				>{\centering\arraybackslash}p{1.3cm}
				>{\centering\arraybackslash}p{1.3cm}
				>{\centering\arraybackslash}p{1.3cm}
				>{\centering\arraybackslash}p{1.3cm}
			}
			
			Dataset & $N$ & $D$ & $C$ & \begin{cellvarwidth}[t]
				\centering
				CART-depth
				
				$d=3$
			\end{cellvarwidth} & \begin{cellvarwidth}[t]
				\centering
				ConTree
				
				$d=3$
			\end{cellvarwidth} & \begin{cellvarwidth}[t]
				\centering
				CART-size
				
				$K=4$
			\end{cellvarwidth} & \begin{cellvarwidth}[t]
				\centering
				CART-size
				
				$K=5$
			\end{cellvarwidth} & \begin{cellvarwidth}[t]
				\centering
				CART-size
				
				$K=6$
			\end{cellvarwidth} & \begin{cellvarwidth}[t]
				\centering
				HDT
				
				$K=4$
			\end{cellvarwidth} & \begin{cellvarwidth}[t]
				\centering
				HDT
				
				$K=5$
			\end{cellvarwidth} & \begin{cellvarwidth}[t]
				\centering
				HDT
				
				$K=6$
			\end{cellvarwidth}\tabularnewline\\
			
			\Xhline{0.9pt}\\

			iris & 147 & 4 & 3 & \begin{cellvarwidth}[t]
	\centering
	98.12/90.67
	
	(1.26/2.49)
\end{cellvarwidth} & \begin{cellvarwidth}[t]
	\centering
	99.82/91.33
	
	(0.34/4.00)
\end{cellvarwidth} & \begin{cellvarwidth}[t]
	\centering
	98.80/88.00
	
	(1.03/4.52)
\end{cellvarwidth} & \begin{cellvarwidth}[t]
	\centering
	99.15/88.00
	
	(0.76/4.52)
\end{cellvarwidth} & \begin{cellvarwidth}[t]
	\centering
	99.32/88.00
	
	(0.64/4.52)
\end{cellvarwidth} & \begin{cellvarwidth}[t]
	\centering
	99.15/96.77
	
	(2.94/12.61)
\end{cellvarwidth} & \begin{cellvarwidth}[t]
	\centering
	\textbf{100}/\textbf{97.77}
	
	(0.00/12.61)
\end{cellvarwidth} & \begin{cellvarwidth}[t]
	\centering
	\textbf{100}/96.77
	
	(0.00/12.61)
\end{cellvarwidth}\tabularnewline
MnkPrb & 432 & 6 & 2 & \begin{cellvarwidth}[t]
	\centering
	82.44/78.62
	
	(3.61/3.89)
\end{cellvarwidth} & \begin{cellvarwidth}[t]
	\centering
	\textbf{89.51}/\textbf{86.44}
	
	(1.43/5.65)
\end{cellvarwidth} & \begin{cellvarwidth}[t]
	\centering
	82.43/78.62
	
	(3.61/3.89)
\end{cellvarwidth} & \begin{cellvarwidth}[t]
	\centering
	83.25/78.16
	
	(2.41/2.72)
\end{cellvarwidth} & \begin{cellvarwidth}[t]
	\centering
	81.54/77.01
	
	(2.60/3.17)
\end{cellvarwidth} & \begin{cellvarwidth}[t]
	\centering
	82.46/81.38
	
	(1.30/5.6)
\end{cellvarwidth} & \begin{cellvarwidth}[t]
	\centering
	82.73/81.61
	
	(1.61/5.75)
\end{cellvarwidth} & \begin{cellvarwidth}[t]
	\centering
	82.96/82.30
	
	(1.50/6.38)
\end{cellvarwidth}\tabularnewline
UKM & 403 & 5 & 5 & \begin{cellvarwidth}[t]
	\centering
	88.76/86.42
	
	(1.38/3.75)
\end{cellvarwidth} & \begin{cellvarwidth}[t]
	\centering
	\textbf{91.55}/\textbf{87.90}
	
	(0.60/3.26)
\end{cellvarwidth} & \begin{cellvarwidth}[t]
	\centering
	85.34/80.00
	
	(0.46/3.86)
\end{cellvarwidth} & \begin{cellvarwidth}[t]
	\centering
	87.39/83.70
	
	(0.64/3.44)
\end{cellvarwidth} & \begin{cellvarwidth}[t]
	\centering
	89.57/87.65
	
	(0.32/1.35)
\end{cellvarwidth} & \begin{cellvarwidth}[t]
	\centering
	73.10/68.89
	
	(10.74/17.00)
\end{cellvarwidth} & \begin{cellvarwidth}[t]
	\centering
	73.60/75.30
	
	(10.17/9.88)
\end{cellvarwidth} & \begin{cellvarwidth}[t]
	\centering
	73.79/75.56
	
	(10.3/9.54)
\end{cellvarwidth}\tabularnewline
TchAsst & 106 & 5 & 4 & \begin{cellvarwidth}[t]
	\centering
	62.38/41.82
	
	(2.21/4.45)
\end{cellvarwidth} & \begin{cellvarwidth}[t]
	\centering
	67.62/46.36
	
	(2.65/10.52)
\end{cellvarwidth} & \begin{cellvarwidth}[t]
	\centering
	58.81/40.91
	
	(1.43/4.98)
\end{cellvarwidth} & \begin{cellvarwidth}[t]
	\centering
	62.14/43.64
	
	(2.31/3.64)
\end{cellvarwidth} & \begin{cellvarwidth}[t]
	\centering
	63.81/42.73
	
	(2.88/4.64)
\end{cellvarwidth} & \begin{cellvarwidth}[t]
	\centering
	78.33/\textbf{62.73}
	
	(1.13/10.85)
\end{cellvarwidth} & \begin{cellvarwidth}[t]
	\centering
	79.76/\textbf{62.73}
	
	(1.19/10.85)
\end{cellvarwidth} & \begin{cellvarwidth}[t]
	\centering
	\textbf{80.48}/59.09
	
	(1.81/8.50)
\end{cellvarwidth}\tabularnewline
RiceCammeo & 3810 & 7 & 2 & \begin{cellvarwidth}[t]
	\centering
	93.06/92.49
	
	(0.28/0.79)
\end{cellvarwidth} & \begin{cellvarwidth}[t]
	\centering
	93.79/92.36
	
	(0.16/0.89)
\end{cellvarwidth} & \begin{cellvarwidth}[t]
	\centering
	92.91/92.52
	
	(0.15/0.67)
\end{cellvarwidth} & \begin{cellvarwidth}[t]
	\centering
	93.03/92.49
	
	(0.28/0.72)
\end{cellvarwidth} & \begin{cellvarwidth}[t]
	\centering
	93.09/92.49
	
	(0.31/0.76)
\end{cellvarwidth} & \begin{cellvarwidth}[t]
	\centering
	94.27/93.49
	
	(0.29/1.36)
\end{cellvarwidth} & \begin{cellvarwidth}[t]
	\centering
	94.35/\textbf{93.57}
	
	(0.29/1.31)
\end{cellvarwidth} & \begin{cellvarwidth}[t]
	\centering
	\textbf{94.38}/\textbf{93.57}
	
	(0.31/1.20)
\end{cellvarwidth}\tabularnewline
Yeast & 1453 & 8 & 10 & \begin{cellvarwidth}[t]
	\centering
	57.86/54.35
	
	(0.62/0.98)
\end{cellvarwidth} & \begin{cellvarwidth}[t]
	\centering
	\textbf{58.50}/\textbf{55.74}
	
	(0.38/0.96)
\end{cellvarwidth} & \begin{cellvarwidth}[t]
	\centering
	55.83/52.16
	
	(0.31/1.78)
\end{cellvarwidth} & \begin{cellvarwidth}[t]
	\centering
	56.99/53.13
	
	(0.42/0.89)
\end{cellvarwidth} & \begin{cellvarwidth}[t]
	\centering
	57.95/54.36
	
	(0.68/1.20)
\end{cellvarwidth} & \begin{cellvarwidth}[t]
	\centering
	50.02/46.83
	
	(0.92/2.73)
\end{cellvarwidth} & \begin{cellvarwidth}[t]
	\centering
	51.23/47.23
	
	(1.02/3.01)
\end{cellvarwidth} & \begin{cellvarwidth}[t]
	\centering
	50.23/48.11
	
	(1.03/2.50)
\end{cellvarwidth}\tabularnewline
WineQuality & 5318 & 11 & 7 & \begin{cellvarwidth}[t]
	\centering
	53.95/53.29
	
	(0.40/0.57)
\end{cellvarwidth} & \begin{cellvarwidth}[t]
	\centering
	\textbf{55.52}/53.29
	
	(0.19/0.81)
\end{cellvarwidth} & \begin{cellvarwidth}[t]
	\centering
	52.96/52.88
	
	(0.39/0.45)
\end{cellvarwidth} & \begin{cellvarwidth}[t]
	\centering
	53.09/52.84
	
	(0.33/0.46)
\end{cellvarwidth} & \begin{cellvarwidth}[t]
	\centering
	53.35/53.08
	
	(0.41/0.61)
\end{cellvarwidth} & \begin{cellvarwidth}[t]
	\centering
	55.21/55.60
	
	(0.37/1.24)
\end{cellvarwidth} & \begin{cellvarwidth}[t]
	\centering
	54.96/\textbf{55.02}
	
	(0.32/0.98)
\end{cellvarwidth} & \begin{cellvarwidth}[t]
	\centering
	55.42/54.32
	
	(1.21/0.42)
\end{cellvarwidth}\tabularnewline
SteelOthers & 1941 & 27 & 2 & \begin{cellvarwidth}[t]
	\centering
	73.30/73.57
	
	(1.09/0.44)
\end{cellvarwidth} & \begin{cellvarwidth}[t]
	\centering
	77.77/73.57
	
	(0.16/1.33)
\end{cellvarwidth} & \begin{cellvarwidth}[t]
	\centering
	73.17/73.37
	
	(0.82/1.06)
\end{cellvarwidth} & \begin{cellvarwidth}[t]
	\centering
	74.38/74.50
	
	(1.55/1.02)
\end{cellvarwidth} & \begin{cellvarwidth}[t]
	\centering
	76.10/75.32
	
	(0.96/0.88)
\end{cellvarwidth} & \begin{cellvarwidth}[t]
	\centering
	76.49/\textbf{75.42}
	
	(0.81/0.74)
\end{cellvarwidth} & \begin{cellvarwidth}[t]
	\centering
	\textbf{77.81}/74.38
	
	(0.32/0.56)
\end{cellvarwidth} & \begin{cellvarwidth}[t]
	\centering
	74.32/72.5
	
	(0.34/1.83)
\end{cellvarwidth}\tabularnewline
DrgCnsAlc & 1885 & 12 & 7 & \begin{cellvarwidth}[t]
	\centering
	69.35/69.44
	
	(0.43/1.91)
\end{cellvarwidth} & \begin{cellvarwidth}[t]
	\centering
	70.62/69.55
	
	(0.50/1.85)
\end{cellvarwidth} & \begin{cellvarwidth}[t]
	\centering
	69.24/69.71
	
	(0.69/2.12)
\end{cellvarwidth} & \begin{cellvarwidth}[t]
	\centering
	69.44/69.34
	
	(0.55/2.02)
\end{cellvarwidth} & \begin{cellvarwidth}[t]
	\centering
	70.03/\textbf{70.08}
	
	(0.78/1.74)
\end{cellvarwidth} & \begin{cellvarwidth}[t]
	\centering
	72.44/69.18
	
	(0.34/1.45)
\end{cellvarwidth} & \begin{cellvarwidth}[t]
	\centering
	72.88/68.96
	
	(0.27/1.42)
\end{cellvarwidth} & \begin{cellvarwidth}[t]
	\centering
	\textbf{73.21}/69.43
	
	(0.43/1.58)
\end{cellvarwidth}\tabularnewline
DrgCnsImp & 1885 & 12 & 7 & \begin{cellvarwidth}[t]
	\centering
	41.22/38.57
	
	(0.52/1.38)
\end{cellvarwidth} & \begin{cellvarwidth}[t]
	\centering
	43.83/38.20
	
	(0.43/3.13)
\end{cellvarwidth} & \begin{cellvarwidth}[t]
	\centering
	40.69/39.47
	
	(0.67/1.87)
\end{cellvarwidth} & \begin{cellvarwidth}[t]
	\centering
	41.07/38.99
	
	(0.69/1.61)
\end{cellvarwidth} & \begin{cellvarwidth}[t]
	\centering
	41.27/38.89
	
	(0.73/1.59)
\end{cellvarwidth} & \begin{cellvarwidth}[t]
	\centering
	46.18/\textbf{42.65}
	
	(2.43/4.38)
\end{cellvarwidth} & \begin{cellvarwidth}[t]
	\centering
	\textbf{46.25}/42.60
	
	(3.18/4.37)
\end{cellvarwidth} & \begin{cellvarwidth}[t]
	\centering
	46.06/40.80
	
	(4.01/2.32)
\end{cellvarwidth}\tabularnewline
DrgCnsSS & 1885 & 12 & 7 & \begin{cellvarwidth}[t]
	\centering
	52.06/51.03
	
	(0.53/0.024)
\end{cellvarwidth} & \begin{cellvarwidth}[t]
	\centering
	54.23/51.57
	
	(0.38/2.22)
\end{cellvarwidth} & \begin{cellvarwidth}[t]
	\centering
	51.54/51.88
	
	(0.36/2.75)
\end{cellvarwidth} & \begin{cellvarwidth}[t]
	\centering
	51.80/51.41
	
	(0.22/2.39)
\end{cellvarwidth} & \begin{cellvarwidth}[t]
	\centering
	51.80/51.41
	
	(0.22/2.39)
\end{cellvarwidth} & \begin{cellvarwidth}[t]
	\centering
	55.48/51.57
	
	(0.50/1.60)
\end{cellvarwidth} & \begin{cellvarwidth}[t]
	\centering
	55.82/51.36
	
	(0.60/1.33)
\end{cellvarwidth} & \begin{cellvarwidth}[t]
	\centering
	\textbf{56.21}/\textbf{52.42}
	
	(0.32/1.24)
\end{cellvarwidth}\tabularnewline
EstObLvl & 2087 & 16 & 7 & \begin{cellvarwidth}[t]
	\centering
	65.07/63.16
	
	(0.12/2.07)
\end{cellvarwidth} & \begin{cellvarwidth}[t]
	\centering
	\textbf{72.51}/\textbf{70.05}
	
	(0.28/1.54)
\end{cellvarwidth} & \begin{cellvarwidth}[t]
	\centering
	61.25/60.62
	
	(0.35/1.40)
\end{cellvarwidth} & \begin{cellvarwidth}[t]
	\centering
	67.31/67.22
	
	(0.31/1.33)
\end{cellvarwidth} & \begin{cellvarwidth}[t]
	\centering
	69.50/68.66
	
	(0.23/1.4)
\end{cellvarwidth} & \begin{cellvarwidth}[t]
	\centering
	51.70/48.95
	
	(1.61/3.17)
\end{cellvarwidth} & \begin{cellvarwidth}[t]
	\centering
	49.61/46.89
	
	(1.75/2.51)
\end{cellvarwidth} & \begin{cellvarwidth}[t]
	\centering
	49.64/49.22
	
	(1.46/2.84)
\end{cellvarwidth}\tabularnewline
AiDS & 2139 & 23 & 2 & \begin{cellvarwidth}[t]
	\centering
	89.23/88.55
	
	(0.15/0.74)
\end{cellvarwidth} & \begin{cellvarwidth}[t]
	\centering
	\textbf{90.24}/\textbf{89.11}
	
	(0.20/0.86)
\end{cellvarwidth} & \begin{cellvarwidth}[t]
	\centering
	89.23/88.55
	
	(0.15/0.74)
\end{cellvarwidth} & \begin{cellvarwidth}[t]
	\centering
	89.36/88.60
	
	(0.27/0.80)
\end{cellvarwidth} & \begin{cellvarwidth}[t]
	\centering
	89.49/88.64
	
	(0.23/0.84)
\end{cellvarwidth} & \begin{cellvarwidth}[t]
	\centering
	86.99/85.93
	
	(1.87/2.23)
\end{cellvarwidth} & \begin{cellvarwidth}[t]
	\centering
	86.99/85.65
	
	(2.19/2.51)
\end{cellvarwidth} & \begin{cellvarwidth}[t]
	\centering
	87.02/85.23
	
	(1.32/2.44)
\end{cellvarwidth}\tabularnewline
AucVer & 2043 & 7 & 2 & \begin{cellvarwidth}[t]
	\centering
	92.08/90.56
	
	(0.65/2.22)
\end{cellvarwidth} & \begin{cellvarwidth}[t]
	\centering
	\textbf{94.98}/\textbf{93.55}
	
	(0.40/1.50)
\end{cellvarwidth} & \begin{cellvarwidth}[t]
	\centering
	91.91/90.27
	
	(0.64/2.16)
\end{cellvarwidth} & \begin{cellvarwidth}[t]
	\centering
	92.83/91.00
	
	(0.52/2.08)
\end{cellvarwidth} & \begin{cellvarwidth}[t]
	\centering
	93.05/91.30
	
	(0.46/2.21)
\end{cellvarwidth} & \begin{cellvarwidth}[t]
	\centering
	90.00/88.36
	
	(0.28/1.53)
\end{cellvarwidth} & \begin{cellvarwidth}[t]
	\centering
	90.11/88.66
	
	(0.48/1.38)
\end{cellvarwidth} & \begin{cellvarwidth}[t]
	\centering
	89.92/88.31
	
	(0.35/1.65)
\end{cellvarwidth}\tabularnewline
Ai4iMF & 10000 & 6 & 2 & \begin{cellvarwidth}[t]
	\centering
	97.41/97.23
	
	(0.05/0.48)
\end{cellvarwidth} & \begin{cellvarwidth}[t]
	\centering
	97.84/97.27
	
	(0.06/0.22)
\end{cellvarwidth} & \begin{cellvarwidth}[t]
	\centering
	97.22/97.06
	
	(0.28/0.39)
\end{cellvarwidth} & \begin{cellvarwidth}[t]
	\centering
	97.49/97.19
	
	(0.26/0.32)
\end{cellvarwidth} & \begin{cellvarwidth}[t]
	\centering
	97.56/97.26
	
	(0.19/0.38)
\end{cellvarwidth} & \begin{cellvarwidth}[t]
	\centering
	\textbf{97.92}/\textbf{98.01}
	
	(0.11/0.17)
\end{cellvarwidth} & \begin{cellvarwidth}[t]
	\centering
	97.88/97.92
	
	(0.12/0.17)
\end{cellvarwidth} & \begin{cellvarwidth}[t]
	\centering
	97.82/97.85
	
	(0.17/0.09)
\end{cellvarwidth}\tabularnewline
VoicePath & 704 & 2 & 2 & \begin{cellvarwidth}[t]
	\centering
	97.16/95.04
	
	(0.48/1.85)
\end{cellvarwidth} & \begin{cellvarwidth}[t]
	\centering
	98.15/96.03
	
	(0.43/1.83)
\end{cellvarwidth} & \begin{cellvarwidth}[t]
	\centering
	97.41/95.04
	
	(0.53/2.24)
\end{cellvarwidth} & \begin{cellvarwidth}[t]
	\centering
	97.58/95.18
	
	(0.47/2.26)
\end{cellvarwidth} & \begin{cellvarwidth}[t]
	\centering
	97.73/95.04
	
	(0.51/2.33)
\end{cellvarwidth} & \begin{cellvarwidth}[t]
	\centering
	98.12/97.89
	
	(0.27/1.42)
\end{cellvarwidth} & \begin{cellvarwidth}[t]
	\centering
	98.30/97.89
	
	(0.27/1.00)
\end{cellvarwidth} & \begin{cellvarwidth}[t]
	\centering
	\textbf{98.40}/\textbf{98.01}
	
	(0.25/1.27)
\end{cellvarwidth}\tabularnewline
WaveForm & 5000 & 21 & 3 & \begin{cellvarwidth}[t]
	\centering
	73.29/70.44
	
	(0.32/0.41)
\end{cellvarwidth} & \begin{cellvarwidth}[t]
	\centering
	\textbf{76.53}/\textbf{73.26}
	
	(0.16/1.03)
\end{cellvarwidth} & \begin{cellvarwidth}[t]
	\centering
	71.37/68.84
	
	(0.45/1.17)
\end{cellvarwidth} & \begin{cellvarwidth}[t]
	\centering
	72.36/69.84
	
	(0.62/1.16)
\end{cellvarwidth} & \begin{cellvarwidth}[t]
	\centering
	72.69/69.96
	
	(0.79/1.17)
\end{cellvarwidth} & \begin{cellvarwidth}[t]
	\centering
	64.37/62.70
	
	(0.82/1.11)
\end{cellvarwidth} & \begin{cellvarwidth}[t]
	\centering
	64.58/62.70
	
	(3.71/1.82)
\end{cellvarwidth} & \begin{cellvarwidth}[t]
	\centering
	66.12/62.70
	
	(2.52/3.64)
\end{cellvarwidth}\tabularnewline
		\end{tabular}
		\par\end{centering}
	\caption{(Part II: $K$ in {4, 5, 6}) Five-fold cross-validation results on the UCI dataset. We compare
		the performance of our HDT algorithm, with $K$ (number of splitting
		rules) ranging from 4 to 6, trained using $\mathit{sodtWSH}$---against
		approximate methods:size- and depth-constrained CART algorithms (CART-size
		and CART-depth), as well as the state-of-the-art optimal axis-parallel
		decision tree algorithm, ConTree. The depth of the CART-depth and
		ConTree algorithms are fixed at 3. Results are reported as mean accuracy on the training and test sets in the format \emph{Training Error
			/ Test Error (Standard Deviation: Train / Test)}. The best-performing
		algorithm in each row is shown in \textbf{bold}.\label{Kgreaterthanthree_II}} 
\end{table}

\end{document}